\documentclass[twoside,11pt]{article}

\usepackage{amsfonts}
\usepackage{bbm}
\usepackage{bm}
\usepackage{color}
\usepackage{dsfont}
\usepackage[shortlabels]{enumitem}
\usepackage{extarrows}
\usepackage{float}
\usepackage{mathdots}
\usepackage{mathrsfs}
\usepackage{mdframed}
\usepackage[caption=false]{subfig}
\usepackage[dvipsnames]{xcolor}

\usepackage{jmlr2e}

\DeclareMathOperator*{\argmin}{argmin}

\newcommand{\col}[2]{\bgroup\color{#1}#2\egroup}

\renewcommand{\d}{\mathop{}\!\mathrm{d}}

\newcommand{\E}{\mathbb{E}}

\newcommand{\mnorm}[1]{{\vert\kern-0.25ex\vert\kern-0.25ex\vert #1 
    \vert\kern-0.25ex\vert\kern-0.25ex\vert}}

\renewcommand{\P}{\mathbb{P}}

\newcommand{\R}{\mathbb{R}}

\DeclareMathOperator{\var}{var}

\def\hyph{-\penalty0\hskip0pt\relax} 

\newif\ifnocolor
\nocolorfalse 

\usepackage[normalem]{ulem} 
\usepackage[colorinlistoftodos]{todonotes} 
\newcommand{\st}[1]{\ifmmode\text{\sout{\ensuremath{#1}}}\else\sout{#1}\fi} 
	\newcommand{\DELETE}[1]{\col{black!30}{\st{#1}}} 
	\newcommand{\DELETETWO}[1]{\col{black!30}{\st{#1}}}
	\newcommand{\DELETELONG}[1]{\col{black!30}{#1}} 
	\newcommand{\DELETELONGTWO}[1]{\col{black!30}{#1}} 
	
\ifnocolor
	\renewcommand{\DELETE}[1]{} 
	\renewcommand{\DELETETWO}[1]{}
	\renewcommand{\DELETELONG}[1]{} 
	\renewcommand{\DELETELONGTWO}[1]{} 
\fi

\usepackage[normalem]{ulem} 
\usepackage[colorinlistoftodos]{todonotes} 

\allowdisplaybreaks

\firstpageno{1}
\usepackage{lastpage}
\jmlrheading{20}{2019}{1-\pageref{LastPage}}{5/18; Revised
4/19}{4/19}{18-278}{Shiqing Yu, Mathias Drton, Ali Shojaie}
\ShortHeadings{Generalized Score Matching}{Yu, Drton and Shojaie}

\begin{document}

\title{Generalized Score Matching for Non-Negative Data}

\author{\name Shiqing Yu \email sqyu@uw.edu \\
       \addr Department of Statistics\\
       University of Washington, Seattle, WA, U.S.A. 
              \AND
       \name Mathias Drton \email md5@uw.edu \\
       \addr Department of Mathematical Sciences\\
       University of Copenhagen, Copenhagen, Denmark\\
       and\\
       \addr Department of Statistics\\
       University of Washington, Seattle, WA, U.S.A. 
       \AND
       \name Ali Shojaie \email ashojaie@uw.edu \\
       \addr Department of Biostatistics\\
       University of Washington, Seattle, WA, U.S.A.
     }

\editor{Aapo Hyvarinen}

\maketitle

\begin{abstract}%
A common challenge in estimating parameters of probability density
functions is the intractability of the normalizing constant. While
in such cases maximum likelihood estimation may be implemented using
numerical integration, the approach becomes computationally
intensive.  The score matching method of \citet{hyv05}
avoids direct calculation of the normalizing constant and yields
closed-form estimates for exponential families of continuous
distributions over $\mathbb{R}^m$. \citet{hyv07} extended the
approach to distributions supported on the non-negative orthant, 
$\mathbb{R}_+^m$. In this paper, we give a generalized form of score
matching for non-negative data that improves estimation
efficiency. As an example, we consider a general class of pairwise
interaction models. Addressing an overlooked inexistence problem, we
generalize the regularized score matching method of
\citet{lin16} and improve its theoretical guarantees for non-negative Gaussian graphical models.
\end{abstract}


\section{Introduction}

Score matching was first developed in \citet{hyv05} for continuous distributions supported on all of $\mathbb{R}^m$.  Consider such a distribution $P_0$, with density $p_0$ and support equal to $\mathbb{R}^m$.
Let $\mathcal{P}$ be a family of distributions with twice continuously differentiable densities. The score matching estimator of $p_0$ using $\mathcal{P}$ as a model is the minimizer of the expected squared $\ell_2$ distance between the gradients of
$\log p_0$ and a log-density from $\mathcal{P}$.  So we minimize the
loss
$\int_{\mathbb{R}^m}p_{0}(\boldsymbol{x})\|\nabla\log
p(\boldsymbol{x})-\nabla\log
p_0(\boldsymbol{x})\|_2^2\d\boldsymbol{x}$ 
with respect to densities $p$ from $\mathcal{P}$.  The loss depends on $p_0$, but integration by parts can be used to rewrite it in a form that can be
approximated by averaging over the sample without knowing $p_0$.  A key feature of score matching is that normalizing constants cancel in gradients of log-densities, allowing for simple treatment of models with intractable normalizing constants.  
For exponential families, the loss is quadratic in the canonical parameter, making optimization straightforward. 

If the considered distributions are supported on a proper subset of $\mathbb{R}^m$, then the integration by parts arguments underlying the score matching estimator may fail due to discontinuities at the boundary of the support.  For data supported on the non-negative orthant $\mathbb{R}_+^m$, \citet{hyv07} addresses this problem by modifying the loss to 
$\int_{\mathbb{R}^m}p_{0}(\boldsymbol{x})\|\nabla\log
p(\boldsymbol{x})\circ \boldsymbol{x}-\nabla\log
p_0(\boldsymbol{x})\circ \boldsymbol{x}\|_2^2\d \boldsymbol{x}$, where $\circ$ denotes entrywise multiplication. In this loss, boundary effects are dampened by multiplying gradients elementwise with the identity functions $x_j$. 

In this paper, we propose \emph{generalized score matching} methods that are based on element\-wise multiplication with functions other than $x_j$. As we show, this can lead to drastically improved estimation accuracy, both theoretically and empirically.  To demonstrate these advantages, we consider a family of 
graphical models on $\mathbb{R}_+^m$, which does not have tractable normalizing constants and hence serves as a practical example.

\emph{Graphical models} specify conditional independence relations for a
random vector $\boldsymbol{X}=(X_{i})_{i\in V}$ indexed by the nodes
of a graph \citep{lau96}.
For undirected graphs, variables $X_{i}$ and $X_{j}$ are required to be conditionally independent given  $(X_{k})_{k\not=i,j}$ if there is no edge between $i$ and $j$. The smallest undirected graph with this property is the \emph{conditional independence graph} of $\boldsymbol{X}$.  Estimation of this graph and associated interaction parameters has been a topic of continued research as reviewed by \citet{drt17}.

Largely due to their tractability, Gaussian graphical models (GGMs)
have gained great popularity.  The conditional independence graph of
a multivariate normal vector
$\boldsymbol{X}\sim\mathcal{N}(\boldsymbol{\mu},\boldsymbol{\Sigma})$
is determined by the \emph{inverse covariance matrix}
$\mathbf{K}\equiv\boldsymbol{\Sigma}^{-1}$, also termed
\emph{concentration} or \emph{precision matrix}.  Specifically, $X_{i}$ and $X_{j}$ are conditionally independent given all other variables  if and only if the $(i,j)$-th and the $(j,i)$-th entries of $\mathbf{K}$ are both zero. This simple relation underlies a rich literature including \citet{drt04}, \citet{mei06}, \citet{yua07} and \citet{fri08}, among others.

More recent work has provided tractable procedures also for
non-Gaussian graphical models. This includes Gaussian copula models
\citep{liu09,dob11,liu12}, Ising models \citep{rav10}, other
exponential family models \citep{che14,yan15}, as well as semi- or
non-parametric estimation techniques \citep{fel13,voo13}. In this
paper, we apply our method to a class of pairwise interaction models
that generalizes non-negative Gaussian random variables, as recently
considered by \citet{lin16} and \citet{yu16}, as well as square root
graphical models proposed by \citet{ino16} when the sufficient
statistic function is a pure power. However, our main ideas can also
be applied for other classes of exponential families whose
support is restricted to a rectangular set.

Our focus will be on \emph{pairwise interaction power models} with
probability distributions having (Lebesgue) densities proportional to
\begin{equation}\label{eq:ab-density}
\exp\left\{-\frac{1}{2a}{\boldsymbol{x}^a}^{\top}\mathbf{K}\boldsymbol{x}^a+\boldsymbol{\eta}^{\top}\frac{\boldsymbol{x}^b-\mathbf{1}_m}{b}\right\}
\end{equation}
on $\mathbb{R}_+^m\equiv[0,\infty)^m$. Here $a>0$ and $b\ge 0$ are known
constants, and $\mathbf{K}\in\mathbb{R}^{m\times m}$ and
$\boldsymbol{\eta}\in\mathbb{R}^{m}$ are unknown parameters of
interest. When $b=0$ we define $(x^b-1)/b\equiv\log x$ and
$\mathbb{R}_+^m\equiv(0,\infty)^m$.  This class of models is motivated
by the form of important univariate distributions for non-negative
data, including gamma and truncated normal distributions.  It provides a
framework for pairwise interaction that is concrete yet rich enough to
capture key differences in how densities may behave at the boundary of
the non-negative orthant, $\mathbb{R}_+^m$.  Moreover, the conditional
independence graph of a random vector $\boldsymbol{X}$ with
distribution as in~(\ref{eq:ab-density}) is determined just as in the
Gaussian case: $X_i$ and $X_j$ are conditionally independent given all
other variables if and only if $\kappa_{ij}=\kappa_{ji}=0$ in the
interaction matrix $\mathbf{K}$.  Section \ref{A General Framework of Pairwise Interaction Models} gives further details on these models. We will 
develop estimators of
$(\boldsymbol{\eta},\mathbf{K})$ in (\ref{eq:ab-density}) and the associated conditional independence graph using the proposed \emph{generalized score matching}.

A special case of (\ref{eq:ab-density}) are truncated Gaussian
graphical models, with $a=b=1$.  Let
$\boldsymbol{\mu}\in\mathbb{R}^m$, and let $\mathbf{K}$ be a positive
definite matrix.  Then a non-negative random vector $\boldsymbol{X}$
follows a truncated normal distribution for mean parameter
$\boldsymbol{\mu}$ and inverse covariance parameter $\mathbf{K}$, in
symbols $\boldsymbol{X}\sim\mathrm{TN}(\boldsymbol{\mu},\mathbf{K})$,
if it has density proportional to
\begin{equation}
  \label{eq:tn-density}
  \exp\left\{-\frac{1}{2}(\boldsymbol{x}-\boldsymbol{\mu})^{\top}\mathbf{K}(\boldsymbol{x}-\boldsymbol{\mu})\right\}
\end{equation}
on $\mathbb{R}_+^m$. We refer to
$\boldsymbol{\Sigma}=\mathbf{K}^{-1}$ as the covariance parameter of
the distribution, and note that the $\boldsymbol{\eta}$ parameter in (\ref{eq:ab-density}) is $\mathbf{K}\boldsymbol{\mu}$.  Another special case of (\ref{eq:ab-density}) is the exponential square root graphical models in \citet{ino16}, where $a=b=1/2$.

\citet{lin16} estimate truncated GGMs based on Hyv\"arinen's modification, with an $\ell_1$ penalty on the entries of $\mathbf{K}$ added to the loss. However, the paper overlooks the fact that the loss can be unbounded from below in the high-dimensional setting even with an $\ell_1$ penalty, such that no minimizer may exist. Since the unpenalized loss is quadratic in the parameter to be estimated, we propose modifying it by adding small positive values to the diagonals of the positive semi-definite matrix that defines the quadratic part, in order to ensure that the loss is bounded and strongly convex and admits a unique minimizer. We apply this to the estimator for GGMs considered in \citet{lin16}, which uses score-matching on $\mathbb{R}^m$, and to the \emph{generalized score matching} estimator for pairwise interaction power models on $\mathbb{R}_+^m$ proposed in this paper. In these cases, we show, both empirically and theoretically, that the consistency results still hold (or even improve) if the positive values added are smaller than a threshold that is readily computable.

The rest of the paper is organized as follows. Section \ref{Score
  Matching} introduces score matching and our proposed
\emph{generalized score matching}. In Section
\ref{Exponential_Families}, we apply generalized score matching to
exponential families, with univariate truncated normal distributions
as an example. \emph{Regularized generalized score matching} for
graphical models is formulated in Section \ref{Regularized Generalized
  Score Matching}. The estimators for pairwise interaction
power models are shown in Section \ref{sec:graph-models-trunc}, while
theoretical consistency results are presented in Section \ref{Theory
  for Graphical Models}, where we treat the probabilistically most
tractable case of truncated GGMs. Simulation results and applications
to RNAseq data are given in Section \ref{Numerical
  Experiments}. Proofs for theorems in Sections \ref{Score
  Matching}--\ref{Theory for Graphical Models} are presented in
Appendices~\ref{Proofs} and~\ref{AppA}. 
Additional experimental results are presented in Appendix \ref{ER}.



\subsection{Notation}

Constant scalars, vectors, and functions are written in lower-case
(e.g., $a$, $\boldsymbol{a}$), random scalars and vectors in
upper-case (e.g., $X$, $\boldsymbol{X}$).  Regular font is used for
scalars (e.g.~$a$, $X$), and boldface for vectors
(e.g.~$\boldsymbol{a}$, $\boldsymbol{X}$).  Matrices are in upright
bold, with constant matrices in upper-case ($\mathbf{K}$,
$\mathbf{M}$) and random matrices holding observations in lower-case
($\mathbf{x}$, $\mathbf{y}$). Subscripts refer to entries in vectors
and columns in matrices. Superscripts refer to rows in matrices.  So
$X_j$ is the $j$-th component of a random vector $\boldsymbol{X}$.
For a data matrix $\mathbf{x}\in\mathbb{R}^{n\times m}$, each row
comprising one observation of $m$ variables/features, $X_{j}^{(i)}$ is
the $j$-th feature for the $i$-th observation.  Stacking the columns
of a matrix $\mathbf{K}=[\kappa_{ij}]_{i,j}\in\mathbb{R}^{q\times r}$
gives its vectorization
$\mathrm{vec}(\mathbf{K})=(\kappa_{11},\ldots,\kappa_{q1},\kappa_{12},\ldots,\kappa_{q2},\ldots,\kappa_{1r},\ldots,\kappa_{qr})^{\top}$.
For a matrix $\mathbf{K}\in\mathbb{R}^{q\times q}$,
$\mathrm{diag}(\mathbf{K})\in\mathbb{R}^q$ denotes its diagonal, and
for a vector $\boldsymbol{v}\in\mathbb{R}^q$,
$\mathrm{diag}(\boldsymbol{v})$ is the $q\times q$ diagonal matrix
with diagonals $v_1,\dots,v_q$.

For $a\geq 1$, the $\ell_{a}$-norm of a vector $\boldsymbol{v}\in\mathbb{R}^q$ is denoted 
\[\|\boldsymbol{v}\|_a=\Bigg(\sum_{j=1}^q |v_j|^a\Bigg)^{1/a},\] with
$\|\boldsymbol{v}\|_{\infty}=\max\limits_{j=1,\ldots,q}|v_j|$.  A
matrix $\mathbf{K}=[\kappa_{ij}]_{i,j}\in\mathbb{R}^{q\times r}$ has Frobenius norm
\[\mnorm{\mathbf{K}}_{F}\equiv\|\mathrm{vec}(\mathbf{K})\|_{2}\equiv\sqrt{\sum_{i=1}^q\sum_{j=1}^r\kappa_{ij}^2},\]
and max norm
$\|\mathbf{K}\|_{\infty}\equiv\|\mathrm{vec}(\mathbf{K})\|_{\infty}\equiv\max\limits_{i,j}|\kappa_{ij}|.$
Its $\ell_a$-$\ell_b$ operator norm  is
\[\mnorm{\mathbf{K}}_{a,b}\equiv\max_{\boldsymbol{x}\neq\boldsymbol{0}}\frac{\|\mathbf{K}\boldsymbol{x}\|_b}{\|\boldsymbol{x}\|_a}\]
with shorthand notation
$\mnorm{\mathbf{K}}_{a}\equiv\mnorm{\mathbf{K}}_{a,a}$; for instance,
$\mnorm{\mathbf{K}}_{\infty}\equiv\max\limits_{i=1.\ldots,q}\sum\limits_{j=1}^{r}|\kappa_{ij}|$. 

For a function $f:\mathbb{R}^m\to\mathbb{R}$, we define
$\partial_j f(\boldsymbol{x})$ as the partial derivative with respect
to 
$x_j$, and
$\partial_{jj}f(\boldsymbol{x})=\partial_j\partial_j
f(\boldsymbol{x})$.  For $\boldsymbol{f}:\mathbb{R}\to\mathbb{R}^m$,
$\boldsymbol{f}(x)=(f_1(x),\ldots,f_m(x))^{\top}$, we let
$\boldsymbol{f}'(x)=(f_1'(x), \ldots, f_m'(x))^{\top}$ be the vector
of derivatives.  Likewise $\boldsymbol{f}''(x)$ is used for second
derivatives.  The symbol $\mathds{1}_A(\cdot)$ denotes the indicator
function of the set $A$, while $\mathbf{1}_n\in\mathbb{R}^n$ is the
vector of all $1$'s. For $\boldsymbol{a}$,
$\boldsymbol{b}\in\mathbb{R}^m$,
$\boldsymbol{a}\circ\boldsymbol{b}\equiv(a_1b_1,\ldots,a_mb_m)^{\top}$.
A density of a distribution is always a probability density function
with respect to Lebesgue measure.  When it is clear from the context,
$\mathbb{E}_0$ denotes the expectation under a true distribution $P_0$.

\section{Score Matching}\label{Score Matching}
In this section, we review the original score matching and develop our generalized score matching estimators.
\subsection{Original Score Matching}
\label{sec:orig-score-match}

Let $\boldsymbol{X}$ be a random vector taking values in $\mathbb{R}^m$ with distribution $P_0$ and density $p_0$. Let $\mathcal{P}$ be a family of distributions of interest with twice continuously differentiable densities supported on $\mathbb{R}^m$. Suppose $P_0\in\mathcal{P}$. 
The \emph{score matching loss} for $P\in\mathcal{P}$, with density $p$, is given by
\begin{equation}\label{eq_sm}
J(P)=\int_{\mathbb{R}^m}p_0(\boldsymbol{x})\|\nabla\log p(\boldsymbol{x})-\nabla\log p_0(\boldsymbol{x})\|_2^2\d\boldsymbol{x}.
\end{equation}
The gradients in~(\ref{eq_sm}) can be thought of as gradients with respect to a hypothetical location parameter, evaluated at the origin \citep{hyv05}.
The loss $J(P)$ is minimized if and only if $P=P_0$, which forms the
basis for estimation of $P_0$.  Importantly, since the loss depends on
$p$ only through its log-gradient, it suffices to know $p$ up to a
normalizing constant. Under mild conditions, (\ref{eq_sm}) can be
rewritten as
\begin{equation}\label{eq_sm_eq}
J(P)=\int_{\mathbb{R}^m}p_0(\boldsymbol{x})\sum\limits_{j=1}^m\left[\partial_{jj}\log p(\boldsymbol{x})+\frac{\left(\partial_{j}\log p(\boldsymbol{x})\right)^2}{2}\right]\d \boldsymbol{x},
\end{equation}
plus a constant independent of $p$.  The integral in (\ref{eq_sm_eq})
can be approximated by a sample average; this alleviates the need for knowing the true density $p_0$, and provides a way to estimate $p_0$.

\subsection{Generalized Score Matching for Non-Negative Data}\label{Generalized Score Matching for Non-Negative Data}

When the true density $p_0$ is supported on a proper subset of
$\mathbb{R}^m$, the integration by parts underlying the equivalence of
(\ref{eq_sm}) and (\ref{eq_sm_eq}) may fail due to discontinuity at the boundary.  For distributions supported on the
non-negative orthant, $\mathbb{R}_+^m$, \citet{hyv07} addressed this
issue by instead minimizing the \emph{non-negative score matching
  loss}
\begin{equation}\label{eq_nn_sm}
J_+(P)=\int_{\mathbb{R}^m_+}p_0(\boldsymbol{x})\|\nabla\log p(\boldsymbol{x})\circ\boldsymbol{x}-\nabla\log p_0(\boldsymbol{x})\circ\boldsymbol{x}\|_2^2\d\boldsymbol{x}.
\end{equation}
This loss can be motivated via gradients with respect to a
hypothetical scale parameter \citep{hyv07}. Under mild conditions,
$J_+(P)$ can again be rewritten in terms of an expectation of a
function independent of $p_0$, thus allowing one to form a sample
loss.

In this work, we consider generalizing the non-negative score matching
loss as follows.

\begin{definition}\label{definition_GSM_loss}
  Let $\mathcal{P}_+$ be the family of distributions of interest, and
  assume every $P\in\mathcal{P}_+$ has a twice continuously
  differentiable density supported on $\mathbb{R}_+^m$. Suppose the
  $m$-variate random vector $\boldsymbol{X}$ has true distribution
  $P_0\in\mathcal{P}_+$, and let $p_0$ be its twice continuously
  differentiable density.  Let
  $h_1,\dots,h_m:\mathbb{R}_+\to\mathbb{R}_+$ be a.s.~positive
  functions that are absolutely continuous in every bounded
  sub-interval of $\mathbb{R}_+$,
  and set $\boldsymbol{h}(\boldsymbol{x})=(h_1(x_1),\dots,h_m(x_m))^{\top}$. For $P\in\mathcal{P}_+$ with density $p$, the \emph{generalized
    $\boldsymbol{h}$-score matching loss} is
\begin{equation}\label{eq_gsm}
J_{\boldsymbol{h}}(P)=\int_{\mathbb{R}_+^m}\frac{1}{2}p_0(\boldsymbol{x})\|\nabla\log p(\boldsymbol{x})\circ \boldsymbol{h}(\boldsymbol{x})^{1/2}-\nabla\log p_0(\boldsymbol{x})\circ \boldsymbol{h}(\boldsymbol{x})^{1/2}\|_2^2\d\boldsymbol{x},
\end{equation}
where $\boldsymbol{h}^{1/2}(\boldsymbol{x})\equiv(h_1^{1/2}(x_1),\ldots,h_m^{1/2}(x_m))^{\top}$.
\end{definition}

\begin{proposition}
  \label{theorem_gsm}
  The distribution $P_0$ is the unique minimizer of
  ${J}_{\boldsymbol{h}}(P)$ for $P\in\mathcal{P}_+$.
\end{proposition}

\begin{proof} 
  First, observe that $J_{\boldsymbol{h}}(P)\geq 0$ and
  $J_{\boldsymbol{h}}(P_0)=0$.  For uniqueness, suppose
  $J_{\boldsymbol{h}}(P_1)=0$ for some $P_1\in\mathcal{P}_+$.  Let
  $p_0$ and $p_1$ be the respective densities.  By
  assumption $p_0(\boldsymbol{x})>0$ a.s.~and
  $h_j^{1/2}(\boldsymbol{x})>0$ a.s.~for all $j=1,\dots,m$.  Therefore,
  we must have
  $\nabla\log p_1(\boldsymbol{x})=\nabla\log
  p_0(\boldsymbol{x})$ a.s., or equivalently,
  $p_1(\boldsymbol{x})=\mathrm{const}\times
  p_0(\boldsymbol{x})$ almost surely in
  $\mathbb{R}_+^m$. Since $p_1$
  and $p_0$ are continuous densities supported on $\mathbb{R}_+^m$, it follows that
  $p_1(\boldsymbol{x})=p_0(\boldsymbol{x})$ for
  all $\boldsymbol{x}\in\mathbb{R}_+^m$.
\end{proof}

Choosing all $h_j(x)=x^2$ recovers the loss from (\ref{eq_nn_sm}).  In
our generalization, we will focus on using functions $h_j$ that are
increasing but are bounded or grow rather slowly.
This will alleviate the need to estimate higher moments,
leading to better practical performance and improved theoretical
guarantees. 

We will consider the following assumptions:
\begin{align*}
(\text{A1}) &\,\,\, p_0(\boldsymbol{x})h_j(x_j)\partial_j\log p(\boldsymbol{x})\left|^{x_j\nearrow+\infty}_{x_j\searrow 0^+}=0\right.,\quad\forall\boldsymbol{x}_{-j}\in\mathbb{R}_{+}^{m-1},\quad \forall p\in\mathcal{P}_+;\\
(\text{A2}) &\,\,\, \mathbb{E}_{p_0}\|\nabla\log p(\boldsymbol{X})\circ \boldsymbol{h}^{1/2}(\boldsymbol{X})\|_2^2<+\infty, \,\,\, \mathbb{E}_{p_0}\|(\nabla\log p(\boldsymbol{X})\circ\boldsymbol{h}(\boldsymbol{X}))'\|_1<+\infty,\,\,\,\forall p\in\mathcal{P}_+,
\end{align*}
where $\partial_j\log p(\boldsymbol{x})\equiv\left.\frac{\partial\log p(\boldsymbol{y})}{\partial y_j}\right|_{\boldsymbol{y}=\boldsymbol{x}}$,\,$f(\boldsymbol{x})\left|_{x_j\searrow 0^+}^{x_j\nearrow+\infty}\right.\equiv\lim_{x_j\nearrow+\infty}f(\boldsymbol{x})-\lim_{x_j\searrow 0}f(\boldsymbol{x})$, ``$\forall p\in\mathcal{P}_+$'' is a shorthand for ``for all $p$ being the density of some $P\in\mathcal{P}_+$'', and the prime symbol denotes component-wise differentiation. While the second half of (A2) was not made explicit in  \citet{hyv05,hyv07}, (A1)-(A2) were both required for integration by parts and Fubini-Tonelli to apply.

Once the forms of $p_0$ and $p$ are given, sufficient conditions for
$\boldsymbol{h}$ for Assumptions (A1)-(A2) to hold are easy to find.
In particular, (A1) and (A2) are easily satisfied and verified for exponential families.

Integration by parts yields the following theorem which shows that
$J_{\boldsymbol{h}}$ from~(\ref{eq_gsm}) is an expectation (under
$P_0$) of a function that does not depend on $p_0$, similar to
(\ref{eq_sm_eq}). The proof is given in Appendix~\ref{Proof of Theorem theorem_GSM_loss_alt}.

\begin{theorem}\label{theorem_GSM_loss_alt}
Under (A1) and (A2), the loss from (\ref{eq_gsm}) equals
\begin{multline}
J_{\boldsymbol{h}}(P)=\int_{\mathbb{R}_+^m}p_0(\boldsymbol{x})\sum_{j=1}^m\left[h_j'(x_j)\partial_j (\log p(\boldsymbol{x}))+h_j(x_j)\partial_{jj}(\log p(\boldsymbol{x}))\phantom{+\frac{1}{2}}\right.\\
\left.+\frac{1}{2}h_j(x_j)\left(\partial_j(\log p(\boldsymbol{x}))\right)^2\right]\d\boldsymbol{x}\label{eq_gsm_eq}
\end{multline}
plus a constant independent of $p$.
\end{theorem}


Given a data matrix $\mathbf{x}\in\mathbb{R}^{n\times m}$ with rows
$\boldsymbol{X}^{(i)}$, we define the sample version of (\ref{eq_gsm_eq})
as
\begin{multline}
\hat{J}_{\boldsymbol{h}}(P)=\frac{1}{n}\sum_{i=1}^n\sum_{j=1}^m\left\{h_j'(X_j^{(i)})\partial_j (\log p(\boldsymbol{X}^{(i)}))\phantom{\left[\frac{1}{2}\left(\partial_j(\log p(\boldsymbol{X}^{(i)}))\right)^2\right]}\right.\\
\left.+h_j(X_j^{(i)})\left[\partial_{jj}(\log p(\boldsymbol{X}^{(i)}))+\frac{1}{2}\left(\partial_j(\log p(\boldsymbol{X}^{(i)}))\right)^2\right]\right\}.\label{eq_gsm_sample}
\end{multline}

\noindent
Subsequently, for a distribution $P$ with density $p$, we let
$J_{\boldsymbol{h}}(p)\equiv J_{\boldsymbol{h}}(P)$.  Similarly, when
a distribution $P_{\boldsymbol{\theta}}$ with density
$p_{\boldsymbol{\theta}}$ is associated to a parameter vector
$\boldsymbol{\theta}$, we write
$J_{\boldsymbol{h}}(\boldsymbol{\theta})\equiv
J_{\boldsymbol{h}}(p_{\boldsymbol{\theta}}) \equiv
J_{\boldsymbol{h}}(P_{\boldsymbol{\theta}})$.  We apply similar
conventions to the sample version $\hat{J}_{\boldsymbol{h}}(P)$.  We
note that this type of loss is also treated in slightly different
settings in \citet{par16} and \citet{alm93}.

\begin{remark}\rm
In the one-dimensional case, using the notation in \citet{par12}, $J_{\boldsymbol{h}}(P)$ and $\hat{J}_{\boldsymbol{h}}(P)$ correspond to $d(P_0,P)$ and $S(x,P)$, respectively, and can be generated by $\phi(x,p,p_1)\equiv-h(x)p_1^2/(2p)$ (c.f.~Equations (39), (51), (53) and Section 10.1 therein). Thus Theorem \ref{theorem_GSM_loss_alt} follows from this correspondence. While (A1) is equivalent to the condition implied by the boundary divergence $d_b=0$ in that paper, (A2), which we assume for invoking Fubini-Tonelli due to multi-dimensionality, is not present. On the other hand, while \citet{par16} treats the multivariate case, it does not cover the connection between our $J_{\boldsymbol{h}}$ and $\hat{J}_{\boldsymbol{h}}$. Since $\phi$ is concave but not strictly concave in $(p,p_1)$, the results in \citet{par16} only imply that $P_0$ is \emph{a} minimizer, a weaker conclusion than Proposition \ref{theorem_gsm}.

\end{remark}




\section{Exponential Families}\label{Exponential_Families}

In this section, we study the case where $\mathcal{P}_+\equiv\{p_{\boldsymbol{\theta}}:\boldsymbol{\theta}\in\boldsymbol{\Theta}\}$ is an
exponential family comprising continuous distributions with support
$\mathbb{R}_+^m$.  More specifically, we consider densities that are indexed by the canonical parameter
$\boldsymbol{\theta}\in\mathbb{R}^{r}$ and have the form
\begin{equation}\label{definition_exponential_family}
\log p_{\boldsymbol{\theta}}(\boldsymbol{x})=\boldsymbol{\theta}^{\top}\boldsymbol{t}(\boldsymbol{x})-\psi(\boldsymbol{\theta})+b(\boldsymbol{x}),\quad\boldsymbol{x}\in\mathbb{R}_+^m,
\end{equation}
where $\boldsymbol{t}(\boldsymbol{x})\in\mathbb{R}_+^{r}$ comprises the sufficient statistics, $\psi(\boldsymbol{\theta})$ is a normalizing constant depending on $\boldsymbol{\theta}$ only, and $b(\boldsymbol{x})$ is the base measure, with $\boldsymbol{t}$ and $b$ a.s.~differentiable with respect to each component. Define $\boldsymbol{t_j'}(\boldsymbol{x})\equiv(\partial_j t_1(\boldsymbol{x}),\ldots,\partial_j t_{r}(\boldsymbol{x}))^{\top}$ and $b_j'(\boldsymbol{x})\equiv\partial_j b(\boldsymbol{x})$.
 
%

\begin{theorem}\label{theorem_GSM_estimator_exp}
Under Assumptions (A1)-(A2) from Section~\ref{Generalized Score Matching for Non-Negative Data}, the empirical generalized $\boldsymbol{h}$-score matching loss (\ref{eq_gsm_sample}) can be rewritten as a quadratic function in $\boldsymbol{\theta}\in\mathbb{R}^{r}$:
\begin{align}\label{eq_gsm_exponential}
\hspace{-0.02in}\hat{J}_{\boldsymbol{h}}(p_{\boldsymbol{\theta}})&=\frac{1}{2}\boldsymbol{\theta}^{\top}\boldsymbol{\Gamma}(\mathbf{x})\boldsymbol{\theta}-\boldsymbol{g}(\mathbf{x})^{\top}\boldsymbol{\theta}+\mathrm{const},\quad\text{where}\\
\boldsymbol{\Gamma}(\mathbf{x})&=\frac{1}{n}\sum_{i=1}^{n}\sum_{j=1}^mh_j(X_j^{(i)})\boldsymbol{t}_j'(\boldsymbol{X}^{(i)})
\boldsymbol{t}_j'(\boldsymbol{X}^{(i)})^{\top}\quad\text{and}\label{def_Gamma}\\
\boldsymbol{g}(\mathbf{x})&=-\frac{1}{n}\sum_{i=1}^n\sum_{j=1}^m\left[h_j(X_j^{(i)})b_j'(\boldsymbol{X}^{(i)})
\boldsymbol{t}'_j(\boldsymbol{X}^{(i)})+h_j(X_j^{(i)})\boldsymbol{t}_j''(\boldsymbol{X}^{(i)})+
h_j'(X_j^{(i)})\boldsymbol{t}_j'(\boldsymbol{X}_i)\right]\label{def_g}
\end{align}
are sample averages of functions of the data matrix $\mathbf{x}$ only. 
\end{theorem}

Define $\boldsymbol{\Gamma}_0\equiv\mathbb{E}_{p_0}\boldsymbol{\Gamma}(\mathbf{x})$, $\boldsymbol{g}_0\equiv\mathbb{E}_{p_0}\boldsymbol{g}(\mathbf{x})$, and $\boldsymbol{\Sigma}_0\equiv\mathbb{E}_{p_0}[(\boldsymbol{\Gamma}(\mathbf{x})\boldsymbol{\theta}_0-g(\mathbf{x}))(\boldsymbol{\Gamma}(\mathbf{x})\boldsymbol{\theta}_0-g(\mathbf{x}))^{\top}]$.
\begin{theorem}\label{theorem_exponential}
Suppose that
 \begin{itemize}
 \item[] \hspace{-.5cm}(C1) \ $\boldsymbol{\Gamma}$ is a.s.~invertible, and  \\[-0.55cm]
  \item[] \hspace{-.5cm}(C2) \ $\boldsymbol{\Gamma}_0$, $\boldsymbol{\Gamma}_0^{-1}$, $\boldsymbol{g}_0$
    and $\boldsymbol{\Sigma}_0$ exist and are entry-wise finite.
  \end{itemize}
  Then the minimizer of (\ref{eq_gsm_exponential}) is a.s.~unique with closed-form solution $\hat{\boldsymbol{\theta}}\equiv\boldsymbol{\Gamma}(\mathbf{x})^{-1}\boldsymbol{g}(\mathbf{x})$. Moreover, 
\begin{align*}
\hat{\boldsymbol{\theta}}\to_{\text{a.s.}}\boldsymbol{\theta}_0
\quad\text{and}\quad\sqrt{n}(\hat{\boldsymbol{\theta}}-\boldsymbol{\theta}_0)\to_d\mathcal{N}_{r}\left(\boldsymbol{0},\boldsymbol{\Gamma}_0^{-1}\boldsymbol{\Sigma}_0\boldsymbol{\Gamma}_0^{-1}\right)\quad\text{as }n\to\infty.
\end{align*}
\end{theorem}

Theorems \ref{theorem_GSM_estimator_exp} and \ref{theorem_exponential}
are proved in Appendix~\ref{Proof of Theorems and Examples in Section Exponential_Families}.
Theorem~\ref{theorem_GSM_estimator_exp} clarifies the quadratic nature
of the loss, and Theorem \ref{theorem_exponential} provides a basis for
asymptotically valid tests and confidence intervals for the parameter
$\boldsymbol{\theta}$.  Note that Condition (C1) holds if and only if
$h_j(X_j)>0$ a.s.~and $[\boldsymbol{t}_j'(\boldsymbol{X}^{(1)}),\ldots,\boldsymbol{t}_j'(\boldsymbol{X}^{(n)})]\in\mathbb{R}^{{r}\times n}$ has rank ${r}$ a.s.~for some $j=1,\ldots,m$.

The conclusion in Theorem \ref{theorem_exponential} indicates that, similar to the estimator in \citet{hyv07} with $h_j(x)=x^2$, the closed-form solution for our generalized $\hat{\boldsymbol{\theta}}$ allows one to consistently estimate the canonical parameter in an exponential family distribution without needing to calculate the often complicated normalizing constant $\psi(\boldsymbol{\theta})$ or resort to numerical methods. Computational details are explicated in Section \ref{Computational Details}.

Below we illustrate the estimator $\hat{\boldsymbol{\theta}}$ in the case of univariate truncated normal distributions. We assume (A1)-(A2) and (C1)-(C2) throughout.

\begin{example}\label{theorem_GSM_normality_tnorm_mu}
Univariate ($m={r}=1$) truncated normal distributions for
mean parameter $\mu$ and variance parameter $\sigma^2$ have density 
\begin{equation}\label{pdf_univariate}
p_{\mu,\sigma^2}(x)\propto\exp\left\{-\frac{(x-\mu)^2}{2\sigma^2}\right\},\quad x\in\mathbb{R}_+.
\end{equation}
If $\sigma^2$ is known but $\mu$ unknown, then writing the density in canonical form as
in (\ref{definition_exponential_family}) yields
\[p_{\theta}(x)\propto\exp\left\{\theta t(x)+b(x)\right\},\quad \theta\equiv\frac{\mu}{\sigma^2},\quad t(x)\equiv x,\quad b(x)=-\frac{x^2}{2\sigma^2}.\]
Given an i.i.d.~sample $X_1,\dots,X_n\sim p_{\mu_0,\sigma^2}$, the generalized $h$-score matching estimator of $\mu$ is
\[\hat{\mu}_{h}\equiv\frac{\sum_{i=1}^nh(X_i)X_i-\sigma^2h'(X_i)}{\sum_{i=1}^n h(X_i)}.
\]
If $\lim_{x\searrow 0^+}h(x)=0$, $\lim_{x\nearrow+\infty}h^2(x)(x-\mu_0)p_{\mu_0,\sigma^2}(x)=0$ and the expectations are finite (for example, when $h(x)=o(\exp(Mx^2))$ for $M<\frac{1}{4\sigma^2}$), then 
\[\sqrt{n}(\hat{\mu}_{h}-\mu_0)\to_d\mathcal{N}\left(0,\frac{\mathbb{E}_{0}[\sigma^2 h^2(X)+\sigma^4 {h'}^{2}(X)]}{\mathbb{E}_{0}^2[h(X)]}\right).\]
We recall that the \emph{Cram\'er-Rao lower bound} (i.e.~the lower bound on the variance of any unbiased estimator) for estimating $\mu_{}$ is \[\frac{\sigma^4}{\var(X-\mu_0)}.\]
\end{example}

\begin{example}\label{theorem_GSM_normality_tnorm_sigma}
  Consider the univariate truncated normal distributions from
  (\ref{pdf_univariate}) in the setting where the mean parameter $\mu$
  is known but the variance parameter $\sigma^2>0$ is unknown.
In canonical form as in (\ref{definition_exponential_family}), we write
\[p_{\theta}(x)\propto\exp\left\{\theta t(x)+b(x)\right\},\quad \theta\equiv\frac{1}{\sigma^2},\quad t(x)\equiv -(x-\mu)^2/2,\quad b(x)=0.\]
Given an i.i.d.~sample $X_1,\dots,X_n\sim p_{\mu,\sigma_0^2}$, the generalized $h$-score matching estimator of $\sigma^2$ is
\[
\hat{\sigma}_{h}^2\equiv\frac{\sum_{i=1}^nh(X_i)(X_i-\mu)^2}{\sum_{i=1}^n
  h(X_i)+h'(X_i)(X_i-\mu)}.
\]
If, in addition to the assumptions in Example \ref{theorem_GSM_normality_tnorm_mu}, $\lim\limits_{x\nearrow+\infty}h^2(x)(x-\mu)^3p_{\mu,\sigma_{0}^{2}}(x)=0$, then 
\[\sqrt{n}(\hat{\sigma}_{h}^2-\sigma_0^2)\to_d\mathcal{N}\left(0,\frac{2\sigma^6_{0}\mathbb{E}_{0}[h^2(X)(X-\mu)^2]+\sigma^8_{0}\mathbb{E}_{0}[h'^{2}(X)(X-\mu)^2]}{\mathbb{E}^2_{0}[h(X)(X-\mu)^2]}\right).\]
Moreover, the \emph{Cram\'er-Rao lower bound} for estimating $\sigma^2_{}$ is \[\frac{4\sigma_0^8}{\var(X-\mu)^2}.\]
\end{example}

\begin{remark}
  \rm In Example \ref{theorem_GSM_normality_tnorm_sigma}, if $\mu_{0}=0$, then $h(x)\equiv 1$ also satisfies (A1)-(A2) and
  (C1)-(C2) and one recovers the sample variance
  $\tfrac{1}{n}\sum_i X_i^2$, which obtains the Cram\'er-Rao lower
  bound.
\end{remark}

In these examples, there is a benefit in using a bounded function $h$, which can be explained as follows. When $\mu\gg\sigma$, there is effectively no truncation to the Gaussian distribution, and our method adapts to using low moments in (\ref{eq_gsm}), since a bounded and increasing $h(x)$ becomes almost constant as it reaches its asymptote for $x$ large.  Hence, we effectively revert to the original score matching (recall Section~\ref{sec:orig-score-match}). In the other cases, the truncation effect is significant and 
our estimator uses higher moments accordingly.


Figure \ref{plot_univariate_mu_var} plots the asymptotic
variance of $\hat{\mu}_{h}$ from Example
\ref{theorem_GSM_normality_tnorm_mu}, with $\sigma=1$
known. Efficiency as measured by the Cram\'er-Rao lower bound divided
by the asymptotic variance is also shown. We see that two truncated
versions of $\log(1+x)$ have asymptotic variance close to the
Cram\'er-Rao
bound. 
This asymptotic variance is also reflective of the variance for smaller
finite samples.
\begin{figure}[htp]
  \centering \includegraphics[scale=0.35]{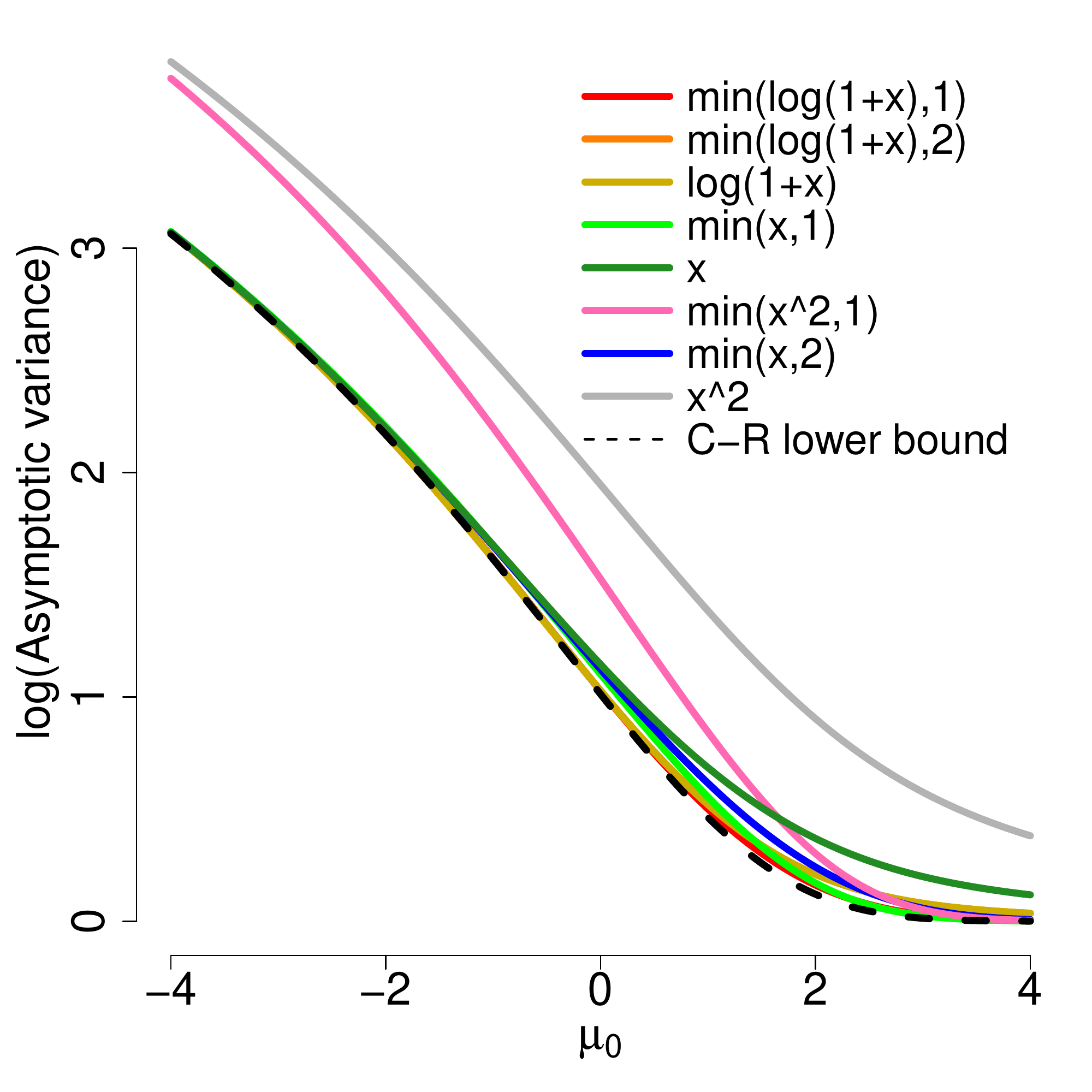}
  \includegraphics[scale=0.35]{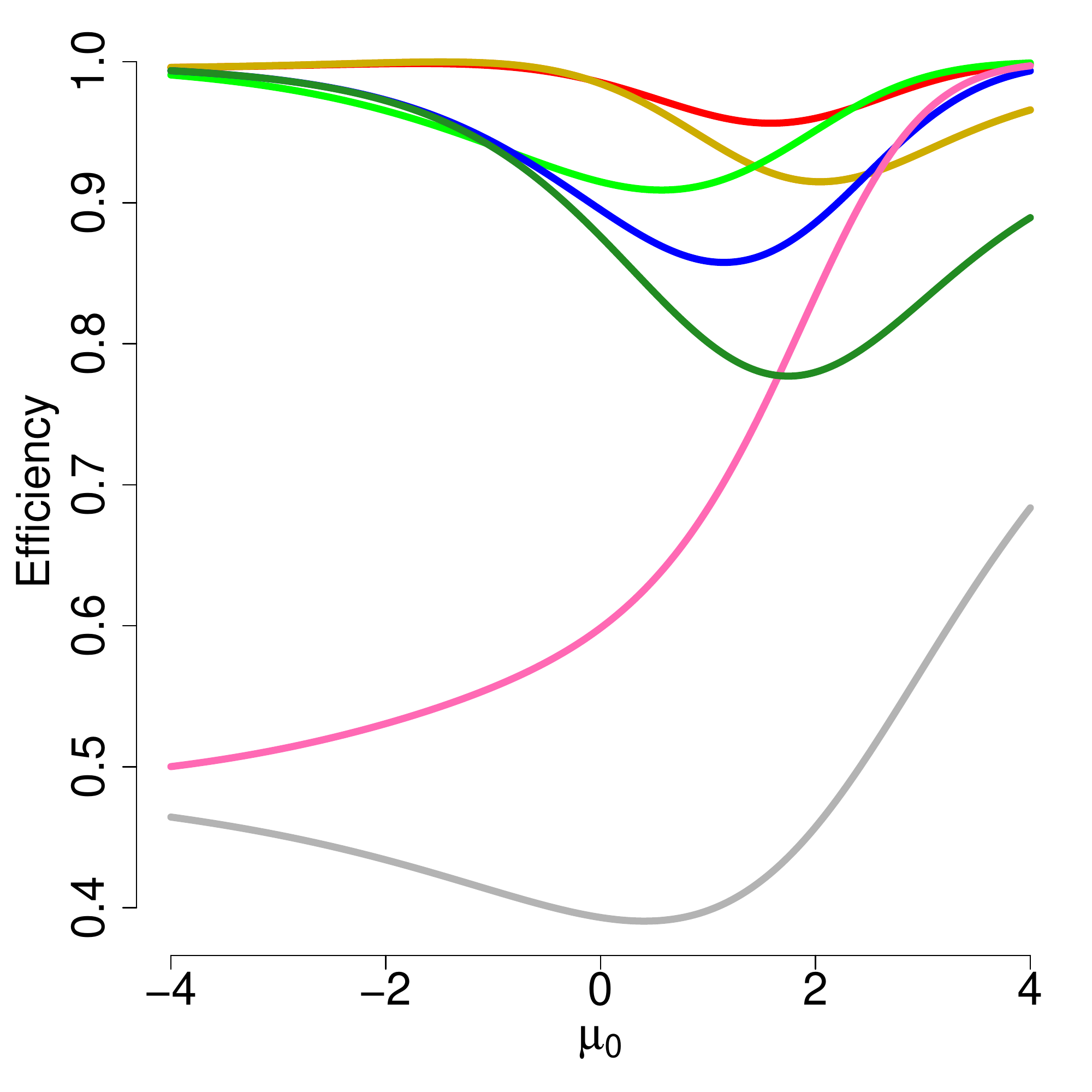}\caption{Log of asymptotic variance and efficiency with respect to the
    Cram\'er-Rao bound for $\hat{\mu}_h$ ($\sigma^{2}=1$
    known).}\label{plot_univariate_mu_var} \centering
  \includegraphics[scale=0.35]{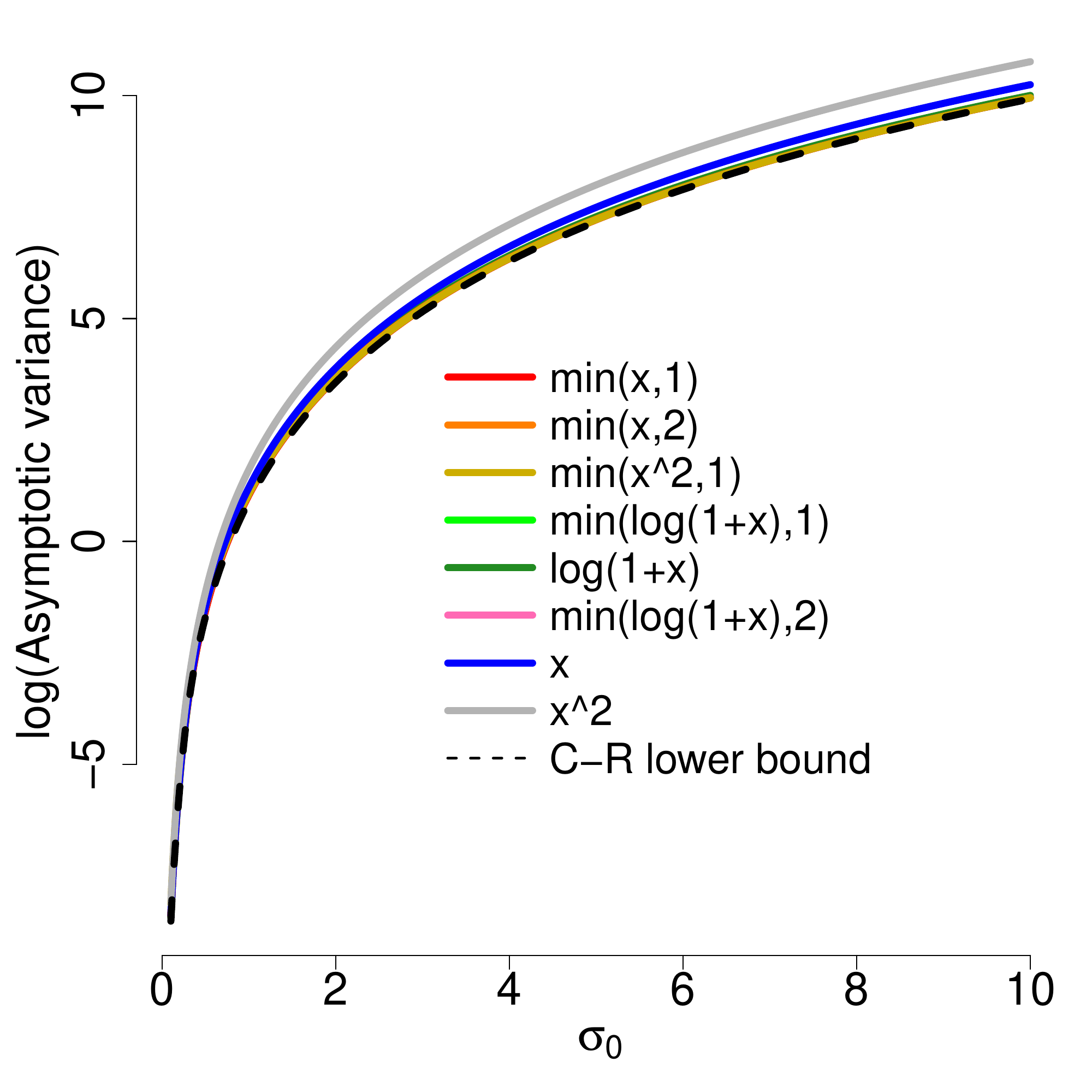}
  \includegraphics[scale=0.35]{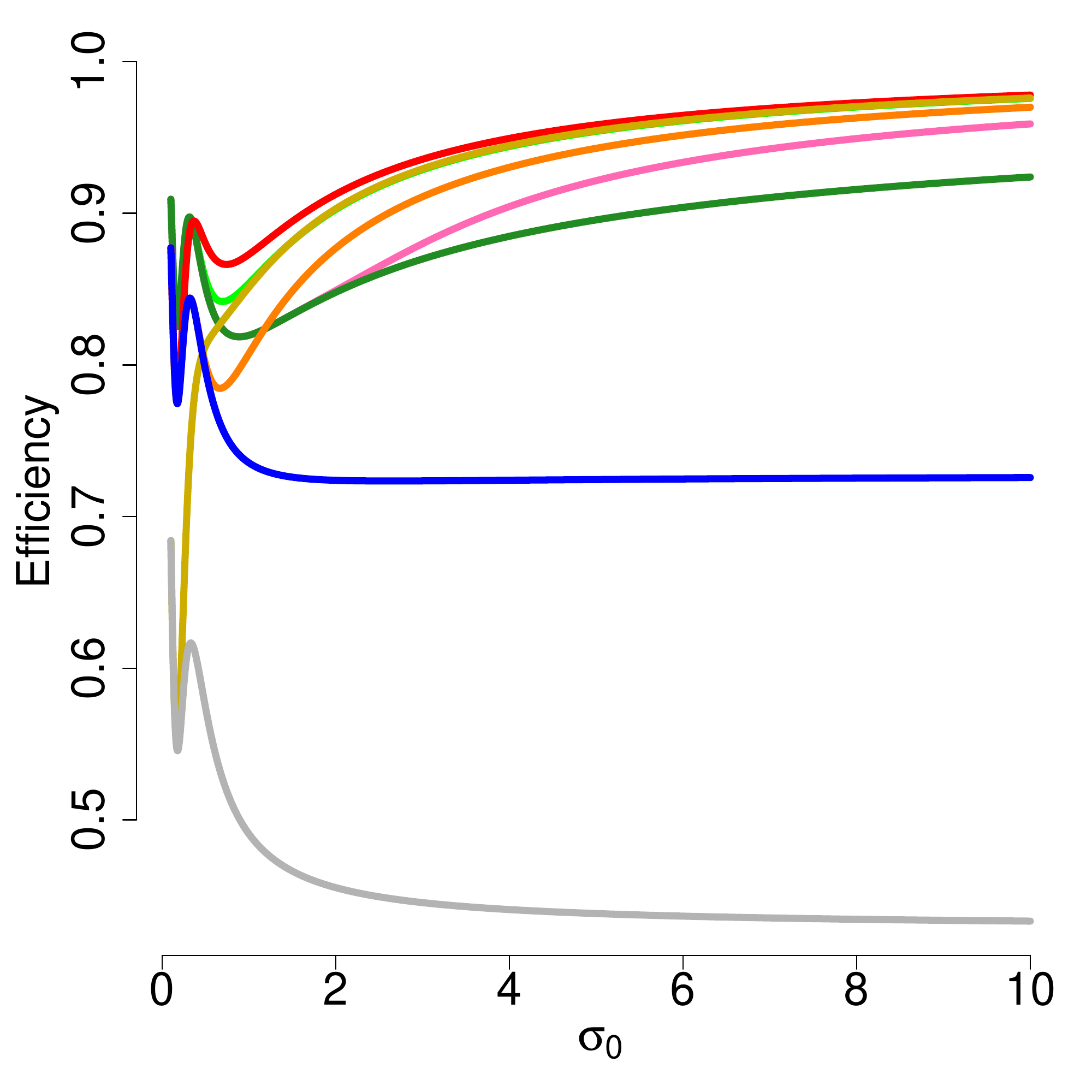}\caption{Log
    of asymptotic variance and efficiency with respect to the
    Cram\'er-Rao bound for $\hat{\sigma}_h^2$ ($\mu=0.5$ known).}
\label{plot_univariate_sigma_var}
\end{figure}

Figure \ref{plot_univariate_sigma_var} is the analog of Figure
\ref{plot_univariate_mu_var} for $\hat{\sigma}^2_{h}$ from Example \ref{theorem_GSM_normality_tnorm_sigma} with
$\mu=0.5$ known.
While the specifics are a bit different the benefits of using bounded
or slowly growing $h$ are again clear.  We note that when
$\sigma$ is small, the effect of truncation to the positive part of
the real line is small.

In both plots we order/color the curves based on their overall efficiency, so they have different colors in one from the other, although the same functions are presented. For all functions presented here (A1)--(A2) and (C1)--(C2) are satisfied.

\section{Regularized Generalized Score Matching}\label{Regularized Generalized Score Matching}

In high-dimensional settings, when the number ${r}$ of parameters to
estimate may be larger than the sample size $n$, it is hard, if not
impossible, to estimate the parameters consistently without turning
to some form of 
regularization. More specifically, for exponential families, condition
(C1) in Section \ref{Exponential_Families} fails when ${r}>n$. A
popular approach is then the use of $\ell_1$ regularization to exploit
possible sparsity.


Let the data matrix $\mathbf{x}\in\mathbb{R}^{n\times m}$ comprise $n$
i.i.d.~samples from distribution $P_0$.  Assume $P_0$ has density
$p_0$ belonging to an exponential family
$\mathcal{P}_+\equiv\{p_{\boldsymbol{\theta}}:\boldsymbol{\theta}\in\boldsymbol{\Theta}\}$,
where $\boldsymbol{\Theta}\subseteq\mathbb{R}^{r}$.  Adding an
$\ell_1$ penalty to (\ref{eq_gsm_exponential}), we obtain the
regularized generalized score matching loss 
\begin{equation}\label{loss_lin}\frac{1}{2}\boldsymbol{\theta}^{\top}\boldsymbol{\Gamma}(\mathbf{x})\boldsymbol{\theta}-\boldsymbol{g}(\mathbf{x})^{\top}\boldsymbol{\theta}+\lambda\|\boldsymbol{\theta}\|_1
\end{equation}
as in \citet{lin16}.  The loss in~(\ref{loss_lin}) involves a
quadratic smooth part as in the familiar lasso loss for linear
regression.  However, although the matrix $\boldsymbol{\Gamma}$ is
positive semidefinite, the regularized loss in~(\ref{loss_lin}) is not
guaranteed to be bounded unless the tuning parameter $\lambda$ is
sufficiently large---a problem that does not occur in lasso.  We note
that here, and throughout, we suppress the dependence on the data
$\mathbf{x}$ for $\boldsymbol{\Gamma}(\mathbf{x})$,
$\boldsymbol{g}(\mathbf{x})$ and derived quantities.

For a more detailed explanation, note that that by (\ref{def_Gamma}),
$\boldsymbol{\Gamma}=\mathbf{H}^{\top}\mathbf{H}$ for some
$\mathbf{H}\in\mathbb{R}^{nm\times r}$.  In the high-dimensional case,
the rank of $\boldsymbol{\Gamma}$, or equivalently $\mathbf{H}$, is at
most $nm<r$.  Hence, $\boldsymbol{\Gamma}$ is not invertible and
$\boldsymbol{g}$ does not necessarily lie in the column span of
$\boldsymbol{\Gamma}$.  Let $\mathrm{Ker}(\boldsymbol{\Gamma})$ be the
kernel of $\boldsymbol{\Gamma}$.  Then there may exist
$\boldsymbol{\nu}\in\mathrm{Ker}(\boldsymbol{\Gamma})$ with
$\boldsymbol{g}^{\top}\boldsymbol{\nu}\not=0$.  In this case, if
\[
  0\leq\lambda<\sup_{\boldsymbol{\nu}\in\mathrm{Ker}(\boldsymbol{\Gamma})}|
  \boldsymbol{g}^{\top}\boldsymbol{\nu}|/\|\boldsymbol{\nu}\|_1, 
\] 
there exists $\boldsymbol{\nu}\in\mathrm{Ker}(\boldsymbol{\Gamma})$
with
$\frac{1}{2}\boldsymbol{\nu}^{\top}\boldsymbol{\Gamma}\boldsymbol{\nu}=0$
and
$-\boldsymbol{g}^{\top}\boldsymbol{\nu}+\lambda\|\boldsymbol{\nu}\|_1<0$.
Evaluating at $\boldsymbol{\theta}(a)=a\cdot\boldsymbol{\nu}$ for
scalar $a>0$, the loss becomes
$a\left(-\boldsymbol{g}^{\top}\boldsymbol{\nu}+\lambda\|\boldsymbol{\nu}\|_1\right)$,
which is negative and linear in $a$, and thus unbounded below. In this
case no minimizer of (\ref{loss_lin}) exists for small values of
$\lambda$. This issue also exists for the estimators from
\cite{zha14} and \cite{liu15}, which correspond to score
matching for GGMs.  We note that in the context of estimating the interaction matrix in pairwise models, $r=m^2$; thus, the condition $nm<r$ reduces to $n<m$, or $n<m+1$ when both $\mathbf{K}$ and $\boldsymbol{\eta}$ are estimated.

To circumvent the unboundedness problem, we add small values
$\gamma_{\ell}>0$ to the diagonal entries of $\boldsymbol{\Gamma}$,
which become $\boldsymbol{\Gamma}_{\ell,\ell}+\gamma_{\ell}$,
$\ell=1,\ldots,r$. This is in the spirit of work such as \citet{led04}
and corresponds to an elastic net-type penalty \citep{zou05} with
weighted $\ell_2$ penalty
$\sum_{\ell=1}^{r}\gamma_{\ell}\theta_{\ell}^2$.  After this
modification, $\boldsymbol{\Gamma}$ is positive definite, our
regularized loss is strongly convex in $\boldsymbol{\theta}$, and a
unique minimizer exists for all $\lambda\geq 0$.  For the special case
of truncated GGMs, we will show that a result on consistent estimation holds if we choose $\gamma_{\ell}=\delta_0\boldsymbol{\Gamma}_{\ell,\ell}$ for a suitably small constant $\delta_0>0$, for which we propose a particular choice to avoid tuning. This choice of $\gamma_{\ell}$ depends on the data through $\boldsymbol{\Gamma}_{\ell,\ell}$.

\begin{definition}
  For
  $\boldsymbol{\gamma}\in\mathbb{R}_+^{r}\backslash\{\boldsymbol{0}\}$,
  let
  $\boldsymbol{\Gamma}_{\boldsymbol{\gamma}}\equiv\boldsymbol{\Gamma}+\mathrm{diag}(\boldsymbol{\gamma})$.
  The \emph{regularized generalized $\boldsymbol{h}$-score matching
    estimator} with tuning parameter $\lambda\ge 0$ and \emph{amplifier}
  $\boldsymbol{\gamma}$ is the estimator
\begin{equation}\label{eq_definition_RGSM_loss_exp}
\hat{\boldsymbol{\theta}}\,\in\argmin\limits_{\boldsymbol{\theta}\in\boldsymbol{\Theta}}\hat{J}_{\boldsymbol{h},\lambda,\boldsymbol{\gamma}}(\boldsymbol{\theta})\equiv\argmin\limits_{\boldsymbol{\theta}\in\boldsymbol{\Theta}}\frac{1}{2}\boldsymbol{\theta}^{\top}\boldsymbol{\Gamma}_{\boldsymbol{\gamma
}}(\mathbf{x})\boldsymbol{\theta}-\boldsymbol{g}(\mathbf{x})^{\top}\boldsymbol{\theta}+\lambda\|\boldsymbol{\theta}\|_1.
\end{equation}
\end{definition}


In the case where $\boldsymbol{\gamma}=(\delta-1)\mathrm{diag}(\boldsymbol{\Gamma})$ for some $\delta>1$, we also call $\delta$ the \emph{multiplier}. We note that $\hat{\boldsymbol{\theta}}$ from~(\ref{eq_definition_RGSM_loss_exp}) is a \emph{piecewise linear} function of $\lambda$ \citep{lin16}.



\section{Score Matching for Graphical Models for Non-negative Data}
\label{sec:graph-models-trunc}
In this section we apply our generalized score matching estimator to a general class of graphical models for non-negative data.

\subsection{A General Framework of Pairwise Interaction Models}\label{A General Framework of Pairwise Interaction Models}


We consider the class of pairwise interaction power models with
density introduced in (\ref{eq:ab-density}).  We recall the form of
the density:
\begin{equation}\label{eq:ab-density2}
p_{\boldsymbol{\eta},\mathbf{K}}(\boldsymbol{x})\propto\exp\left(-\frac{1}{2a}{\boldsymbol{x}^a}^{\top}\mathbf{K}\boldsymbol{x}^a+\boldsymbol{\eta}^{\top}\frac{\boldsymbol{x}^b-\mathbf{1}_m}{b}\right)\mathds{1}_{\mathbb{R}_+^m}(\boldsymbol{x}),
\end{equation}
where $a$ and $b$ are known constants, and the interaction matrix
$\mathbf{K}$ and the vector $\boldsymbol{\eta}$ are parameters.  When
$b=0$, we use the convention that $\frac{x^0-1}{0}\equiv\log x$ and
apply the logarithm element-wise.  Our focus will be on the
interaction matrix $\mathbf{K}$ that determines the conditional
independence graph through its support
$S(\mathbf{K})\equiv\{(i,j):\kappa_{ij}\neq 0\}$.  However, unless
$\boldsymbol{\eta}$ is known or assumed to be zero, we also need to
estimate $\boldsymbol{\eta}$ as a nuisance parameter. In the case where we assume $\boldsymbol{\eta}\equiv\boldsymbol{0}$ is known (i.e.~the linear part $(\boldsymbol{x}^b-\mathbf{1}_m)/b$ is not present), we call the distribution (and the corresponding estimator) a \emph{centered} distribution (estimator), in contrast to the general case termed \emph{non-centered} when we assume $\boldsymbol{\eta}\neq\boldsymbol{0}$ or unknown.

We first give a set of sufficient conditions for the density to be
valid, i.e., the right-hand side of (\ref{eq:ab-density2}) to be
integrable.  The proof is given in Appendix~\ref{Proof of Theorems in Section sec:graph-models-trunc}.

\begin{theorem}\label{thm_normalizability}
Define conditions 
\begin{enumerate}[label=(CC\arabic*),leftmargin=\widthof{(CC3)}+\labelsep]
\item $\mathbf{K}$ is \emph{strictly co-positive}, i.e., $\boldsymbol{v}^{\top}\mathbf{K}\boldsymbol{v}>0$ for all $\boldsymbol{v}\in\mathbb{R}_+^m\backslash\{\boldsymbol{0}\}$;
\item $2a>b>0$;
\item $a>0$, $b=0$, and $\eta_j>-1$ for $j=1,\dots,m$ ($\boldsymbol{\eta}\succ-\mathbf{1}_m$).
\end{enumerate}
In the non-centered case, if (CC1) and one of (CC2) and (CC3) holds,
then the function on the right-hand side of (\ref{eq:ab-density2}) is
integrable over $\mathbb{R}_+^m$. In the centered case, (CC1) and
$a>0$ are sufficient.
\end{theorem}

We emphasize that (CC1) is a weaker condition than positive
definiteness.  Criteria for strict co-positivity are discussed in
\citet{val86}.

\subsection{Implementation for Different Models}
In this section we give some implementation details for the
regularized generalized $\boldsymbol{h}$-score matching estimator
defined in (\ref{eq_definition_RGSM_loss_exp}) applied to the pairwise
interaction models from (\ref{eq:ab-density2}).
We again let
$\boldsymbol{\Psi}\equiv\left(\mathbf{K}^{\top},
  \boldsymbol{\eta}\right)^{\top}\in\mathbb{R}^{(m+1)\times m}$.  The
unregularized loss is then
\begin{align*}
\hat{J}_{\boldsymbol{h}}(P)&=\frac{1}{2}\mathrm{vec}(\boldsymbol{\Psi})^{\top}\boldsymbol{\Gamma}(\boldsymbol{x})\mathrm{vec}(\boldsymbol{\Psi})-\boldsymbol{g}(\boldsymbol{x})^{\top}\mathrm{vec}(\boldsymbol{\Psi}).
\end{align*}
The general form of the matrix $\boldsymbol{\Gamma}$ and the vector
$\boldsymbol{g}$ in the loss were given in equations
(\ref{eq_gsm_exponential})--(\ref{def_g}).
Here $\boldsymbol{\Gamma}\in\mathbb{R}^{(m+1)m\times (m+1)m}$ is block-diagonal, with the $j$-th $\mathbb{R}^{(m+1)\times (m+1)}$ block
\begin{align}
\boldsymbol{\Gamma}_j(\mathbf{x})&\equiv\begin{bmatrix}\boldsymbol{\Gamma}_{11,j} & \boldsymbol{\Gamma}_{12,j} \\ \boldsymbol{\Gamma}_{12,j}^{\top} & \Gamma_{22,j}\end{bmatrix}\label{eq_Gammasub}
\\
&\equiv\frac{1}{n}\sum\limits_{i=1}^n
\begin{bmatrix}
h_j\left(X_j^{(i)}\right){X_j^{(i)}}^{2a-2}{\boldsymbol{X}^{(i)}}^a{{\boldsymbol{X}^{(i)}}^a}^{\top} & -h_j\left(X_j^{(i)}\right){X_j^{(i)}}^{a+b-2}{\boldsymbol{X}^{(i)}}^a\\
-h_j\left(X_j^{(i)}\right){X_j^{(i)}}^{a+b-2}{{\boldsymbol{X}^{(i)}}^a}^{\top} & h_j\left(X_j^{(i)}\right){X_j^{(i)}}^{2b-2}
\end{bmatrix}\nonumber\\
&=\frac{1}{n}\mathbf{y}^{\top}\mathbf{y},\quad\mathbf{y}\equiv\begin{bmatrix}-(\sqrt{\boldsymbol{h}_j(\boldsymbol{X}_j)}\circ\boldsymbol{X}_j^{a-1})\circ\mathbf{x}^a & \sqrt{\boldsymbol{h}_j(\boldsymbol{X}_j)}\circ\boldsymbol{X}_j^{b-1}\end{bmatrix}\in\mathbb{R}^{n,m+1},\label{eq_Gamma_y}
\end{align}
where the $\circ$ product between a vector and a matrix means an
elementwise multiplication of the vector with each \emph{column} of
the matrix, and
$\boldsymbol{h}_j(\boldsymbol{X}_j)\equiv[h_j(X_j^{(1)}),\dots,h_j(X_j^{(n)})^{\top}]\in\mathbb{R}^m$.
Furthermore, $\boldsymbol{g}\equiv\begin{bmatrix}\mathrm{vec}(\mathbf{g}_1) \\ \boldsymbol{g}_2\end{bmatrix}\in\mathbb{R}^{(m+1)m}$, where $\mathbf{g}_1$ and $\boldsymbol{g}_2$ correspond to each entry of $\mathbf{K}$ and $\boldsymbol{\eta}$, respectively. The $j$-th column of $\mathbf{g}_1\in\mathbb{R}^{m\times m}$, written as $\boldsymbol{g}_{1,j}(\mathbf{x})$, is
\[
\frac{1}{n}\sum_{i=1}^n \bigg(h_j'\left(X_j^{(i)}\right){X_j^{(i)}}^{a-1}+(a-1)h_j\left(X_j^{(i)}\right){X_j^{(i)}}^{a-2}\bigg){\boldsymbol{X}^{(i)}}^a+ah_j\left(X_j^{(i)}\right){X_j^{(i)}}^{2a-2}\boldsymbol{e}_{j,m},
\]
where $\boldsymbol{e}_{j,m}$ is the $m$-vector with 1 at the $j$-th position and 0 elsewhere, and the $j$-th entry of $\boldsymbol{g}_2\in\mathbb{R}^m$ is
\[g_{2,j}=\frac{1}{n}\sum_{i=1}^n-h_j'\left(X_j^{(i)}\right){X_j^{(i)}}^{b-1}-(b-1)h_j\left(X_j^{(i)}\right){X_j^{(i)}}^{b-2}.\]


These formulae also hold for $b=0$ since $\boldsymbol{\Gamma}$ and $\boldsymbol{g}$ only depend on the gradient of the log density, and $\frac{\d (x^{b}-1)/b}{\d x}=x^{b-1}$ also holds for $b=0$. In the centered case where we know $\boldsymbol{\eta}_0\equiv\boldsymbol{0}$, we only estimate $\mathbf{K}\in\mathbb{R}^{m\times m}$, and $\boldsymbol{\Gamma}\in\mathbb{R}^{m^2\times m^2}$ is still block-diagonal, with the $j$-th block being the $\boldsymbol{\Gamma}_{11,j}$ submatrix in (\ref{eq_Gammasub}), while $\boldsymbol{g}$ is just $\mathrm{vec}(\mathbf{g}_1)$. Since $b$ only appears in the $\boldsymbol{\eta}$ part of the density, the formulae only depend on $a$ in the centered case.

We emphasize that it is indeed necessary to introduce amplifiers
$\boldsymbol{\gamma}\succ \boldsymbol{0}$ or a multiplier $\delta>1$ in
addition to the $\ell_1$ penalty.  It is clear
from~(\ref{eq_Gamma_y}) that
$\mathrm{rank}(\boldsymbol{\Gamma}_j)\le\min\{n,m+1\}$ (or $\min\{n,m\}$ if centered).
Thus, $\boldsymbol{\Gamma}$ is non-invertible when $n\leq m$ (or $n<m$ if centered) and
$\boldsymbol{g}$ need not lie in its column span.  

We claim that including amplifiers/multipliers for the submatrices
$\boldsymbol{\Gamma}_{11,j}$ only is sufficient for unique existence of a solution for all penalty parameters $\lambda\ge 0$.  To see this,
consider any nonzero vector $\boldsymbol{\nu}\in\mathbb{R}^{m+1}$.
Partition it as
$\boldsymbol{\nu}\equiv(\boldsymbol{\nu}_1,\boldsymbol{\nu}_2)$ with
$\boldsymbol{\nu}_1\in\mathbb{R}^{m}$.  Let
$\boldsymbol{\Gamma}_{j,\boldsymbol{\gamma}}$ be our amplified version of the
matrix $\boldsymbol{\Gamma}_j$ from~(\ref{eq_noncentered_gamma}), so
\[
  \boldsymbol{\Gamma}_{j,\boldsymbol{\gamma}} \;=\;
  \begin{pmatrix}
    \boldsymbol{\Gamma}_{11,j}+\mathrm{diag}(\gamma_1,\dots,\gamma_m) &
  \boldsymbol{\Gamma}_{12,j}\\
  \boldsymbol{\Gamma}_{12,j}^\top& \Gamma_{22,j}
\end{pmatrix}.
\]
As $\boldsymbol{\Gamma}_j$ itself is positive semidefinite, we find that if at least
one of the first $m$ entries of $\boldsymbol{\nu}$ is nonzero then
\[
  \boldsymbol{\nu}^{\top}\boldsymbol{\Gamma}_{j,\boldsymbol{\gamma}}\boldsymbol{\nu}\;\ge\;
  \boldsymbol{\nu}^{\top}\boldsymbol{\Gamma}_{j}\boldsymbol{\nu} + \sum_{k=1}^m
  \nu_k^2\boldsymbol{\gamma}_k \;\ge\; \sum_{k=1}^m
  \nu_k^2\boldsymbol{\gamma}_k \;>\;0.
\]
If only the last entry of  $\boldsymbol{\nu}$ is nonzero then
\[
  \boldsymbol{\nu}^{\top}\Gamma_{j,\boldsymbol{\gamma}}\boldsymbol{\nu}\;=\;
  \nu_{m+1}^2\Gamma_{22,j}  \;>\;0
\]
almost surely; recall that
$\Gamma_{22,j}=\tfrac{1}{n}\sum_{i=1}^nh_j\left(X_{j}^{(i)}\right)X_j^{2b-2}$.
We conclude that $\boldsymbol{\Gamma}_{j,\boldsymbol{\gamma}}$ (and thus the entire amplified $\boldsymbol{\Gamma}$) is a.s.~positive
definite, which ensures unique existence of the loss minimizer.

Given the formulae for $\boldsymbol{\Gamma}$ and $\boldsymbol{g}$, one adds the $\ell_1$ penalty on $\boldsymbol{\Psi}$ to get the regularized loss (\ref{eq_definition_RGSM_loss_xi}). Our methodology readily accommodates two different choices of the penalty parameter $\lambda$ for $\mathbf{K}$ and $\boldsymbol{\eta}$. This is also theoretically supported for truncated GGMs, since if the ratio of the respective values $\lambda_{\mathbf{K}}$ and $\lambda_{\boldsymbol{\eta}}$ is fixed, the proof of the theorems in Section \ref{Theory for Graphical Models} can be easily modified by replacing $\boldsymbol{\eta}$ by $(\lambda_{\boldsymbol{\eta}}/\lambda_{\mathbf{K}})\boldsymbol{\eta}$.
  To avoid picking two tuning parameters, one may also
  choose to remove the penalty on $\boldsymbol{\eta}$ altogether by
  profiling out $\boldsymbol{\eta}$ and solve for
  $\hat{\boldsymbol{\eta}}\equiv\boldsymbol{\Gamma}_{22}^{-1}\left(\boldsymbol{g}_2-\boldsymbol{\Gamma}_{12}^{\top}\mathrm{vec}(\mathbf{\hat{K}})\right)$,
  with $\mathbf{\hat{K}}$ the minimizer of the profiled loss
\begin{equation}\label{loss-noncentered-profiled}
\hat{J}_{\boldsymbol{h},\lambda,\boldsymbol{\gamma},\mathrm{profile}}(\mathbf{K})\equiv
\frac{1}{2}\mathrm{vec}(\mathbf{K})^{\top}\boldsymbol{\Gamma}_{\boldsymbol{\gamma},11.2}
\mathrm{vec}(\mathbf{K})
-(\boldsymbol{g}_1-\boldsymbol{\Gamma}_{12}\boldsymbol{\Gamma}_{22}^{-1}
\boldsymbol{g}_2)^{\top}\mathrm{vec}(\mathbf{K})+\lambda\|\mathbf{K}\|_1,
\end{equation}
where the Schur complement $\boldsymbol{\Gamma}_{\boldsymbol{\gamma},11.2}\equiv\boldsymbol{\Gamma}_{\boldsymbol{\gamma},11}-
\boldsymbol{\Gamma}_{12}\boldsymbol{\Gamma}_{22}^{-1}\boldsymbol{\Gamma}_{12}^{\top}$ is a.s.~positive definite such that the profiled estimator exists a.s.~for all $\lambda\ge 0$.  This profiled approach corresponds to choosing $\lambda_{\boldsymbol{\eta}}/\lambda_{\mathbf{K}}=0$. A detailed theoretical analysis of the profiled estimator is beyond the scope of
this paper, however.  We note that in the other extreme, with $\lambda_{\boldsymbol{\eta}}/\lambda_{\mathbf{K}}=+\infty$, the non-centered estimator reduces to the estimator from the centered case.

\begin{example}
\label{Truncated Non-centered GGMs}
The truncated normal
model comprises the density
\begin{equation}\label{pdf-noncentered} 
p_{\boldsymbol{\mu},\mathbf{K}}(\boldsymbol{x})\;\propto\;\exp\left\{-\frac{1}{2}(\boldsymbol{x}-\boldsymbol{\mu})^{\top}\mathbf{K}(\boldsymbol{x}-\boldsymbol{\mu})\right\}\mathds{1}_{[0,\infty)^m}(\boldsymbol{x}).
\end{equation}
This corresponds to (\ref{eq:ab-density2}) with $a=b=1$, and $\boldsymbol{\eta}=\mathbf{K}\boldsymbol{\mu}$.
The 
$j$-th $(m+1)\times (m+1)$ block of $\boldsymbol{\Gamma}(\mathbf{x})$ is
\begin{equation}\label{eq_noncentered_gamma}
\frac{1}{n}\begin{bmatrix}
\mathbf{x}^{\top}\mathrm{diag}(\boldsymbol{h}_j(\boldsymbol{X}_j))\mathbf{x}  &  -\mathbf{x}^{\top}\boldsymbol{h}_j(\boldsymbol{X}_j) \\ -\boldsymbol{h}_j(\boldsymbol{X}_j)^{\top}\mathbf{x} &  \boldsymbol{h}_j(\boldsymbol{X}_{j})^{\top}\mathbf{1}_n
\end{bmatrix}.
\end{equation}
Partitioning the vector $\boldsymbol{g}(\mathbf{x})$ into $m$ subvectors
$\boldsymbol{g}_j(\mathbf{x})\in\mathbb{R}^{m+1}$, where the entries of
$\boldsymbol{g}_j(\mathbf{x})$ correspond to column
$\boldsymbol{\Psi}_j$, the $k$-th entry of
$\boldsymbol{g}_j(\mathbf{x})$ is
\begin{equation}\label{eq_noncentered_g}
  g_{jk}(\mathbf{x})\equiv
  \begin{cases}
    \tfrac{1}{n}\sum_{i=1}^nh'_j\left(X_j^{(i)}\right)X_k^{(i)} 
    &\text{ if } k\le m,\; k\not= j,\\
    \tfrac{1}{n}\sum_{i=1}^nh'_j\left(X_j^{(i)}\right)X_k^{(i)} 
    +h_j\left(X_{j}^{(i)}\right)
    &\text{ if } k= j,\\
    -\tfrac{1}{n}\sum_{i=1}^nh_j'\left(X_{j}^{(i)}\right)
    &\text{ if } k=m+1.
  \end{cases}
\end{equation}
\end{example}

\begin{example}
The exponential square-root graphical model in \citet{ino16} has
\[p_{\boldsymbol{\eta},\mathbf{K}}(\boldsymbol{x})\propto\exp\left(-\sqrt{\boldsymbol{x}}^{\top}\mathbf{K}\sqrt{\boldsymbol{x}}+2\boldsymbol{\eta}^{\top}\sqrt{\boldsymbol{x}}\right)\mathds{1}_{[0,\infty)^m}(\boldsymbol{x}),\]
which corresponds to (\ref{eq:ab-density2}) with $a=b=1/2$. We refer to this as the \emph{exponential} model. In this case, the $j$-th $\mathbb{R}^{(m+1)\times (m+1)}$ block of $\boldsymbol{\Gamma}$ is
\[\boldsymbol{\Gamma}_j(\mathbf{x})\equiv\frac{1}{n}\sum_{i=1}^n\frac{h_j\left(X_j^{(i)}\right)}{X_j^{(i)}}\left(\begin{matrix}-\sqrt{\boldsymbol{X}^{(i)}}\\ 1\end{matrix}\right)\left(-\sqrt{\boldsymbol{X}^{(i)}}^{\top},1\right)\]
and $\boldsymbol{g}=\mathrm{vec}(\mathbf{g}_0)$, where the $j$-th column of $\mathbf{g}_0\in\mathbb{R}^{(m+1)\times m}$ is
\[\boldsymbol{g}_j(\mathbf{x})\equiv\frac{1}{n}\sum_{i=1}^n\frac{2h'_j\left(X_j^{(i)}\right)X_j^{(i)}-h_j\left(X_j^{(i)}\right)}{2{X_j^{(i)}}^{3/2}}\left(\begin{matrix}\sqrt{\boldsymbol{X}^{(i)}}\\ -1\end{matrix}\right)+\frac{h_j\left(X_j^{(i)}\right)}{2X_j^{(i)}}\boldsymbol{e}_{j,m+1}.\]
\end{example}

\begin{example}
If $a=1/2$ and $b=0$, then (\ref{eq:ab-density2}) becomes
\begin{equation}\label{eq:gamma-density}
p_{\boldsymbol{\eta},\mathbf{K}}(\boldsymbol{x})\propto\exp\left(-\sqrt{\boldsymbol{x}}^{\top}\mathbf{K}\sqrt{\boldsymbol{x}}+\boldsymbol{\eta}^{\top}\log(\boldsymbol{x})\right)\mathds{1}_{(0,\infty)^m}(\boldsymbol{x}).
\end{equation}
If $\mathbf{K}$ is diagonal in this case, then
$\boldsymbol{X}\sim p_{\boldsymbol{\eta},\mathbf{K}}$ has independent
entries with $X_j$ following the gamma distribution with rate
$\kappa_{jj}$ and shape $\eta_j+1$, which gives an intuition for
condition (CC3) $\eta_j>-1$ in Theorem \ref{thm_normalizability}. We
can thus view (\ref{eq:gamma-density}) as a multivariate gamma
distribution with pairwise interactions among the covariates, and call this the \emph{gamma} model.  For this model, the $j$-th block
of $\boldsymbol{\Gamma}$ is
\[\boldsymbol{\Gamma}_j(\mathbf{x})\equiv\frac{1}{n}\sum_{i=1}^n\frac{h_j\left(X_j^{(i)}\right)}{{X_j^{(i)}}^2}
\left(\begin{matrix}-\sqrt{X_j^{(i)}\boldsymbol{X}^{(i)}}\\ 1\end{matrix}\right)\left(-\sqrt{X_j^{(i)}\boldsymbol{X}^{(i)}}^{\top},1\right)
\]
and the part of $\boldsymbol{g}$ corresponding to $\mathbf{K}_j$ is 
\[\boldsymbol{g}_{1,j}(\mathbf{x})\equiv\frac{1}{n}\sum_{i=1}^n\frac{2h'_j\left(X_j^{(i)}\right)X_j^{(i)}-h_j\left(X_j^{(i)}\right)}{2{X_j^{(i)}}^{3/2}}\sqrt{\boldsymbol{X}^{(i)}}+\frac{h_j\left(X_j^{(i)}\right)}{2X_j^{(i)}}\boldsymbol{e}_{j,m},\]
while the part for $\eta_j$ is 
\[g_{2.j}(\mathbf{x})=\frac{1}{n}\sum_{i=1}^n\frac{h_j\left(X_j^{(i)}\right)}{{X_j^{(i)}}^2}-\frac{h_j'\left(X_j^{(i)}\right)}{X_j^{(i)}}.\]
\end{example}

We note that the $\boldsymbol{\Gamma}_{11,j}$ sub-matrix of $\boldsymbol{\Gamma}_j$ and the $\boldsymbol{g}_{1,j}$ sub-vector of $\boldsymbol{g}_{j}$ for the gamma model are the same as those for the exponential model, since $a=1/2$ in both cases and the parts involving $\mathbf{K}$ in the densities are the same.

\subsection{Computational Details}\label{Computational Details}

In the most general exponential family setting, as in Eq.~(\ref{eq_gsm_exponential})--(\ref{def_g}) in Theorem \ref{theorem_GSM_estimator_exp}, the time complexity for forming $\boldsymbol{\Gamma}\in\mathbb{R}^{r\times r}$ and $\boldsymbol{g}\in\mathbb{R}^r$ is $\mathcal{O}\left(nm(f_{b'}(m)+r^2+r(f_{t'}(m)+f_{t''}(m)))\right)$. Here $f_{b'}(m)$ is the average time complexity for calculating $\partial_j b(\boldsymbol{x})$ over $j=1,\dots,m$, and similarly $f_{t'}(m)$ for $\partial_j t_{\ell}(\boldsymbol{x})$ and $f_{t''}(m)$ for $\partial_{jj} t_{\ell}(\boldsymbol{x})$ over $j=1,\dots,m$ and $\ell=1\dots,r$. In many applications, however, these three functions would be constant in $m$, thus giving an $\mathcal{O}(nmr^2)$ computational complexity, with the dominating term coming from the operations for $\boldsymbol{t}_j'{\boldsymbol{t}_j'}^{\top}$ in $\boldsymbol{\Gamma}$ since $\boldsymbol{\Gamma}$ is of dimension $r\times r$. 

For pairwise interaction power models, $r=m^2$ and the formula above becomes $\mathcal{O}(nm^5)$. However, since $\boldsymbol{\Gamma}$ is block-diagonal with only $m^3$ nonzero entries and by the special form of $\boldsymbol{t}(\boldsymbol{x})=\boldsymbol{x}^a{\boldsymbol{x}^a}^{\top}$, the true complexity is in fact $\mathcal{O}(nm^3)$.

While the introduction of the $\ell_1$ penalty inevitably precludes the estimator from having a closed-form solution and introduces non-differentiability, state-of-art numerical optimization algorithms, such as coordinate-descent \citep{fri07}, can be applied for fast estimation. To speed up estimation, one can usually use warm starts using the solution from the previous $\lambda$'s, as well as lasso-type strong screening rules \citep{tib12} to eliminate components of $\hat{\boldsymbol{\theta}}$ that are known a priori to have zero estimates.

In our implementation for pairwise interaction models of Section \ref{A General Framework of Pairwise Interaction Models} (that will become available in an R package), we optimize our loss functions with respect to a symmetric matrix $\hat{\mathbf{K}}$; in the non-centered case the vector $\hat{\boldsymbol{\eta}}$ is also included. We use a coordinate-descent method analogous to Algorithm 2 in \citet{lin16}, where in each step we update each element of $\hat{\mathbf{K}}$ and $\hat{\boldsymbol{\eta}}$ based on the other entries from the previous steps, while maintaining symmetry. In our simulations in Section \ref{Numerical Experiments} we always scale the data matrix by column $\ell_2$ norms before proceeding to estimation. Note that estimation of $\hat{\mathbf{K}}$ without symmetry can be parallelized as the loss can be decomposed into a sum over the columns.

\subsection{Choice of the Function \texorpdfstring{$\boldsymbol{h}$}{h}}
In this subsection we discuss the requirements on the function $\boldsymbol{h}$ as well as some reasonable choices of $\boldsymbol{h}$.
\subsubsection{Requirements on \texorpdfstring{$\boldsymbol{h}$}{h}}\label{Requirement of h}
In Section \ref{Generalized Score Matching for Non-Negative Data}, we presented two assumptions (A1) and (A2) under which the generalized score-matching loss is valid, i.e., the integration by parts is justified and Theorem \ref{theorem_GSM_loss_alt} holds. In this section, we present some sufficient (and nearly necessary) requirements on $\boldsymbol{h}$ such that (A1) and (A2) are satisfied.

\begin{definition}\label{def_h}
Suppose $\boldsymbol{h}:\mathbb{R}_+^m\to\mathbb{R}_+^m$ with $\boldsymbol{h}(\boldsymbol{x})=(h_1(x_1),\dots,h_m(x_m))^{\top}$. We write that $\boldsymbol{h}\in\mathcal{H}_{a,b}$ (for simplicity we omit the dependency on $m$) if for all $j=1,\dots,m$:
\begin{enumerate}[i)]
\item $h_j$ is absolutely continuous in every bounded sub-interval of $\mathbb{R}_+$, and thus has derivative $h_j'$ a.s.;
\item $h_j(x)>0$ a.s.~on $\mathbb{R}_+$;
\item $h_j$ and $h_j'$ are both bounded by some piecewise powers of $x$ a.s.~on $\mathbb{R}_+$; 
\item \label{def_h_lim0} \mbox{$\lim\limits_{x\searrow 0+}h_j(x)/x_j^q=0$, where $q\equiv\begin{cases}
\max\{1-a,1-b\} & \text{if }b>0,\\ 1-\eta_{0,j} & \text{if }b=0.\end{cases}$}
\end{enumerate}
\end{definition}

\begin{theorem}\label{thm_h}
  Assume every $P$ in the family of distribution $\mathcal{P}_+$ satisfies (CC1)--(CC3) and thus has finite normalizing constants. If $\boldsymbol{h}\in\mathcal{H}_{a,b}$, then (A1) and (A2) are satisfied.
\end{theorem}

In centered models, where $\boldsymbol{\eta}\equiv\boldsymbol{0}$, we can assume $b=2a$ and iv) in the definition of $\mathcal{H}_{a,2a}$ has $q=1-a$. For truncated GGMs, $a=b=1$, so \ref{def_h_lim0} in Definition~\ref{def_h} is simply $\lim_{x_j\searrow 0^+}h_j(x_j)=0$.

In the case of $b=0$, $\boldsymbol{\eta}$ is an unknown parameter, and (CC3) requires each of its component to be greater than $-1$. If one has prior information on $\boldsymbol{\eta}$ or restricts the parameter space for $\boldsymbol{\eta}$, the requirement reduces to $h_j(x_j)=o(x_j^{1-\eta_{0,j}})$ as $x_j\searrow 0^+$.
Otherwise, it suffices to require $h_j(x_j)=o(x_j^2)$.  Note that this is only a condition for $x_j\searrow 0^+$, and the globally quadratic behavior of $h_j(x_j)=x_j^2$ from the original score matching is not needed on the entire $\mathbb{R}_+$, leaving
opportunities for improvements.

\subsubsection{Reasonable Choices of \texorpdfstring{$\boldsymbol{h}$}{h}}\label{Reasonable Choices of h}
Assume a common univariate $h$ for all components in
$\boldsymbol{h}$. Inspired by Theorem \ref{thm_h}, we consider $h$ that behaves like a power of $x$ both as $x\nearrow+\infty$ and as $x\searrow 0^+$. Since the requirements on the two tails are separate, we can choose $h$ to be a piecewise defined function that joins two
powers with possibly different degrees.  In other words, $h(x)=\min(x^{p_1},cx^{p_2})$ for some powers
$p_1\geq p_2\geq 0$ and constant $c>0$.
Only one constant $c$ is
required since generalized score matching is invariant to scaling of
$h$.  In determining the exact power of $p_1$ we have the following
considerations:
\begin{enumerate}[a)]
\item In the centered case:
\begin{enumerate}[(i)]
\item (A1) and (A2): Theorem \ref{thm_h} requires that $p_1\geq 1-a$.
\item ``Controlled $\boldsymbol{\Gamma}$ and $\boldsymbol{g}$  for
  $\boldsymbol{x}^a$'':  We propose avoiding poles at the origin for the entries of
  $\boldsymbol{\Gamma}$ and $\boldsymbol{g}$.  The formula for
  $\boldsymbol{\Gamma}_{11}$ in (\ref{eq_Gamma_y}) shows that to this end
  $\sqrt{h(x)}x^{a-1}$ needs to have a non-negative degree  . This
  requires $p_1\geq 2-2a$.  
  The formula for $\boldsymbol{g}_1$ similarly shows that $h'(x)x^{a-1}$, $h(x)x^{a-2}$
  and $h(x)x^{2a-2}$ all need to have a non-negative degree for small
  $x$.  This requires $p_1\geq 2-a$.
\end{enumerate}
\item In the non-centered case, in addition to (i) and (ii),
\begin{enumerate}[(i)]
\setcounter{enumii}{2}
\item (A1) and (A2): Theorem \ref{thm_h} requires $p_1\geq \max\{1-a,1-b\}$ for $b>0$, or $1-\min_j\eta_{0,j}$ for $b=0$.
\item ``Controlled $\boldsymbol{\Gamma}$ and $\boldsymbol{g}$ for $\boldsymbol{x}^b$'': From the definition of $\boldsymbol{\Gamma}_{22}$ and $\boldsymbol{g}_2$ and by the same reasoning as above, $\sqrt{h(x)}x^{b-1}$, $h'(x)x^{b-1}$ and $h(x)x^{b-2}$ need to be non-negative powers of $x$, thus requiring $p_1\geq \max\{2-b,2-2b\}=2-b$.
\end{enumerate}
\end{enumerate}

The choice of $p_2$, is only relevant for large data
points.  Our main consideration is then merely how well $\boldsymbol{\Gamma}$ and $\boldsymbol{g}$ concentrate on their true population values (Theorem \ref{theorem_lin}).  From this perspective, our intuition is that $p_2$ should be chosen small so that the tails of the distributions of the entries of $\boldsymbol{\Gamma}$ and $\boldsymbol{g}$ are well-behaved.  Thus, we can choose $p_2=0$, in
which case $h(x)=\min(x^{p_1},c)$ is a truncated power.

\subsection{Tuning Parameter Selection}\label{Parameter Tuning}
By treating the unpenalized loss (i.e.,~$\lambda=0$,
$\boldsymbol{\gamma}=0$) as a negative log-likelihood, we may use the extended Bayesian Information Criterion (eBIC) to choose the tuning parameter \citep{che08,foy10}. Consider the centered case as an example.
Let $\hat{S}^{\lambda}\equiv\{(i,j):\hat{\kappa}_{ij}^{\lambda}\neq 0,i<j\}$, where $\hat{\mathbf{K}}^{\lambda}$ be the estimate associated with tuning parameter $\lambda$. The eBIC is then
\begin{align*}
\mathrm{eBIC}(\lambda)=&-n\mathrm{vec}(\hat{\mathbf{K}})^{\top}\boldsymbol{\Gamma}(\mathbf{x})\mathrm{vec}(\hat{\mathbf{K}})+2n\boldsymbol{g}(\mathbf{x})^{\top}\mathrm{vec}(\hat{\mathbf{K}})+|\hat{S}^{\lambda}|\log n+2\log\left(\begin{matrix}p(p-1)/2 \\ |\hat{S}^{\lambda}|\end{matrix}\right),
\end{align*}
where $\hat{\mathbf{K}}$ can be either the original estimate associated with $\lambda$, or a refitted solution obtained by restricting the support to $\hat{S}^{\lambda}$.


We use the eBIC instead of the ordinary BIC (Bayesian Information Criterion) since the BIC tends to choose an overly complex model when the model space is large, as encountered in the high-dimensional setting.  The extension in eBIC comes from the last term in the above display which can be motivated by a prior distribution under which the number of edges in the conditional independence graph is uniformly distributed; see also \cite{bogdan11} and \cite{foy15}.

\section{Theory for Graphical Models}\label{Theory for Graphical Models}

In our regularized generalized score matching framework, we introduced the amplifiers/ multipliers to address the inexistence problem. We also proposed using a general function $\boldsymbol{h}$ in place of $\boldsymbol{x}^2$ as a means to improve estimation accuracy. This section provides a theoretical analysis of these two aspects.

In Section \ref{Theory for Pairwise Interaction Models}, we present the theory for our regularized generalized score matching estimators for general pairwise interaction models before going into the details for the special cases of (truncated) GGMs.
Next, we show that a specific choice of amplifiers/multipliers yields consistent estimation without the need for tuning. This point is important even in the case of Gaussian models on all of $\mathbb{R}^m$. Therefore, in Section~\ref{sec:revis-gauss-score} we digress from non-negative data and consider 
the original score matching of \cite{hyv05} for centered Gaussian distributions.
Finally, in Section~\ref{Generalized Score Matching for Truncated GGMs}, we derive probabilistic results for $\hat{\boldsymbol{\Psi}}$ based on Theorem \ref{theorem_lin}, justifying the benefits of using a general bounded $\boldsymbol{h}$ over $\boldsymbol{x}^2$ in the non-negative setting. As the most important models from the class of pairwise interaction power models over $\mathbb{R}_+^m$, we only treat truncated GGMs since they have the most tractable concentration bounds; this case  also provides a comparison to Corollary 2 in \citet{lin16}, which uses $\boldsymbol{x}^2$.

\subsection{Theory for Pairwise Interaction Models}\label{Theory for Pairwise Interaction Models}

The graphical models we treat are parametrized by the interaction matrix $\mathbf{K}$ and the coefficients $\boldsymbol{\eta}$ on $(\boldsymbol{x}^b-\mathbf{1}_m)/b$. 
It is convenient to accommodate this setting with a matrix-valued parameter
$\mathbf{\Psi}\in\mathbb{R}^{r_1\times r_2}$ (in place of
$\boldsymbol{\theta}$) and specify our regularized
$\boldsymbol{h}$-score matching loss as
\begin{equation}\label{eq_definition_RGSM_loss_xi}
\hat{J}_{\boldsymbol{h},\lambda,\boldsymbol{\gamma}}(\boldsymbol{\Psi})\equiv\argmin\limits_{\boldsymbol{\Psi}\in\mathbb{R}^{r_1\times r_2}}\frac{1}{2}\mathrm{vec}(\boldsymbol{\Psi})^{\top}\boldsymbol{\Gamma}_{\boldsymbol{\gamma}}(\mathbf{x})\mathrm{vec}(\boldsymbol{\Psi})-\boldsymbol{g}(\mathbf{x})^{\top}\mathrm{vec}(\boldsymbol{\Psi})+\lambda\|\boldsymbol{\Psi}\|_1.
\end{equation}

In the non-centered case we thus take $\mathbf{\Psi}=[\mathbf{K},\boldsymbol{\eta}]^{\top}\in\mathbb{R}^{m(m+1)\times m}$. In the centered case, $\mathbf{\Psi}$ is simply the $m\times m$ interaction matrix $\mathbf{K}$. Following related prior
work such as \cite{lin16}, for ease of proof we allow the matrix $\mathbf{K}$ to be
nonsymmetric, which allows us to decouple optimization over the
different columns of $\mathbf{K}$ or $\boldsymbol{\Psi}$, while in our implementations we ensure that $\mathbf{K}$ is symmetric.

\begin{definition}\label{def_constants_centered}
  Let
  $\boldsymbol{\Gamma}_0\equiv\mathbb{E}_0\boldsymbol{\Gamma}(\mathbf{x})$
  and $\boldsymbol{g}_0\equiv\mathbb{E}_0\boldsymbol{g}(\mathbf{x})$
  be the population versions of $\boldsymbol{\Gamma}(\mathbf{x})$ and
  $\boldsymbol{g}(\mathbf{x})$ under the distribution given by a true
  parameter matrix $\mathbf{\Psi}_0$.  The support of a matrix
  $\mathbf{\Psi}$ is
  $S(\mathbf{\Psi})\equiv\{(i,j):\psi_{ij}\neq 0\}$, and we let
  $S_0=S(\mathbf{\Psi}_0)$.  For a matrix $\mathbf{\Psi}_0$, we
  define $d_{\mathbf{\Psi}_0}$ to be the maximum number of non-zero
  entries in any column, and
  $c_{\mathbf{\Psi}_0}\equiv\mnorm{\mathbf{\Psi}_0}_{\infty,\infty}$.
  Writing $\mathbf{\Gamma}_{0,AB}$ for the $A\times B$ submatrix of
  $\boldsymbol{\Gamma}_0$, we define
  \begin{equation}\label{def_c_gamma}
    c_{\boldsymbol{\Gamma}_0}\equiv
    \mnorm{(\boldsymbol{\Gamma}_{0,S_0S_0})^{-1}}_{\infty,\infty}.
  \end{equation}
  Finally, $\boldsymbol{\Gamma}_0$ satisfies \emph{the
    irrepresentability condition with incoherence parameter
    $\alpha\in(0,1]$ and edge set $S_0$} if
\begin{equation}\label{irrepresentability}
\mnorm{\boldsymbol{\Gamma}_{0,S_0^cS_0}(\boldsymbol{\Gamma}_{0,S_0S_0})^{-1}}_{\infty,\infty}\leq (1-\alpha).
\end{equation}
\end{definition}


Our analysis of the regularized generalized $\boldsymbol{h}$-score
matching estimator builds on the following theorem taken from
\citet[Theorem 1]{lin16}.  

\begin{theorem} 
  \label{theorem_lin}
  Suppose $\boldsymbol{\Gamma}_{0}$ has
  $\boldsymbol{\Gamma}_{0,S_0S_0}$ invertible and satisfies the
  irrepresentability condition (\ref{irrepresentability})
  with
  incoherence parameter $\alpha\in(0,1]$. Assume that
\begin{equation}\label{Max_constants}\|\boldsymbol{\Gamma}_{\boldsymbol{\gamma}}(\mathbf{x})-\boldsymbol{\Gamma}_0\|_{\infty}<\epsilon_1,\quad\quad\quad\quad\|\boldsymbol{g}(\mathbf{x})-\boldsymbol{g}_0\|_{\infty}<\epsilon_2,
\end{equation}
with $d_{\mathbf{\Psi}_0}\epsilon_1\leq\alpha/(6c_{\boldsymbol{\Gamma}_0})$. If 
\[\lambda>\frac{3(2-\alpha)}{\alpha}\max\{c_{\mathbf{\Psi}_0}\epsilon_1,\epsilon_2\},\]
then the following holds:
\begin{enumerate}[(a)]
\item The regularized generalized $\boldsymbol{h}$-score matching estimator $\hat{\mathbf{\Psi}}$ minimizing (\ref{eq_definition_RGSM_loss_xi}) is unique, with support $\hat{S}\equiv S(\hat{\mathbf{\Psi}})\subseteq S_0$, and satisfies
\begin{align*}
\|\hat{\mathbf{\Psi}}-\mathbf{\Psi}_0\|_{\infty}&\leq\frac{c_{\boldsymbol{\Gamma}_0}}{2-\alpha}\lambda.
\end{align*}
\item If 
\[\min\limits_{1\leq j<k\leq m}|\mathbf{\Psi}_{0,jk}|>\frac{c_{\boldsymbol{\Gamma}_0}}{2-\alpha}\lambda,\]
then $\hat{S}=S_0$ and $\mathrm{sign}(\hat{\mathbf{\Psi}}_{jk})=\mathrm{sign}(\mathbf{\Psi}_{0.jk})$ for all $(j,k)\in S_0$.
\end{enumerate}
\end{theorem}
This result is deterministic, and the improvement of our generalized
estimator over the one in \citet{lin16} is in its probabilistic guarantees, as shown for truncated GGMs in Theorems \ref{corollary1} and \ref{corollary2} in Section \ref{Generalized Score Matching for Truncated GGMs}.  Before going into these examples, we
state a general corollary.  

\begin{corollary}\label{cor_L2-consistency}
  Under the assumptions of Theorem \ref{theorem_lin}, the matrix
  $\hat{\mathbf{\Psi}}$ minimizing (\ref{eq_definition_RGSM_loss_xi})
  satisfies
\begin{align*}
\|\hat{\mathbf{\Psi}}-\mathbf{\Psi}_0\|_F & \;\leq\; \frac{c_{\boldsymbol{\Gamma}_0}}{2-\alpha}\lambda\sqrt{|S_0|}\;\leq\; \frac{c_{\boldsymbol{\Gamma}_0}}{2-\alpha}\lambda\sqrt{d_{\mathbf{\Psi}_0}m},\\
\|\hat{\mathbf{\Psi}}-\mathbf{\Psi}_0\|_2&\;\leq\;\frac{c_{\boldsymbol{\Gamma}_0}}{2-\alpha}\lambda\min(\sqrt{|S_0|},d_{\mathbf{\Psi}_0}).
\end{align*}
\end{corollary}
 

\subsection{Revisiting Gaussian Score Matching}
\label{sec:revis-gauss-score}

In this section we consider estimating the inverse covariance matrix
$\mathbf{K}$ of a centered Gaussian distribution $\mathrm{N}(\boldsymbol{0},\mathbf{K})$, which of course has
density proportional to~(\ref{eq:tn-density}) on all of
$\mathbb{R}^m$.   As shown, e.g., in Example 1 of \citet{lin16},  the
$\ell_1$-regularized score matching loss then takes the form
\begin{equation}\label{eq_gaussian}
\frac{1}{2}\mathrm{tr}(\mathbf{KKxx}^{\top})-\mathrm{tr}(\mathbf{K})+\lambda\|\mathbf{K}\|_1,
\end{equation}
which can be written as (\ref{loss_lin}) with $\boldsymbol{\theta}=\mathrm{vec}(\mathbf{K})$, $\boldsymbol{\Gamma}=\mathrm{diag}(\mathbf{xx}^{\top},\dots,\mathbf{xx}^{\top})$ and $\boldsymbol{g}=\mathrm{vec}(\mathbf{I}_m)$. Thus, in general, the kernel of $\boldsymbol{\Gamma}$ need not be orthogonal to $\boldsymbol{g}$, and for $\lambda$ small the loss can be unbounded below as discussed above. Hence, an amplifier/multiplier on the diagonals of $\boldsymbol{\Gamma}$ is needed. We have the following theorem on the estimator using the amplification.

\begin{theorem}\label{thm_gau}
Suppose the data matrix $\mathbf{x}$ holds $n$ i.i.d.~copies of
$\boldsymbol{X}\sim\mathrm{N}(\boldsymbol{0},\mathbf{K}_0)$.
Adopt the amplifying in Section \ref{Regularized Generalized Score Matching} and redefine the loss in (\ref{eq_gaussian}) as
\begin{equation}
\frac{1}{2}\mathrm{tr}(\mathbf{KKG})-\mathrm{tr}(\mathbf{K})+\lambda\|\mathbf{K}\|_1,\quad\mathbf{G}_{jk}=(\mathbf{xx}^{\top})_{jk}\left(\mathds{1}_{\{j\neq k\}}+(\delta-1)\mathds{1}_{\{j=k\}}\right),
\end{equation}
where $1<\delta<2-\left(1+80\sqrt{\log m/n}\right)^{-1}$.  Let $\hat{\mathbf{K}}$ be the resulting estimator.  Let $c^*\equiv 12800\left(\max_j\boldsymbol{\Sigma}_{0,jj}\right)^2$ and $c_1=4c_{\boldsymbol{\Gamma}_0}/\alpha$.  If for some $\tau>2$, the regularization parameter and the sample size satisfy
\begin{align*}
  \lambda&>(2c_{\mathbf{K}_0}(2-\alpha)\sqrt{c^*(\tau\log m+\log
  4)/n})/\alpha,\\
  n&>\max(c^*c_1^2d_{\mathbf{K}_0}^2,2)(\tau\log m+\log 4),
\end{align*}
then
     $\|\hat{\mathbf{K}}-\mathbf{K}_0\|_{\infty}\leq\frac{c_{\boldsymbol{\Gamma_0}}}{2-\alpha}\lambda$
     with probability $1-m^{2-\tau}$.
\end{theorem}
In Corollary 1 of \citet{lin16} the same results were shown with $c^*\equiv 3200\left(\max_j\boldsymbol{\Sigma}_{0,jj}\right)^2$ when a unique minimizer exists, but the existence was not guaranteed.

\subsection{Generalized Score Matching for Truncated GGMs}\label{Generalized Score Matching for Truncated GGMs}
Next, we provide theory for the regularized generalized $\boldsymbol{h}$-score matching estimator $\hat{\boldsymbol{\Psi}}$ in the special case of truncated GGMs. Again, assume a common $h$ for all components in
$\boldsymbol{h}$.
\begin{theorem}\label{corollary1}
  Suppose the data matrix $\mathbf{x}$ holds $n$ i.i.d.~copies of
  $\boldsymbol{X}\sim\mathrm{TN}(\boldsymbol{0},\mathbf{K}_0)$, where the mean parameter is known to be zero.
 Assume that $\boldsymbol{h}\in\mathcal{H}_{1,1}$ and that $0\leq h\leq M$, $0\leq h'\leq M'$ a.s.~for constants $M,M'$, and choose $\boldsymbol{\gamma}=(\delta-1)\mathrm{diag}(\boldsymbol{\Gamma})$ with
  \[
    1<\delta<C(n,m)\equiv 2-\left(1+4e\max\{6\log m/n,\sqrt{6\log
        m/n}\}\right)^{-1}.
  \]
  Suppose that the $\boldsymbol{\Gamma}_{0,S_0S_0}$ block of $\boldsymbol{\Gamma}_{0}$ is invertible and $\boldsymbol{\Gamma}_0$ satisfies the irrepresentability condition (\ref{irrepresentability}) with $\alpha\in(0,1]$ and true edge set $S_0$. Define
  $c_{\boldsymbol{X}}\equiv
  2\max_j\bigg(2\sqrt{(\mathbf{K}_0^{-1})_{jj}}+\sqrt{e}\,\mathbb{E}_{0}
  X_j\bigg)$.  If
  for $\tau>3$ the sample size and the regularization
  parameter satisfy
\begin{align}
n&>\mathcal{O}\left(\tau\log m\max\left\{\frac{M^2c_{\boldsymbol{\Gamma}_0}^2c_{\boldsymbol{X}}^4d_{\mathbf{K}_0}^2}{\alpha^2},\frac{Mc_{\boldsymbol{\Gamma}_0}c_{\boldsymbol{X}}^2d_{\mathbf{K}_0}}{\alpha}\right\}\right),\\
\lambda&>\mathcal{O}\left[(Mc_{\mathbf{K}_0}c_{\boldsymbol{X}}^2+M'c_{\boldsymbol{X}}+M)\left(
\sqrt{\frac{\tau\log m}{n}}+\frac{\tau\log m}{n}\right)\right],
\end{align}
 then the following statements hold with probability $1-m^{3-\tau}$:
\begin{enumerate}[(a)]
\item The regularized generalized $\boldsymbol{h}$-score matching estimator $\hat{\mathbf{K}}$ that minimizes (\ref{eq_definition_RGSM_loss_xi}) is unique, has its support included in the true support, $\hat{S}\equiv S(\hat{\mathbf{K}})\subseteq S_0$, and satisfies
\begin{align*}
\|\hat{\mathbf{K}}-\mathbf{K}_0\|_{\infty}&\leq\frac{c_{\boldsymbol{\Gamma}_0}}{2-\alpha}\lambda,
  \\
\mnorm{\hat{\mathbf{K}}-\mathbf{K}_0}_{F}&\leq\frac{c_{\boldsymbol{\Gamma}_0}}{2-\alpha}\lambda\sqrt{|S_0|\color{black}{}},\\
\mnorm{\hat{\mathbf{K}}-\mathbf{K}_0}_{2}&\leq\frac{c_{\boldsymbol{\Gamma}_0}}{2-\alpha}\lambda\min(\sqrt{|S_0|\color{black}{}},d_{\mathbf{K}_0}),
\end{align*}
where $c_{\boldsymbol{\Gamma}_0}$ is defined in (\ref{def_c_gamma}).
\item Moreover, if
\[\min_{j,k:(j,k)\in S_0}|\kappa_{0,jk}|>\frac{c_{\boldsymbol{\Gamma}_0}}{2-\alpha}\lambda,\]
then $\hat{S}=S_0$ and $\mathrm{sign}(\hat{\kappa}_{jk})=\mathrm{sign}(\kappa_{0,jk})$ for all $(j,k)\in S_0$.
\end{enumerate}
\end{theorem}

The theorem is proved in Appendix~\ref{Proof of Theorems in Section Theory for Graphical Models}, where details on the
dependencies on constants are provided.
A key ingredient of the proof is a tail bound on
$\|\boldsymbol{\Gamma}_{\boldsymbol{\gamma}}-\boldsymbol{\Gamma}_0\|_{\infty}$, which features products of the $X_j^{(i)}$'s. In \citet{lin16}, the products are up to fourth order.  Using bounded $\boldsymbol{h}$, our products automatically calibrates to a quadratic polynomial when the observed values are large, and resort to higher moments only when they are small.  This leads to improved bounds and convergence rates, underscored in the new requirement on the sample size $n$, which should be compared to $n\ge\mathcal{O}(d_{\mathbf{K}_0}^2(\log m^{\tau})^8)$ in \citet{lin16}.

For the non-centered case, by definition,
$c_{\boldsymbol{\Psi}_0}\equiv\mnorm{\boldsymbol{\Psi_0}^{\top}}_{\infty,\infty}\leq
c_{\mathbf{K}_0}+\|\boldsymbol{\eta}_0\|_{\infty}$,
$d_{\boldsymbol{\Psi}_0}\leq d_{\mathbf{K}_0}+1$.  The proof given for
Theorem \ref{corollary1} goes through again here, and we have the
following consistency results.  
\begin{theorem}\label{corollary2}
  Suppose the data matrix holds $n$ i.i.d.~copies of
  $\boldsymbol{X}\sim\mathrm{TN}(\boldsymbol{\mu}_0,\mathbf{K}_0)$.
  Assume that $\boldsymbol{h}\in\mathcal{H}_{1,1}$ and that $0\leq h\leq M$, $0\leq h'\leq M'$ a.s.~for constants $M,M'$.  Let $\boldsymbol{\gamma}$ be a vector of amplifiers that are
  non-zero only for the diagonal entries of the matrices
  $\boldsymbol{\Gamma}_{11,j}$, amplifying those by
  $(\delta-1)\mathrm{diag}(\boldsymbol{\Gamma}_{11,j})$ with
  \[
    1<\delta<C(n,m)\equiv 2-\left(1+4e\max\{6\log m/n,\sqrt{6\log
        m/n}\}\right)^{-1}.
  \]
  Suppose further that $\boldsymbol{\Gamma}_{0,S_0S_0}$ is invertible
  and satisfies the irrepresentability condition
  (\ref{irrepresentability}) with $\alpha\in(0,1]$.  Define
  $c_{\boldsymbol{X}}\equiv2\max_j\left(2\sqrt{(\mathbf{K}_0^{-1})_{jj}}+\sqrt{e}\,\mathbb{E}_{0}
  X_j\right)$.  Suppose for $\tau>3$ the sample size and the regularization
  parameter satisfy
\begin{align}
n&>\mathcal{O}\left(\tau\log m\max\left\{\frac{M^2c_{\boldsymbol{\Gamma}_0,\boldsymbol{\Psi}_0}^2c_{\boldsymbol{X}}^4d_{\boldsymbol{\Psi}_0}^2}{\alpha^2},\frac{Mc_{\boldsymbol{\Gamma}_0,\boldsymbol{\Psi}_0}c_{\boldsymbol{X}}^2d_{\boldsymbol{\Psi}_0}}{\alpha}\right\}\right),\\
\lambda&>\mathcal{O}\left[(Mc_{\boldsymbol{\Psi}_0}c_{\boldsymbol{X}}^2+M'c_{\boldsymbol{X}}+M)\left(
\sqrt{\frac{\tau\log m}{n}}+\frac{\tau\log m}{n}\right)\right],
\end{align}
where $c_{\boldsymbol{\Gamma}_0,\boldsymbol{\Psi}_0}$ is $c_{\boldsymbol{\Gamma}_0}$ as in (\ref{def_c_gamma})  but with notation $\boldsymbol{\Psi}_0$ to differentiate it from the centered case.
Then the following statements hold with probability $1-m^{3-\tau}$:
\begin{enumerate}[(a)]
\item The regularized generalized $\boldsymbol{h}$-score matching estimator $\hat{\boldsymbol{\Psi}}$ that minimizes (\ref{eq_definition_RGSM_loss_xi}) is unique, has its support included in the true support, $\hat{S}\equiv S(\hat{\boldsymbol{\Psi}})\subseteq S_0$, and satisfies
\begin{alignat*}{3}
\|\hat{\mathbf{K}}-\mathbf{K}_0\|_{\infty}&\leq\frac{c_{\boldsymbol{\Gamma}_0,\boldsymbol{\Psi}_0}}{2-\alpha}\lambda,&&\|\hat{\boldsymbol{\eta}}-\boldsymbol{\eta}_0\|_{\infty}\leq\frac{c_{\boldsymbol{\Gamma}_0,\boldsymbol{\Psi}_0}}{2-\alpha}\lambda,\\
\mnorm{\hat{\mathbf{K}}-\mathbf{K}_0}_{F}&\leq\frac{c_{\boldsymbol{\Gamma}_0, \boldsymbol{\Psi}_0}}{2-\alpha}\lambda\sqrt{|S_0|},&&\mnorm{\hat{\boldsymbol{\eta}}-\boldsymbol{\eta}_0}_{F}\leq\frac{c_{\boldsymbol{\Gamma}_0, \boldsymbol{\Psi}_0}}{2-\alpha}\lambda\sqrt{|S_0|},\\
\mnorm{\hat{\mathbf{K}}-\mathbf{K}_0}_{2}&\leq\frac{c_{\boldsymbol{\Gamma}_0,\boldsymbol{\Psi}_0}}{2-\alpha}\lambda\min\left(\sqrt{|S_0|},d_{\boldsymbol{\Psi}_0}\right),\quad  &&\mnorm{\hat{\boldsymbol{\eta}}-\boldsymbol{\eta}_0}_{2}\leq\frac{c_{\boldsymbol{\Gamma}_0,\boldsymbol{\Psi}_0}}{2-\alpha}\lambda\min\left(\sqrt{|S_0|},d_{\boldsymbol{\Psi}_0}\right).
\end{alignat*}
\item Moreover, if
\[\min_{j,k:(j,k)\in S_0}|\kappa_{0,jk}|>\frac{c_{\boldsymbol{\Gamma}_0}}{2-\alpha}\lambda\quad\quad\text{and}\quad\quad\min_{j:(m+1,j)\in S_0}|\eta_{0,j}|>\frac{c_{\boldsymbol{\Gamma}_0}}{2-\alpha}\lambda,\]
then $\hat{S}=S_0$ and $\mathrm{sign}(\hat{\kappa}_{jk})=\mathrm{sign}(\kappa_{0,jk})$ for all $(j,k)\in S_0$ and $\mathrm{sign}(\hat{\eta}_j)=\mathrm{sign}(\eta_{0j})$ for $(m+1,j)\in S_0$.
\end{enumerate}
\end{theorem}

\begin{remark}
  \rm The quantity $c_{\boldsymbol{X}}$ in Theorem~\ref{corollary2}
  depends on $\mathbb{E}_{0}X_j$, which in turn depends on the
  structure of both $\boldsymbol{\mu}_0$ and $\mathbf{K}_0$.  If
  $\mu_{0,j}$ is large compared to ${(\mathbf{K}_0)^{-1}_{jj}}$, then
  $c_{\boldsymbol{X}}$ seems to scale as $\boldsymbol{\mu}_0$, which
  negatively impacts the guarantees stated in
  Theorem~\ref{corollary2}.  However, as in the one-dimensional case
  for estimation of $\mu_0$ (Example~\ref{theorem_GSM_normality_tnorm_mu}), our
  estimator should automatically adapt to the large mean parameter.
  This suggests that it might be possible to improve our analysis
  involving $c_{\boldsymbol{X}}$.
\end{remark}

\section{Numerical Experiments}\label{Numerical Experiments}

In this section, we compare the performance of our estimator with different choices of $\boldsymbol{h}$ to the existing approaches for pairwise interaction power models. In our simulation experiments, we consider $m=100$ variables and $n=80$ and $n=1000$ samples, corresponding to high- and low-dimensional settings. We also tried intermediate sample sizes between these two extremes, but found no interesting result worth reporting. For $n=80$, amplification is necessary. Except in Section~\ref{Choice of multiplier}, the amplifier is set based on Theorem~\ref{corollary1} to $\delta=C(n,m)=1.8647$ for truncated GGMs. The same amplifier is also used for settings with other $a$ and $b$. For $n=1000$, we consider $\delta=1$, i.e., no amplification, and $\delta=C(n,m)=1.6438$ (again, based on Theorem~\ref{corollary1}). Throughout, we assume a common univariate $h$ for all components in $\boldsymbol{h}$.

%


\subsection{Structure of \texorpdfstring{$\mathbf{K}$}{K}}\label{Structure of K}
The underlying interaction matrices are selected as follows:
Proceeding as in Section~4.2 of \citet{lin16}, the graph is chosen to have 10 disconnected subgraphs, each containing $m/10$ nodes. Thus, $\mathbf{K}_0$ is block-diagonal. In each block, each lower-triangular element is set to $0$ with probability $1-\pi$ for some $\pi\in(0,1)$, and is otherwise drawn from $\text{Uniform}[0.5,1]$. The upper triangular elements are determined by symmetry.  The diagonal elements of $\mathbf{K}_0$ are chosen as a common positive value such that the minimum eigenvalue of $\mathbf{K}_0$ is $0.1$.

We generate 5 different true precision matrices $\mathbf{K}_0$, and run 10 trials with each of these precision matrices.  For $n=1000$, we choose $\pi=0.8$, 
which is in accordance with \citet{lin16}. 
For $n=80$, we set $\pi=0.2$. 
This way $n/(d_{\mathbf{K}_0}^2\log m)$ is roughly constant; recall Theorems \ref{corollary1} and \ref{corollary2} for truncated GGMs.

In Appendix~\ref{ER}, we report results on \emph{Erd\"os-R\'enyi
  graphs}, which lead to similar conclusions.

\subsection{Truncated GGMs}\label{Simulations_Truncated GGMs}
Given our focus on truncated GGMs and their relevance in graphical modeling applications, we start with experiments for these models.

\subsubsection{Choice of \texorpdfstring{$\boldsymbol{h}$}{h}}\label{Choice of h}

Our estimator requires choosing a function
$\boldsymbol{h}:\mathbb{R}_+^m\to\mathbb{R}_+^m$. For simplicity, we will always specify $\boldsymbol{h}(\boldsymbol{x})=(h(x_1),\ldots,h(x_m))$ for a single non-decreasing univariate function
$h:\mathbb{R}_+\to\mathbb{R}_+$, i.e.~all coordinates share the same $h$ function.

As previously explained, $\boldsymbol{h}\in\mathcal{H}_{a,b}$ is a sufficient condition for assumptions (A1)-(A2), as well as (C1)-(C2) in the case of unregularized estimators. Only in the proofs of our theoretical guarantees in Section~\ref{Theory for Graphical Models} for truncated GGMs, did we require $h$ to be bounded and to have bounded derivatives. As motivated by the discussion in Section~\ref{Reasonable Choices of h}, we consider truncated and untruncated powers, $\min(x,c)$ and $x$ (since $2-a=2-b=1$); we evaluate this choice by contrasting them with powers $x^{1.5}$ and $x^2$. 
We also explore functions like $\log(1+x)$ that seem natural and are linear near $0$. In particular, we make a further comparison to functions linear near $0$ with a finite asymptote as $x\nearrow+\infty$ but differentiable everywhere: MCP- \citep{fan01} and SCAD-like \citep{zha10} functions defined below. The results we report are based on selections of best performing choices of $h$.
\[
\mathrm{SCAD}(x;\lambda,\gamma)\equiv\begin{cases} \lambda x & \text{if }0\leq x\leq \lambda,\\ \frac{2\gamma\lambda x-x^2-\lambda^2}{2(\gamma-1)} & \text{if }\lambda<x<\gamma\lambda,\\\frac{\lambda^2(\gamma+1)}{2} & \text{if }x\geq\gamma\lambda;\end{cases}\hfill \mathrm{MCP}(x;\lambda,\gamma)\equiv\begin{cases}\lambda x-\frac{x^2}{2\gamma} & \text{if }0\leq x\leq\gamma\lambda, \\ \frac{1}{2}\gamma\lambda^2 & \text{if }x>\gamma\lambda.\end{cases}
\]

We do not observe any clear relationship between features such as convexity, differentiability or the slope of $h$ at $0$, and performance of the estimator. Nonetheless, for many choices of rather simple functions $h$, our estimator provides a significant improvement over existing methods. In particular, most $h$ functions that behave linearly for small $x$, namely $\log(1+x)$ and $x$ and their truncations, and additionally MCP and SCAD, always perform better than $x^{1.5}$ and $x^2$. This agrees with our discussion in  Section~\ref{Reasonable Choices of h}, where $2-a=1$ is a reasonable choice of the power for small $x$; also see Section \ref{Other Models}. However, we conclude that there is no real gain from making the function smoother by using MCP or SCAD.\\


\noindent\emph{Truncated Centered GGMs}\label{Chooseh_Truncated Centered GGMs}:\indent For data from a truncated centered Gaussian distribution, we compare our generalized score matching estimator 
with various choices of $h$, to \emph{SpaCE JAM} (SJ,
\citealp{voo13}), which estimates graphs using additive models for conditional means, a pseudo-likelihood method \emph{SPACE}
\citep{pen09} in the reformulation of \cite{kha15}, \emph{graphical lasso} (GLASSO, \citealp{yua07,fri08}), the \emph{neighborhood selection} estimator (NS) of \citet{mei06}, and \emph{nonparanormal SKEPTIC} \citep{liu12} with Kendall's $\tau$. Recall that the choice of $h(x)=x^2$ corresponds to the estimator from \citet{lin16}. 

\begin{figure}[p]
\centering
\subfloat[$n=80$, $\mathrm{mult}=1.8647$]
{\includegraphics[trim={0 0.5in 1.3in 0},clip,scale=0.26]{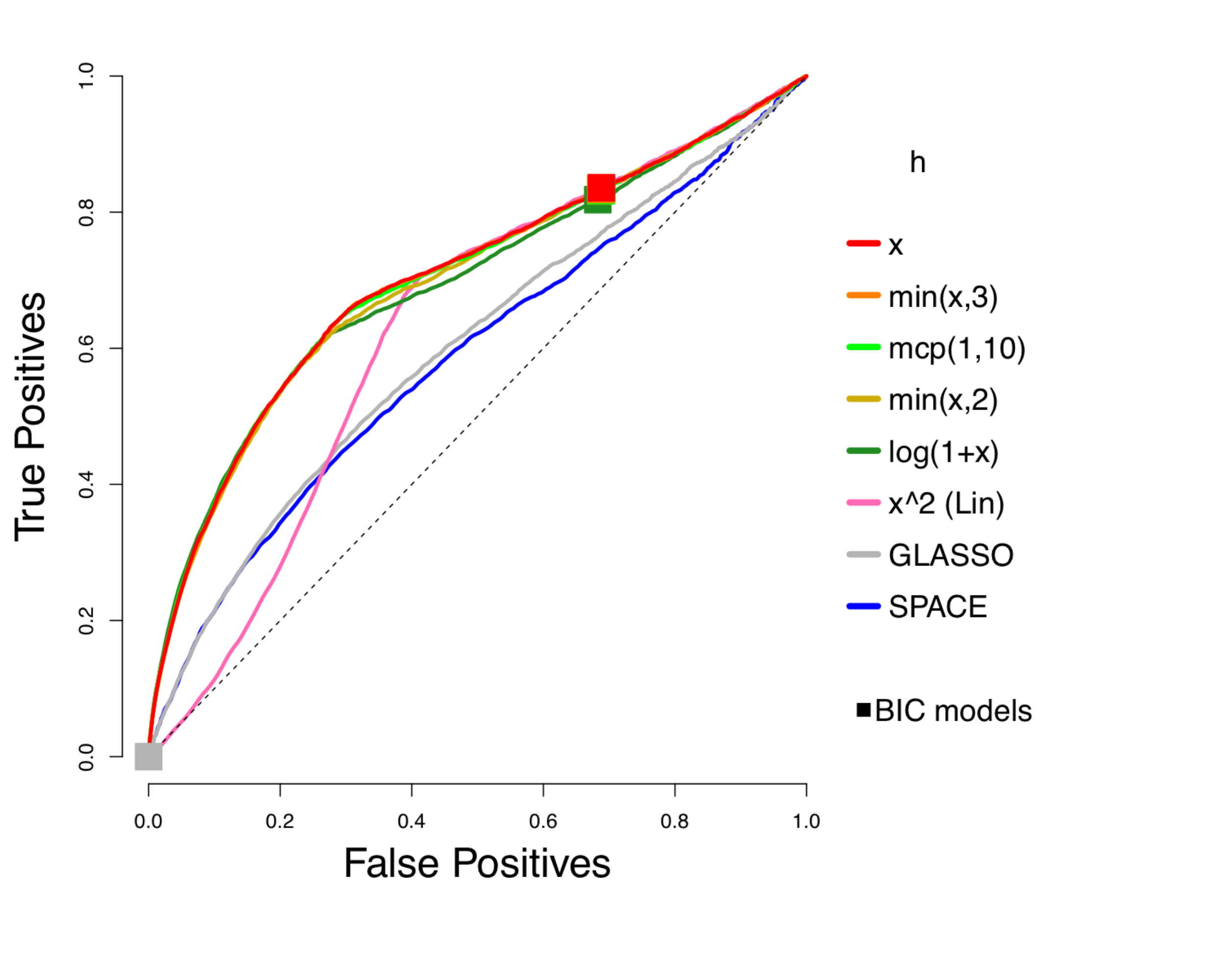}\hspace{-0.02in}}
\subfloat[$n=1000$, $\mathrm{mult}=1$]
{\includegraphics[trim={0 0.5in 2in 0},clip,scale=0.26]{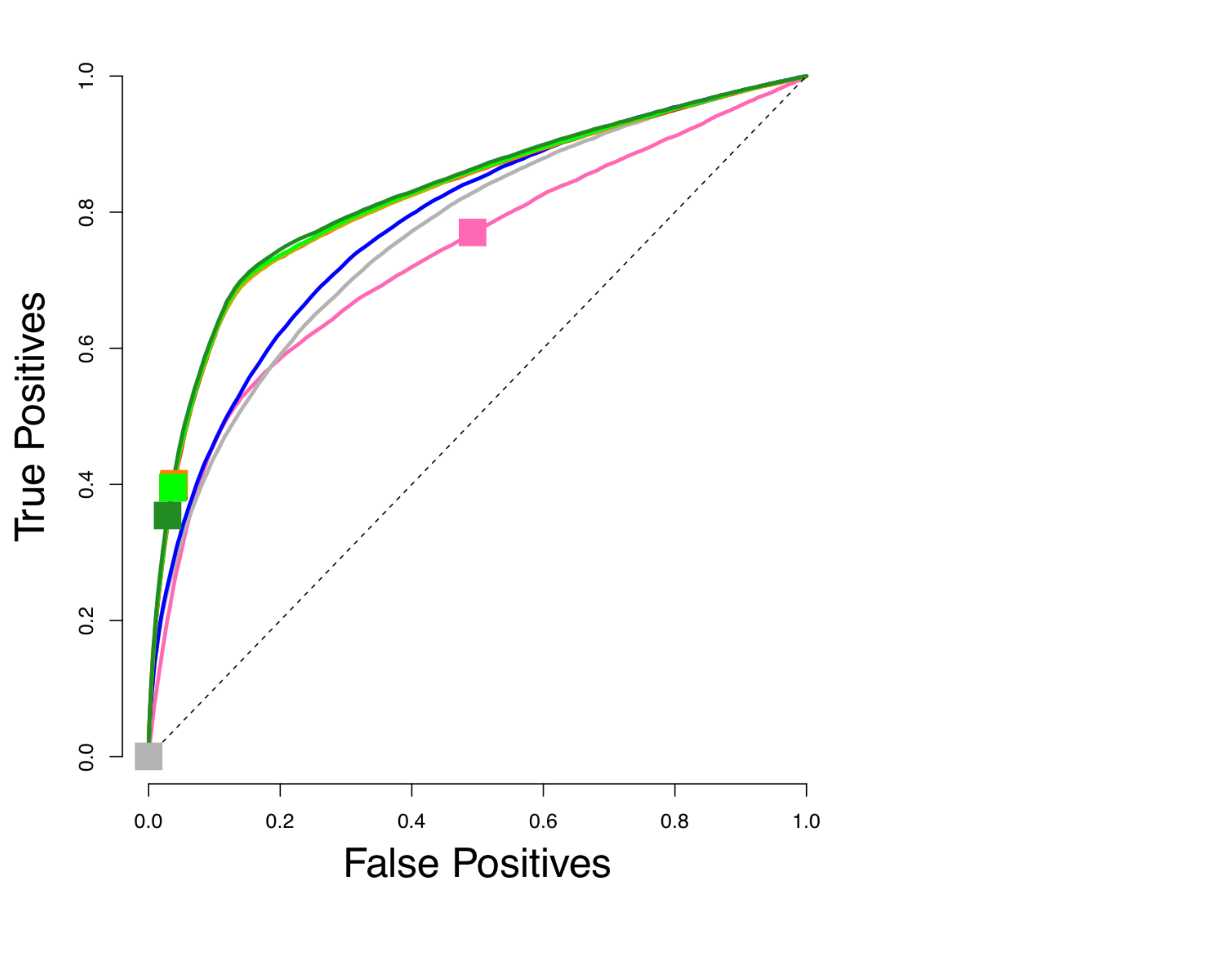}\hspace{-0.3in}}
\subfloat[$n=1000$, $\mathrm{mult}=1.6438$]
{\includegraphics[trim={0 0.5in 3in 0},clip,scale=0.26]{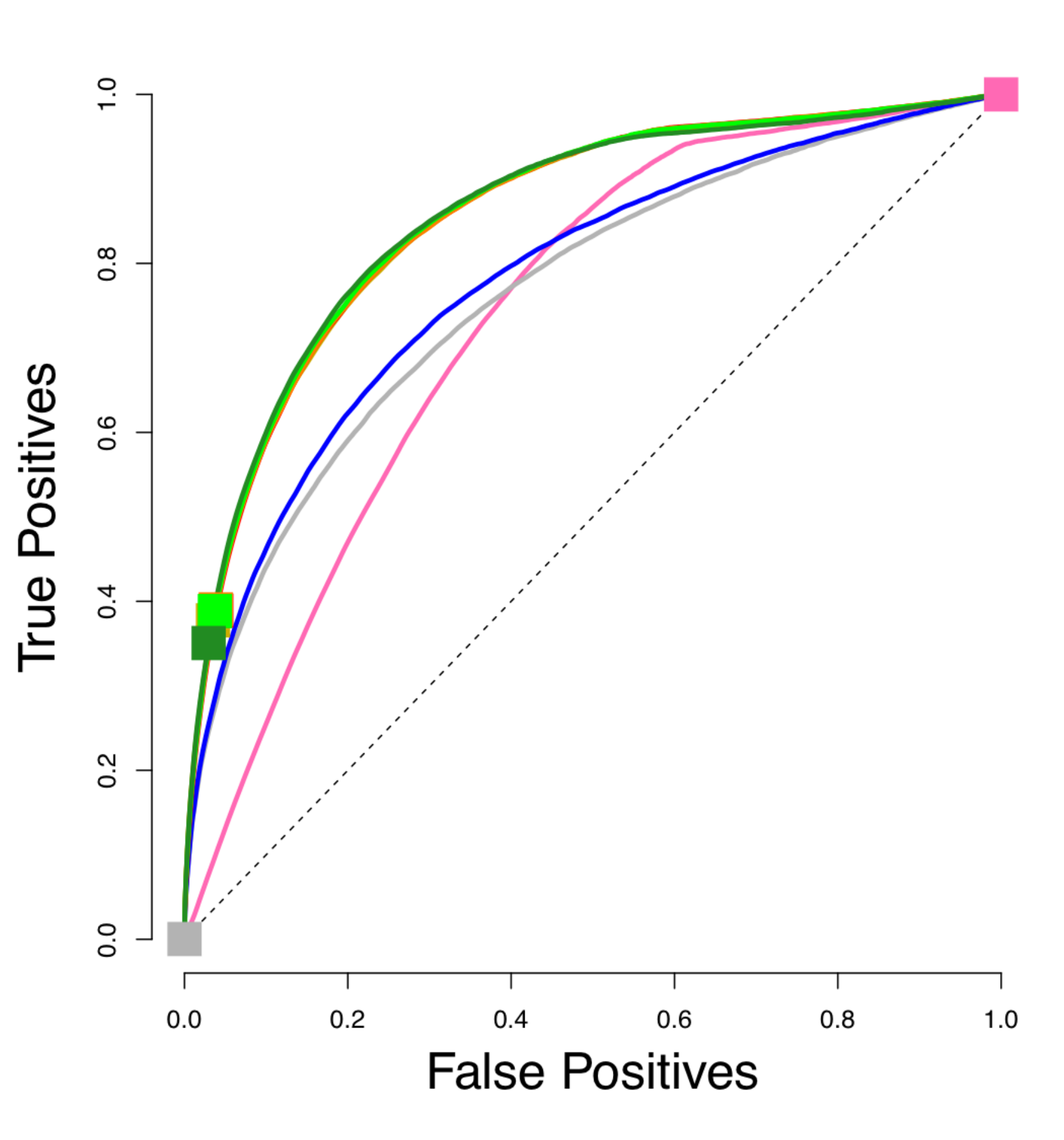}}
\caption{Average ROC curves of our centered estimator with various choices of $h$, compared to SPACE and GLASSO, for the truncated centered GGM case; $m=100$ variables and $n=80$ or $1000$ samples are considered. Squares indicate average true positive rate (TPR) and false positive rate (FPR) of models picked by eBIC with refitting for the estimator in the same color.}
\label{plot_chooseh_centered}
\end{figure}

\vspace{-0.0in}
\begin{table}[ht]
\begin{center}
\scalebox{0.85}{
\begin{tabular}{|c|c|c|c|c|c|}
\hline
\multicolumn{6}{|c|}{Centered, $n=80$, multiplier 1.8647}
\tabularnewline
\hline\hline
\multicolumn{3}{|c|}{$\min(\log(1+x),c)$} & \multicolumn{3}{c|}{$\min(x,c)$}
\tabularnewline
\hline
$c$ & Mean & sd & $c$ & Mean & sd
\tabularnewline\hline
$\infty$ & 0.694 & 0.033 & $\infty$ & 0.702 & 0.031
\tabularnewline
2 & 0.694 & 0.033  & 3 & 0.702 & 0.031
\tabularnewline
1 & 0.692 & 0.033 & 2 &  0.698 & 0.033
\tabularnewline
0.5 & 0.664 & 0.038  & 1 & 0.686 & 0.030
\tabularnewline
\hline
\multicolumn{3}{|c|}{$\mathrm{MCP}(1,c)$} & \multicolumn{3}{c|}{$\mathrm{SCAD}(1,c)$}
\tabularnewline
\hline
$c$ & Mean & sd & $c$ & Mean & sd
\tabularnewline
\hline
10 & 0.701 & 0.032 & 10 & 0.702 & 0.031
\tabularnewline
5 & 0.700 & 0.032 & 5 & 0.701 & 0.032
\tabularnewline
1 & 0.672 & 0.036 & 2 & 0.696 & 0.033
\tabularnewline
\hline
\multicolumn{3}{|c|}{$x^{1.5}$: (0.683, 0.030)} & \multicolumn{3}{|c|}{$x^2$: (0.630, 0.029)}
\tabularnewline
\hline
\hline
\multicolumn{3}{|c|}{GLASSO (0.600,0.032)} & \multicolumn{3}{|c|}{SPACE: (0.587, 0.031)}
\tabularnewline
\hline
\multicolumn{3}{|c|}{NS: (0.587,0.031)} & \multicolumn{3}{|c|}{SJ: (0.540,0.036)}
\tabularnewline
\hline
\end{tabular}}
\vskip0.0in  
\scalebox{0.85}{
\begin{tabular}{|c|c|c|c|c|c|}
\hline
\multicolumn{6}{|c|}{Centered, $n=1000$, multiplier 1}
\tabularnewline
\hline\hline
\multicolumn{3}{|c|}{$\min(\log(1+x),c)$} & \multicolumn{3}{c|}{$\min(x,c)$}
\tabularnewline
\hline
$c$ & Mean & sd & $c$ & Mean & sd
\tabularnewline\hline
2 & 0.826 & 0.015 & 2 & 0.820 & 0.014
\tabularnewline
$\infty$ & 0.826 & 0.015  & 3  & 0.820 & 0.015
\tabularnewline
1 & 0.824 & 0.014 & $\infty$ & 0.819 & 0.015
\tabularnewline 
0.5 & 0.804 & 0.015 & $1$ & 0.817 & 0.014
\tabularnewline
\hline
\multicolumn{3}{|c|}{$\mathrm{MCP}(1,c)$} & \multicolumn{3}{c|}{$\mathrm{SCAD}(1,c)$}
\tabularnewline
\hline
$c$ & Mean & sd & $c$ & Mean & sd
\tabularnewline
\hline
5 & 0.824 & 0.015 & 2 & 0.823 & 0.014
\tabularnewline
10 & 0.822 & 0.015 & 5 & 0.822 & 0.015
\tabularnewline
1 & 0.810 & 0.015 & 10 & 0.821 & 0.015
\tabularnewline
\hline
\multicolumn{3}{|c|}{$x^{1.5}$: (0.782,0.014)} & \multicolumn{3}{|c|}{$x^2$: (0.732,0.015)}
\tabularnewline
\hline
\hline
\multicolumn{3}{|c|}{SPACE: (0.780,0.015)} & \multicolumn{3}{|c|}{NS: (0.779,0.015)}
\tabularnewline
\hline
\multicolumn{3}{|c|}{GLASSO (0.764,0.014)} & \multicolumn{3}{|c|}{SJ: (0.703,0.015)}
\tabularnewline
\hline
\end{tabular}
\begin{tabular}{|c|c|c|c|c|c|}
\hline
\multicolumn{6}{|c|}{Centered, $n=1000$, multiplier 1.6438}
\tabularnewline
\hline\hline
\multicolumn{3}{|c|}{$\min(\log(1+x),c)$} & \multicolumn{3}{c|}{$\min(x,c)$}
\tabularnewline
\hline
$c$ & Mean & sd & $c$ & Mean & sd
\tabularnewline\hline
$\infty$ & 0.857 & 0.011 & 3 & 0.855 & 0.011
\tabularnewline
2 & 0.857 & 0.011 & $\infty$ & 0.855 & 0.011
\tabularnewline
1 & 0.855 & 0.011 & 2 & 0.854 & 0.011
\tabularnewline 
0.5 & 0.833 & 0.012 & 1 & 0.847 & 0.011
\tabularnewline
\hline
\multicolumn{3}{|c|}{$\mathrm{MCP}(1,c)$} & \multicolumn{3}{c|}{$\mathrm{SCAD}(1,c)$}
\tabularnewline
\hline
$c$ & Mean & sd & $c$ & Mean & sd
\tabularnewline
\hline
5 & 0.857 & 0.011 & 5 & 0.856 & 0.011
\tabularnewline
10 & 0.856 & 0.011 & 10 & 0.855 & 0.011
\tabularnewline
1 & 0.840 & 0.012 & 2 & 0.855 & 0.011
\tabularnewline
\hline
\multicolumn{3}{|c|}{$x^{1.5}$: (0.812,0.011)} & \multicolumn{3}{|c|}{$x^2$: (0.736,0.011)}
\tabularnewline
\hline
\hline
\multicolumn{3}{|c|}{SPACE: (0.780,0.015)} & \multicolumn{3}{|c|}{NS: (0.779,0.015)}
\tabularnewline
\hline
\multicolumn{3}{|c|}{GLASSO (0.764,0.014)} & \multicolumn{3}{|c|}{SJ: (0.703,0.015)}
\tabularnewline
\hline
\end{tabular}}
\end{center}
\vspace{-0.1in}
\caption{Mean and standard deviation of areas under the ROC curves (AUC) using different estimators in the centered setting, with $n=80$ and multiplier $1.8647$, or $n=1000$ and multiplier $1$ and $1.6438$. Methods include our estimator with different choices of $h$,  GLASSO, SPACE, neighborhood selection (NS), and Space JAM (SJ).}\label{table_chooseh_centered}
\vspace{-0.3in}
\end{table}

The ROC (\emph{receiver operating characteristic}) curves for different estimators are shown in Figure~\ref{plot_chooseh_centered} on Page~\pageref{plot_chooseh_centered}. Each plotted curve corresponds to the average of 50 ROC curves, where the averaging is based on the vertical averaging from Algorithm~3 in \citet{faw06}, and is mean AUC-preserving. The $x$ and $y$ axes of each ROC curve represent the false positive and true positive rates at varying levels of penalty parameter $\lambda$, defined as

\[\mathrm{FPR}\equiv\frac{|\hat{S}_{\text{off}}\backslash S_{0,\text{off}}|}{m(m-1)-|S_{0,\text{off}}|}\quad\quad\text{and}\quad\quad \mathrm{TPR}\equiv\frac{|\hat{S}_{\text{off}}\cap S_{0,\text{off}}|}{|S_{0,\text{off}}|},\]
\noindent where $S_{0,\text{off}}\equiv\{(i,j):i\neq j\wedge\kappa_{0,ij}\neq 0\}$, and $\hat{S}_{\text{off}}\equiv\{(i,j):i\neq j\wedge\hat{\kappa}_{ij}\neq 0\}$.

To reduce clutter, we only report the results for the top performing competing methods. In particular, results for nonparanormal SKEPTIC are omitted, as the method always performs the worst in our experiments. The corresponding means and standard deviations of AUCs (\emph{areas under the curves}) over 50 curves are given in Table~\ref{table_chooseh_centered}. 

Looking at the mean AUCs, with the standard deviations in mind, all choices of $h$ considered here perform better than $h(x)=x^2$ from \citet{hyv07} and \citet{lin16} and the competing methods. The results for $n=1000$ in Table \ref{table_chooseh_centered} also show that the multiplier does help improve the AUCs, a matter to be discussed in Section \ref{Choice of multiplier}.

\newpage

\noindent\emph{Truncated Non-Centered GGMs}\label{Chooseh_Truncated Non-Centered GGMs}: \indent We generate data from a truncated non\hyph centered Gaussian distribution with both parameters $\boldsymbol{\mu}$ and $\mathbf{K}$
unknown. In each trial, we form the true $\mathbf{K}_0$ as in
Section~\ref{Structure of K}, and generate each component of
$\boldsymbol{\mu}_0$ independently from the normal distribution with
mean $0$ and standard deviation $0.5$.

We compare the performance of our \emph{profiled} estimator based on (\ref{loss-noncentered-profiled}), with different $h$ functions, but with no penalty on $\boldsymbol{\eta}\equiv\mathbf{K}\boldsymbol{\mu}$, to SPACE, SpaCE JAM (SJ), GLASSO, and neighborhood selection (NS). As before, we consider 50 trials. Representative ROC curves are plotted in Figure~\ref{plot_chooseh_profiled}, and the corresponding AUCs are summarized in Table~\ref{table_chooseh_profiled}.

\begin{figure}[p]
\centering
\vspace{-0.4in}
\subfloat[$n=80$, $\mathrm{mult}=1.8647$]
{\includegraphics[trim={0 0.5in 1.3in 0},clip,scale=0.26]{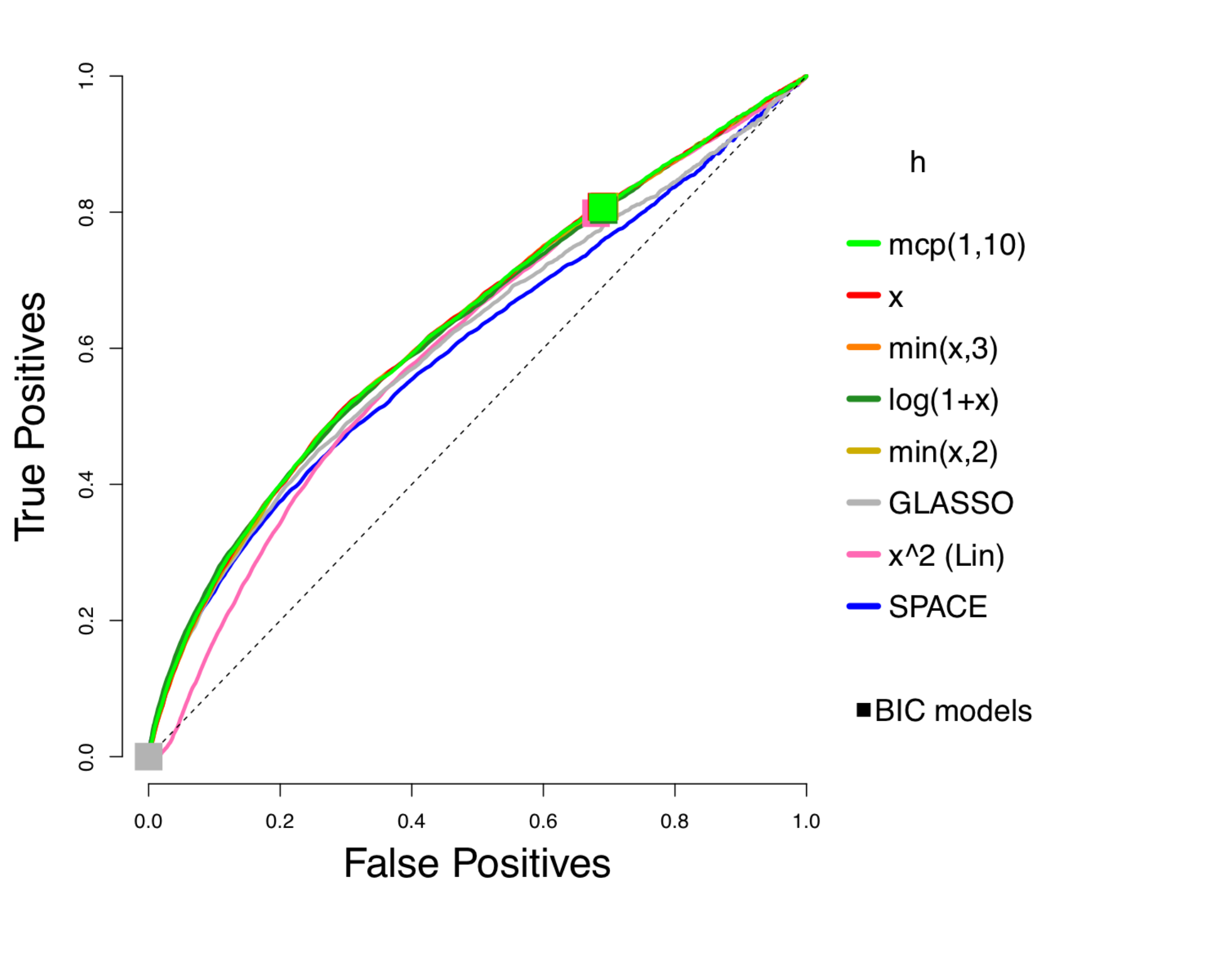}\hspace{-0.02in}}
\subfloat[$n=1000$, $\mathrm{mult}=1$]
{\includegraphics[trim={0 0.5in 2in 0},clip,scale=0.26]{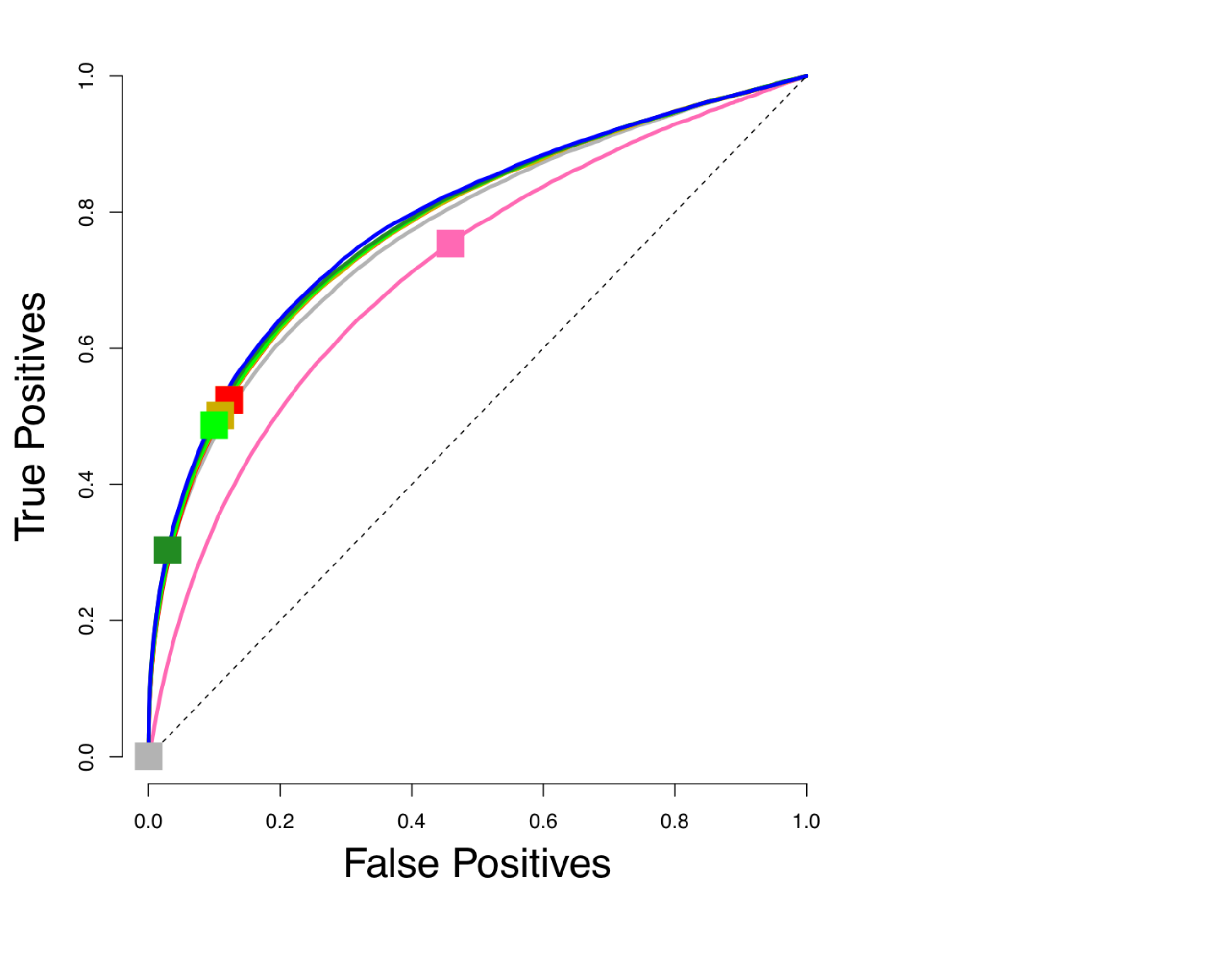}\hspace{-0.3in}}
\subfloat[$n=1000$, $\mathrm{mult}=1.6438$]
{\includegraphics[trim={0 0.5in 3in 0},clip,scale=0.26]{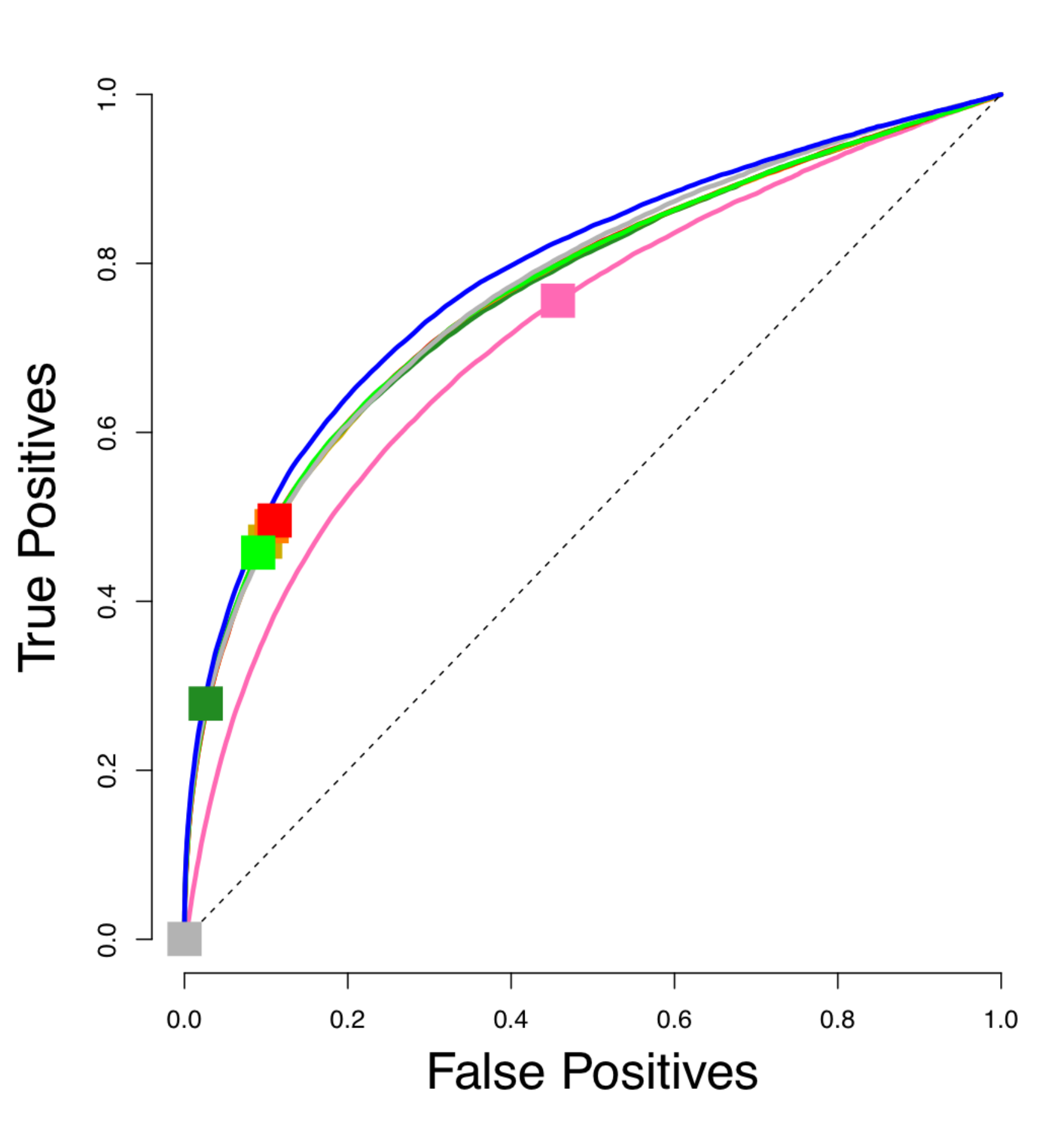}}
\caption{Average ROC curves of our non-centered profiled estimator with various choices of $h$, compared to SPACE and GLASSO, for the truncated non-centered GGM case; $m=100$ variables and $n=80$ or $1000$ samples are considered. Squares indicate average true positive rate (TPR) and false positive rate (FPR) of models picked by eBIC with refitting for the estimator in the same color.}
\label{plot_chooseh_profiled}
\end{figure}

\begin{table}[H]
\begin{center}
\scalebox{0.85}{
\begin{tabular}{|c|c|c|c|c|c|}
\hline
\multicolumn{6}{|c|}{Non-centered profiled, $n=80$, multiplier 1.8647}
\tabularnewline
\hline\hline
\multicolumn{3}{|c|}{$\min(\log(1+x),c)$} & \multicolumn{3}{c|}{$\min(x,c)$}
\tabularnewline
\hline
$c$ & Mean & sd & $c$ & Mean & sd
\tabularnewline\hline
$\infty$ & 0.632 & 0.032 & $\infty$ & 0.634 & 0.032
\tabularnewline
2 & 0.632 & 0.032 & 3 & 0.634 & 0.032
\tabularnewline
1 & 0.631 & 0.032 & 2 & 0.632 & 0.032
\tabularnewline 
0.5 & 0.619 & 0.033 & 1 & 0.628 & 0.032
\tabularnewline
\hline
\multicolumn{3}{|c|}{$\mathrm{MCP}(1,c)$} & \multicolumn{3}{c|}{$\mathrm{SCAD}(1,c)$}
\tabularnewline
\hline
$c$ & Mean & sd & $c$ & Mean & sd
\tabularnewline
\hline
10 & 0.634 & 0.032 & 5 & 0.634 & 0.032
\tabularnewline
5 & 0.634 & 0.032 & 10 & 0.634 & 0.032
\tabularnewline
1 & 0.622 & 0.032 & 2 & 0.634 & 0.032
\tabularnewline
\hline
\multicolumn{3}{|c|}{$x^{1.5}$: (0.623,0.031)} & \multicolumn{3}{|c|}{$x^2$: (0.607,0.030)}
\tabularnewline
\hline
\hline
\multicolumn{3}{|c|}{GLASSO: (0.614,0.029)} & \multicolumn{3}{|c|}{NS: (0.604,0.028)}
\tabularnewline
\hline
\multicolumn{3}{|c|}{SPACE: (0.602,0.029)} & \multicolumn{3}{|c|}{SJ: (0.561,0.036)}
\tabularnewline
\hline
\end{tabular}}
\vskip0.2in
\scalebox{0.85}{
\begin{tabular}{|c|c|c|c|c|c|}
\hline
\multicolumn{6}{|c|}{Non-centered profiled, $n=1000$, multiplier 1}
\tabularnewline
\hline\hline
\multicolumn{3}{|c|}{$\min(\log(1+x),c)$} & \multicolumn{3}{c|}{$\min(x,c)$}
\tabularnewline
\hline
$c$ & Mean & sd & $c$ & Mean & sd
\tabularnewline\hline
$\infty$ & 0.783 & 0.020 & 2 & 0.779 & 0.020
\tabularnewline
2 & 0.783 & 0.020  & $\infty$ & 0.779 & 0.020
\tabularnewline
1 & 0.782 & 0.020 & 3 & 0.779 & 0.020
\tabularnewline 
0.5 & 0.767 & 0.021 & $0.5$ & 0.758 & 0.020
\tabularnewline
\hline
\multicolumn{3}{|c|}{$\mathrm{MCP}(1,c)$} & \multicolumn{3}{c|}{$\mathrm{SCAD}(1,c)$}
\tabularnewline
\hline
$c$ & Mean & sd & $c$ & Mean & sd
\tabularnewline
\hline
5 & 0.782 & 0.020 & 2 & 0.780 & 0.020
\tabularnewline
10 & 0.780 & 0.020 & 5 & 0.780 & 0.020
\tabularnewline
1 & 0.771 & 0.021 & 10 & 0.779 & 0.020
\tabularnewline
\hline
\multicolumn{3}{|c|}{$x^{1.5}$: (0.751,0.019)} & \multicolumn{3}{|c|}{$x^2$: (0.713,0.018)}
\tabularnewline
\hline
\hline
\multicolumn{3}{|c|}{SPACE: (0.786,0.020)} & \multicolumn{3}{|c|}{NS: (0.785,0.02)}
\tabularnewline
\hline
\multicolumn{3}{|c|}{GLASSO (0.770,0.019)} & \multicolumn{3}{|c|}{SJ: (0.720,0.019)}
\tabularnewline
\hline
\end{tabular}
\begin{tabular}{|c|c|c|c|c|c|}
\hline
\multicolumn{6}{|c|}{Non-centered profiled, $n=1000$, multiplier 1.6438}
\tabularnewline
\hline\hline
\multicolumn{3}{|c|}{$\min(\log(1+x),c)$} & \multicolumn{3}{c|}{$\min(x,c)$}
\tabularnewline
\hline
$c$ & Mean & sd & $c$ & Mean & sd
\tabularnewline\hline
$\infty$ & 0.764 & 0.018 & $\infty$ & 0.766 & 0.019
\tabularnewline
2 & 0.764 & 0.018 & 3 & 0.765 & 0.019
\tabularnewline
1 & 0.762 & 0.018 & 2 & 0.764 & 0.018
\tabularnewline 
0.5 & 0.738 & 0.018 & 1 & 0.753 & 0.018
\tabularnewline
\hline
\multicolumn{3}{|c|}{$\mathrm{MCP}(1,c)$} & \multicolumn{3}{c|}{$\mathrm{SCAD}(1,c)$}
\tabularnewline
\hline
$c$ & Mean & sd & $c$ & Mean & sd
\tabularnewline
\hline
10 & 0.766 & 0.019 & 10 & 0.766 & 0.019
\tabularnewline
5 & 0.766 & 0.019 & 5 & 0.766 & 0.019
\tabularnewline
1 & 0.745 & 0.018 & 2 & 0.763 & 0.018
\tabularnewline
\hline
\multicolumn{3}{|c|}{$x^{1.5}$: (0.748,0.018)} & \multicolumn{3}{|c|}{$x^2$: (0.718,0.017)}
\tabularnewline
\hline
\hline
\multicolumn{3}{|c|}{SPACE: (0.786,0.020)} & \multicolumn{3}{|c|}{NS: (0.785,0.020)}
\tabularnewline
\hline
\multicolumn{3}{|c|}{GLASSO (0.770,0.019)} & \multicolumn{3}{|c|}{SJ: (0.720,0.019)}
\tabularnewline
\hline
\end{tabular}}
\end{center}\caption{Mean and standard deviation of AUC using different profiled estimators in the non-centered setting, with $n=80$ and multiplier $1.8647$, or $n=1000$ and multipliers $1$ and $1.6438$. Methods include our estimator with different choices of $h$,  GLASSO, SPACE, neighborhood selection (NS), and Space JAM (SJ).}\label{table_chooseh_profiled}
\end{table}

Even without tuning the extra penalty parameter on $\boldsymbol{\eta}\equiv\mathbf{K}\boldsymbol{\mu}$, our profiled estimator beats the competing methods by a large margin when $n=80$. 
With multipliers $1$ and $n=1000$, our estimators still do better than Space JAM and GLASSO, and have performance comparable to other competing methods. It might appear that the performance of our estimators deteriorate with a multiplier larger than $1$; however, as we will see, there can be significant improvement in AUCs if we tune an additional parameter for the multiplier.
As in the centered case, the leading $h$ functions in each category perform similarly, and the exact choice is not crucial.  Subsequently, we will simply use $h(x)=\min(x,3)$.

\subsubsection{Choice of multiplier}\label{Choice of multiplier}
\noindent\emph{Truncated Centered GGMs}\label{Choosemult_Truncated Centered GGMs}:\indent  In Figure~\ref{plot_mixed}, the ROC curves for GLASSO, SPACE, and our estimator with $h(x)=\min(x,3)$, but with different levels of amplification, via different choices of multipliers $\delta$, are compared for the centered case of Section~\ref{Chooseh_Truncated Centered GGMs}. 

While Theorem~\ref{corollary1} guarantees consistency only for $\delta<C(n,m)$, we observe that there can be a gain from going beyond the
\emph{upper-bound multiplier} $C(n,m)$, which is $1.8647$ for $n=80$ and $1.6438$ for $n=1000$ (when $n=1000$, $C(n,m)$ turns out to be the best-performing multiplier). However, the effect deteriorates fast as the multiplier grows larger. The figure suggests that while some additional gains are possible by tuning over the choice of multiplier, the \emph{upper-bound multiplier} is a good default.

\begin{figure}[H]
\centering
\vspace{-0.0in}
\subfloat[$n=80$]
{\includegraphics[scale=0.28]{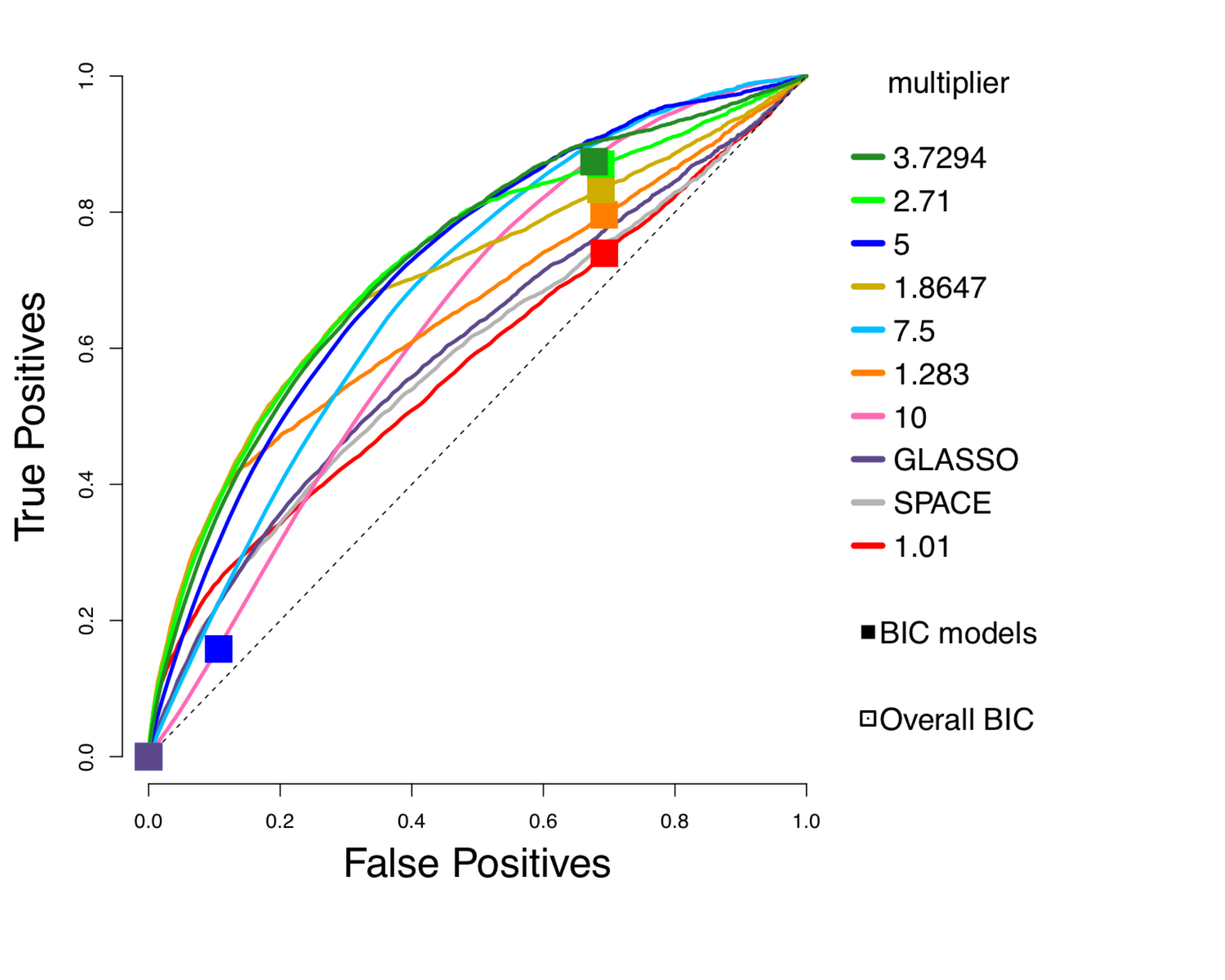}\hspace{-0.02in}}
\subfloat[$n=1000$]
{\includegraphics[scale=0.28]{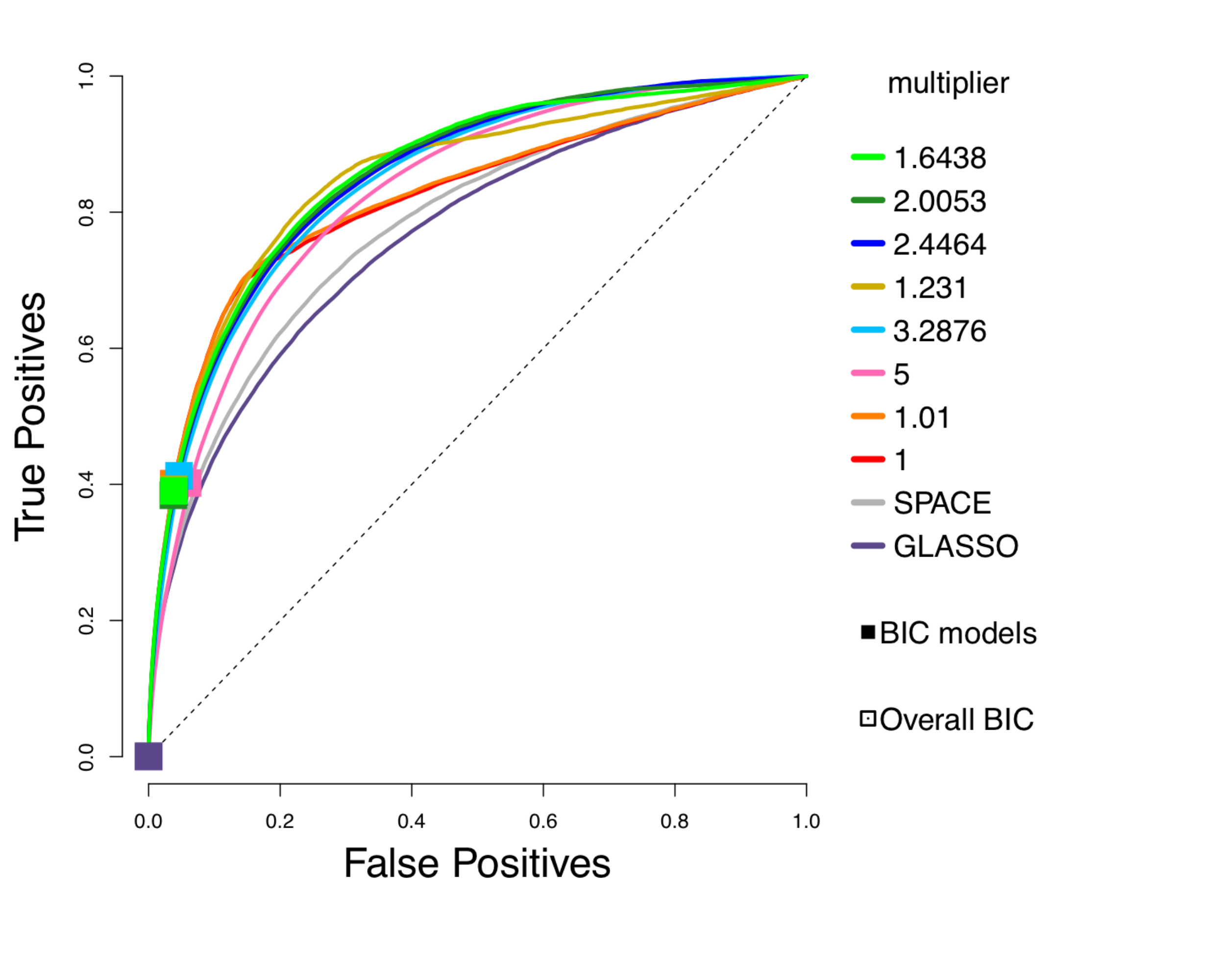}}
\caption{Performance of $\min(x,3)$ for truncated centered GGMs using different multipliers, compared to GLASSO and SPACE, in the centered setting, $n=80$ or 1000.}
\label{plot_mixed}
\end{figure}


\noindent\emph{Truncated Non-Centered GGMs}
\label{Choosemult_Truncated Non-Centered GGMs}:\indent  In Figure~\ref{plot_lr}, we consider the non-centered case of  Section~\ref{Chooseh_Truncated Non-Centered GGMs}, and use the non-profiled estimator; that is, the non-centered estimator with $\ell_1$ penalty on both $\mathbf{K}$ and $\boldsymbol{\eta}\equiv\mathbf{K}\boldsymbol{\mu}$. The ROC curves are compared to competing methods GLASSO and SPACE. For the choice of amplification in our estimator, we consider the upper-bound multiplier $C(n,m)$ from Theorem~\ref{corollary2} as the default. We refer to this as $\emph{high}$ amplification.  We also consider lower amplification, with $\delta=2-(1+24e\log m/n)^{-1}$, referred to as \emph{medium}.  For $n=1000$, we also consider a \emph{low} multiplier $1$, which corresponds to no amplification.  We compare these possible defaults to a finer grid of multipliers of which we show some representatives in the plots.

We see that among our defaults, the upper-bound choice $C(n,m)$
performs best. Some additional gains are possible by tuning
the multiplier over a grid of values containing this choice.
Moreover, we see that it can be beneficial to tune over both $\lambda_{\mathbf{K}}$ and $\lambda_{\mathbf{K}}/\lambda_{\boldsymbol{\eta}}$.


We remark that while for each run, the best model picked by BIC falls on the ROC curve, a few squares are off the curve in Figure~\ref{plot_lr}~(c). This is because these squares correspond to the average of the true and false positive rates of the chosen BIC models over 50 runs, potentially due to multimodality of the distribution of the models. Nonetheless, in all cases, the average of the models picked by BIC tuned over both $\lambda_{\mathbf{K}}$ and $\lambda_{\mathbf{K}}/\lambda_{\boldsymbol{\eta}}$ looks reasonable.

\begin{figure}[H]
\centering
\vspace{-0.2in}
\subfloat[$n=80$, $\mathrm{mult}=1.7897$]
{\includegraphics[trim={0 0.5in 1.3in 0},clip,scale=0.27]{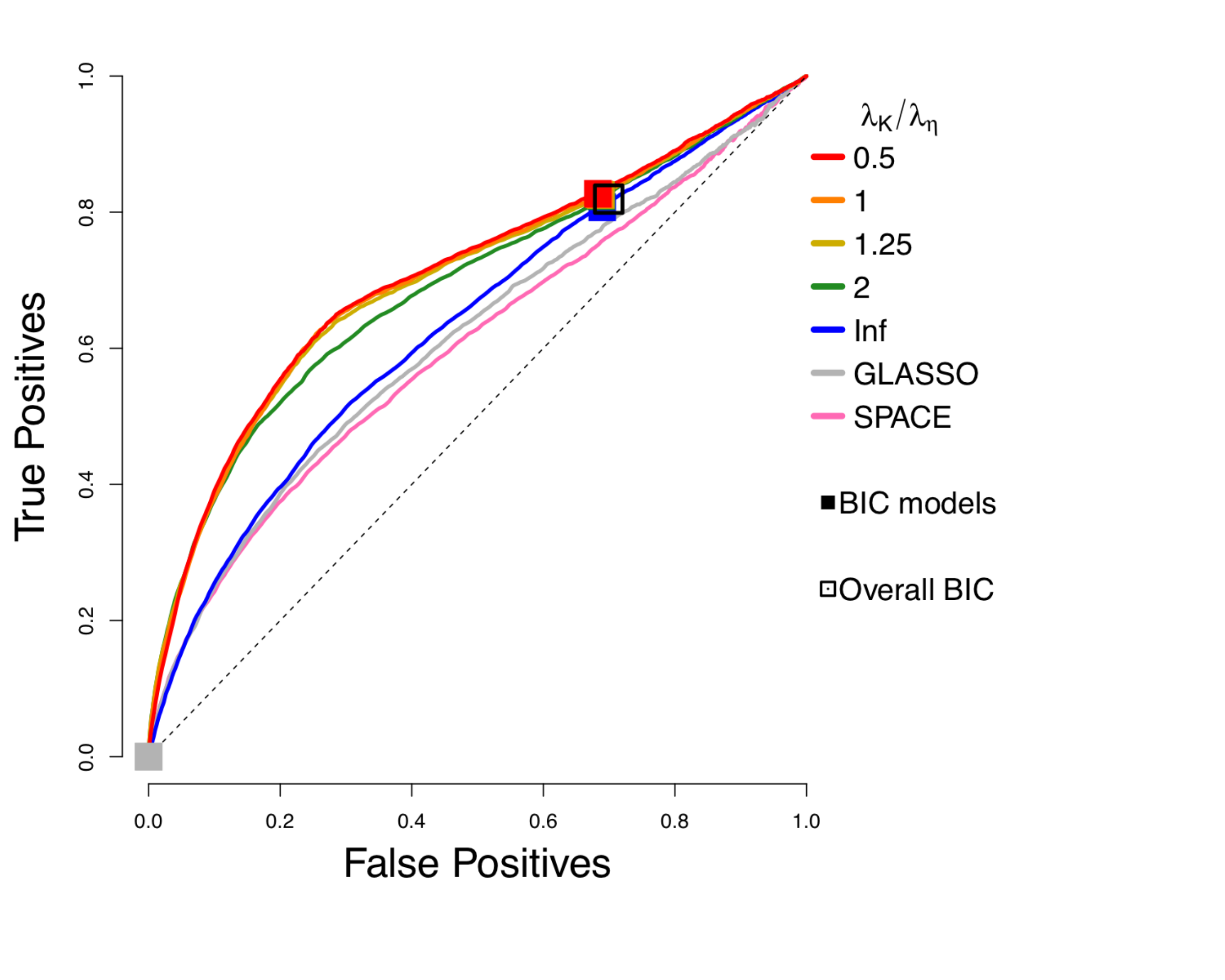}\hspace{-0.02in}}
\subfloat[$n=80$, $\mathrm{mult}=1.8647$]
{\includegraphics[trim={0 0.5in 2in 0},clip,scale=0.27]{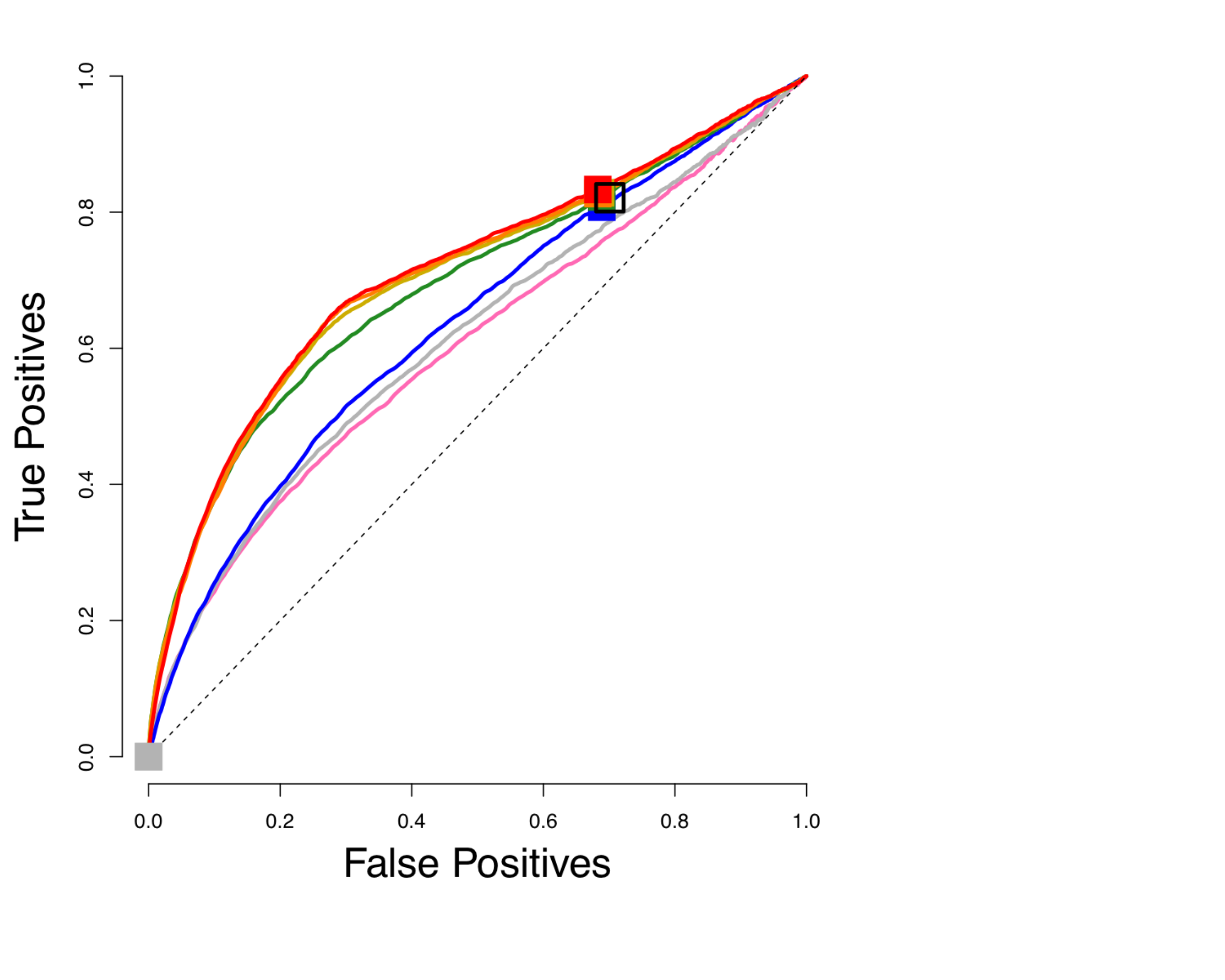}\hspace{-0.3in}}
\\ \vspace{-0.1in}
\subfloat[$n=1000$, $\mathrm{mult}=1$]
{\includegraphics[trim={0 0.5in 2in 0},clip,scale=0.27]{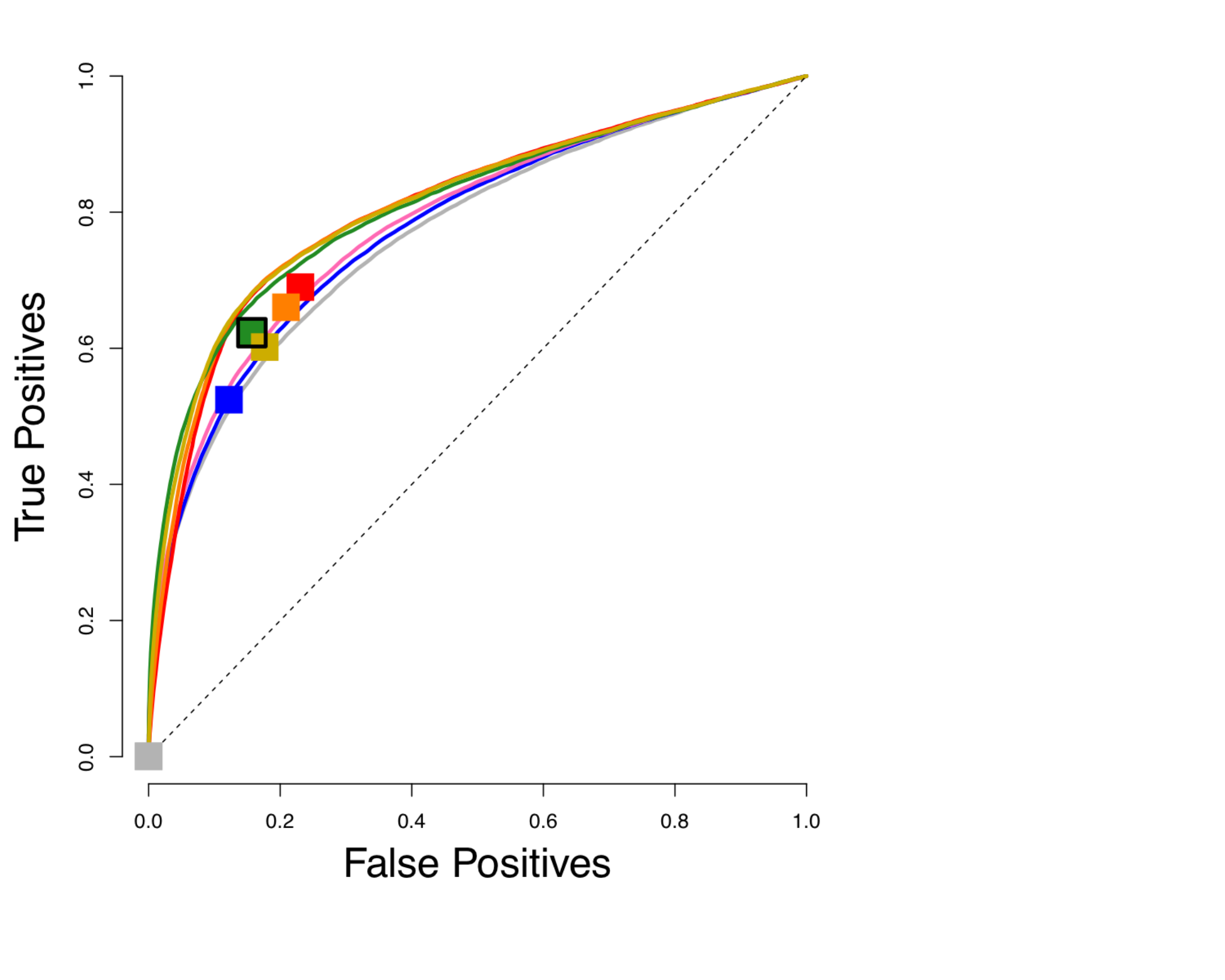}\hspace{-0.3in}}
\subfloat[$n=1000$, $\mathrm{mult}=1.2310$]
{\includegraphics[trim={0 0.5in 2in 0},clip,scale=0.27]{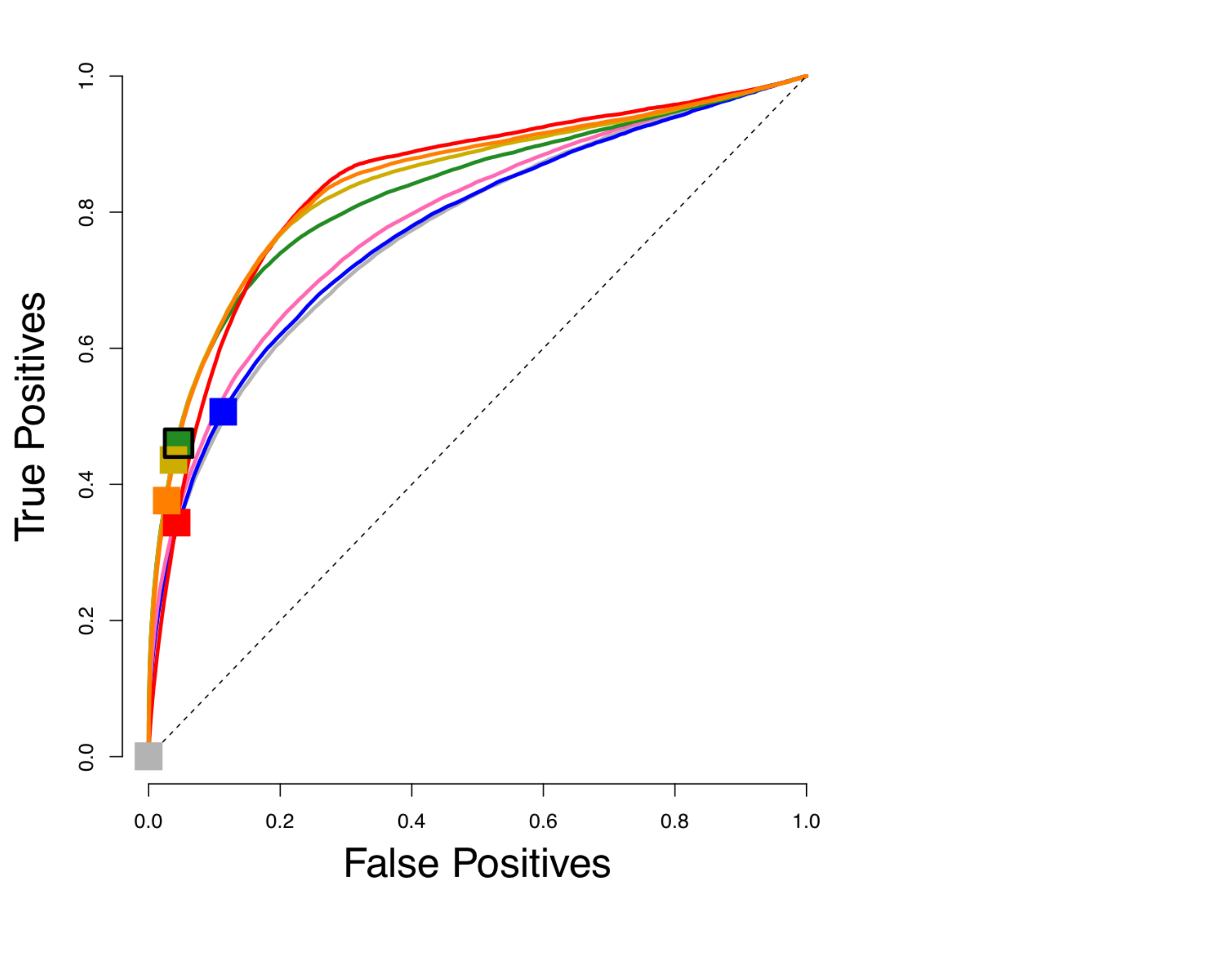}\hspace{-0.3in}}\subfloat[$n=1000$, $\mathrm{mult}=1.6438$]
{\includegraphics[trim={0 0.5in 2in 0},clip,scale=0.27]{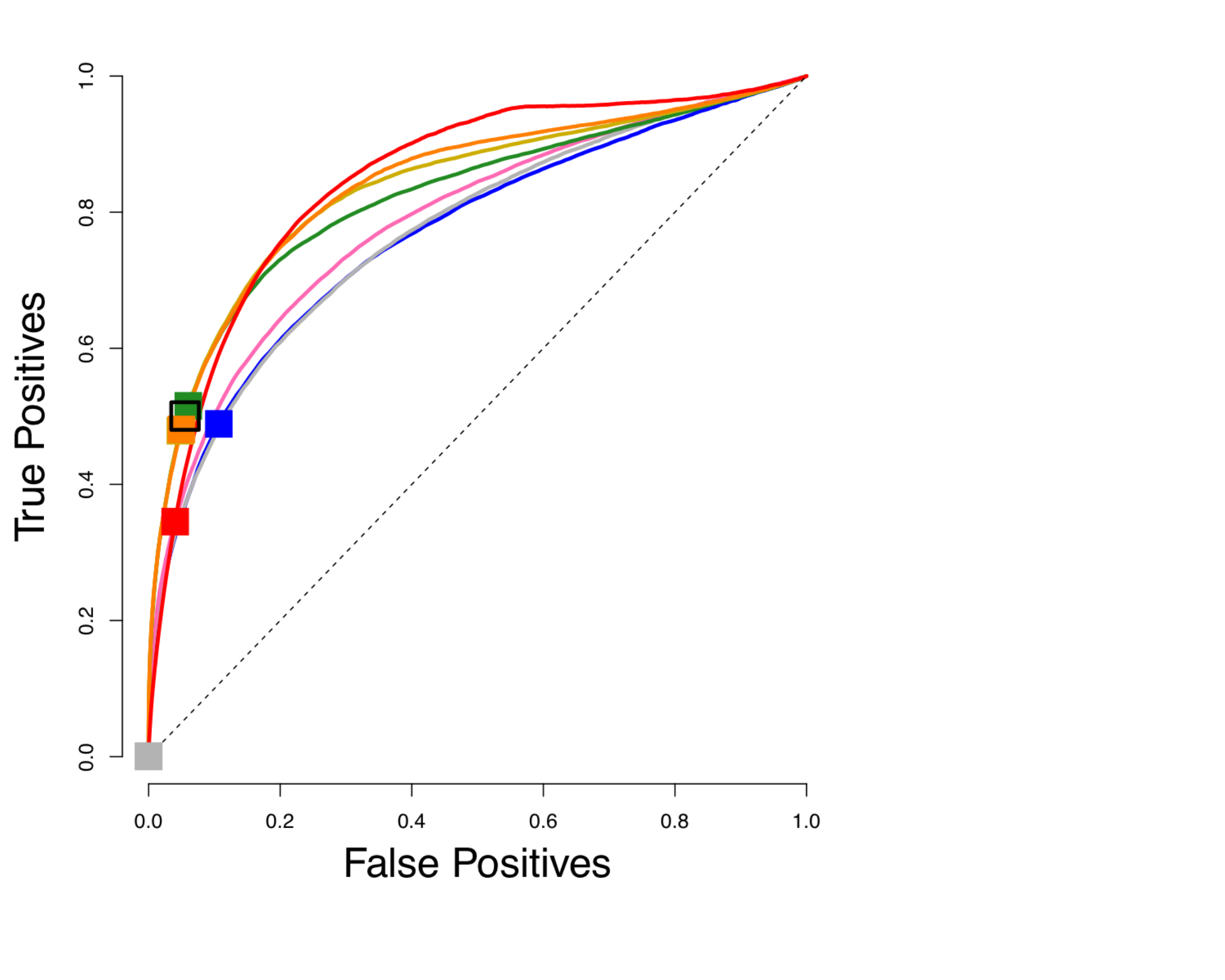}\hspace{-0.3in}}
\caption{Performance of the non-centered estimator with
  $h(x)=\min(x,3)$. Each curve corresponds to a different choice of
  $\lambda_{\mathbf{K}}/\lambda_{\boldsymbol{\eta}}$. 
  Squares indicate models picked by eBIC with refit.  The square with
  black outline has the highest eBIC among all models (combinations of
  $\lambda_{\mathbf{K}}$, $\lambda_{\boldsymbol{\eta}}$). Multipliers correspond to medium or high for $n=80$, and low, medium or high for $n=1000$, respectively.}
\label{plot_lr}
\vspace{-0.3in}
\end{figure}

\newpage

\subsection{Other \texorpdfstring{$a/b$}{a/b} Models}\label{Other Models}
We now turn to the non-Gaussian ($a\neq 1$ or $b\neq 1$) setting. Based on the observations in Section~\ref{Reasonable Choices of h}, we focus on functions of type $\min(x^p,c)$ for some power $p>0$ and truncation point $c>0$. For simplicity, for the non-centered models we use the profiled estimator (\ref{loss-noncentered-profiled}) (i.e., $\lambda_{\boldsymbol{\eta}}=0$) and use the multiplier $C(n,m)$ in Theorem~\ref{corollary1} for truncated GGMs as a guidance. We note that tuning over the $\lambda_{\boldsymbol{\eta}}$ parameter and the multiplier can potentially give a significant improvement as seen in Section~\ref{Simulations_Truncated GGMs}.

These simulations suggest that among the class of functions of the form $\min(x^p,c)$, $x^{2-a}$ or $\min(x^{2-a},c)$ with a moderately large $c$ can be used as the default choice of $h(x)$. This agrees with our findings in Section~\ref{Choice of h}. We note that bounded $h$ functions were only used in the proof for truncated GGMs, and picking a moderately large truncation point can correspond to having an untruncated power.

\subsubsection{Exponential Setting}
For the exponential models, $a=b=1/2$. Since $a=b$, for both centered and non-centered settings, based on the principle in Section~\ref{Reasonable Choices of h}, choosing $h(x)=\min(x^{3/2},c)$ satisfies (A1) and (A2) and also ensures that entries in $\boldsymbol{\Gamma}$ and $\boldsymbol{g}$ are bounded (for small $x$), while choosing $h(x)=\min(\sqrt{x},c)$ only guarantees (A1) and (A2).

In Figure~\ref{plot_exp}, we present the AUCs for the ROC curves of edge recovery with different choices of $h(x)=\min(x^{\mathrm{pow}},c)$. As before, we set $n=80$ or $1000$ and $m=100$, but we use an $\boldsymbol{\eta}_0$ with each component uniformly equal to $-0.5$, $0$ or $0.5$; for $\boldsymbol{\eta}_0\equiv\boldsymbol{0}$, we assume this information is known and use the centered estimator. The results suggest that $\mathrm{pow}=3/2=2-a$ is the best choice of power. For this optimal choice, the performance improves with larger $c$, so $x^{2-a}$ gives the best results. For sub-optimal powers, including truncation gives better results.

\subsubsection{Gamma Setting}
The centered gamma models reduce to the centered exponential models. Thus, in this section, we only consider the non-centered settings, with  $a=1/2$, $b=0$. From Section~\ref{Reasonable Choices of h}, we have the following choices:
\begin{itemize}
\item $\min(x^2,c)$ both satisfies (A1)--(A2) and ensures $\boldsymbol{\Gamma}$ and $\boldsymbol{g}$ are bounded;
\item $\min(x^{\max\{3/2,1-\min_j\eta_{0,j}\}},c)$ ensures (A1)--(A2) and bounds $\boldsymbol{\Gamma}_{11}$ and $\boldsymbol{g}_1$; by default without prior information on $\boldsymbol{\eta}_0$ this is $\min(x^2,c)$;
\item $\min(x^{3/2},c)$ satisfies both conditions on the interaction part only ($\boldsymbol{x}^a$), but does not guarantee (A1)--(A2);
\item $\min(x^{1/2},c)$ satisfies the sufficient conditions for (A1)--(A2) on the interaction only.
\end{itemize}
The results are shown in Figure~\ref{plot_gamma}, where we consider $n=80$, $1000$, and $\boldsymbol{\eta}=\pm 0.5\mathbf{1}_{100}$. They suggest that $\mathrm{pow}=2-a=1.5$ works consistently well, although slightly outperformed by 1 and 1.25 in one case. As in the exponential case, with the optimal power it is beneficial to choose a large truncation point, or work with an untruncated power $x^{1.5}$. We conclude that the performance is likely only dependent on the $(2-a)$ power requirement for the ${\boldsymbol{x}^a}^{\top}\mathbf{K}\boldsymbol{x}^a$ part or $2-\min_{j}\eta_{0,j}$; simulations in the next section rule out the possibility of the latter.

\begin{figure}[htp]
\centering
\vspace{-0.7in}
\subfloat[$n=80$, $\boldsymbol{\eta}=-0.5\mathbf{1}_{100}$, profiled estimator]
{\includegraphics[scale=0.30]{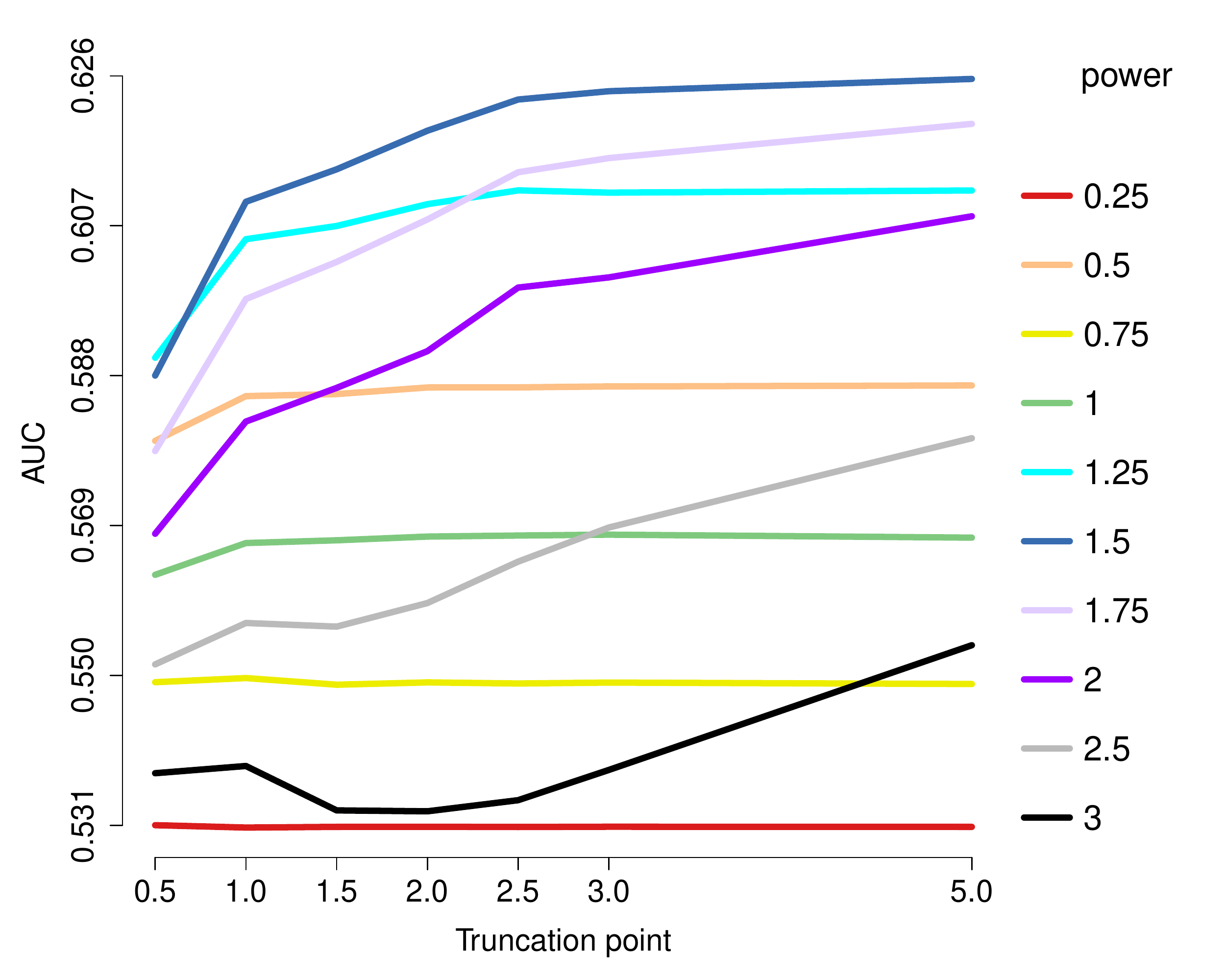}\hspace{-0.02in}}
\subfloat[$n=1000$, $\boldsymbol{\eta}=-0.5\mathbf{1}_{100}$, profiled estimator]
{\includegraphics[scale=0.30]{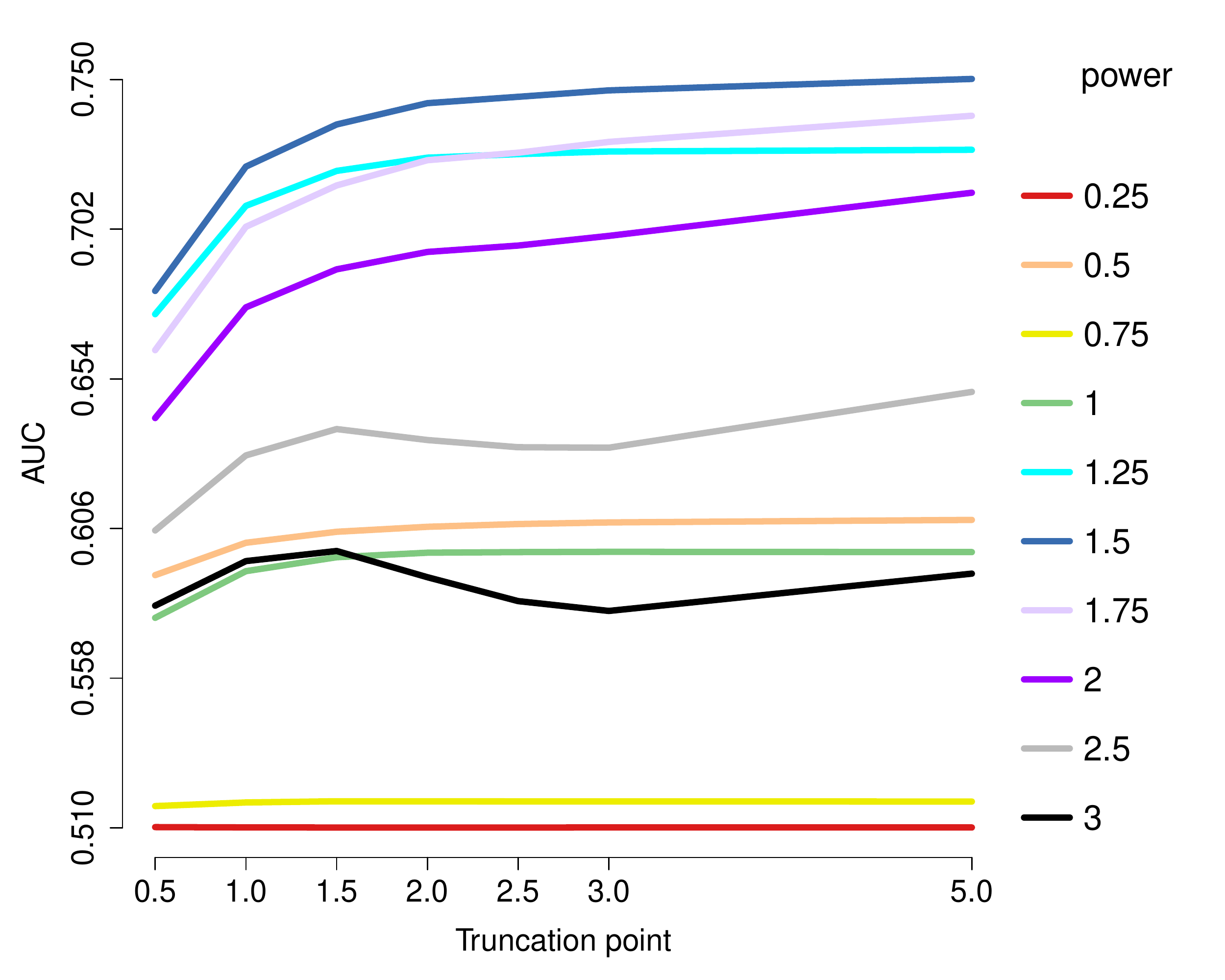}\hspace{-0.02in}}
\\ \vspace{-0.1in}
\subfloat[$n=80$, $\boldsymbol{\eta}=\boldsymbol{0}$, centered estimator]
{\includegraphics[scale=0.30]{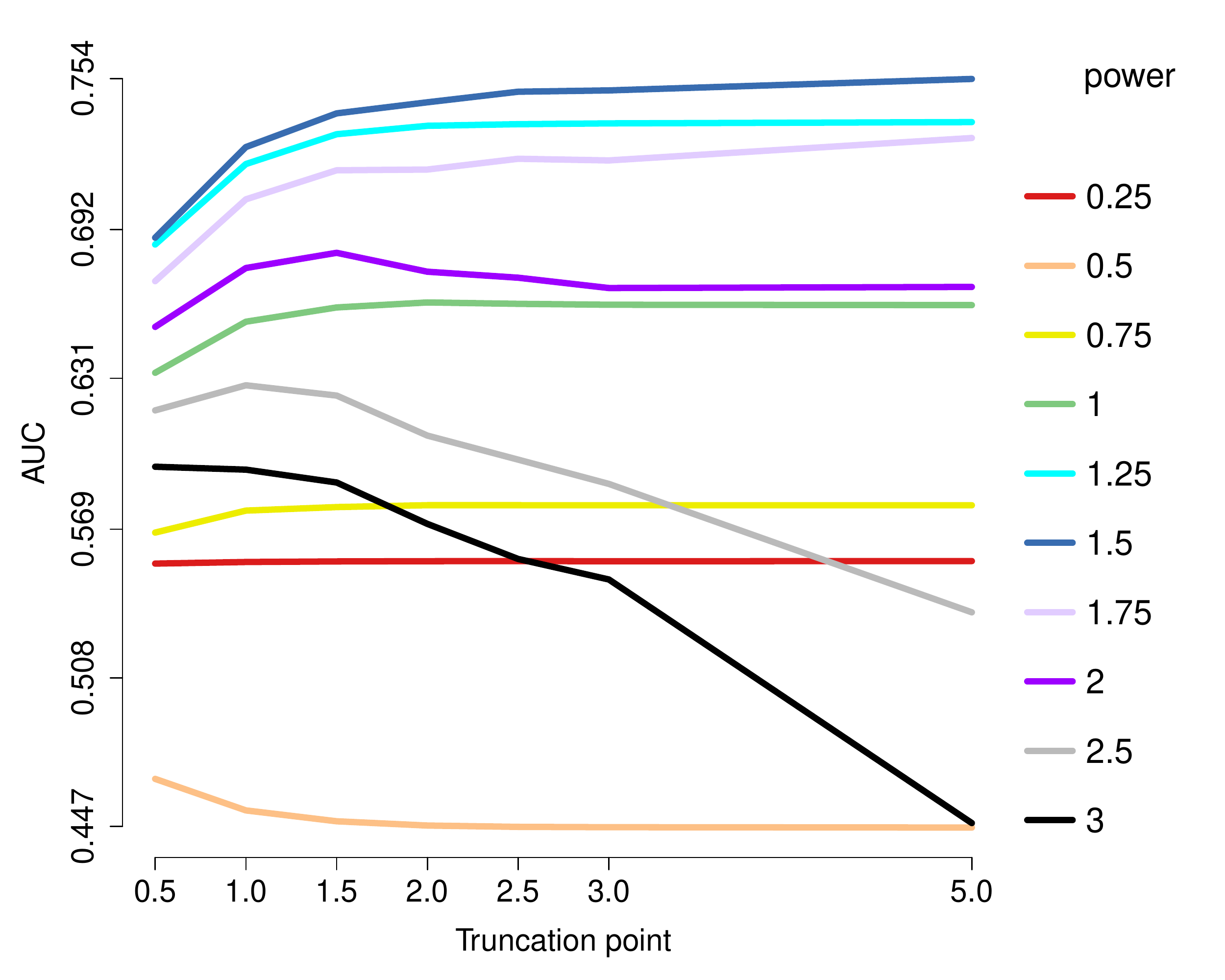}\hspace{-0.02in}}
\subfloat[$n=1000$, $\boldsymbol{\eta}=\boldsymbol{0}$, centered estimator]
{\includegraphics[scale=0.30]{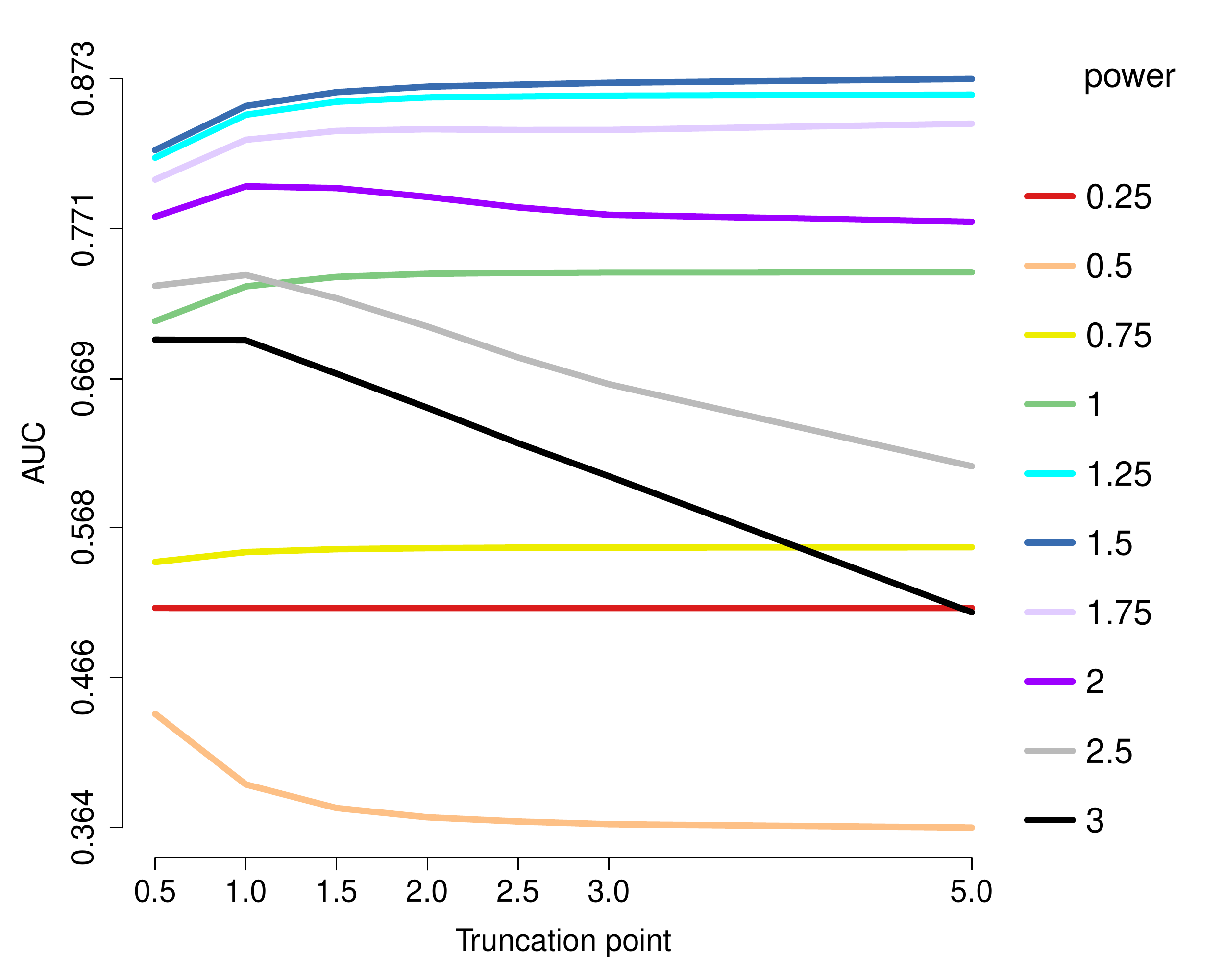}\hspace{-0.02in}}
\\ \vspace{-0.1in}
\subfloat[$n=80$, $\boldsymbol{\eta}=0.5\mathbf{1}_{100}$, profiled estimator]
{\includegraphics[scale=0.30]{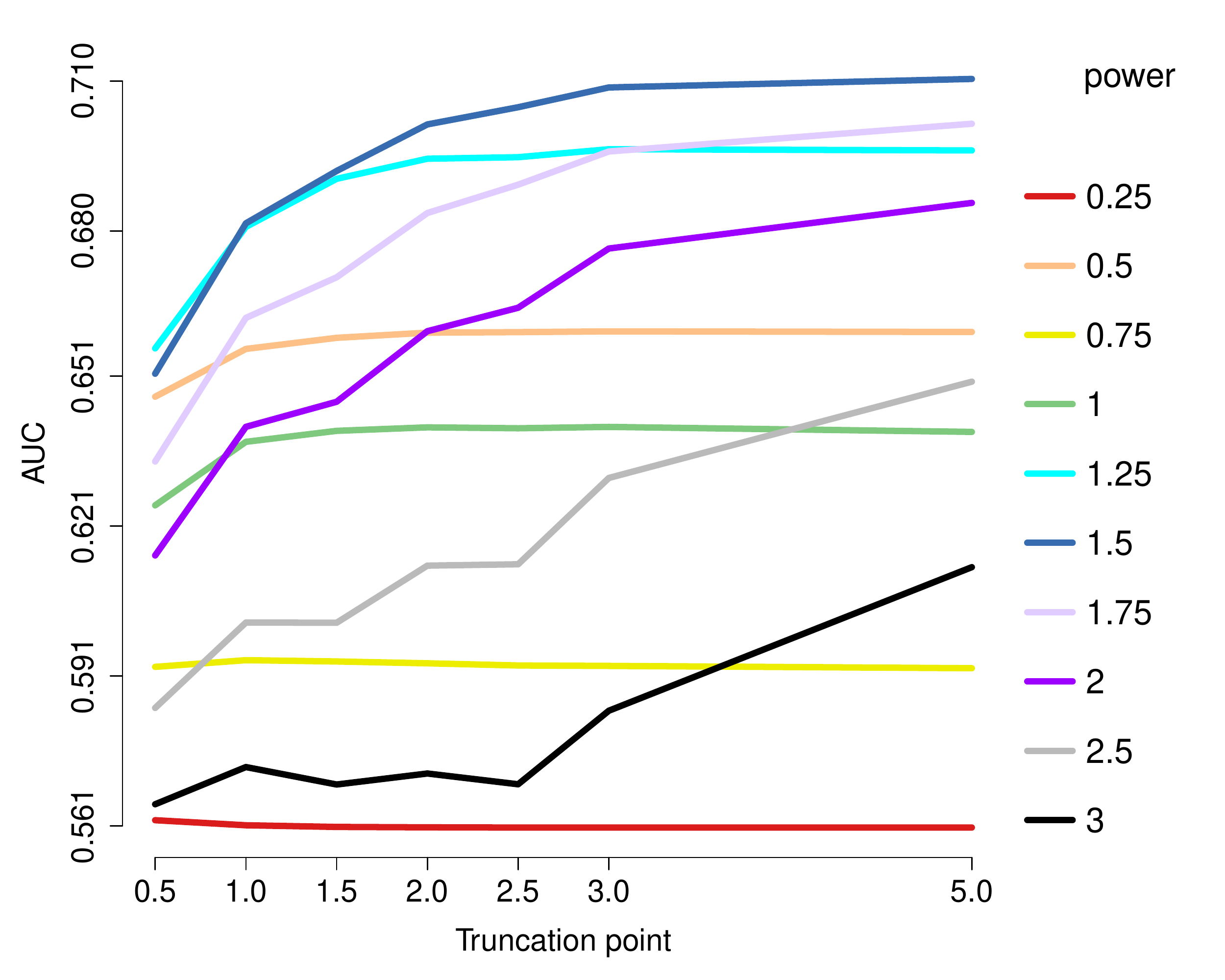}\hspace{-0.02in}}
\subfloat[$n=1000$, $\boldsymbol{\eta}=0.5\mathbf{1}_{100}$, profiled estimator]
{\includegraphics[scale=0.30]{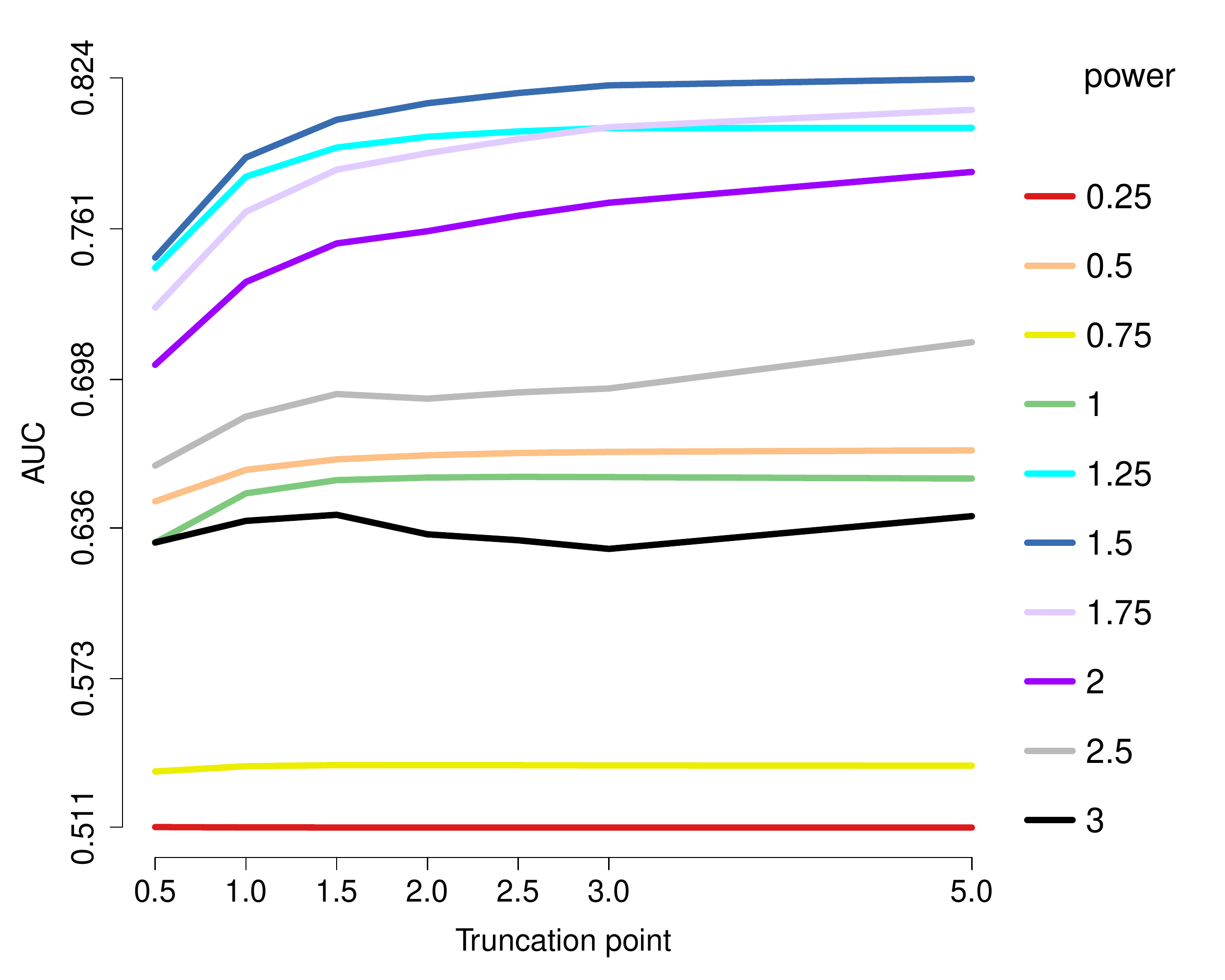}\hspace{-0.02in}}
\caption{AUCs for edge recovery using generalized score matching for the exponential models. Each curve represents a different choice of power $p$ in $h(x)=\min(x^p,c)$, and the $x$ axis marks the truncation point $c$. Colors are sorted by $p$.}\label{plot_exp}
\end{figure}

\vspace{-0.3in}
\subsubsection{Other Choices of \texorpdfstring{$a$}{a} and \texorpdfstring{$b$}{b}}
In this section, we consider other choices of $a$ and $b$. Specifically, $a=3/2$ and $b=1/2$ or $0$. These combinations are chosen to confirm, in a more extreme setting, that the performance is mainly determined by requirements on the power based on $a$, which correspond to choosing a power of $1-a$ or $2-a$, but not those on $b$ (or on $\boldsymbol{\eta}$ when $b=0$) that correspond to $1-b$ and $2-b$. The relationship between these  two settings is analogous to that between the exponential and gamma models (same $a$, $b$ nonzero/zero). 

The results are shown in Figures~\ref{plot_ab_1.5_0.5} and \ref{plot_ab_1.5_0}, and indeed confirm that $x^{2-a}=x^{0.5}$ consistently gives the optimal results, even though $\boldsymbol{\eta}^{\top}\boldsymbol{x}^b$ is in favor of $x^{2-b}=x^{1.5}$ for $b=0.5$, and $\boldsymbol{\eta}^{\top}\log(\boldsymbol{x})$ is in favor of $x^2$ or at least $x^{1-\min_j\eta_{0,j}}$ when $b=0$. 
There are two possible explanations for the optimality of $2-a$ over $\max\{2-a,2-b\}$ or $\max\{2-a,1-\min_j\eta_{0,j}\}$: (1) The AUC metric is measured only on our interest, edge recovery for the interaction matrix, which only depends on $\boldsymbol{x}^a$; (2) using the profiled estimator weakens the effect of $b$.

\begin{figure}[ht]
\centering
\vspace{-0.2in}
\subfloat[$n=80$, $\boldsymbol{\eta}=-0.5\mathbf{1}_{100}$, profiled estimator]
{\includegraphics[scale=0.30]{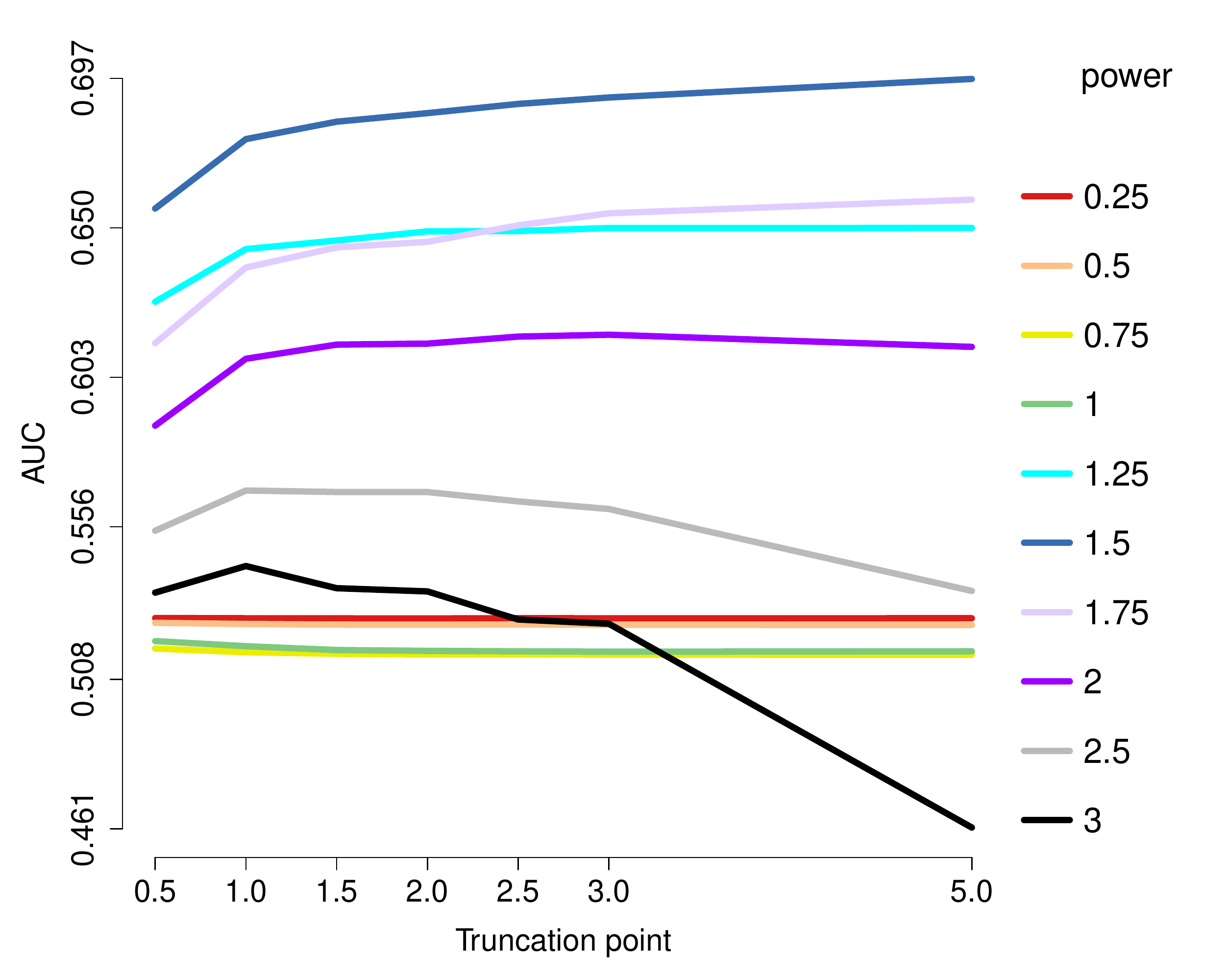}\hspace{-0.02in}}
\subfloat[$n=1000$, $\boldsymbol{\eta}=-0.5\mathbf{1}_{100}$, profiled estimator]
{\includegraphics[scale=0.30]{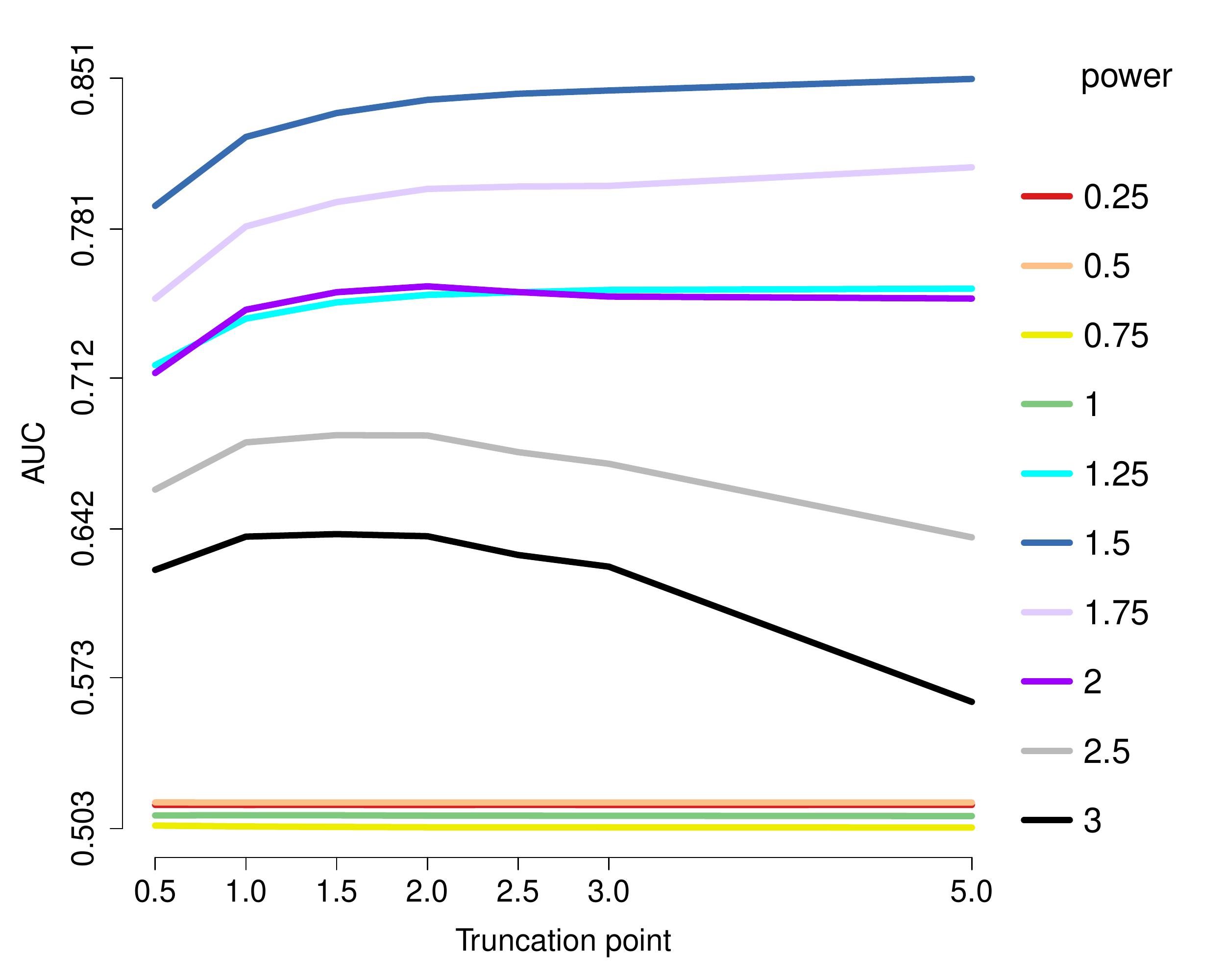}\hspace{-0.02in}}
\\ \vspace{-0.1in}
\subfloat[$n=80$, $\boldsymbol{\eta}=0.5\mathbf{1}_{100}$, profiled estimator]
{\includegraphics[scale=0.30]{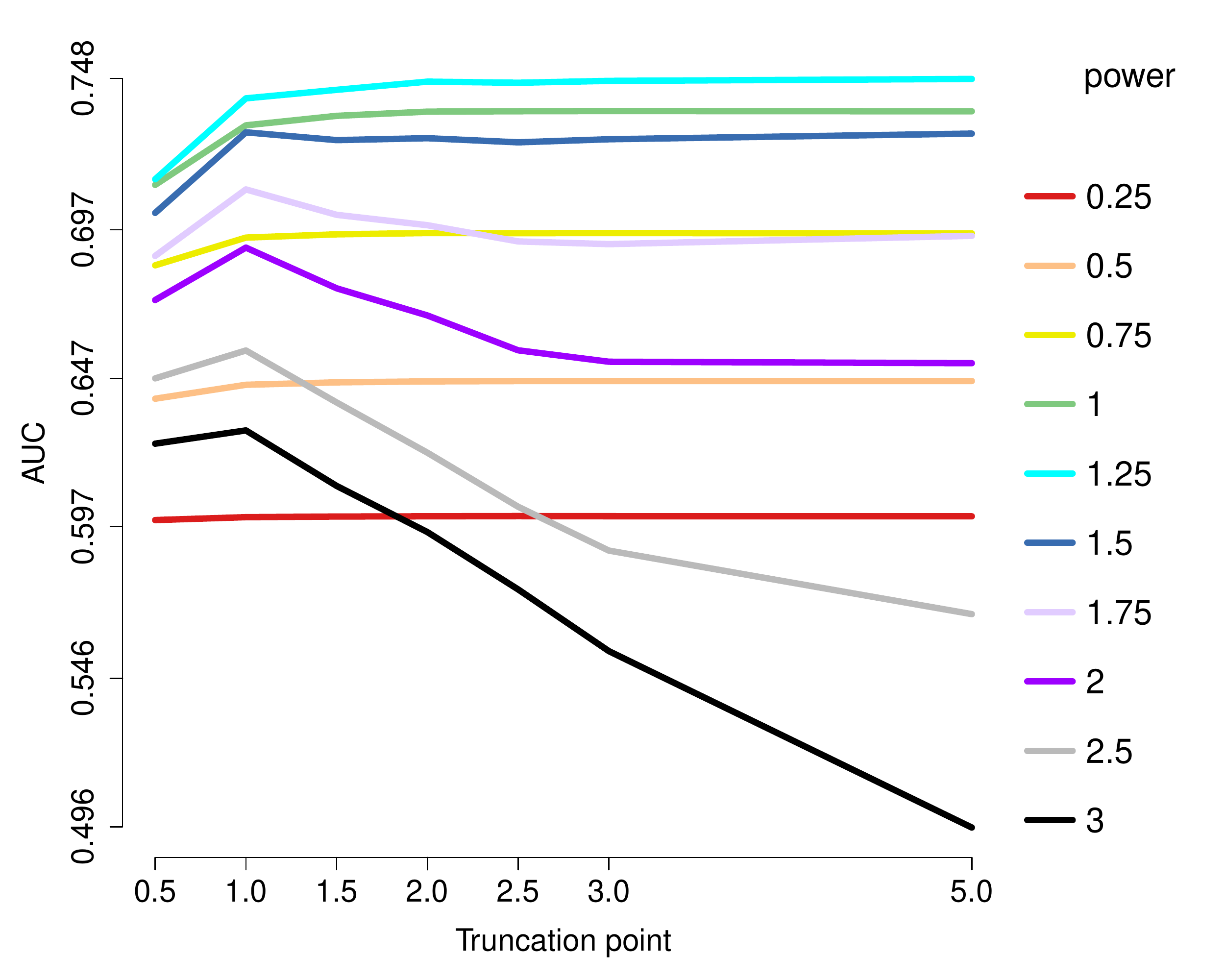}\hspace{-0.02in}}
\subfloat[$n=1000$, $\boldsymbol{\eta}=0.5\mathbf{1}_{100}$, profiled estimator]
{\includegraphics[scale=0.30]{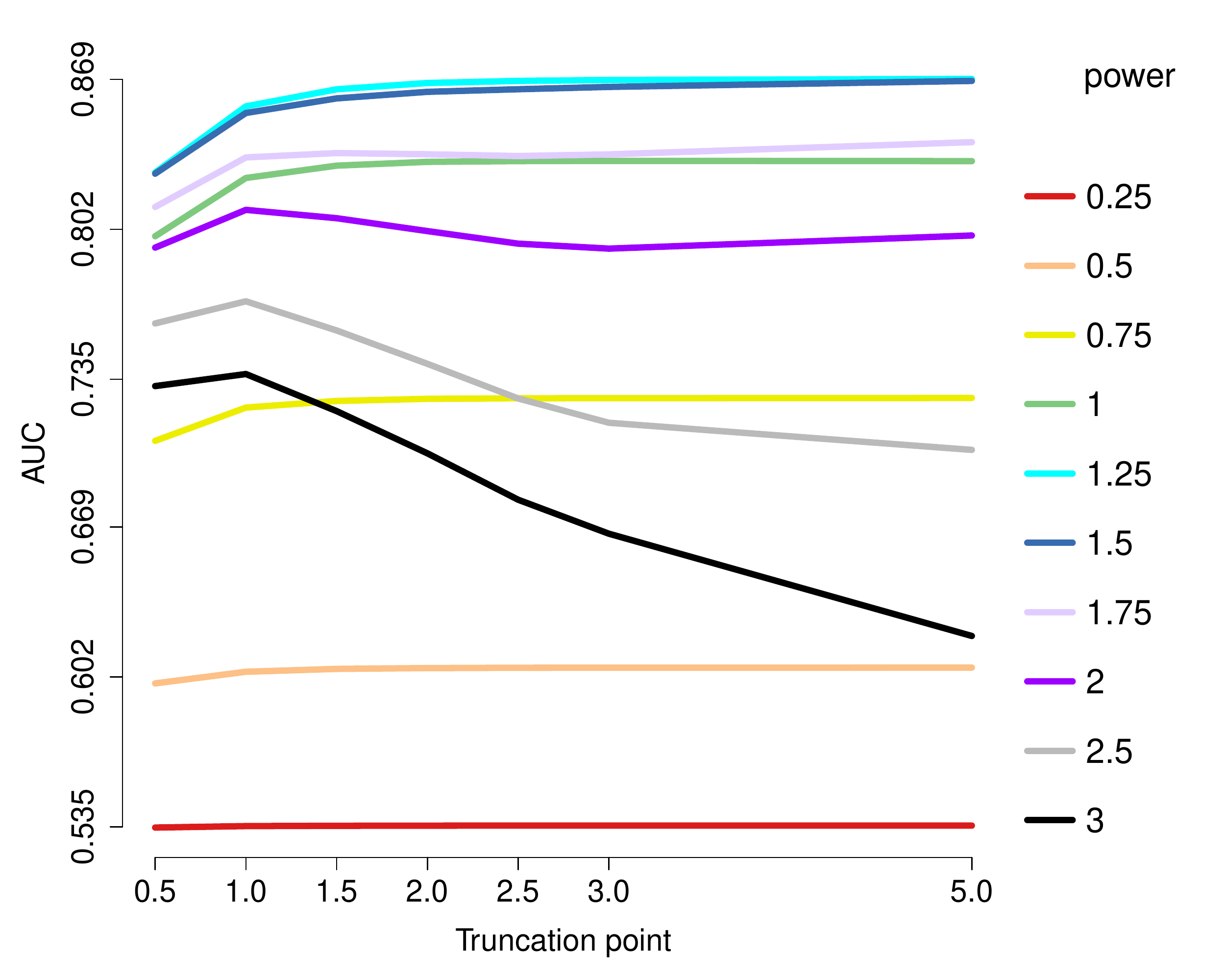}\hspace{-0.02in}}
\caption{AUCs for edge recovery using generalized score matching for the gamma models. Each curve represents a different choice of power $p$ in $h(x)=\min(x^p,c)$, and the $x$ axis marks the truncation point $c$. Colors are sorted by $p$.}
\vspace{-0.7in}
\label{plot_gamma}
\end{figure}

\begin{figure}[htp]
\centering
\vspace{-0.7in}
\subfloat[$n=80$, $\boldsymbol{\eta}=-0.5\mathbf{1}_{100}$, profiled estimator]
{\includegraphics[scale=0.30]{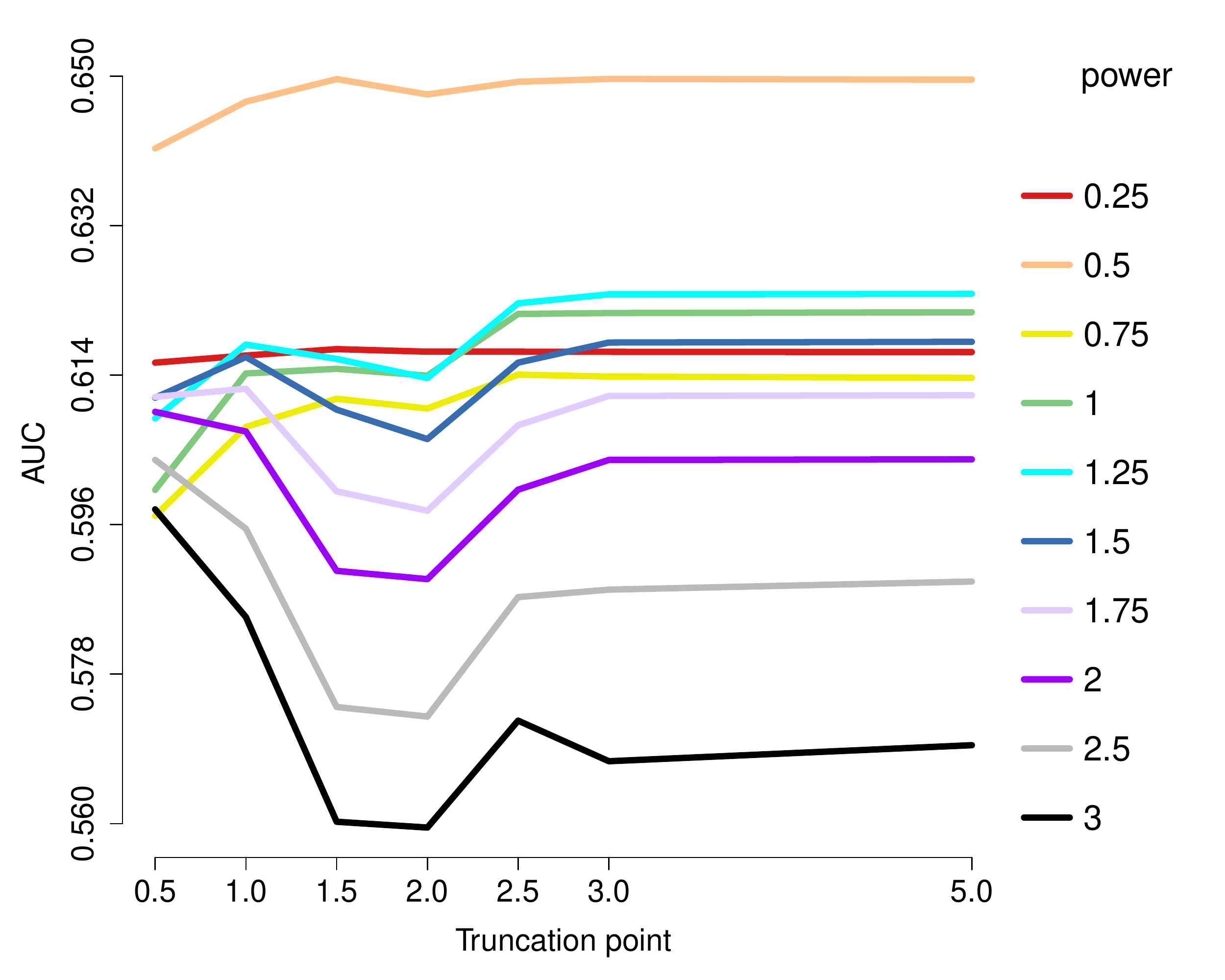}\hspace{-0.02in}}
\subfloat[$n=1000$, $\boldsymbol{\eta}=-0.5\mathbf{1}_{100}$, profiled estimator]
{\includegraphics[scale=0.30]{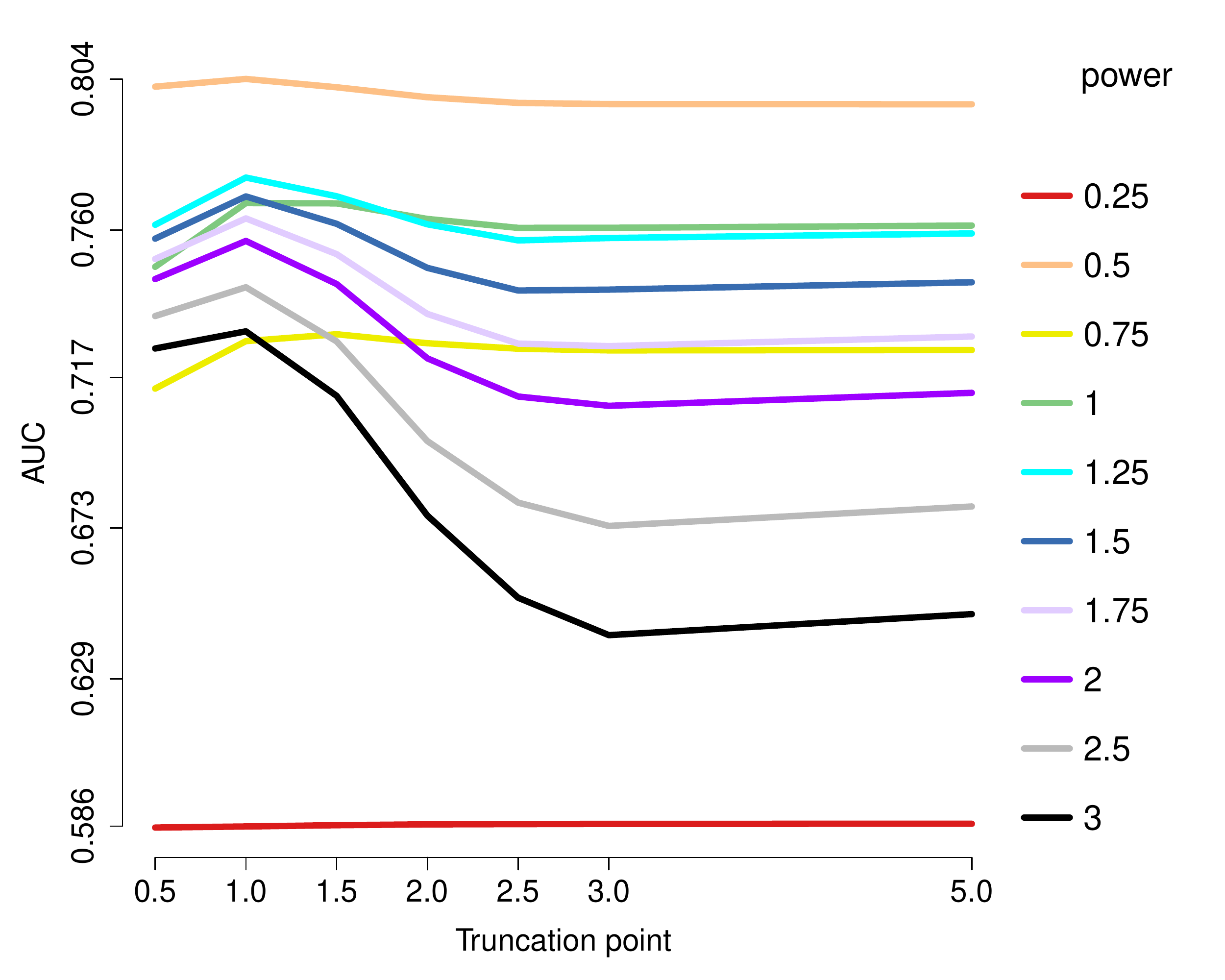}\hspace{-0.02in}}
\\ \vspace{-0.1in}
\subfloat[$n=80$, $\boldsymbol{\eta}=\boldsymbol{0}$, centered estimator]
{\includegraphics[scale=0.30]{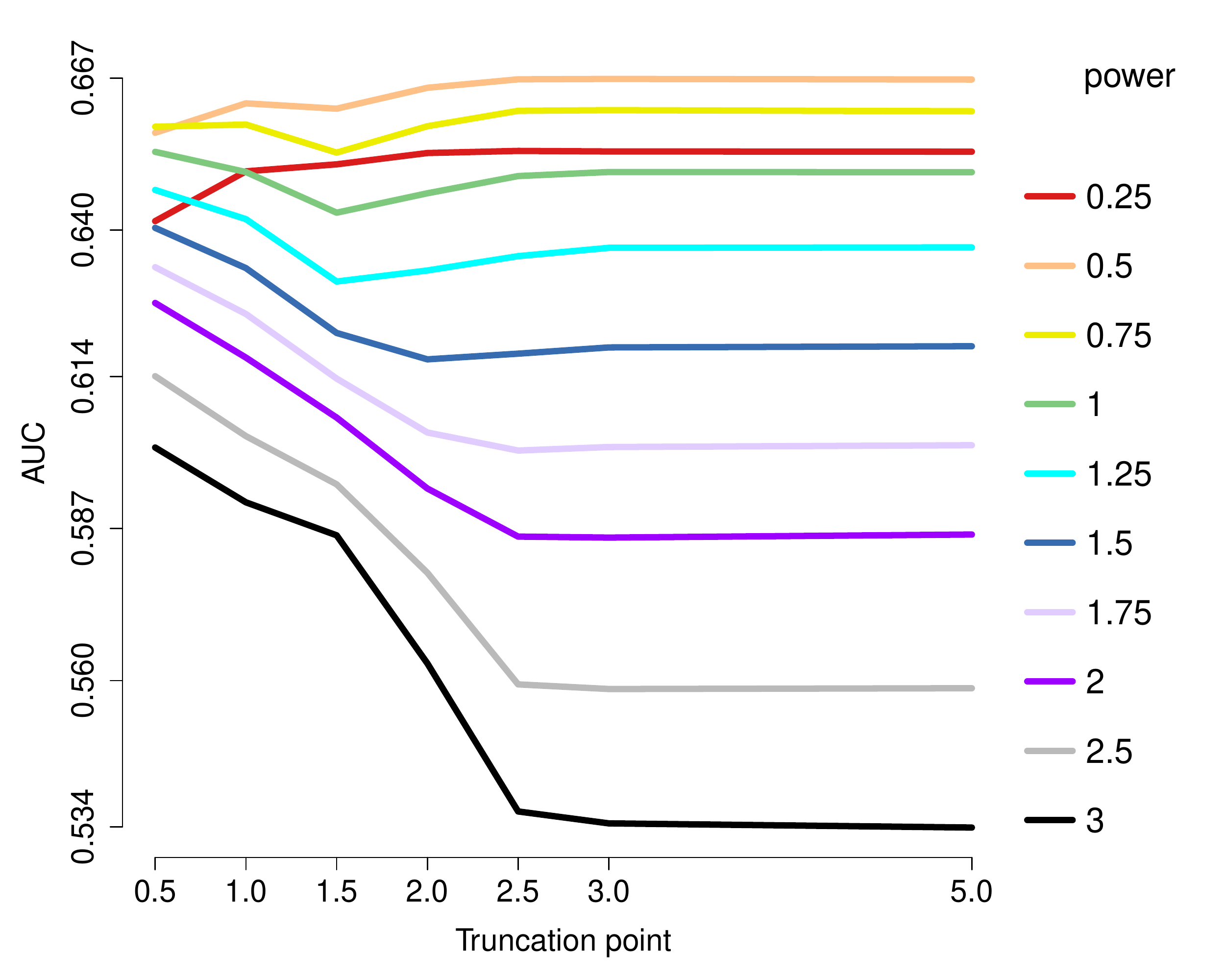}\hspace{-0.02in}}
\subfloat[$n=1000$, $\boldsymbol{\eta}=\boldsymbol{0}$, centered estimator]
{\includegraphics[scale=0.30]{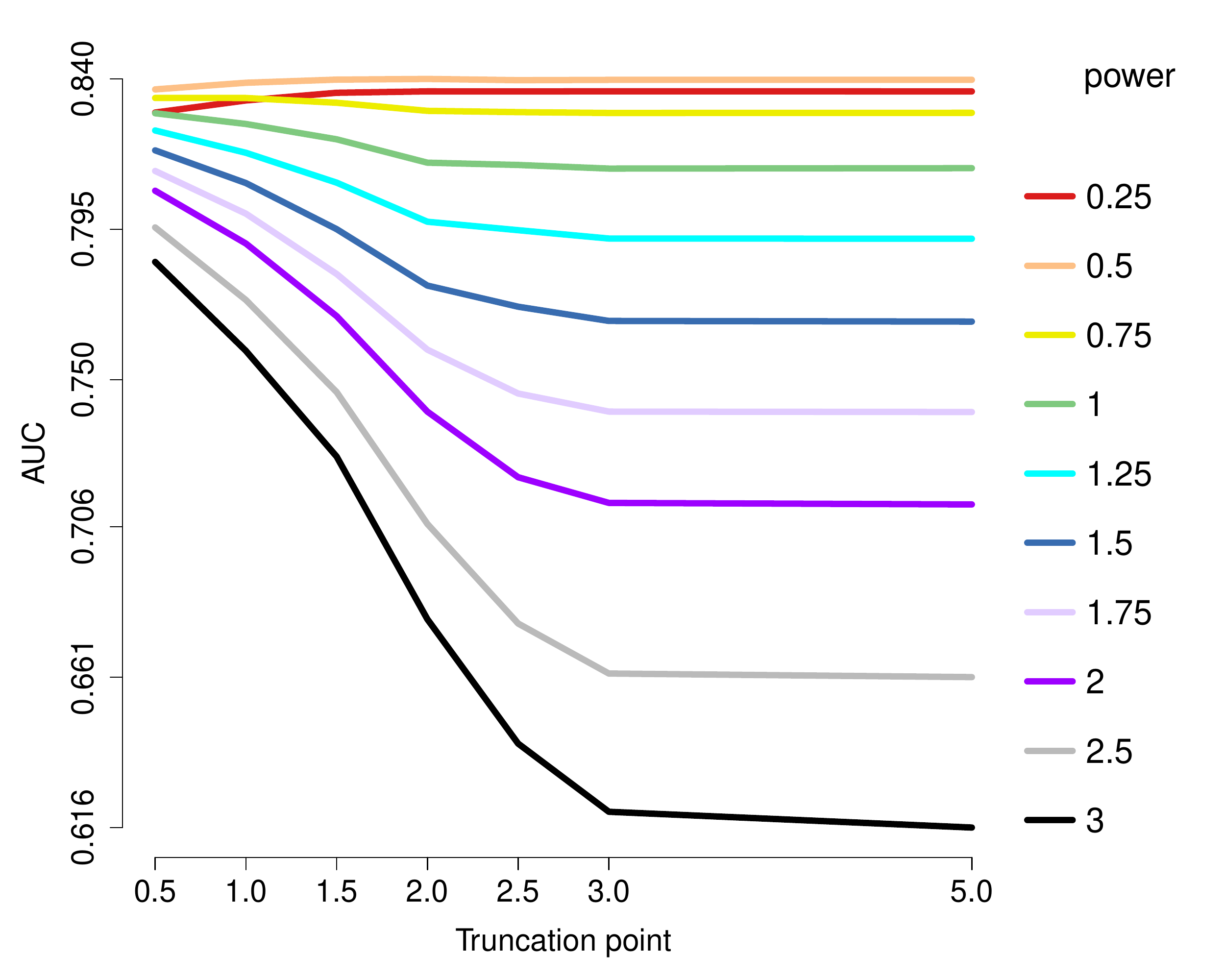}\hspace{-0.02in}}
\\ \vspace{-0.1in}
\subfloat[$n=80$, $\boldsymbol{\eta}=0.5\mathbf{1}_{100}$, profiled estimator]
{\includegraphics[scale=0.30]{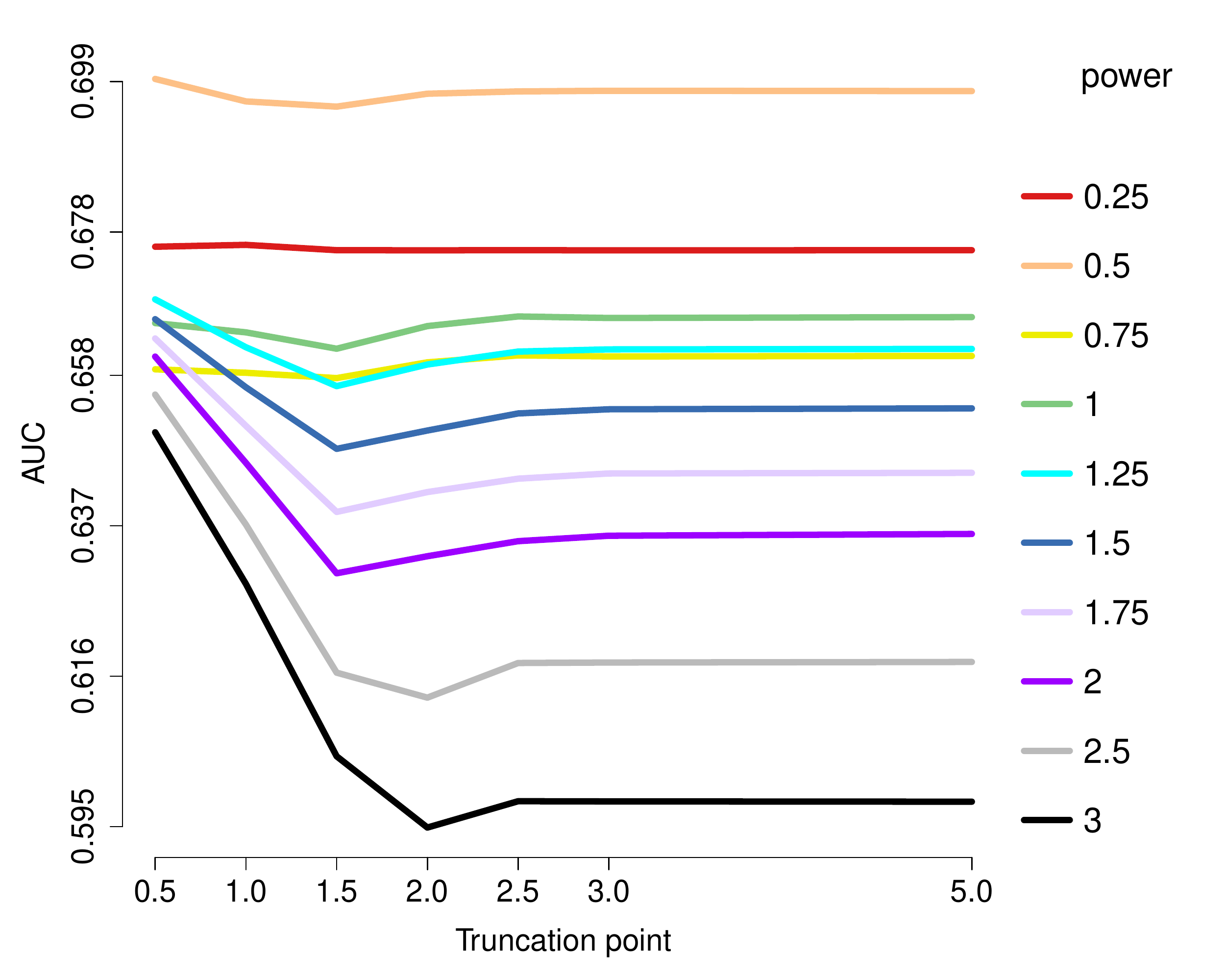}\hspace{-0.02in}}
\subfloat[$n=1000$, $\boldsymbol{\eta}=0.5\mathbf{1}_{100}$, profiled estimator]
{\includegraphics[scale=0.30]{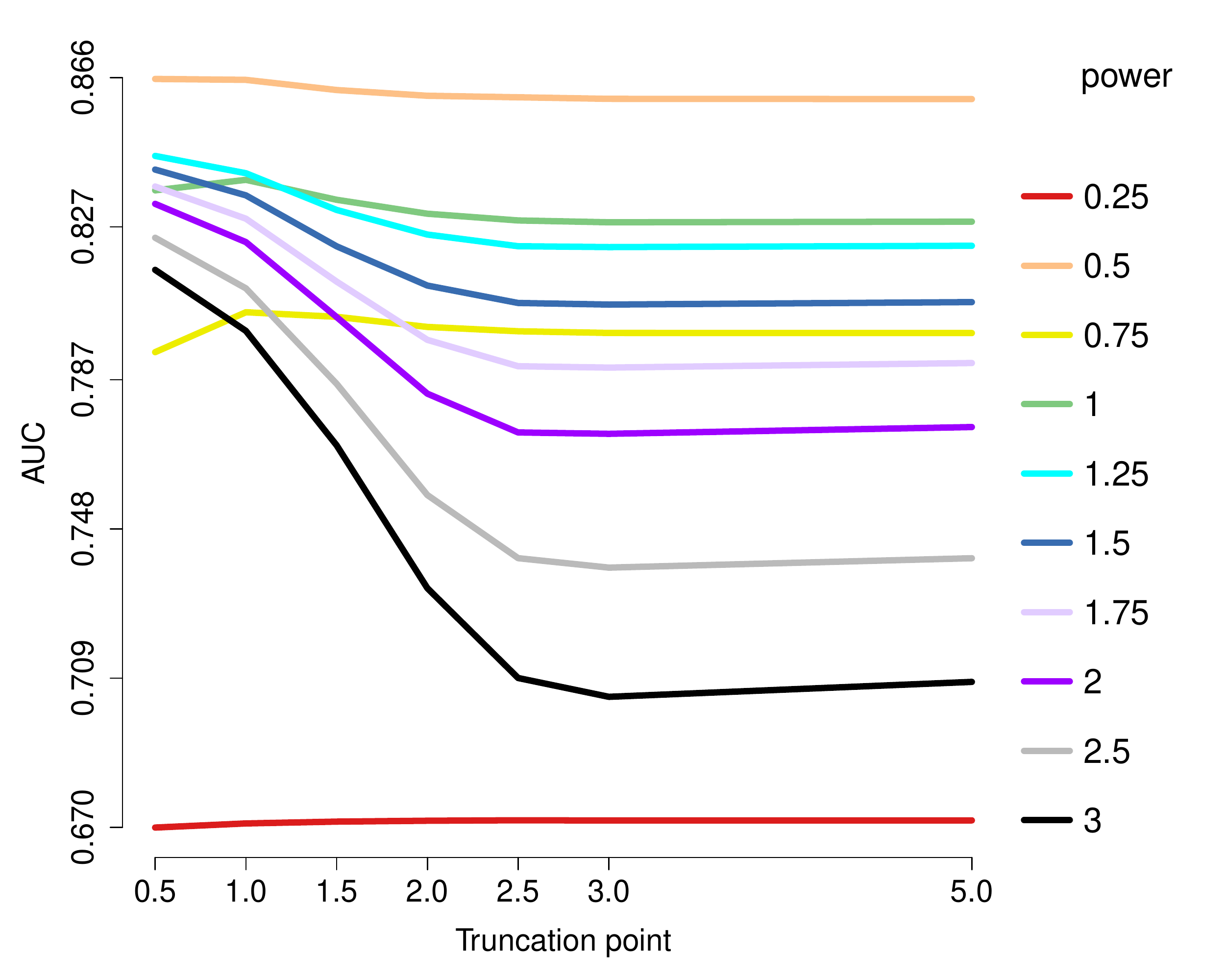}\hspace{-0.02in}}
\caption{AUCs for edge recovery using generalized score matching for $a=3/2$, $b=1/2$. Each curve represents a different choice of power $p$ in $h(x)=\min(x^p,c)$, and the $x$ axis marks the truncation point $c$. Colors are sorted by $p$.}
\label{plot_ab_1.5_0.5}
\end{figure}

\begin{figure}[htp]
\centering
\vspace{-0.0in}
\subfloat[$n=80$, $\boldsymbol{\eta}=-0.5\mathbf{1}_{100}$, profiled estimator]
{\includegraphics[scale=0.30]{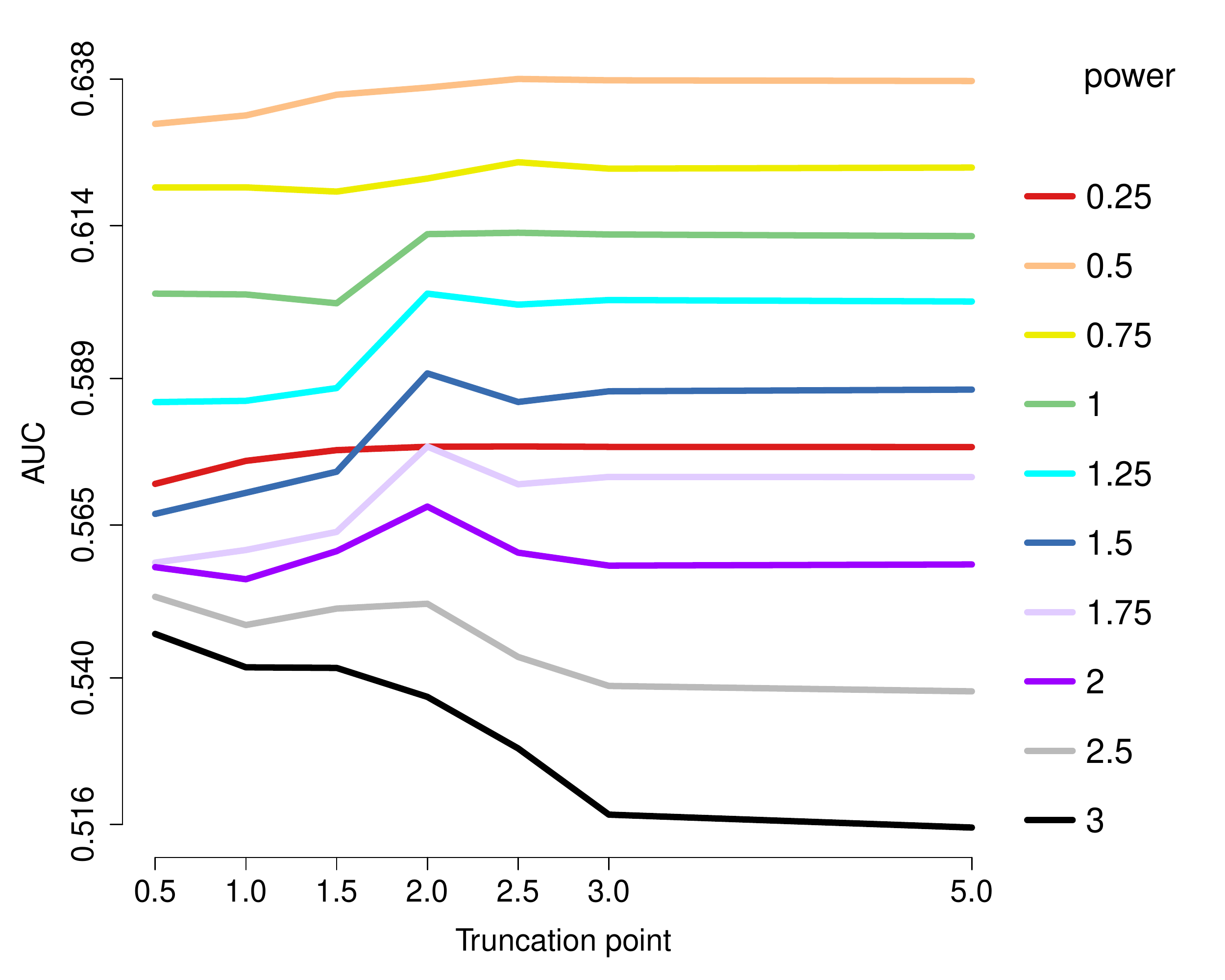}\hspace{-0.02in}}
\subfloat[$n=1000$, $\boldsymbol{\eta}=-0.5\mathbf{1}_{100}$, profiled estimator]
{\includegraphics[scale=0.30]{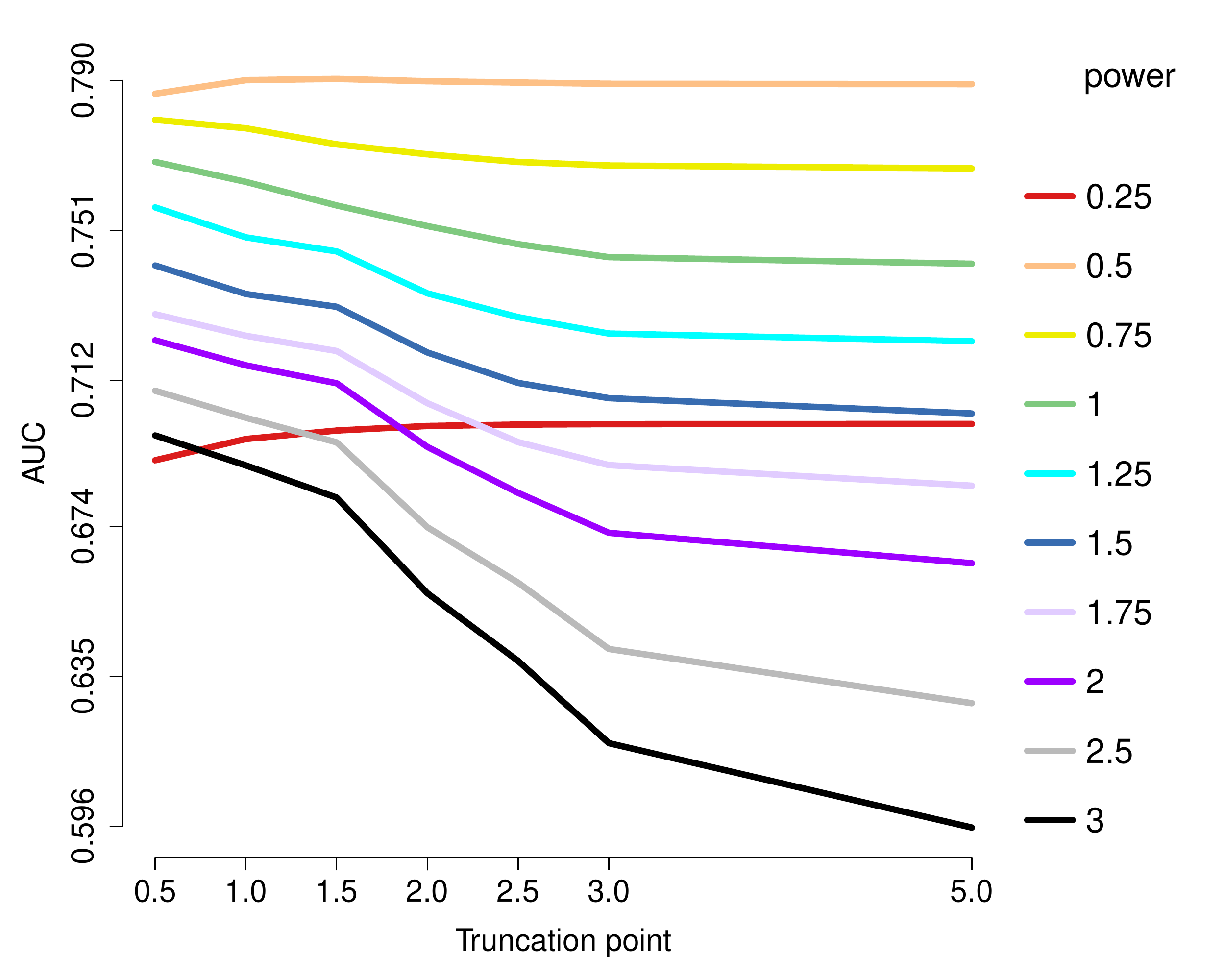}\hspace{-0.02in}}
\\ \vspace{-0.1in}
\subfloat[$n=80$, $\boldsymbol{\eta}=0.5\mathbf{1}_{100}$, profiled estimator]
{\includegraphics[scale=0.30]{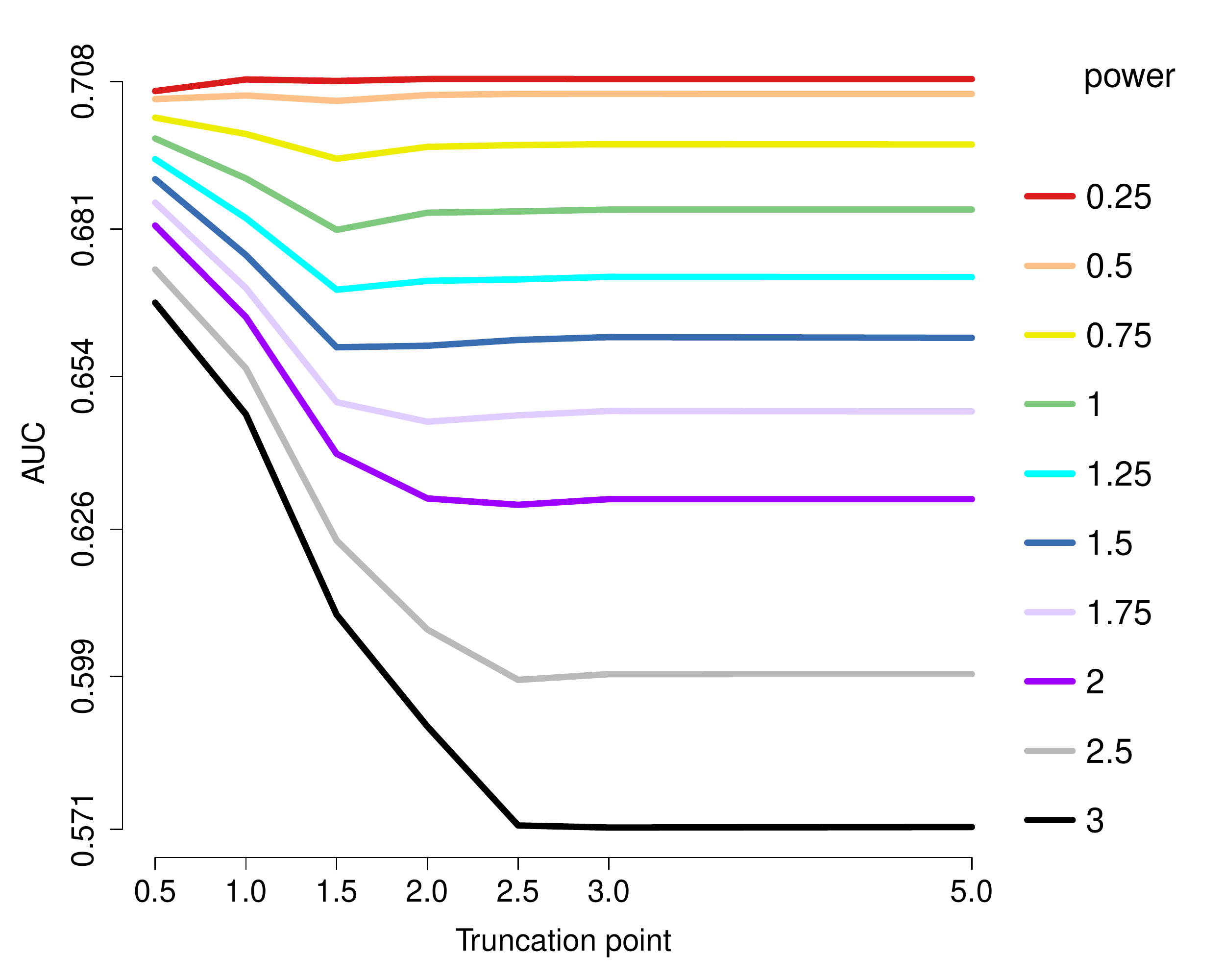}\hspace{-0.02in}}
\subfloat[$n=1000$, $\boldsymbol{\eta}=0.5\mathbf{1}_{100}$, profiled estimator]
{\includegraphics[scale=0.30]{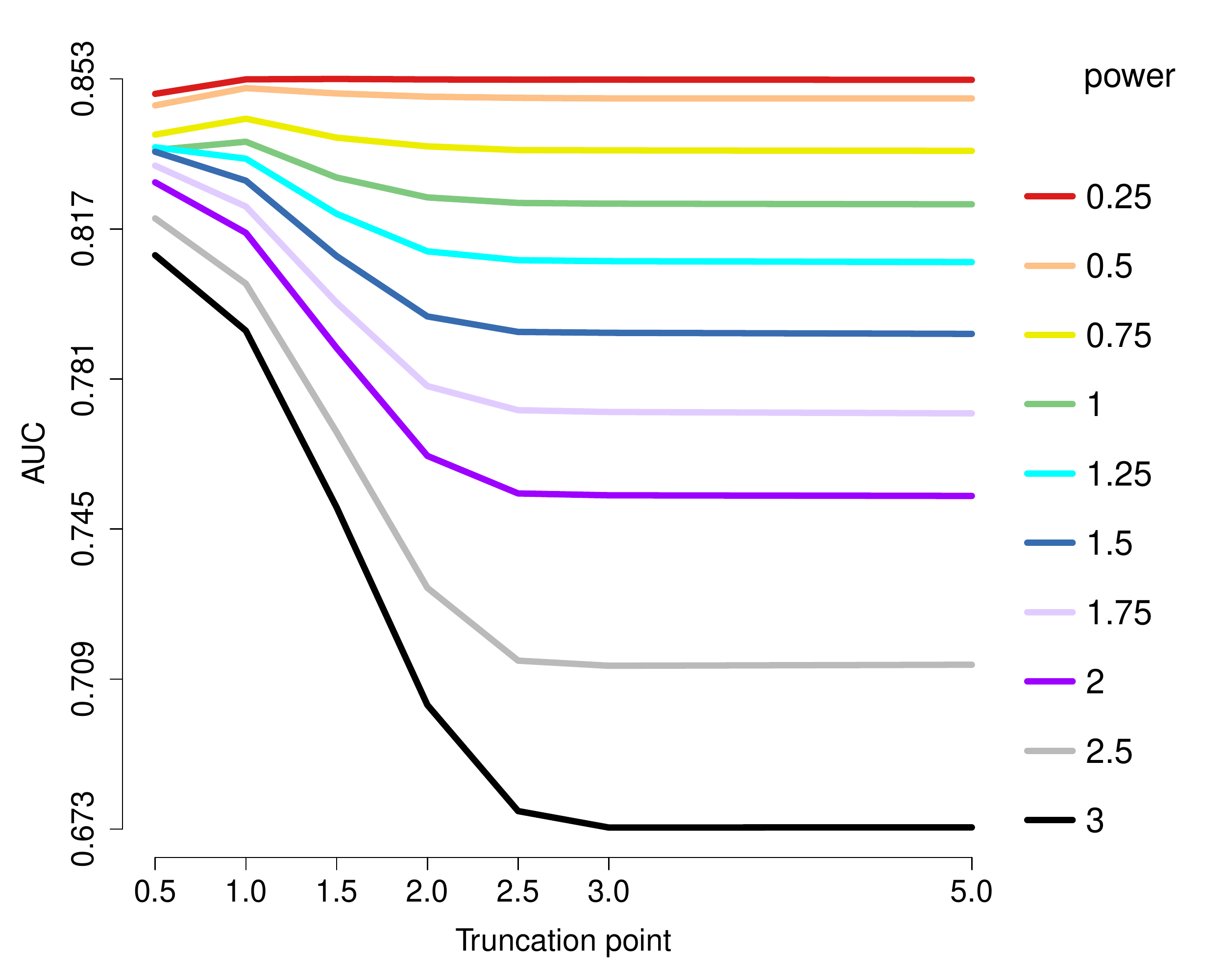}\hspace{-0.02in}}
\caption{AUCs for edge recovery using generalized score matching for $a=3/2$, $b=0$. Each curve represents a different choice of power $p$ in $h(x)=\min(x^p,c)$, and the $x$ axis marks the truncation point $c$. Colors are sorted by $p$.}
\label{plot_ab_1.5_0}
\end{figure}
\newpage

\subsection{RNAseq Data}\label{RNA}

In this section we apply our regularized generalized $\boldsymbol{h}$-score matching estimator for truncated non-centered GGMs to RNAseq data also studied in \citet{lin16}, since the same model is considered therein. The data consists of $n=487$ prostate adenocarcinoma samples from The Cancer Genome Atlas (TCGA) data set.  Following \citet{lin16}, we focus on $m=333$ genes that belong to the known cancer pathways in the Kyoto Encyclopedia of Genes and Genomes (KEGG) and that have no more than $10\%$ missing values.
Missing values are set to $0$.  We choose $h(x)=\min(x,3)$ and use the
upper-bound multiplier (\emph{high}), as discussed in Section
\ref{Choice of multiplier}.  For simplicity, we use the profiled
estimator, and choose the regularization parameter $\lambda$ so that
the estimated graph has exactly $m=333$ edges, all these choices being
as in \citet{lin16}. 

We compare our graph to the one in \citet{lin16}, which corresponds to $h(x)=x^2$ with no multiplier. Shown in Figure \ref{RNA_graphs_simp} are the estimated graphs, with their intersection in the middle. To improve visualization, isolated nodes are removed and the layouts are optimized for each plot. Red-colored points are the ``hub nodes'', namely nodes with degree at least 10. In Figure \ref{RNA_graphs}, we plot the same graphs in a fixed layout optimized for the graph corresponding to $h=\min(x,3)$, and include the isolated nodes.

Out of 333 edges, the two estimated graphs share 117 edges in
common. Assuming that edges are placed at random between nodes and the
two graphs are independent, the distribution of the number $R$ of
common edges follow a hypergeometric distribution, so
$P(R=r)=\frac{\left(\begin{smallmatrix}m\\r\end{smallmatrix}\right)\left(\begin{smallmatrix}m(m-1)/2-m\\
      m-r\end{smallmatrix}\right)}{\left(\begin{smallmatrix}m(m-1)/2\\
      m\end{smallmatrix}\right)}$.   For $m=333$ the probability of at
least 117 common edges is essentially zero. The large number of shared edges between the two methods can be explained by the fact that they both minimize the same {underlying score-matching loss}.

\begin{figure}[htp]
\centering
\subfloat[$\min(x,3)$]{\includegraphics[scale=0.16]{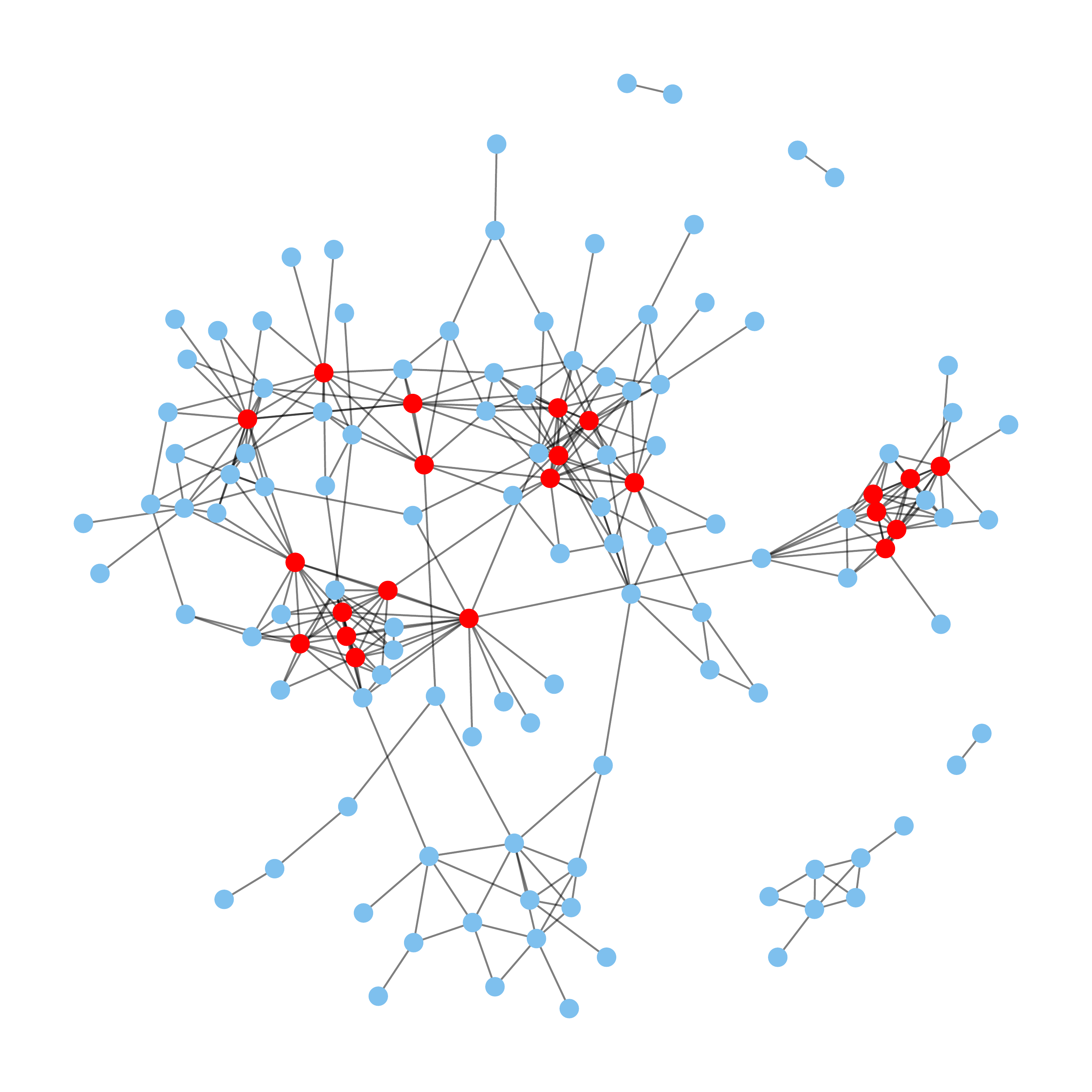}}
\subfloat[Common edges]{\includegraphics[scale=0.16]{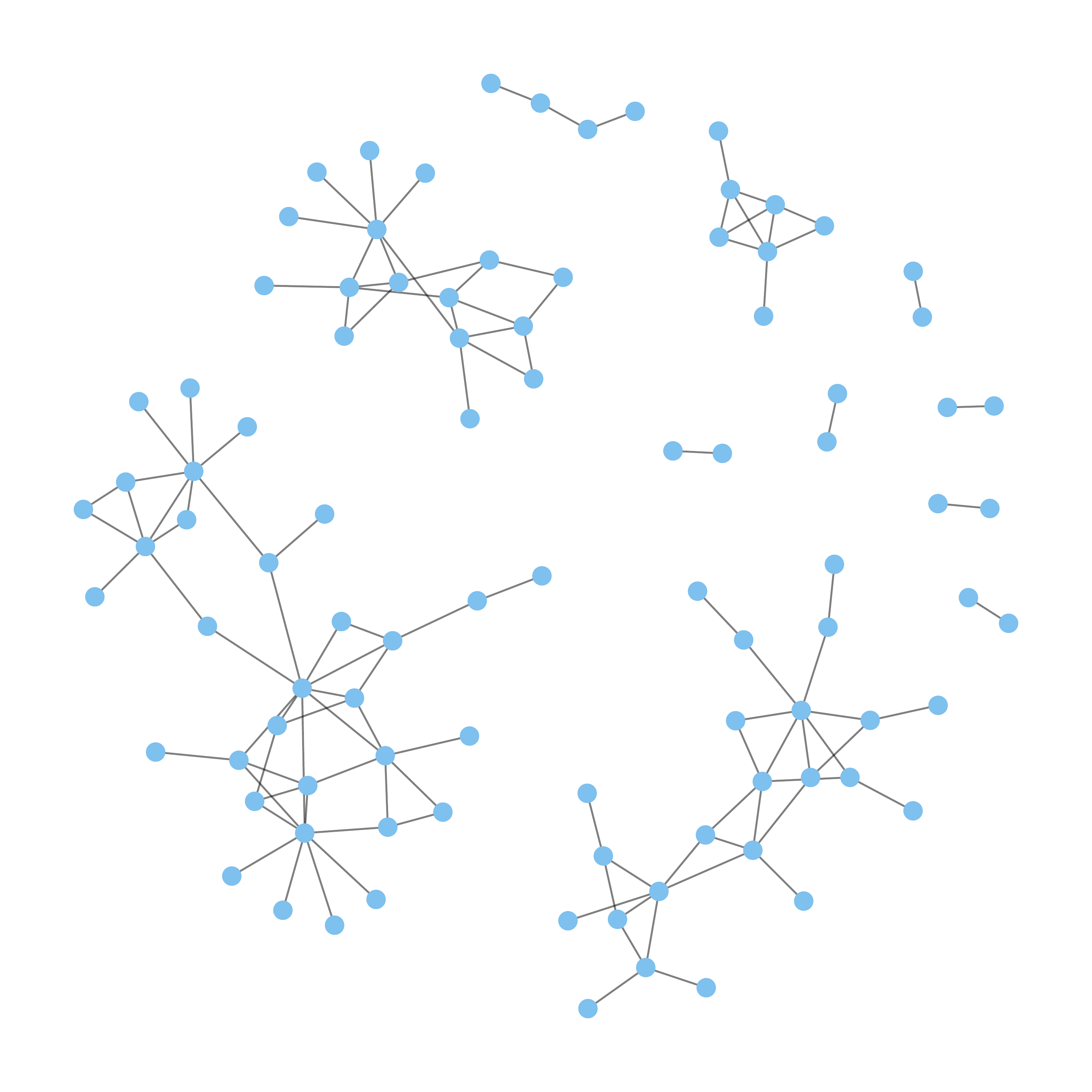}}
\subfloat[Lin et al.]{\includegraphics[scale=0.16]{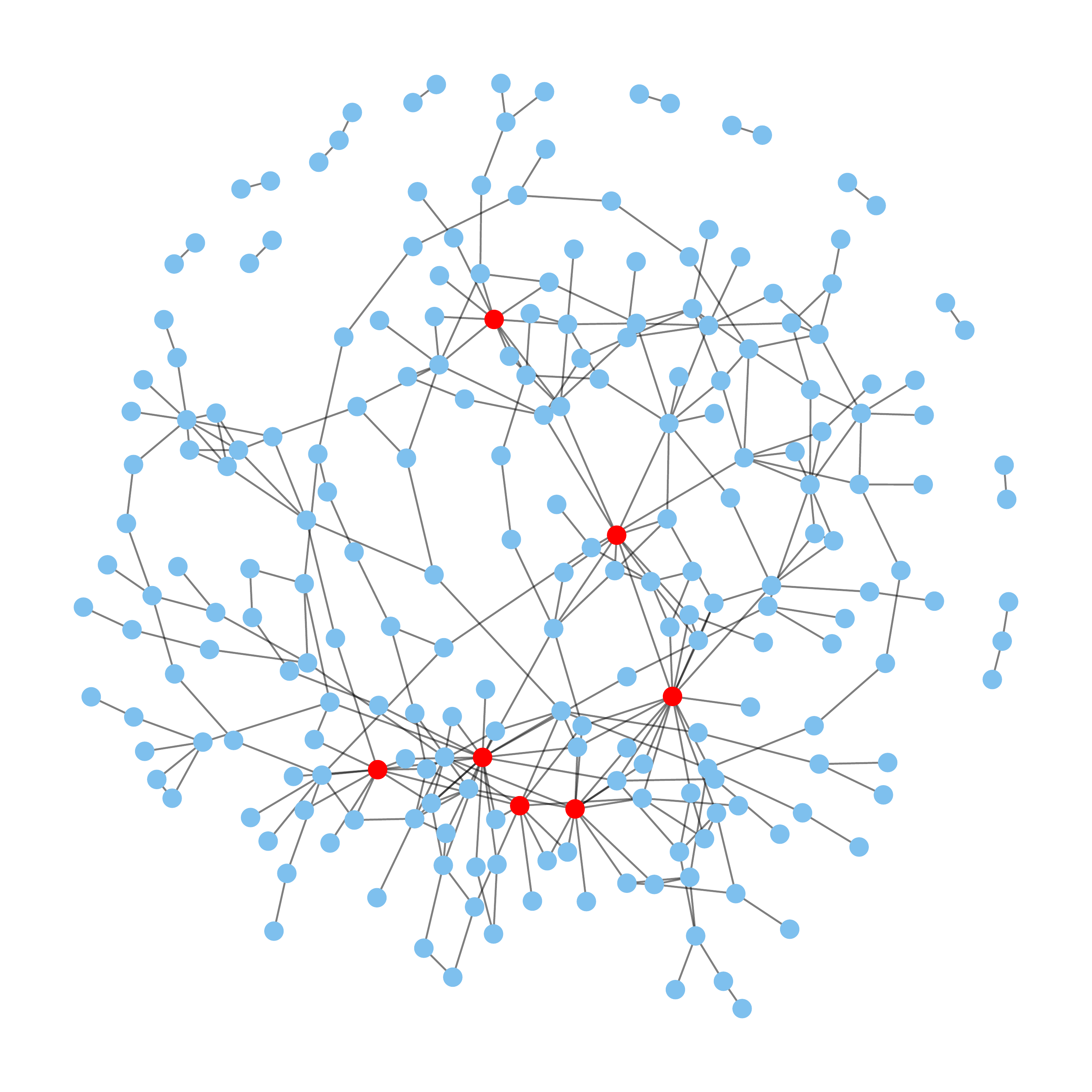}}
\caption{Graphs estimated by regularized generalized score matching estimator with $h(x)=\min(x,3)$ with upper-bound multiplier (left) and $h(x)=x^2$ with no multiplier \citep[right]{lin16}, and their intersection graph (middle). Isolated nodes with no edges are removed, and the layout is optimized for each plot. In (a) and (c), red points indicate nodes with degree at least 10 (``hub nodes'').}\label{RNA_graphs_simp}
\end{figure}

\begin{figure}[htp]
\centering
\subfloat[$\min(x,3)$]{\includegraphics[scale=0.16]{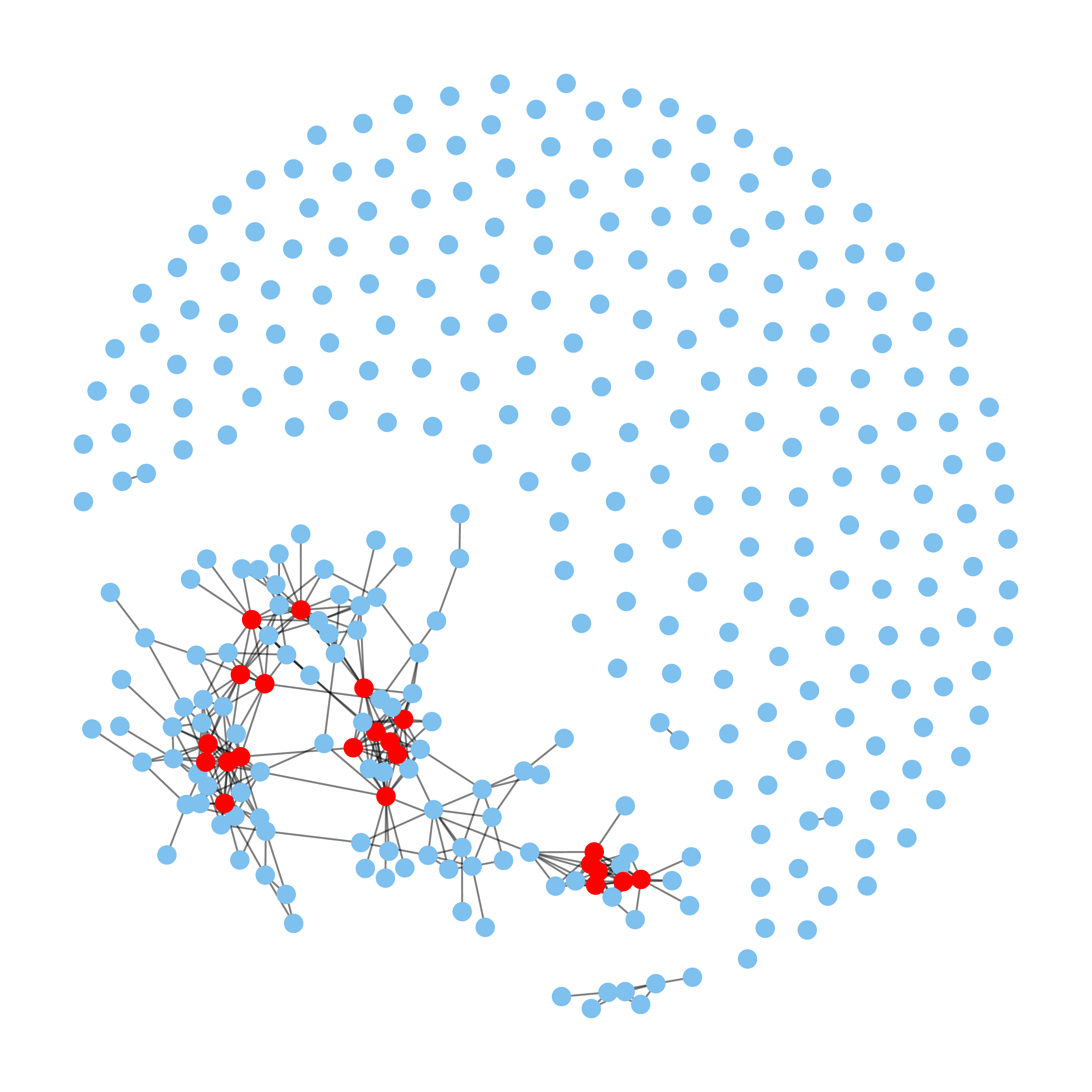}}
\subfloat[Common edges]{\includegraphics[scale=0.16]{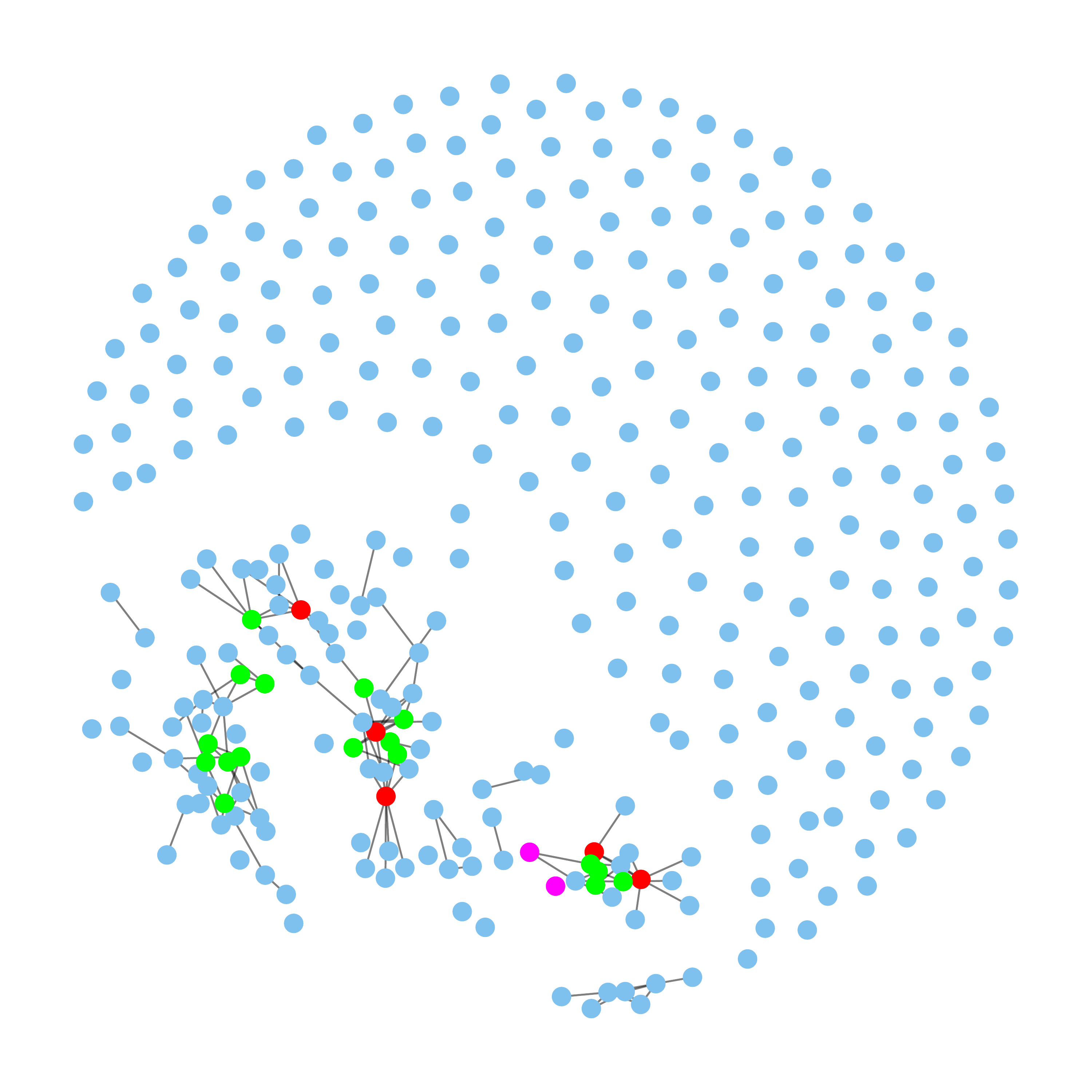}}
\subfloat[Lin et al.]{\includegraphics[scale=0.16]{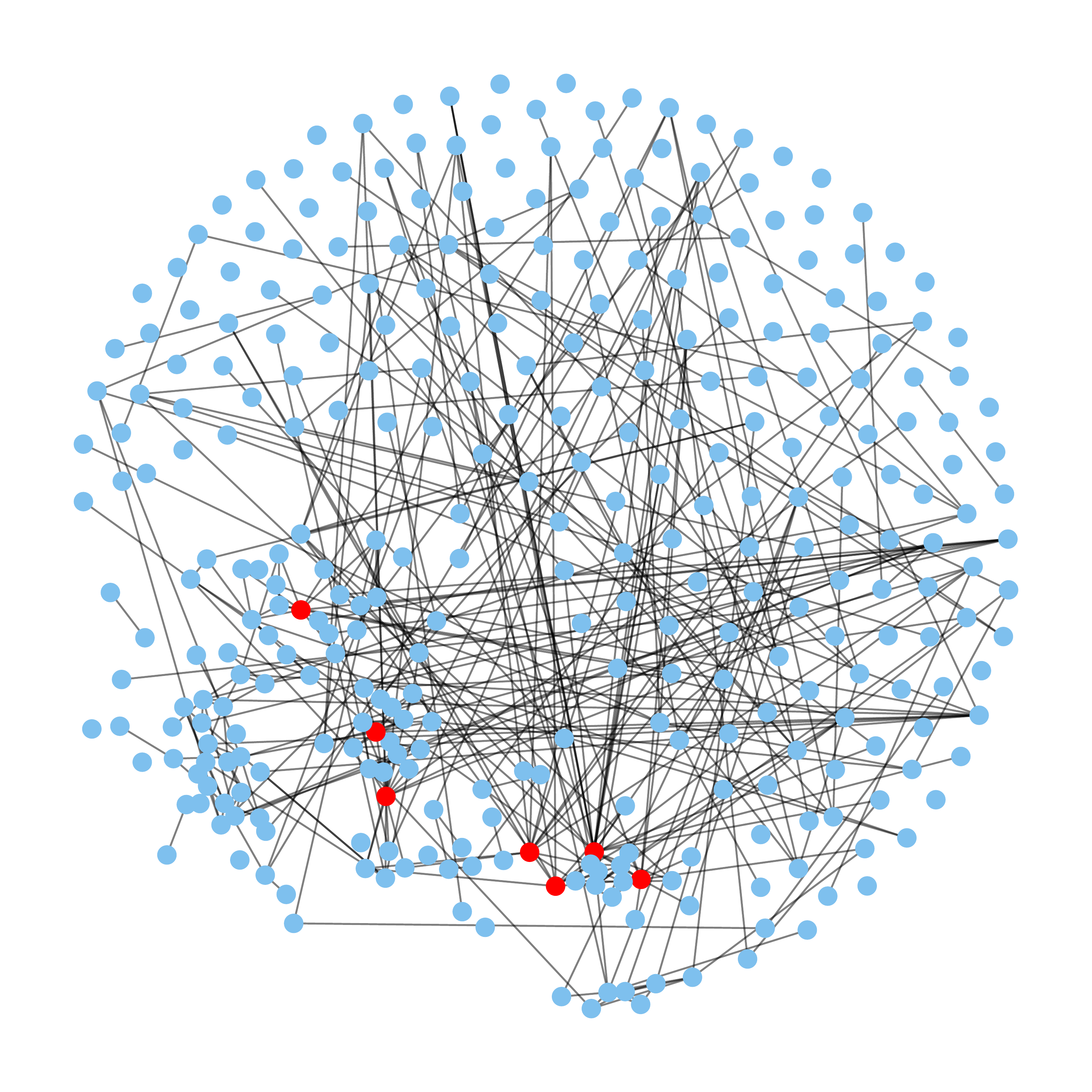}}
\caption{Graphs estimated by regularized generalized score matching estimator with $h(x)=\min(x,3)$ with upper-bound multiplier (left) and $h(x)=x^2$ with no multiplier \citep[right]{lin16}, and their intersection graph (middle). Isolated nodes are included and the layout is fixed across plots and optimized for graph (a). In (b) the red nodes are hub nodes shared by both graphs, the green ones are hub nodes in graph (a) only, and the magenta ones are hub nodes in graph (c) only.}\label{RNA_graphs}
\end{figure}

The graph using $h(x)=\min(x,3)$ has much more isolated nodes (204) than the other (108), and has a slightly smaller max degree (16 versus 19). 
Table \ref{RNA_table} provides another way of comparing between the two graphs by listing the genes with the highest node degrees.

\begin{table}[htp]
\centering
\begin{tabular}{c|c}
\hline
$\min(x,3)$ with multiplier 1.63 & Lin et al.
\tabularnewline
\hline
\textbf{LAMB3 (16)} & CCNE2 (19)\\
\textbf{PIK3CG (16)} & \textbf{PIK3CG (16)} \\
MMP2 (15) & BRCA2 (13) \\
GLI2 (13) & \textbf{BIRC5 (12)}\\
LAMA4 (13) & \textbf{LAMB3 (10)}\\
\textbf{PDGFRB (13)} &  \textbf{PIK3CD (10)}\\
\textbf{PIK3CD (13)} & SKP2 (10)\\
RASSF5 (13) & HRAS (9)\\
\textbf{BIRC5 (12)} & STAT5B (9)\\
FLT3 (12) & \textbf{GSTP1(8)} \\
\textbf{GSTP1 (12)} & \textbf{PDGFRB (8)} \\ 
LAMA2 (12)& \\ 
RAC2 (12) & 
\tabularnewline
\hline
\end{tabular}
\caption{List of genes with the highest node degrees in each estimated graph.}
\label{RNA_table}
\end{table}
In Table \ref{RNA_table} we list the top ten genes in terms of node degree for both estimated graphs. Due to ties, 13 genes are listed for $h(x)=\min(x,3)$ and 11 for \citet{lin16}. As noted in \citet{lin16}, genes with high node-degrees  are known to be important in biological networks \citep{car04,jeo01,han04}. Among these top genes, six are common in both graphs, and are discussed in \citet{lin16}. We next elaborate on the evidence supporting the first four of the newly discovered genes.

\begin{itemize}
\item MMP2 (Matrix metalloproteinase 2): According to \citet{tru03},
  increased MMP-2 expression is an independent predictor of decreased
  prostate cancer disease-free survival. \citet{mor05} state that
  activity of MMP-2 can be useful in diagnosis, therapy, and
  assessment of malignant progression in prostate cancer.
\item GLI2 (GLI family zinc finger 2): GLI2 is a primary mediator of the hedgehog signaling pathway, which has been reported in prostate cancer, and plays a critical role in the malignant phenotype of prostate cancer cells \citep{thi07}. Its increased level of expression is also related to AI prostate cancer, and may be a therapeutic target in castrate-resistant prostate cancer \citep{nar08}.
\item LAMA4 (Laminin subunit alpha 4): LAMA4 is consistently upregulated in benign prostatic hyperplasia when compared to normal prostate tissues \citep{luo02}.
\item RASSF5 (RAS association domain family member 5): The combination of RASSF5 along with four other DNA methylation markers can effectively differentiate between benign prostate biopsy cores from non-cancer patients and cancer cores, and can be used to identify patients at risk without repeat biopsies \citep{bri14}.
\end{itemize}

We note that the two methods indeed use different estimators (different $h$ functions and multipliers), and it is thus not surprising to see that some of the top genes by one method are not among those for the other. In particular, CCNE2, BRCA2, SKP2 and STAT5B, while previously reported as newly discovered in \citet{lin16}, are dropped by our new analysis. Testing and inference (potentially using bootstrapping) is an important problem but is beyond the scope of this paper.


\section{Discussion}

In this paper, we proposed a generalization of the score matching
estimator of \citet{hyv07}, based on scaling the log-gradients to be
matched with a suitably chosen function $\boldsymbol{h}$.  The
generalization retains the advantages of Hyv\"arinen's method:
Estimates can be computed without knowledge of normalizing constants,
and for canonical parameters of exponential families, the estimation
loss is a quadratic function.

For high-dimensional exponential family graphical models, following
 \citet{lin16}, we add an $\ell_1$ penalty to regularize
the generalized score matching loss.
One practical issue that is overlooked in \citet{lin16} is the fact
that the score matching loss can be unbounded below for a small tuning
parameter, when the dimension $m$ exceeds the sample size
$n$. We fix this issue by amplifying the diagonal entries in the
quadratic matrix in the definition of the generalized score matching
loss by a factor/multiplier, and we give an upper bound on that
multiplier that guarantees consistency.

As examples we consider \emph{pairwise interaction power models} on
the non-negative orthant $\mathbb{R}_+^m$.  Specifically, the
considered models are exponential families in which the log density is
the sum of pairwise interactions between entries in of powers
$\boldsymbol{x}^a$ plus linearly weighted effects $\boldsymbol{x}^b$,
or $\log(\boldsymbol{x})$ when $b=0$.  Our main interest is in the
matrix of interaction parameters whose support determines the
distributions' conditional independence graph.  The considered
framework covers truncated normal distributions ($a=b=1$),
exponential square root graphical models ($a=b=1/2$) from
\citet{ino16}, as well as a class of multivariate gamma distributions
($a=1/2$, $b=0$).  

In the case of multivariate truncated normal distributions, where
the conditional independence graph is given by the underlying Gaussian
inverse covariance matrix, the sample size required for the
consistency of our method using bounded $\boldsymbol{h}$ is
$\Omega(d^2\log m)$, where $d$ is the degree of the graph.  This
matches the rates for Gaussian graphical models in \citet{rav11} and
\citet{lin16}.  In contrast, the sample complexity for truncated
Gaussian models given in \citet{lin16} is $\Omega(d^2\log^8 m)$.

For the considered class of pairwise interaction models, we recommend
using the function $\boldsymbol{h}$ with coordinates
$h_j(x)=\min\big(x^{2-a},c\big)$ for some moderately large $c$, or
simply $h_j(x)=x^{2-a}$.  While this choice is effective, it would be an
interesting problem for future work to develop a method that
adaptively chooses an optimized function $\boldsymbol{h}$ from data.


\acks{This work was partially supported by grant DMS/NIGMS-1561814 from the National Science Foundation (NSF). AS also gratefully acknowledges funding by grant R01-GM114029 from the National Institute of Health (NIH).}

\newpage

\appendix
\renewcommand{\thesection}{\Alph{section}}
\section{Proofs}\label{Proofs}

\subsection{Proof of Theorem \ref{theorem_GSM_loss_alt}}\label{Proof of Theorem theorem_GSM_loss_alt}
The following integration by parts lemma is used in the proof of Theorem \ref{theorem_GSM_loss_alt}.

\begin{lemma}\label{lemma_integration_by_parts}
Let $f,g:\mathbb{R}_+\to\mathbb{R}$ be functions that are absolutely continuous in every bounded sub-interval of $\mathbb{R}_+$. Then
\[
\lim_{x\nearrow+\infty}f(x)g(x)-\lim_{x\searrow0^+}f(x)g(x)\\=\int_0^{\infty}f(\boldsymbol{x})\frac{\d g(x)}{\d x}\d x+\int_{0}^{\infty}g(\boldsymbol{x})\frac{\d f(x)}{\d x}\d x.
\]
\end{lemma}

\begin{proof}
This is an analog of Lemma 4 from \citet{hyv05} that can be proved by integrating the partial derivatives, and follows from the fundamental theorem of calculus for absolutely continuous functions and the product rule. In particular, we work on integrals in bounded $[0,c]$, where the product of two absolutely continuous functions in a bounded interval is again absolutely continuous, and the result is then obtained by letting $c\nearrow+\infty$.
\end{proof}

\begin{proof}[Proof of Theorem \ref{theorem_GSM_loss_alt}]
Recall the following assumptions from Section \ref{Generalized Score Matching for Non-Negative Data}:\\
\resizebox{\linewidth}{!}{\begin{minipage}{\linewidth}
\begin{align*}
(\text{A1}) &\,\,\, p_0(\boldsymbol{x})h_j(x_j)\partial_j\log p(\boldsymbol{x})\left|^{x_j\nearrow+\infty}_{x_j\searrow 0^+}=0\right.,\,\,\, \forall\boldsymbol{x}_{-j}\in\mathbb{R}_{+}^{m-1},\quad \forall p\in\mathcal{P}_+;\\
(\text{A2}) &\,\,\,\mathbb{E}_{p_0}\|\nabla\log p(\boldsymbol{X})\circ \boldsymbol{h}^{1/2}(\boldsymbol{X})\|_2^2<+\infty, \,\,\, \mathbb{E}_{p_0}\|(\nabla\log p(\boldsymbol{X})\circ\boldsymbol{h}(\boldsymbol{X}))'\|_1<+\infty,\,\,\, \forall p\in\mathcal{P}_+.
\end{align*}
\end{minipage}}\\\\
Without explicitly writing the domains $\mathbb{R}_+$ or $\mathbb{R}_+^m$ in all integrals, by (\ref{eq_gsm}) we have
\begin{alignat*}{3}
J_{\boldsymbol{h}}(p)&=\frac{1}{2}\int p_0(\boldsymbol{x})\left[\|\nabla\log p(\boldsymbol{x})\circ\boldsymbol{h}^{1/2}(\boldsymbol{x})\|_2^2\right.
\\&\quad\quad\left.-2(\nabla\log p(\boldsymbol{x})\circ\boldsymbol{h}^{1/2}(\boldsymbol{x}))^{\top}(\nabla\log p_0(\boldsymbol{x})\circ\boldsymbol{h}^{1/2}(\boldsymbol{x}))
+\|\nabla\log p_0(\boldsymbol{x})\circ\boldsymbol{h}^{1/2}(\boldsymbol{x})\|_2^2\right]\d\boldsymbol{x}\\
&=\underbrace{\frac{1}{2}\int p_0(\boldsymbol{x})\sum_{j=1}^mh_j(x_j)\left(\frac{\partial \log p(\boldsymbol{x})}{\partial x_j}\right)^2\d\boldsymbol{x}}_{\equiv\,A}+\underbrace{\frac{1}{2}\int p_0(\boldsymbol{x})\sum_{j=1}^{m}h_j(x_j)\left(\frac{\partial\log p_0(\boldsymbol{x})}{\partial x_j}\right)^2\d\boldsymbol{x}}_{\equiv\,C}\\
&\quad\quad\quad\quad\underbrace{-\int p_0(\boldsymbol{x})\sum_{j=1}^mh_j(x_j)\frac{\partial\log p(\boldsymbol{x})}{\partial x_j}\frac{\partial\log p_0(\boldsymbol{x})}{\partial x_j}\d\boldsymbol{x}}_{\equiv\,B},
\end{alignat*}
where $A$ will simply appear in the final display as is, $C$ is a constant as it only involves the true pdf $p_0$, and we wish to simplify $B$ by integration by parts. We can split the integral into these three parts since $A$ and $C$ are assumed finite in the first part of (A2), and the integrand in $B$ is integrable since $|2ab|\leq a^2+b^2$. Thus, by linearity and Fubini's theorem, we can write
\begin{align*}
B&=-\sum_{j=1}^m\int p_0(\boldsymbol{x})h_j(x_j)\frac{\partial \log p(\boldsymbol{x})}{\partial x_j}\frac{\partial \log p_0(\boldsymbol{x})}{\partial x_j}\d\boldsymbol{x}\\
&=-\sum_{j=1}^m\int\left[\int p_0(\boldsymbol{x})h_j(x_j)\frac{\partial \log p(\boldsymbol{x})}{\partial x_j}\frac{\partial \log p_0(\boldsymbol{x})}{\partial x_j}\d x_j\right]\d\boldsymbol{x}_{-j}.
\end{align*}
By the fact that $\frac{\partial \log p_0(\boldsymbol{x})}{\partial x_j}=\frac{1}{p_0(\boldsymbol{x})}\frac{\partial p_0(\boldsymbol{x})}{\partial x_j}$, this can be simplified to
\[B=-\sum_{j=1}^m\int\left[\int\frac{\partial p_0(\boldsymbol{x})}{\partial x_j}h_j(x_j)\frac{\partial\log p(\boldsymbol{x})}{\partial x_j}\d x_j\right]\d \boldsymbol{x}_{-j}.\]

But, we assume $p_0$ and $p$ are twice continuously differentiable, for every $j=1,\dots,m$ and fixed $\boldsymbol{x}_{-j}\in\mathbb{R}_+^{m-1}$. Hence, in every bounded sub-interval of $\mathbb{R}_+$, $p_0(\boldsymbol{x}_{-j};x_j)$ is an absolutely continuous function of $x_j$, $\partial_j\log p(\boldsymbol{x}_{-j},x_j)=\partial_j p(\boldsymbol{x}_{-j},x_j)/p(\boldsymbol{x}_{-j},x_j)$ is a continuously differentiable (and hence absolutely continuous) function of $x_j$ by the quotient rule. Thus $h_j(x_j)\partial_j\log p(\boldsymbol{x}_{-j};x_j)$ is also absolutely continuous by the absolute continuity assumption on $h_j$. Then, by Lemma \ref{lemma_integration_by_parts}, where we take $f\equiv p_0(\boldsymbol{x}_{-j};x_j)$ and $g\equiv h_j(x_j)\partial_j\log p(\boldsymbol{x}_{-j};x_j)$ as functions of $x_j$, followed by assumption (A1),
\begin{align*}
B&=-\sum_{j=1}^m\int\left[\lim_{a\nearrow+\infty,b\searrow 0^+}\left[p_0(\boldsymbol{x}_{-j};a)h_j(a)\partial_j\log p(\boldsymbol{x}_{-j},a)-p_0(\boldsymbol{x}_{-j};b)h_j(b)\partial_j\log p(\boldsymbol{x}_{-j},b)\right]\right.\\
&\quad\quad\quad\quad\quad\quad-\left.\int p_0(\boldsymbol{x})\frac{\partial\left(h_j(x_j)\partial_j\log p(\boldsymbol{x})\right)}{\partial x_j}\d x_j\right]\d\boldsymbol{x}_{-j}\\
&=\sum_{j=1}^m\int\left[\int p_0(\boldsymbol{x})\frac{\partial(h_j(x_j)\partial_j \log p(\boldsymbol{x}))}{\partial x_j}\d x_j\right]\d\boldsymbol{x}_{-j}.
\end{align*}
Justified by the second half of (A2), by Fubini-Tonelli and linearity again
\begin{align*}
B&=\sum_{j=1}^m\int p_0(\boldsymbol{x})\frac{\partial(h_j(x_j)\partial_j \log p(\boldsymbol{x}))}{\partial x_j}\d \boldsymbol{x},\\
&=\sum_{j=1}^m\int h'_j(x_j) \frac{\partial\log p(\boldsymbol{x})}{\partial x_j}p_0(\boldsymbol{x})\d\boldsymbol{x}+\sum_{j=1}^{m}\int h_j(x_j)\frac{\partial^2\log p(\boldsymbol{x})}{\partial x_j^2}p_0(\boldsymbol{x})\d \boldsymbol{x}.
\end{align*}
Thus,
\begin{align*}
&J_{\boldsymbol{h}}(p)\\
=\,&B+A+C\\
=\,&\int_{\mathbb{R}_+^m}p_0(\boldsymbol{x})\sum_{j=1}^m\left[h_j'(x_j)\frac{\partial\log p(\boldsymbol{x})}{\partial x_j}+h_j(x_j)\frac{\partial^2\log p(\boldsymbol{x})}{\partial x_j^2}+\frac{1}{2}h_j(x_j)\left(\frac{\partial\log p(\boldsymbol{x})}{\partial x_j}\right)^2\right]\d\boldsymbol{x}+C,
\end{align*}
where $C$ is a constant that does not depend on $p$.
\end{proof}

\subsection{Proof of Theorems and Examples in Section \ref{Exponential_Families}}\label{Proof of Theorems and Examples in Section Exponential_Families}

\begin{proof}[Proof of Theorem \ref{theorem_GSM_estimator_exp}]
For 
exponential families and under the assumptions, the empirical loss
$\hat{J}_{\boldsymbol{h}}(p_{\boldsymbol{\theta}})$ in
(\ref{eq_gsm_sample}) becomes (up to an additive constant)
\newpage
\begin{align*}
&\,\hat{J}_{\boldsymbol{h}}(p_{\boldsymbol{\theta}})\\
=&\,\frac{1}{n}\sum_{i=1}^n\sum_{j=1}^m\left[h_j'(X_{j}^{(i)})\frac{\partial \log p_{\boldsymbol{\theta}}(\boldsymbol{X}^{(i)})}{\partial X_{j}^{(i)}}+h_j(X_{j}^{(i)})\frac{\partial^2 \log p_{\boldsymbol{\theta}}(\boldsymbol{X}^{(i)})}{\partial (X_{j}^{(i)})^2}\right.\\
&\hspace{3.7in}\left.+
	\frac{1}{2}h_j(X_{j}^{(i)})\left(\frac{\partial \log p_{\boldsymbol{\theta}}(\boldsymbol{X}^{(i)})}{\partial X_{j}^{(i)}}\right)^2\right]\\
=&\,\frac{1}{n}\sum_{i=1}^n\sum_{j=1}^m\left[h_j'(X_j^{(i)})(\boldsymbol{\theta}^{\top}\boldsymbol{t}_j'(\boldsymbol{X}^{(i)})+b_j'(\boldsymbol{X}^{(i)}))+h_j(X_j^{(i)})(\boldsymbol{\theta}^{\top}\boldsymbol{t}_j''(\boldsymbol{X}^{(i)})+b_j''(\boldsymbol{X}^{(i)}))\right.\\
&\quad\quad\quad\quad\,\,\,+\left.\frac{1}{2}h_j(X_j^{(i)})(\boldsymbol{\theta}^{\top}\boldsymbol{t}'_j(\boldsymbol{X}^{(i)})+b'_j(\boldsymbol{X}^{(i)}))^2\right]\hspace{0.4in}\\
=&\,\frac{1}{n}\left\{\frac{1}{2}\boldsymbol{\theta}^{\top}\left[\sum_{i=1}^n\sum_{j=1}^m
h_j(X_j^{(i)})\boldsymbol{t}_j'(\boldsymbol{X}^{(i)})
\boldsymbol{t}_j'(\boldsymbol{X}^{(i)})^{\top}\right]\boldsymbol{\theta}+\right.\\
&\left.\left[\sum_{i=1}^n\sum_{j=1}^mh_j(X_j^{(i)})b_j'(\boldsymbol{X}^{(i)})\boldsymbol{t}_j'(\boldsymbol{X}^{(i)})+h_j(X_j^{(i)})
\boldsymbol{t}_j''(\boldsymbol{X}^{(i)})+h_j'(X_j^{(i)})\boldsymbol{t}_j'
(\boldsymbol{X}^{(i)})\right]^{\top}\boldsymbol{\theta}\right\}+\text{const},
\end{align*}
which is quadratic in $\boldsymbol{\theta}$. Let
\begin{align}
\boldsymbol{\Gamma}(\mathbf{x})&=\frac{1}{n}\sum_{i=1}^{n}\sum_{j=1}^mh_j(X_j^{(i)})\boldsymbol{t}_j'(\boldsymbol{X}^{(i)})
\boldsymbol{t}_j'(\boldsymbol{X}^{(i)})^{\top},\\
\boldsymbol{g}(\mathbf{x})&=-\frac{1}{n}\sum_{i=1}^n\sum_{j=1}^m\left[h_j(X_j^{(i)})b_j'(\boldsymbol{X}^{(i)})
\boldsymbol{t}'_j(\boldsymbol{X}^{(i)})+h_j(X_j^{(i)})\boldsymbol{t}_j''(\boldsymbol{X}^{(i)})+
h_j'(X_j^{(i)})\boldsymbol{t}_j'(\boldsymbol{X}_i)\right].
\end{align}
Then we can write $\hat{J}_{\boldsymbol{h}}(p_{\boldsymbol{\theta}})=\frac{1}{2}\boldsymbol{\theta}^{\top}\boldsymbol{\Gamma}(\mathbf{x})\boldsymbol{\theta}-\boldsymbol{g}(\mathbf{x})^{\top}\boldsymbol{\theta}+\mathrm{const}$.
\end{proof}

\begin{proof}[Proof of Theorem \ref{theorem_exponential}]
By Theorem \ref{theorem_GSM_estimator_exp}, $\hat{J}_{\boldsymbol{h}}(p_{\boldsymbol{\theta}})=\frac{1}{2}\boldsymbol{\theta}^{\top}\boldsymbol{\Gamma}\boldsymbol{\theta}-\boldsymbol{g}^{\top}\boldsymbol{\theta}+\mathrm{const}$. The minimizer of $\hat{J}_{\boldsymbol{h}}(p_{\boldsymbol\theta})$ is thus available in the unique closed form $\hat{\boldsymbol{\theta}}\equiv\boldsymbol{\Gamma}(\mathbf{x})^{-1}\boldsymbol{g}(\mathbf{x})$ as long as $\boldsymbol{\Gamma}$ is invertible (C1).
Since $\boldsymbol{\Gamma}$ and $\boldsymbol{g}$ are sample averages, 
the weak law of large numbers yields that 
$\boldsymbol{\Gamma}\to_p\mathbb{E}_{p_0}\boldsymbol{\Gamma}\equiv\boldsymbol{\Gamma}_0$ and $\boldsymbol{g}\to_p\mathbb{E}_{p_0}\boldsymbol{g}\equiv\boldsymbol{g}_0$,
where existence of $\boldsymbol{\Gamma}_0$ and $\boldsymbol{g}_0$ is assumed in (C2). Since $J_{\boldsymbol{h}}(p_{\boldsymbol{\theta}})=\mathbb{E}[\hat{J}_{\boldsymbol{h}}(p_{\boldsymbol{\theta}})]=\mathbb{E}[\frac{1}{2}\boldsymbol{\theta}^{\top}\boldsymbol{\Gamma}(\mathbf{x})\boldsymbol{\theta}-\boldsymbol{g}(\mathbf{x})^{\top}\boldsymbol{\theta}]=\frac{1}{2}\boldsymbol{\theta}^{\top}\boldsymbol{\Gamma}_0\boldsymbol{\theta}-\boldsymbol{g}_0\boldsymbol{\theta}$ and we know $\boldsymbol{\theta}_0$ minimizes $J_{\boldsymbol{h}}(p_{\boldsymbol{\theta}})$ by definition, by the first-order condition we must have $\boldsymbol{\Gamma}_0\boldsymbol{\theta}_0=\boldsymbol{g}_0$.  Then by the Lindeberg-L\'evy central limit theorem, 
\[\sqrt{n}(\boldsymbol{g}(\mathbf{x})-\boldsymbol{\Gamma}(\mathbf{x})\boldsymbol{\theta}_0)\to_d\mathcal{N}_m(\boldsymbol{0},\boldsymbol{\Sigma}_0),\]
where $\boldsymbol{\Sigma}_0\equiv\mathbb{E}_{p_0}[(\boldsymbol{\Gamma}(\mathbf{x})\boldsymbol{\theta}_0-\boldsymbol{g}(\mathbf{x}))(\boldsymbol{\Gamma}(\mathbf{x})\boldsymbol{\theta}_0-\boldsymbol{g}(\mathbf{x}))^{\top}]$, as long as $\boldsymbol{\Sigma}_0$ exists (C2).
Thus, by Slutsky's theorem, 
\[\sqrt{n}(\hat{\boldsymbol{\theta}}-\boldsymbol{\theta}_0)\equiv\sqrt{n}(\boldsymbol{\Gamma}(\mathbf{x})^{-1}(\boldsymbol{g}(\mathbf{x})-\boldsymbol{\Gamma}(\mathbf{x})\boldsymbol{\theta}_0))\to_d\mathcal{N}_{r}(\boldsymbol{0},\boldsymbol{\Gamma}_0^{-1}\boldsymbol{\Sigma}\boldsymbol{\Gamma}_0^{-1}),\]
as long as $\boldsymbol{\Gamma}_0$ is invertible (C2).

For the second half of the theorem, (C2) $\mathbb{E}_{p_0}\boldsymbol{\Gamma}(\mathbf{x})<\infty$ and $\mathbb{E}_{p_0}\boldsymbol{g}(\mathbf{x})<\infty$ implies $\mathbb{E}_{p_0}|\boldsymbol{\Gamma}(\mathbf{x})|<\infty$ and $\mathbb{E}_{p_0}|\boldsymbol{g}(\mathbf{x})|<\infty$, so by strong law of large numbers (and a union bound on at most $k^2$ null sets)
\[\boldsymbol{\Gamma}(\mathbf{x})\to_{\text{a.s.}}\boldsymbol{\Gamma}_0,\quad \boldsymbol{g}(\mathbf{x})\to_{\text{a.s.}}\boldsymbol{g}_0.\]
Then outside a null set,
\[\hat{\boldsymbol{\theta}}\equiv\boldsymbol{\Gamma}(\mathbf{x})^{-1}\boldsymbol{g}(\mathbf{x})\to_{\mathrm{a.s.}}\boldsymbol{\Gamma}_0^{-1}\boldsymbol{g}_0=\boldsymbol{\theta}_0.
\]
\end{proof}

\begin{proof}[Proof for Example \ref{theorem_GSM_normality_tnorm_mu}]
We choose to estimate $\theta\equiv\mu/\sigma^2$. Then by (\ref{def_Gamma}) and (\ref{def_g}), 
\begin{align*}
\hat{\mu}_{h}&=\sigma^2\hat{\theta}\equiv\sigma^2\Gamma(\boldsymbol{x})^{-1}g(\boldsymbol{x})\\
&=-\sigma^2\left[\sum_{i=1}^nh(X_i)t'(X_i)^2\right]^{-1}\left[\sum_{i=1}^nh(X_i)b'(X_i)t'(X_i)+h(X_i)t''(X_i)+h'(X_i)t'(X_i)\right]\\
&=-\sigma^2\left[\sum_{i=1}^nh(X_i)\right]^{-1}\left[\sum_{i=1}^n-h(X_i)\frac{X_i}{\sigma^2}+h'(X_i)\right].
\end{align*}
By Theorem \ref{theorem_exponential},
\begin{align*}\sqrt{n}(\hat{\mu}_{h}-\mu_0)&\to_d \mathcal{N}\left(0,\frac{\sigma^4\mathbb{E}_0\left[-h(X)\frac{X-\mu_0}{\sigma^2}+h'(X)\right]^2}{\mathbb{E}_0^2[h(X)]}\right)\\
&\sim \mathcal{N}\left(0,\frac{\mathbb{E}_0\left[-h(X)(X-\mu_0)+\sigma^2 h'(X)\right]^2}{\mathbb{E}_0^2[h(X)]}\right).
\end{align*}
By integration by parts, (suppressing the dependence of $p_{\mu_0}$ on $\mu_0$)
\begin{align*}
&\,\mathbb{E}_0[h(X)h'(X)(X-\mu_0)]\\
=&\,\int_0^{\infty}h'(x)h(x)(x-\mu_0)p(x)\d x\\
=&\int_0^{\infty}h(x)(x-\mu_0)p(x)\d h(x)\\
=&\,\left.h^2(x)(x-\mu_0)p(x)\right|_0^{\infty}-\int h(x)\d h(x)(x-\mu_0)p(x)\\
=&\,-\int h^2(x)p(x)\d x-\int h(x)h'(x)(x-\mu_0)p(x)\d x+\int h^2(x)\frac{(x-\mu_0)^2}{\sigma^2}p(x)\d x,
\end{align*}
where the last step follows from the assumptions $\lim\limits_{x\searrow 0^+}h(x)=0$ and $\lim\limits_{x\nearrow+\infty}h^2(x)(x-\mu_0)p_{\mu_0}(x)=0$. So
\begin{align}
\mathbb{E}_0[h(X)h'(X)(X-\mu_0)]&=\frac{\mathbb{E}[h^2(X)((X-\mu_0)^2/\sigma^2-1)]}{2}.\label{eq_proof_thm_GSM_normality_tnorm_mu}
\end{align}
The asymptotic variance is thus 
\begin{align*}
&\,\,\frac{\mathbb{E}_0\left[-h(X)(X-\mu_0)+\sigma^2 h'(X)\right]^2}{\mathbb{E}_0^2[h(X)]}\\
=&\,\,\frac{\mathbb{E}_0\left[h^2(X)(X-\mu_0)^2-2\sigma^2h^2(X)\left((X-\mu_0)^2/\sigma^2-1\right)/2+\sigma^4 {h'}^{2}(X)\right]}{\mathbb{E}_0^2[h(X)]}\\
=&\,\,\frac{\mathbb{E}_0[\sigma^2 h^2(X)+\sigma^4 {h'}^{2}(X)]}{\mathbb{E}_0^2[h(X)]}.
\end{align*}
The Cram\'er-Rao lower bound follows from taking the second derivative of $\log p_{\mu_0}$ with respect to $\mu_0$.
\end{proof}

\begin{proof}[Proof for Example \ref{theorem_GSM_normality_tnorm_sigma}]
We estimate $\theta\equiv 1/\sigma^2$. By (\ref{def_Gamma}) and (\ref{def_g}), 
\begin{align*}
\hat{\theta}&\equiv\Gamma(\boldsymbol{x})^{-1}g(\boldsymbol{x})\\
&=-\left[\sum_{i=1}^nh(X_i)t'(X_i)^2\right]^{-1}\left[\sum_{i=1}^nh(X_i)b'(X_i)t'(X_i)+h(X_i)t''(X_i)+h'(X_i)t'(X_i)\right]\\
&=\left[\sum_{i=1}^nh(X_i)(X_i-\mu)^2\right]^{-1}\left[\sum_{i=1}^nh(X_i)+h'(X_i)(X_i-\mu)\right].
\end{align*}
By Theorem \ref{theorem_exponential}, $\sqrt{n}(\hat{\theta}-\theta)\to_d \mathcal{N}(0,\varsigma^2)$, where
{\begin{align*}
\varsigma^2&\equiv\frac{\mathbb{E}_0\left[h(X)((X-\mu)^2/\sigma_0^2-1)-h'(X)(X-\mu)\right]^2}{\mathbb{E}_0^2[h(X)(X-\mu)^2]}\\
&=%
\frac{1}{\mathbb{E}_0^2[h(X)(X-\mu)^2]}\bigg(\mathbb{E}_0[h^2(X)(X-\mu)^4/\sigma_0^4-2h^2(X)(X-\mu)^2/\sigma_0^2+h^2(X)
\\&\quad\quad\quad\quad\quad\quad\quad+{h'}^{2}(X)(X-\mu)^2-2h(X)h'(X)(X-\mu)^3/\sigma_0^2+2h(X)h'(X)(X-\mu)\bigg).
\end{align*}}
By integration by parts, (suppressing the dependence of $p_{\sigma_0^2}$ on $\sigma_0^2$)
\begin{align*}
&\,\mathbb{E}_0[h(X)h'(X)(X-\mu)^3]\\
=&\,\int_0^{\infty}h'(x)h(x)(x-\mu)^3p(x)\d x\\
=&\int_0^{\infty}h(x)(x-\mu)^3p(x)\d h(x)\\
=&\,\left.h^2(x)(x-\mu)^3p(x)\right|_0^{\infty}-\int h(x)\d h(x)(x-\mu)^3p(x)\\
=&\,-\int h(x)h'(x)(x-\mu)^3p(x)\d x-3\int h^2(x)(x-\mu)^2p(x)\d x+\int h^2(x)\frac{(x-\mu)^4}{\sigma_0^2}p(x)\d x,
\end{align*}
where the last step follows from the assumptions $\lim\limits_{x\searrow 0^+}h(x)=0$ and $\lim\limits_{x\nearrow+\infty}h^2(x)(x-\mu)^3p_{\sigma_0^2}(x)=0$. Combining this with (\ref{eq_proof_thm_GSM_normality_tnorm_mu}) we get
\[\sqrt{n}(\hat{\theta}-\theta)\to_d\mathcal{N}(0,\varsigma^2)\sim\mathcal{N}\left(0,\frac{2\mathbb{E}_0[h^2(X)(X-\mu)^2/\sigma_0^2]+\mathbb{E}_0[{h'}^{2}(X-\mu)^2]}{\mathbb{E}_0^2[h(X)(X-\mu)^2]}\right),\]
and so by the delta method, for $\hat{\sigma}_k^2\equiv\hat{\theta}^{-1}$,
\[\sqrt{n}(\hat{\sigma}_{h}^2-\sigma_0^2)\to_d\mathcal{N}\left(0,\frac{2\sigma_0^6\mathbb{E}_0[h^2(X)(X-\mu)^2]+\sigma_0^8\mathbb{E}_0[{h'}^{2}(X-\mu)^2]}{\mathbb{E}_0^2[h(X)(X-\mu)^2]}\right).\]
The Cram\'er-Rao lower bound follows from taking the second derivative of $\log p_{\sigma_0^2}$ with respect to $\sigma_0^2$.
\end{proof}

\subsection{Proof of Theorems in Section \ref{sec:graph-models-trunc}}\label{Proof of Theorems in Section sec:graph-models-trunc}

\begin{proof}[Proof of Theorem \ref{thm_normalizability}]~\\
\emph{Case $b\neq 0$:} We use a strategy similar to that of \citet{ino16}.  Let $\mathcal{V}_1=\{\boldsymbol{v}:\|\boldsymbol{v}\|_1=1,\boldsymbol{v}\in\mathbb{R}_+^m\}$. Then by Fubini-Tonelli the normalizing constant is,
\begin{align*}
\phantom{=}&\int_{\mathbb{R}_+^m}\exp\left(\boldsymbol{\eta}^{\top}\frac{\boldsymbol{x}^b-\mathbf{1}_m}{b}-\frac{1}{2a}{\boldsymbol{x}^a}^{\top}\mathbf{K}\boldsymbol{x}^a\right)\d\boldsymbol{x}\\
=&\int_{\mathcal{V}_1}\int_{0}^{\infty}\exp\left(\boldsymbol{\eta}^{\top}\frac{z^b\boldsymbol{v}^b-\mathbf{1}_m}{b}-\frac{1}{2a}z^{2a}{\boldsymbol{v}^{a}}^{\top}\mathbf{K}\boldsymbol{v}^{a}\right)\d z\d\boldsymbol{v}\\
\propto&\int_{\mathcal{V}_1}\int_{0}^{\infty}\exp\Big(z^b(\boldsymbol{\eta}^{\top}\boldsymbol{v}^b)/b-z^{2a}({\boldsymbol{v}^{a}}^{\top}\mathbf{K}\boldsymbol{v}^{a})/(2a)\Big)\d z\d\boldsymbol{v}.
\end{align*}
Here $\mathcal{V}_1$ is compact and the inner integral, if finite, is continuous in $\boldsymbol{v}$. It thus suffices to show that the inner integral is finite at every single $\boldsymbol{v}\in\mathcal{V}_1$.

Fixing $\boldsymbol{v}\in\mathcal{V}_1$, write $A\equiv A(\boldsymbol{v})\equiv{\boldsymbol{v}^{a}}^{\top}\mathbf{K}\boldsymbol{v}^a/(2a)$ and $B\equiv B(\boldsymbol{v})\equiv(\boldsymbol{\eta}^{\top}\boldsymbol{v}^b)/b$. We need to show that
\[N(A,B,a,b)\equiv\int_0^{\infty}\exp\left(-Az^{2a}+Bz^b\right)\d z<+\infty.\]
Recall that (CC1) $\boldsymbol{v}^{\top}\mathbf{K}\boldsymbol{v}>0$ for all $\boldsymbol{v}\in\mathbb{R}_+^m\backslash\{\boldsymbol{0}\}$, so $A>0$.

\begin{enumerate}[(i)]
\item Suppose $B\leq 0$. Then $N(A,B,a,b)\leq \int_0^{\infty}\exp(-Az^{2a})\d z=A^{-a/2}\Gamma(1+1/(2a))$, a finite constant since $A>0$ and $a>0$.
\item Suppose $B>0$.  We first want  to bound $\exp(-Az^{2a}+Bz^b)\leq N_0\exp(-Az^{2a}/2)$ by some finite constant $N_0>0$, so that $N(A,B,a,b)\leq N_0\int_0^{\infty}\exp(-Az^{2a}/2)\d z$, a finite constant for $a>0$. Thus, it remains to give conditions so that $\exp(-Az^{2a}/2+Bz^b)$ is bounded by some finite constant $N_0$, which by continuity only requires a finite limit as $z\searrow 0$ and as $z\nearrow+\infty$. As $z\nearrow+\infty$, $Bz^b\nearrow+\infty$, while $-Az^{2a}/2\searrow -\infty$. We thus need $b<2a$ so that the sum of the two does not go to positive infinity. On the other hand, as $z\searrow 0$, $-Az^{2a}/2\nearrow 0$, so we need $b>0$, otherwise $z^b\nearrow+\infty$. In conclusion, we require that $2a>b>0$.
\end{enumerate}
It thus suffices to require (CC1) and (CC2) $2a>b>0$ to eliminate restrictions on $B$, and hence on $\boldsymbol{\eta}$. That is, $\boldsymbol{\eta}$ can take value in the entirety of $\mathbb{R}^m$.\\

\noindent\emph{Case $b=0$}: Again in (CC1) we assume $\boldsymbol{v}^{\top}\mathbf{K}\boldsymbol{v}>0$ for all $\boldsymbol{v}\in\mathbb{R}_+^m\backslash\{\boldsymbol{0}\}$. Since $\mathcal{V}_2\equiv\{\boldsymbol{v}:\|\boldsymbol{v}\|_2=1,\boldsymbol{v}\in\mathbb{R}_+^m\}$ is compact and $\boldsymbol{v}^{\top}\mathbf{K}\boldsymbol{v}$ is continuous in $\boldsymbol{v}$ and strictly positive on $\mathcal{V}_2$, the image of $\mathcal{V}_2$ under $\boldsymbol{v}^{\top}\mathbf{K}\boldsymbol{v}$ is a compact subset of $(0,\infty)$, i.e.~$N_{\mathbf{K}}\equiv\min_{\boldsymbol{v}\in\mathbb{R}_+^m\backslash\{\boldsymbol{0}\}}\boldsymbol{v}^{\top}\mathbf{K}\boldsymbol{v}/\boldsymbol{v}^{\top}\boldsymbol{v}\equiv\min_{\boldsymbol{v}\in\mathcal{V}_2}\boldsymbol{v}^{\top}\mathbf{K}\boldsymbol{v}>0$. We thus have


\begin{align*}
\phantom{=}&\int_{\mathbb{R}_+^m}\exp\left(\boldsymbol{\eta}^{\top}\log(\boldsymbol{x})-\frac{1}{2a}{\boldsymbol{x}^a}^{\top}\mathbf{K}\boldsymbol{x}^a\right)\d\boldsymbol{x}\\
\leq&\int_{\mathbb{R}_+^m}\exp\left(\boldsymbol{\eta}^{\top}\log(\boldsymbol{x})-\frac{N_{\mathbf{K}}}{2a}{\boldsymbol{x}^a}^{\top}\boldsymbol{x}^a\right)\d\boldsymbol{x}\\
=&\prod_{j=1}^m\int_{0}^{\infty}\exp\left(\eta_j\log(x_j)-\frac{N_{\mathbf{K}}}{2a}x_j^{2a}\right)\d x_j\\
=&\prod_{j=1}^m\left[\Gamma\left(\frac{\eta_j+1}{2a}\right)\frac{(N_{\mathbf{K}}/2a)^{-\frac{\eta_j+1}{2a}}}{2a}\right],
\end{align*}
where the integration follows by change of variable and requires $a>0$. Assuming $a>0$, the last quantity is finite if and only if $\boldsymbol{\eta}\succ-\mathbf{1}_m$, by definition of the gamma function.

In conclusion, given conditions (CC1) $\min_{\boldsymbol{v}\in\mathbb{R}_+^m\backslash\{\boldsymbol{0}\}}\boldsymbol{v}^{\top}\mathbf{K}\boldsymbol{v}>0$, (CC2) $2a>b>0$, and (CC3) $a>0$, $b=0$ and $\boldsymbol{\eta}\succ-\mathbf{1}_m$, the unnormalized density (\ref{eq:ab-density2}) has a finite normalizing constant when (CC1) and (CC2) both hold, or (CC1) and (CC3) both hold.

The centered settings, where the term involving $\boldsymbol{x}^b$ is excluded, can be considered as a special case of both (1) and (2) with $\boldsymbol{\eta}\equiv \boldsymbol{0}$, and thus (CC1) and $a>0$ are sufficient.
\end{proof}

\begin{proof}[Proof of Theorem \ref{thm_h}]
Recall assumptions (A1) and (A2):
\begin{enumerate}[label=(A\arabic*)]
\item $p_0(\boldsymbol{x})h_j(x_j)\partial_j\log p(\boldsymbol{x})\left|^{x_j\nearrow+\infty}_{x_j\searrow 0^+}=0\right.,\quad\forall\boldsymbol{x}_{-j}\in\mathbb{R}_{+}^{m-1},\quad \forall p\in\mathcal{P}_+$;
\item $\mathbb{E}_{p_0}\|\nabla\log p(\boldsymbol{X})\circ \boldsymbol{h}^{1/2}(\boldsymbol{X})\|_2^2<+\infty, \,\,\, \mathbb{E}_{p_0}\|(\nabla\log p(\boldsymbol{X})\circ\boldsymbol{h}(\boldsymbol{X}))'\|_1<+\infty,\,\,\,\forall p\in\mathcal{P}_+$.
\end{enumerate}
Let $\mathbf{K}_0$ and $\boldsymbol{\eta}_0$ be the true parameters so that $p_0\in\mathcal{P}_+$, with $\mathcal{P}_+$ corresponding to a parameter space in which all parameters satisfy the conditions for a finite normalizing constant. We now give sufficient conditions for $h$ to satisfy (A1) and (A2).\\

\noindent \emph{Conditions for (A1):}
Fix $j=1,\dots,m$ and $\boldsymbol{x}_{-j}\in\mathbb{R}_+^{m-1}$. We show that the conditions on $h_j$ imply that the limits go to $0$ as $x_j\nearrow+\infty$ and as $x_j\searrow 0^+$, which is stronger than (A1); in fact, from (\ref{A1_proof_ABC}) below, the limits cannot go to a nonzero finite constant assuming an $h$ with polynomial tail, since $a>0$ and $B_1\equiv\kappa_{0,jj}>0$ for all $j$. Now,
\begin{align}
& \phantom{\propto=} p_0(\boldsymbol{x})h_j(x_j)\partial_j\log p(\boldsymbol{x})\nonumber\\
& \propto h_j(x_j)\exp\left(-\frac{1}{2a}{\boldsymbol{x}^a}^{\top}\mathbf{K}_0\boldsymbol{x}^a+\boldsymbol{\eta}_0^{\top}\frac{\boldsymbol{x}^b-\mathbf{1}_m}{b}\right)\partial_j\left(-\frac{1}{2a}{\boldsymbol{x}^a}^{\top}\mathbf{K}\boldsymbol{x}^a+\boldsymbol{\eta}^{\top}\frac{\boldsymbol{x}^b-\mathbf{1}_m}{b}\right)\nonumber\\
& \propto h_j(x_j)\exp\left(-\frac{1}{a}(\boldsymbol{k}_{0,j,-j}^{\top}\boldsymbol{x}^{a}_{-j})x_j^a-\frac{\kappa_{0,jj}}{2a}x_j^{2a}+\eta_{0,j}\frac{x_j^b-1}{b}\right)\times\nonumber \\
&\hspace{2.43in}\left(-\boldsymbol{k}_{j,-j}^{\top}\boldsymbol{x}^{a}_{-j}x_j^{a-1}-\kappa_{jj}x_j^{2a-1}+\eta_j x_j^{b-1}\right)\nonumber\\
& \equiv h_j(x_j)\exp\left(\frac{A_1x_j^a}{a}+\frac{B_1x_j^{2a}}{2a}+C_1\frac{x_j^b-1}{b}\right)\left(A_2x_j^{a-1}+B_2x_j^{2a-1}+C_2x_j^{b-1}\right),\label{A1_proof_ABC}
\end{align}
where $A_1\equiv-\boldsymbol{k}_{0,j,-j}^{\top}\boldsymbol{x}^{a}_{-j}$, $A_2\equiv-\boldsymbol{k}_{j,-j}^{\top}\boldsymbol{x}^{a}_{-j}$, $B_1\equiv-\kappa_{0,jj}<0$ and $B_2\equiv-\kappa_{jj}<0$ by condition (CC1). Finally $C_1\equiv\eta_{0,j}$, $C_2\equiv\eta_j$.
\begin{enumerate}[(1)]
\item Let $x_j\nearrow+\infty$. If $b>0$, since $2a>b>0$ and $B_1<0$, the exponential term in (\ref{A1_proof_ABC}) decreases to $0$ exponentially and its reciprocal dominates any polynomial functions. Thus, the entire product goes to $0$ if $h_j(x_j)$ grows no faster than polynomially as $x_j\nearrow+\infty$. If $b=0$, the $C_1\log x_j$ term is again dominated by $B_1x_j^{2a}/(2a)$, and the same conclusion holds.
\item Let $x_j\searrow 0$.
\begin{enumerate}[(i)]
\item Let $b>0$. Then the exponential term in (\ref{A1_proof_ABC}) goes to constant $\exp(-C_1/b)$, and we only need 
\begin{equation}\label{A1_proof_sub}
\lim_{x_j\searrow 0^+}h_j(x_j)(A_2x_j^{a-1}+B_2x_j^{2a-1}+C_2x_j^{b-1})=0.
\end{equation}
\begin{itemize}
\item If $a>1$ and $b>1$, the second term in (\ref{A1_proof_sub}) is a polynomial with three terms having powers $\geq\min\{a-1,b-1\}$. The product goes to zero if and only if $h_j(x_j)=o(x_j^{\max\{1-a,1-b\}})$ as $x_j\searrow 0$. Note that this is satisfied by any $h_j$ that has a finite right limit at $0$.
\item If $a=1$ and $b\geq 1$, or $a\geq 1$ and $b=1$, then the second term in (\ref{A1_proof_sub}) is a polynomial of non-negative power plus a potentially nonzero constant. A 
sufficient condition for (\ref{A1_proof_sub}) is thus $\lim_{x_j\searrow 0}h_j(x_j)=0$.
\item If $a<1$ or $b<1$, then the second part in (\ref{A1_proof_sub}) is a polynomial having terms with negative degree $\geq\min\{a-1,b-1\}$. To counteract this a 
sufficient condition is $h_j(x_j)=o(x_j^{\max\{1-a,1-b\}})$.
\end{itemize}
In conclusion, $\lim_{x_j\searrow 0^+}p_0(\boldsymbol{x})h_j(x_j)\partial_j\log p(\boldsymbol{x})=0$ if and only if \[\lim_{x_j\searrow 0^+}h_j(x_j)/x_j^{\max\{1-a,1-b\}}=0.\]
\item Now assume $b=0$. Then, (\ref{A1_proof_ABC}) now becomes
\[h_j(x_j)\exp\left(A_1x_j^a/a+B_1x_j^{2a}/(2a)+C_1\log x_j\right)\big(A_2x_j^{a-1}+B_2x_j^{2a-1}+C_2/x_j\big).\]
With $C_1\log x_j$ dominating, the exponential part scales as $x_j^{C_1}$. We thus require
\[
\lim_{x_j\searrow 0^+}h_j(x_j)(A_2x_j^{a-1+C_1}+B_2x_j^{2a-1+C_1}+C_2x_j^{C_1-1})=0,\]
which by the previous discussion on (\ref{A1_proof_sub}) holds if and only if
\[\lim_{x_j\searrow 0^+}h_j(x_j)/x_j^{1-C_1}=0\]
since $1-a-C_1<1-C_1$. 
\end{enumerate}
\end{enumerate}
In summary, (A1) is satisfied if $h_j(x_j)$ grows at most polynomially as $x_j\nearrow +\infty$, and $\lim_{x_j\searrow 0^+}h_j(x_j)/x_j^{\max\{1-a,1-b\}}=0$ if $b>0$, or $\lim_{x_j\searrow 0^+}h_j(x_j)/x_j^{1-\eta_{0,j}}=0$ if $b=0$.\\

\noindent \emph{Conditions for (A2):} For (A2), we consider powers of $x$ as the $h$ functions for simplicity; conclusions for other functions that have the same tail behavior (big-O scaling) as $x\searrow0$ and $x\nearrow+\infty$ follow similarly. Sufficiency results for piecewise power functions follow by partitioning, and similarly for other functions $h$ whose function values and derivatives can be bounded by those of some piecewise power function (e.g.~truncated powers), since (A2) is on integrability of products involving positive powers of $h$ and $h'$.

Let $\mathbf{K}_0$ and $\boldsymbol{\eta}_0$ be the true parameters from the parameter space that satisfies the conditions for finite normalizing constant. By part (2) of the proof of Theorem \ref{thm_normalizability}, the assumption that $\mathbf{K}_0$ satisfies (CC1) implies that $\min_{\boldsymbol{v}\in\mathbb{R}_+^m\backslash\{\boldsymbol{0}\}}\boldsymbol{v}^{\top}\mathbf{K}_0\boldsymbol{v}/\boldsymbol{v}^{\top}\boldsymbol{v}\equiv N_{\mathbf{K}_0}>0$. Then we have the following decomposition
\begin{align*}
p_{\mathbf{K}_0,\boldsymbol{\eta}_0}(\boldsymbol{x})&\equiv\exp\left(-\frac{1}{2a}{\boldsymbol{x}^a}^{\top}\mathbf{K}_0\boldsymbol{x}^a+\boldsymbol{\eta}_0^{\top}\frac{\boldsymbol{x}^b-\mathbf{1}_m}{b}\right)\\
&\leq\prod_{j=1}^m\exp\left(-\frac{N_{\mathbf{K}_0}}{2a}x^{2a}_j+\eta_{0,j}\frac{x_j^b-1}{b}\right).
\end{align*}

Then for any other $\mathbf{K}$ and $\boldsymbol{\eta}$ in the parameter space, for the first part of (A2) it suffices to show for any $j=1,\dots,m$ that $D<\infty$, where
\begin{align*}
D&\equiv \int_{\mathbb{R}_+^m}\prod_{j=1}^m\exp\left(-\frac{N_{\mathbf{K}_0}}{2a}x^{2a}_j+\eta_{0,j}\frac{x_j^b-1}{b}\right)h_j(x_j)\times\\
&\hspace{2in}\bigg(-\kappa_{jj}x^{2a-1}_j-\sum_{i\neq j}\kappa_{ji}x_i^{a}x_j^{a-1}+\eta_j x_j^{b-1}\bigg)^2\d\boldsymbol{x}\\
&\geq \int_{\mathbb{R}_+^m}p_0(\boldsymbol{x})h_j(x_j)\left(\partial_j\log p(\boldsymbol{x})\right)^2\d \boldsymbol{x}.
\end{align*}
Note that 
\begin{align*}
&\phantom{=}\left(-\kappa_{jj}x^{2a-1}_j-\sum_{i\neq j}\kappa_{ji}x_i^{a}x_j^{a-1}+\eta_j x_j^{b-1}\right)^2\\
&=\kappa_{jj}^2x_j^{4a-2}+\sum_{i\neq j,\ell\neq j}\kappa_{ji}\kappa_{j\ell}x_i^{a}x_{\ell}^ax_j^{2a-2}+\eta_j^2x_j^{2b-2}+2\sum_{i\neq j}\kappa_{jj}\kappa_{ji}x_i^ax_j^{3a-2}
\\&\hspace{2.0in}-2\sum_{i\neq j}\kappa_{ji}\eta_j x_i^ax_j^{a+b-2}-2\kappa_{jj}\eta_jx_j^{2a+b-2}.
\end{align*}
Thus, plugging this back in the definition of $D$, we can split $D$ into a sum of six terms $D_1$ through $D_6$, each of which is a sum of terms of the form
\begin{align*}
\phantom{=}&\int_{\mathbb{R}_+^m}\prod_{k=1}^m\exp\left(-\frac{N_{\mathbf{K}_0}}{2a}x_k^{2a}+\eta_{0,k}\frac{x_k^b-1}{b}\right)h_j(x_j)x_i^{\text{pow}_i}x_{\ell}^{\text{pow}_{\ell}}x_j^{\text{pow}_j}\d \boldsymbol{x}\\
=&\prod_{k\neq j}\int_{0}^{\infty}\exp\left(-\frac{N_{\mathbf{K}_0}}{2a}x_k^{2a}+\eta_{0,k}\frac{x_k^b-1}{b}\right)x_k^{\mathrm{pow}_k}\d x_k
\\ &\hspace{1.5in}\times \int_{0}^{\infty}\exp\left(-\frac{N_{\mathbf{K}_0}}{2a}x_j^{2a}+\eta_{0,j}\frac{x_j^b-1}{b}\right)h_j(x_j)x_j^{\mathrm{pow}_j}\d x_j
\end{align*}
times a constant involving $\mathbf{K}$ and $\eta_j$, where $\mathrm{pow}_k\geq 0$ for each $k\neq j$. We have thus decomposed the integral into a product of univariate integrals. Note that
\[\int_{0}^{\infty}\exp\left(-\frac{N_{\mathbf{K}_0}}{2a}x_i^{2a}+\eta_{0,i}\frac{x_i^b-1}{b}\right)x_i^{\mathrm{pow}_i}\d x_i\]
is finite for all $\mathrm{pow}_i\geq 0$ regardless of whether $b$ is nonzero, since we assumed $\mathbf{K}_0$ and $\boldsymbol{\eta}_0$ to lie in the parameter space with a finite normalizing constant. Indeed, if $b>0$ then the terms in the exponential is a regular polynomial with positive degree and a negative leading term; if $b=0$ then integrability follows from $\eta_{0,i}+\mathrm{pow}_i\geq\eta_{0,i}>-1$. Thus, we only need to consider the univariate integral that involve the $x_j$ terms, namely
\[\int_{0}^{\infty}\exp\left(-\frac{N_{\mathbf{K}_0}}{2a}x_j^{2a}+\eta_{0,j}\frac{x_j^b-1}{b}\right)h_j(x_j)x_j^{\mathrm{pow}_j}\d x_j,\]
where $\mathrm{pow}_j$ takes value in $\{4a-2,2a-2,2b-2,3a-2,a+b-2,2a+b-2\}\subseteq[2\min\{a,b\}-2,4a-2]$. 
We split the integral into two parts over $[0,1]$ and $[1,\infty]$, respectively. 
\begin{itemize}
\item If $b>0$, on $[0,1]$ the exponential part is bounded above and below by positive constants, and for (A1) we require $h_j(x)=o(x^{1-\min\{a,b\}})$ as $x\searrow 0^+$, so the integrand is $o(x^{\min\{a,b\}-1})=o(x^{-1})$ and is thus integrable on $[0,1]$. The integrand on $[1,\infty)$ is integrable as  in (A1) we assume $h$ to grow at most polynomially. 
\item If $b=0$, $\mathrm{pow}_j\in[-2,4a-2]$ and the integrand becomes \[\exp(-N_{\mathbf{K}_0}x_j^{2a}/(2a))h_j(x_j)x_j^{\mathrm{pow}_j+\eta_{0,j}}.\] On $[0,1]$, (A1) requires $h_j(x)=o(x^{1-\min_j\eta_{0,j}})$, so $h_j(x_j)x_j^{\mathrm{pow}_j+\eta_{0,j}}=o(x^{-1})$ and the integrand is again integrable. Integrability on $[1,\infty)$ follows similarly to the case with $b>0$.
\end{itemize}

Now consider the second part of (A2). By definition $\mathbb{E}_{p_0}\|(\nabla\log p(\boldsymbol{X})\circ\boldsymbol{h}(\boldsymbol{X}))'\|_1$ equals
\begin{align*}
&\,\int_{\mathbb{R}_+^m}p_0(\boldsymbol{x})\sum_{j=1}^m\left|h_j'(X_j)\partial_j\log p(\boldsymbol{X})+h_j(X_j)\partial_j^2\log p(\boldsymbol{X})\right|\d\boldsymbol{x}\\
\leq&\,\sum_{j=1}^m\int_{\mathbb{R}_+^m}\exp\left(-\frac{N_{\mathbf{K}_0}}{2a}x^{2a}_j+\eta_{0,j}\frac{x_j^b-1}{b}\right)\left|h_j'(x_j)\left(-\kappa_{jj}x^{2a-1}_j-\sum_{i\neq j}\kappa_{ji}x_i^{a}x_j^{a-1}+\eta_j x_j^{b-1}\right)\right.\\
&\hspace{0.8in}\left.+h_j(x_j)\left(-\kappa_{jj}(2a-1)x^{2a-2}_j-\sum_{i\neq j}\kappa_{ji}(a-1)x_i^{a}x_j^{a-2}+(b-1)\eta_j x_j^{b-2}\right)\right|\d\boldsymbol{x}.
\end{align*}
By the triangle inequality and the fact that $h_j\geq 0$ and $h_j'\geq 0$, similar to the proof for the first part, for each $j$ the integral can be bounded by a sum of six integrals, each of the form
\[\mathrm{const}\times\int_{\mathbb{R}_+^m}\prod_{k=1}^m\exp\left(-\frac{N_{\mathbf{K}_0}}{2a}x_k^{2a}+\eta_{0,k}\frac{x_k^b-1}{b}\right)h_j(x_j)x_i^{\text{pow}_i}x_j^{\text{pow}_j}\d \boldsymbol{x},\]
or with $h_j$ replaced by $h_j'$. Finiteness thus follows from the same type of discussion by noting that $h_j(x)=o\left(x^{1-\min\{a,b\}}\right)$ and $h'_j(x)=o\left(x^{-\min\{a,b\}}\right)$.

We conclude that if the true and the proposed parameters give densities with finite normalizing constants, and if $h$ satisfies assumption (A1), then (A2) is automatically satisfied.

In the centered case where we assume $\boldsymbol{\eta}\equiv \boldsymbol{0}$, we only need $\lim_{x_j\searrow 0^+}h_j(x_j)/x_j^{1-a}=0$ as it is a special case with $b=2a$.
\end{proof}

\subsection{Proof of Theorems in Section \ref{Theory for Graphical Models}}\label{Proof of Theorems in Section Theory for Graphical Models}

\begin{proof}[Proof of Corollary \ref{cor_L2-consistency}]
By Theorem \ref{theorem_lin}, under assumptions in that theorem, the support of $\hat{\mathbf{\Psi}}$ is a subset of the true support of $\mathbf{\Psi}_0$, and $\|\hat{\mathbf{\Psi}}-\mathbf{\Psi}_0\|_{\infty}\leq\frac{c_{\boldsymbol{\Gamma}_0}}{2-\alpha}\lambda$. Since $\mathbf{\Psi}_0$ has $|S_0|$ nonzero entries, \[\mnorm{\hat{\mathbf{\Psi}}-\mathbf{\Psi}_0}_F=\left[\sum_{\mathbf{\Psi}_{0,jk}\neq 0}(\hat{\mathbf{\Psi}}_{jk}-\mathbf{\Psi}_{0,jk})^2\right]^{1/2}\leq\sqrt{|S_0|}\|\hat{\mathbf{\Psi}}-\mathbf{\Psi}_0\|_{\infty}\leq\frac{c_{\boldsymbol{\Gamma}_0}}{2-\alpha}\lambda\sqrt{|S_0|}.\]
Similarly, by the definition of matrix $\ell_{\infty}$-$\ell_{\infty}$ norm, 
\[\mnorm{\hat{\mathbf{\Psi}}-\mathbf{\Psi}_0}_{2}\leq\mnorm{\hat{\mathbf{\Psi}}-\mathbf{\Psi}_0}_{\infty}=\max_{j=1,\ldots,m}\sum_{k=1}^m|\hat{\mathbf{\Psi}}_{jk}-\mathbf{\Psi}_{0,jk}|\leq \frac{c_{\boldsymbol{\Gamma}_0}}{2-\alpha}\lambda d_{\mathbf{\Psi}_0}.\]
The result follows by also noting that $\mnorm{\hat{\mathbf{\Psi}}-\mathbf{\Psi}_0}_{2}\leq\mnorm{\hat{\mathbf{\Psi}}-\mathbf{\Psi}_0}_{F}$.
\end{proof}

\begin{proof}[Proof of Theorem \ref{thm_gau}]
The proof is based on Theorem \ref{theorem_lin} and a probabilistic bound on $\|\boldsymbol{\Gamma}_{\boldsymbol{\gamma}}-\boldsymbol{\Gamma}_0\|_{\infty}$, where in the case of centered Gaussian $\boldsymbol{\Gamma}=\mathrm{diag}(\mathbf{xx}^{\top},\dots,\mathbf{xx}^{\top})$. Denote $\boldsymbol{\Sigma}_0=\mathbf{K}_0^{-1}$. In particular, given $\tau>2$ we wish to show that for $\epsilon=80\sqrt{2}c_0\max_j(\Sigma_{0,jj})$, assuming $c_0\equiv\sqrt{(\tau\log m+\log 4)/n}<1/\sqrt{2}$,
\[\mathbb{P}\left(\left|n^{-1}\sum_{i=1}^nX_j^{(i)}X_{k}^{(i)}+\gamma_{1j}\mathds{1}_{\{j=k\}}-\mathbb{E}X_jX_k\right|>\epsilon\right)\leq m^{2-\tau},\]
and so the results follow from Theorem \ref{theorem_lin}.

By Lemma 1 of \citet{rav11}, since $X_j/\sqrt{\Sigma_{0,jj}}$ is Gaussian with mean $0$ and standard deviation $1$,
\[\P\left(\left|n^{-1}\sum_{i=1}^nX_j^{(i)}X_{k}^{(i)}-\mathbb{E}X_jX_k\right|>t\right)\leq 4\exp\left(-\frac{nt^2}{3200\max_j(\Sigma_{0,jj})^2}\right)\]
for $t\in (0,40\max_j(\Sigma_{0,jj}))$. Denote the event as $\mathcal{E}_{j,k}(t)$. Note that $\E X_j^2\leq\max_j\Sigma_{0,jj}=\epsilon/(80\sqrt{2}c_0)$. Then letting $t=\epsilon/2$ and conditioning on the complement of $\mathcal{E}_{j,j}(\epsilon/2)$, we have
\[n^{-1}\sum_{i=1}^n {X_j^{(i)}}^2\leq\mathbb{E}X_j^2+\epsilon/2\leq\frac{\epsilon}{2}\left(1+\frac{1}{40\sqrt{2}c_0}\right).\]
Thus, choosing $\gamma_{\ell j}=(\delta-1)\sum_{i=1}^n {X_j^{(i)}}^2/n$ for $\ell=1,\dots,m$ ($\boldsymbol{\Gamma}$ has $m$ identical blocks) with $1<\delta<1+\left(1+1/(40\sqrt{2}c_0)\right)^{-1}$, by the triangle inequality and a union bound we have
\[\mathbb{P}\left(\max_{j,k}\left|n^{-1}\sum_{i=1}^nX_j^{(i)}X_{k}^{(i)}+\gamma_{1j}\mathds{1}_{\{j=k\}}-\mathbb{E}X_jX_k\right|>\epsilon\right)\leq\P(\mathcal{E}_{j,k}(\epsilon/2))=m^{2-\tau}.\]
Since $\tau>2$,  it holds that $1+\left(1+1/(40\sqrt{2}c_0)\right)^{-1}=2-(1+40\sqrt{2}c_0)^{-1}$ is larger than $2-(1+80\sqrt{\log m/n})^{-1}\equiv C(n,m)$, so it is safe to choose any $\delta\in(1,C(n,m))$. Thus by the requirement on $\epsilon$, the theorem statement holds when $n>\max(c^*c_1^2d_{\mathbf{K}}^2,2)(\tau\log m+\log 4)$ with $c^*=12800\max_j(\Sigma_{0,jj})^2$.
\end{proof}

\begin{proof}[Proof of Theorem \ref{corollary1}]
The proof of Theorem \ref{theorem_lin} from \citet{lin16} does not rely on the fact that the original $\boldsymbol{\Gamma}$ is an unbiased estimator for the population $\boldsymbol{\Gamma}_0$, but instead only requires one to bound $\|\boldsymbol{\Gamma}-\boldsymbol{\Gamma}_0\|_{\infty}$. Thus, for $\boldsymbol{\Gamma}_{\boldsymbol{\gamma}}=\boldsymbol{\Gamma}+\mathrm{diag}(\boldsymbol{\gamma})$, by Theorem \ref{theorem_lin} it suffices to prove that for any $\tau>3$, we can bound $\|\boldsymbol{\Gamma}(\mathbf{x})+\mathrm{diag}(\boldsymbol{\gamma}(\mathbf{x}))-\boldsymbol{\Gamma}_0\|_{\infty}$ by some $\epsilon_1$ and $\|\boldsymbol{g}(\mathbf{x})-\boldsymbol{g}_0\|_{\infty}$ by some $\epsilon_2$, uniformly with probability $1-m^{3-\tau}$. 
Recall from (\ref{eq_noncentered_gamma}) that the $j^{\mathrm{th}}$ block of $\boldsymbol{\Gamma}_{\boldsymbol{\gamma}}\in\mathbb{R}^{m^2\times m^2}$ has $(k,\ell)$-th entry 
\[n^{-1}\sum_{i=1}^n X_k^{(i)}X_{\ell}^{(i)}h_{j}\left(X_{j}^{(i)}\right)+\gamma_{kj}\cdot\mathds{1}_{\{k=\ell\}}.\]
The entry in $\boldsymbol{g}\in\mathbb{R}^{m^2}$ (obtained by linearizing a $m\times m$ matrix) corresponding to $(j,k)$ is 
\[n^{-1}\sum_{i=1}^nX^{(i)}_kh_j'\left(X^{(i)}_j\right)+n^{-1}\mathds{1}_{\{j=k\}}\sum_{i=1}^nh_j\left(X^{(i)}_j\right).\]
Denote $M\equiv\max_j\sup_{x>0}h_j(x)$, $M'\equiv\max_j\sup_{x>0}h_j'(x)$, and $c_{\boldsymbol{X}}\equiv 2\max_j(2\sqrt{\Sigma_{jj}}+\sqrt{e}\,\mathbb{E}_0 X_j)$. Using results for sub-Gaussian random variables from Lemma \ref{lemma_sub-norms}.2 in Appendix \ref{AppA}, we have for any $t_{1}>0$,
\begin{small}
\begin{align*}
\P\left(\left|n^{-1}\sum_{i=1}^n X_k^{(i)}X_{\ell}^{(i)}h_{j}\left(X_j^{(i)}\right)-\mathbb{E}_0 X_{k}X_{\ell}h_{j}(X_j)\right|>t_{1}\right)&\leq2\exp\left(-\min\left(\frac{nt_{1}^2}{2M^2c_{\boldsymbol{X}}^4},\frac{nt_{1}}{2Mc_{\boldsymbol{X}}^2}\right)\right).
\end{align*}
\end{small}
Thus, choosing $\epsilon_{1}\equiv 2Mc_{\boldsymbol{X}}^2c_{n,m}$, where $c_{n,m}\equiv\max\left\{\frac{2(\log m^{\tau}+\log 6)}{n},\sqrt{\frac{2(\log m^{\tau}+\log 6)}{n}}\right\}$, for $\gamma_{kj}\leq \epsilon_{1}/2$, we have
\begin{align}
&\,\P\left(\left|n^{-1}\sum_{i=1}^n X_k^{(i)}X_{\ell}^{(i)}h_{j}\left(X_j^{(i)}\right)+\gamma_{kj}\mathds{1}_{\{k=\ell\}}-\mathbb{E}_0 X_{k}X_{\ell}h_{j}(X_j)\right|>\epsilon_{1}\right)\nonumber\\
\leq&\,\P\left(\left|n^{-1}\sum_{i=1}^n X_k^{(i)}X_{\ell}^{(i)}h_{j}\left(X_j^{(i)}\right)-\mathbb{E}_0 X_{k}X_{\ell}h_{j}(X_j)\right|>\epsilon_{1}/2\right)\label{eq_Eepsilon}\\
\leq &\,2\exp\left(-\min\left(\frac{n\epsilon_{1}^2}{8M^2c_{\boldsymbol{X}}^4},\frac{n\epsilon_{1}}{4Mc_{\boldsymbol{X}}^2}\right)\right)\leq\frac{1}{3m^{\tau}}\label{eq_ep1_gamma}.
\end{align}
Denote the event inside the probability in (\ref{eq_Eepsilon}) as $\mathcal{E}_{k,\ell,j}(\epsilon_1/2)$.\\
By definition, 
\[c_{\boldsymbol{X}}^2=4\max_k\left(4\Sigma_{kk}+4\sqrt{e}\sqrt{\Sigma_{kk}}\,\mathbb{E}_0X_k+e(\mathbb{E}_0X_k)^2\right)\geq 4e\max_k\left(\Sigma_{kk}+(\mathbb{E}_0X_k)^2\right).\] 
By Lemmas \ref{lem_subgau_subexp}.2 and \ref{lemma_sub-norms}.1 from Appendix \ref{AppA}, $\var(X_k)\leq\Sigma_{kk}$, so $c_{\boldsymbol{X}}^2\geq 4e\max_k \mathbb{E}_0 X_k^2\geq 4e\mathbb{E}_0X_k^2 h_{j}(X_j)/M$. Thus, setting $\epsilon_1=2Mc_{\boldsymbol{X}}^2c_{n,m}$, on the complement of $\mathcal{E}_{k,k,j}(\epsilon_1/2)$ we have
\[n^{-1}\sum_{i=1}^n{X_k^{(i)}}^2h_j\left(X_j^{(i)}\right)\leq \mathbb{E}_0X_k^2h_j(X_j)+\epsilon_1/2\leq \frac{\epsilon_1}{2}\left(1+\frac{1}{4ec_{n,m}}\right).\]
Then 
\begin{equation}\label{eq_find_gamma}
\frac{1}{1+1/(4ec_{n,m})}\frac{1}{n}\sum_{i=1}^n{X_k^{(i)}}^2h_j\left(X_j^{(i)}\right)\leq\epsilon_1/2
\end{equation}
on the complement of $\mathcal{E}_{k,k,j}(\epsilon_1/2)$, again with $c_{n,m}\equiv\max\bigg\{\frac{2(\log m^\tau+\log 6)}{n},\sqrt{\frac{2(\log m^{\tau}+\log 6)}{n}}\bigg\}.$
Note that the multiplier on the left of (\ref{eq_find_gamma}) is increasing in $c_{n,m}$, and that $2(\log m^{\tau}+\log 6)>6\log m$ by the assumption that $\tau>3$. 
Thus, if we let 
\[\gamma_{kj}\equiv\frac{1}{1+1/\left(4e\max\left\{6\log m/n,\sqrt{6\log m/n}\right\}\right)}\frac{1}{n}\sum_{i=1}^n{X_k^{(i)}}^2h_j\left(X_j^{(i)}\right),\]
which is just a constant multiple of the $(k,k)$-th entry of $\boldsymbol{\Gamma}_j$ itself, with the constant explicitly calculable and a function of $p$ and $n$ only, then for $k=\ell$
\[\P\left(\left|n^{-1}\sum_{i=1}^n X_k^{(i)}X_{\ell}^{(i)}h_{j}\left(X_j^{(i)}\right)+\gamma_{kj}-\mathbb{E}_0 X_{k}X_{\ell}h_{j}(X_j)\right|\geq\epsilon_1\right)\leq\P(\mathcal{E}_{k,\ell,j}(\epsilon_1/2))\leq\frac{1}{3m^{\tau}}.\]
Since this also holds for $k\neq\ell$ without the $\gamma_{kj}$ term, by a union bound over $m^3$ events,
\begin{equation}\label{eq_ep1}
\P\left(\max_{j,k,\ell}\left|n^{-1}\sum_{i=1}^n X_k^{(i)}X_{\ell}^{(i)}h_{j}\left(X_j^{(i)}\right)+\gamma_{kj}\mathds{1}_{\{k=\ell\}}-\mathbb{E}_0 X_{k}X_{\ell}h_{j}(X_j)\right|\geq\epsilon_1\right)\leq\frac{1}{3m^{\tau-3}}.
\end{equation}
Now, on the other hand, Lemma \ref{lemma_sub-norms}.1 and Hoeffding's inequality give for any $t_{2,1},t_{2,2}>0$ that
\begin{align*}
\P\left(\left|n^{-1}\sum_{i=1}^nX^{(i)}_kh_j'\left(X^{(i)}_j\right)-\mathbb{E}_0 X_kh_j'(X_j)\right|\geq t_{2,1}\right)&\leq2\exp\left(-\frac{nt_{2,1}^2}{2M{'^2}c_{\boldsymbol{X}}^2}\right),\\
\P\left(\left|n^{-1}\sum_{i=1}^nh_j\left(X^{(i)}_j\right)-\mathbb{E}_0h_j(X_j)\right|\geq t_{2,2}\right)&\leq2\exp\left(-2nt_{2,2}^2/M^2\right).
\end{align*}
Choosing $\epsilon_{2,1}\equiv \sqrt{2}M'c_{\boldsymbol{X}}\sqrt{\frac{\log m^{\tau-1}+\log 6}{n}}$, $\epsilon_{2,2}\equiv M\sqrt{\frac{\log m^{\tau-2}+\log 6}{2n}}$ and taking union bounds over $m^2$, and $m$ events, respectively, we have
\begin{align}
\P\left(\max_{j,k}\left|n^{-1}\sum_{i=1}^nX^{(i)}_kh_j'\left(X^{(i)}_j\right)-\mathbb{E}_0 X_kh_j'(X_j)\right|\geq\epsilon_{2,1}\right)&\leq\frac{1}{3m^{\tau-3}},\label{eq_ep21}\\
\P\left(\max_j\left|n^{-1}\sum_{i=1}^nh_j\left(X^{(i)}_j\right)-\mathbb{E}_0h_j(X_j)\right|\geq\epsilon_{2,2}\right)&\leq\frac{1}{3m^{\tau-3}}.\label{eq_ep22}
\end{align}
Hence, by (\ref{eq_ep1}) (\ref{eq_ep21}) (\ref{eq_ep22}), with probability at least $1-m^{3-\tau}$, $\|\boldsymbol{\Gamma}_{\boldsymbol{\gamma}}(\mathbf{x})-\boldsymbol{\Gamma}_0\|_{\infty}<\epsilon_1$ and $\|\boldsymbol{g}(\mathbf{x})-\boldsymbol{g}_0\|_{\infty}<\epsilon_2\equiv\epsilon_{2,1}+\epsilon_{2,2}.$ Consider any $\tau>3$, and let
\begin{align*}
c_2&\equiv\frac{6}{\alpha}c_{\boldsymbol{\Gamma}_0},
\\
n&>\max\{2M^2c_{\boldsymbol{X}}^4c_2^2d_{\mathbf{K}_0}^2(\tau\log m+\log 6), 2Mc_{\boldsymbol{X}}^2c_2d_{\mathbf{K}_0}(\tau\log m+\log 6)\},\\
\lambda&>\frac{3(2-\alpha)}{\alpha}\max\{c_{\mathbf{K}_0}\epsilon_1,\epsilon_2\}\\
&\equiv\frac{3(2-\alpha)}{\alpha}\max\left\{4Mc_{\mathbf{K}_0}c_{\boldsymbol{X}}^2\frac{(\log m^{\tau}+\log 6)}{n},\right.\\
&\quad\left.2Mc_{\mathbf{K}_0}c_{\boldsymbol{X}}^2\sqrt{\frac{2(\log m^{\tau}+\log 6)}{n}},\sqrt{2}M'c_{\boldsymbol{X}}\sqrt{\frac{\log m^{\tau-1}+\log 6}{n}}+M\sqrt{\frac{\log m^{\tau-2}+\log 6}{2n}}\right\}.
\end{align*}
Then $d_{\mathbf{K}_0}\epsilon_1\leq\alpha/(6c_{\boldsymbol{\Gamma}_0})$ and the results follow from Theorem \ref{theorem_lin}.
\end{proof}

\begin{proof}[Proof of Theorem \ref{corollary2}]
Similar to the proof of Theorem \ref{corollary1}, by Theorem \ref{theorem_lin} it suffices to prove that for any $\tau>3$, we can bound $\|\boldsymbol{\Gamma}_{\boldsymbol{\gamma}}(\mathbf{x})-\boldsymbol{\Gamma}_0\|_{\infty}$ by some $\epsilon_1$ and $\|\boldsymbol{g}(\mathbf{x})-\boldsymbol{g}_0\|_{\infty}$ by some $\epsilon_2$, uniformly with probability $1-m^{3-\tau}$. 
Recall that $\boldsymbol{\Gamma}\in\mathbb{R}^{(m^2+m)\times (m^2+m)}$ is a rearrangement of $\boldsymbol{\Gamma}^{(*)}$, which is in turn formed by $\boldsymbol{\Gamma}_{11}\in\mathbb{R}^{m^2\times m^2}$, $\boldsymbol{\Gamma}_{12}\in\mathbb{R}^{m^2\times m}$, $\boldsymbol{\Gamma}_{12}^{\top}$ and $\boldsymbol{\Gamma}_{22}\in\mathbb{R}^{m\times m}$, all of which are block-diagonal with $m$ blocks.\\
The $j^{\mathrm{th}}$ block of $\boldsymbol{\Gamma}_{11}\in\mathbb{R}^{m^2\times m^2}$ has $(k,\ell)$-th entry 
\[n^{-1}\sum_{i=1}^n X_k^{(i)}X_{\ell}^{(i)}h_{j}\left(X_{j}^{(i)}\right),\]
the $k^{\mathrm{th}}$ entry in the $j^{\mathrm{th}}$ block of $\boldsymbol{\Gamma}_{12}$ is 
\[-n^{-1}\sum_{i=1}^{n}X_k^{(i)}h_j\left(X_j^{(i)}\right),\]
the $j^{\mathrm{th}}$ diagonal entry of $\boldsymbol{\Gamma}_{22}$ is 
\[n^{-1}\sum_{i=1}^{n}h_j\left(X_j^{(i)}\right).\]
On the other hand, $\boldsymbol{g}\in\mathbb{R}^{(m^2+m)}$ is a rearrangement of $\boldsymbol{g}^{(*)}\equiv [\boldsymbol{g}_1^{\top},\boldsymbol{g}_2^{\top}]^{\top}$, where the entry in $\boldsymbol{g}_1\in\mathbb{R}^{m^2}$ (obtained by linearizing a $m\times m$ matrix) corresponding to $(j,k)$, is 
\[n^{-1}\sum_{i=1}^nX^{(i)}_kh_j'\left(X^{(i)}_j\right)+n^{-1}\mathds{1}_{\{j=k\}}\sum_{i=1}^nh_j\left(X_j^{(i)}\right),\]
while the $j$-th component of $\boldsymbol{g}_2\in\mathbb{R}^m$ is 
\[-n^{-1}\sum_{i=1}^nh_j'\left(X_j^{(i)}\right).\]
Recalling that the bounds in Lemma \ref{lemma_sub-norms} also hold when $\boldsymbol{\mu}\neq 0$, we may then use bounds similar to those in the proof of Theorem \ref{corollary1}, and use union bounds to arrive at analogous consistency results, modulus different constants. The amplifiers $\boldsymbol{\gamma}$ can be incorporated analogously.
\end{proof}

\section{Auxiliary Lemmas and Definitions}
\label{AppA}

In this appendix, to simplify notation, when it is clear from the context, the operator $\mathbb{E}$ is defined as the expectation under the true distribution, unless otherwise noted.

\begin{definition}[Sub-Gaussian and Sub-Exponential Variables]~\\
The \emph{sub-Gaussian} $(r=2)$ and \emph{sub-exponential} $(r=1)$ \emph{norms} of a random variable are 
\[\|X\|_{\psi_r}\equiv\sup_{q\geq 1}q^{-1/r}(\E|X|^{rq})^{1/(rq)}\equiv\sup_{q\geq 1}q^{-1/r}\|X\|_{rq}.\]
If $\|X\|_{\psi_2}<\infty$ we say $X$ is \emph{sub-Gaussian}; if $\|X\|_{\psi_1}<\infty$ we call $X$ \emph{sub-exponential}.
For a \emph{zero-mean} sub-Gaussian random variable $X$ also define the \emph{sub-Gaussian parameter}
\[\tau(X)=\inf\{\tau\geq 0:\mathbb{E}\exp(tX)\leq\exp(\tau^2 t^2/2),\,\forall t\in\R\}.\]
\end{definition}

The definition of sub-Gaussian norm here allows for a non-centered
variable and differs from the one in \citet{ver12}, which uses
$\|X\|_q$.  Instead, it coincides with $\theta_2$ in
\citet{bul00}. The sub-Gaussian parameter is defined  as in \citet{bul00} and the sub-exponential norm as in \citet{ver12}.

\begin{lemma}[Properties of Sub-Gaussian and Sub-Exponential Variables]\label{lem_subgau_subexp}
\noindent
\begin{enumerate}[1)]
\item \label{lem_subgau_subexp_norm} For any $X$ and $r=1,2$, $\|X-\mathbb{E} X\|_{\psi_r}\leq 2\|X\|_{\psi_r}$ and $\|X\|_{\psi_r}\leq\|X-\mathbb{E} X\|_{\psi_r}+|\mathbb{E} X|$, as long as the expectation and norms are finite. 
\item\label{lem_subgau_subexp_equivalence} \citep{bul00} $\tau(X)$ is
  a norm on the space of all zero-mean sub-Gaussian variables; so
  $\tau(X+Y)\leq\tau(X)+\tau(Y)$.
If $X$ is zero-mean sub-Gaussian, then $\mathrm{var}(X)\leq\tau^2(X)$, $\|X\|_{\psi_2}\leq2\tau(X)/\sqrt{e}$, $\tau(X)\leq\sqrt{e}\|X\|_{\psi_2}$.
If $X_1,\ldots,X_n$ are i.i.d.~zero-mean sub-Gaussian, $\tau\left(n^{-1}\sum_{i=1}^nX_i\right)\leq n^{-1/2}\tau(X_i)$.
\item\label{lem_subgau_subexp_product} If 
$X_1$ and $X_2$ are sub-Gaussian (not necessarily independent) with $\|X_1\|_{\psi_2}\leq K_1$ and $\|X_2\|_{\psi_2}\leq K_2$, then $X_1X_2$ is sub-exponential with $\|X_1X_2\|_{\psi_1}\leq K_1K_2$.
\item\label{lem_subgau_subexp_bounded_moments} \citep{bul00} If $X$ is
  zero-mean sub-Gaussian and $q>0$, then
\[\E|X|^q\leq 2(q/e)^{q/2}\tau^q(X).\]
\item\label{lem_subgau_subexp_gaubound} \citep{bul00} If $X_1,\ldots,X_n$ are independent zero-mean, sub-Gaussian variables, then for any $\epsilon>0$,
\begin{align*}
\P(|X_1|\geq\epsilon)&\leq 2\exp\left(-\frac{\epsilon^2}{2\tau^2(X_1)}\right),\\
\P\left(\left|n^{-1}\sum_{i=1}^nX_i\right|>\epsilon\right)&\leq 2\exp\left(-\frac{n\epsilon^2}{2\max_i\tau^2(X_i)}\right).
\end{align*}
\item\label{lem_subgau_subexp_expbound} \citep{ver12} If $X_1,\ldots,X_n$ are independent zero-mean sub-exponential random variables with $K\geq\max_i\|X_i\|_{\psi_1}$, 
then for any $\epsilon>0$,
\begin{align*}
\P(|X_1|\geq\epsilon)&\leq 2\exp\left(-\min\left(\frac{\epsilon^2}{8e^2K^2},\frac{\epsilon}{4eK}\right)\right),\\
\P\left(\left|n^{-1}\sum_{i=1}^nX_i\right|\geq\epsilon\right)&\leq 2\exp\left(-\min\left(\frac{n\epsilon^2}{8e^2K^2},\frac{n\epsilon}{4e K}\right)\right).
\end{align*}
\item\label{lem_subgau_subexp_moment_based} \citep{bou13} If for $X_i$ i.i.d.~there exists some $B>0$ such that
\[\sup_{q\geq 2}\left(\frac{\E|X|^q}{q!}\right)^{1/q}\leq B/2\]
then for all $\epsilon>0$,
\[\P\left(\left|n^{-1}\sum_{i=1}^n(X_i-\mathbb{E} X_i)\right|\geq\epsilon\right)\leq 2\exp\left(-\min\left(\frac{n\epsilon^2}{2B^2},\frac{n\epsilon}{2B}\right)\right).\]
\end{enumerate}
\end{lemma}
\begin{proof}

  \noindent 1)
  For $r=1,2$, by the triangle inequality, $\|X-\mathbb{E}
  X\|_{\psi_r}\leq\|X\|_{\psi_r}+\|\mathbb{E}
  X\|_{\psi_r}=\|X\|_{\psi_r}+|\mathbb{E}
  X|\leq\|X\|_{\psi_r}+\E|X|\leq 2\|X\|_{\psi_r}$, where in the last
  step we used the definition of $\|\cdot\|_{\psi_r}$ with $q=1$ for
  $r=1$ and $\E|X|\leq(\E|X|^2)^{1/2}$ with $q=2$ for $r=2$. On the
  other hand, $\|X\|_{\psi_r}\leq\|X-\mathbb{E}
  X\|_{\psi_r}+\|\mathbb{E} X\|_{\psi_r}=\|X-\mathbb{E}
  X\|_{\psi_r}+|\mathbb{E} X|$.
\smallskip

  \noindent 2)
  These follow from Theorems 1.2 and 1.3 and Lemmas 1.2 and 1.7 from
  \citet{bul00}, and $\sqrt[4]{3.1}e^{9/16}/\sqrt{2}\approx
  1.6467\leq1.6487\approx\sqrt{e}$.
\smallskip

  \noindent 3)
  By H\"older's inequality (or Cauchy-Schwarz),
\begin{align*}
\|X_1X_2\|_{\psi_1}&=\sup_{q\geq 1}q^{-1}(\E|X_1X_2|^q)^{1/q}=\sup_{q\geq 1}q^{-1}(\E|X_1^qX_2^q|)^{1/q}\\
&\leq\sup_{q\geq 1}q^{-1}\left[(\E|X_1|^{2q})^{1/2}(\E|X_2|^{2q})^{1/2}\right]^{1/q}\\
&\leq\sup_{q\geq 1}\left[q^{-1/2}(\E|X_1|^{2q})^{1/2q}\right]\sup_{q\geq 1}\left[q^{-1/2}(\E|X_2|^{2q})^{1/2q}\right]\\
&=\|X_1\|_{\psi_2}\|X_2\|_{\psi_2}\leq K_1K_2.
\end{align*}
\smallskip

\noindent 4-6) These are Lemma 1.4 and Theorem 1.5 in \citet{bul00},
and a consequence of Corollary 5.17 in \citet{ver12}.
\smallskip

\noindent
7) By Theorem 2.10 of \citet{bou13} wherein we let $v\equiv nB^2/2$ and $c\equiv B/2$, we have 
\[\P\left(\left|n^{-1}\sum_{i=1}^n(X_i-\mathbb{E} X_i)\right|\geq\epsilon\right)\leq 2\exp\left(-\frac{n\epsilon^2}{B^2+B\epsilon}\right)\]
for all $\epsilon>0$. (Theorem 2.10 gives an one-sided bound; bound for the other side is obtained by taking $X_i=-X_i$). The inequality follows by splitting into cases $\epsilon\leq B$ and $\epsilon>B$.
\end{proof}

\begin{lemma}\label{lemma_sub-norms}
Suppose $\boldsymbol{X}$ follows a truncated normal distribution on $\mathbb{R}_+^m$ with parameters $\boldsymbol{\mu}$ and $\boldsymbol{\Sigma}=\mathbf{K}^{-1}\succ\mathbf{0}$. Let $\boldsymbol{X}^{(1)},\ldots,\boldsymbol{X}^{(n)}$ be i.i.d.~copies of $\boldsymbol{X}$, with $j$-th component of the $i$-th copy being $X^{(i)}_j$. Then 
\begin{enumerate}
\item For $j=1,\ldots,p$, $\tau(X_j-\mathbb{E} X_j)\leq \sqrt{\Sigma_{jj}}$. That is, the sub-Gaussian parameter of any marginal distribution of $\boldsymbol{X}$, after centering, is bounded by the square root of its corresponding diagonal entry in the covariance parameter $\boldsymbol{\Sigma}$. Then, for any $\epsilon>0$,
\[\P\left(\left|n^{-1}\sum_{i=1}^nX^{(i)}_j-\mathbb{E} X_j\right|>\epsilon\right)\leq2\exp\left(-\frac{n\epsilon^2}{2\Sigma_{jj}}\right).\]
In particular, if $h_0$ is a function bounded by $M_0$, then for any $\epsilon>0$,
\begin{align*}
\hspace{-0.1in}\P\left(\left|n^{-1}\sum_{i=1}^nX^{(i)}_jh_0\left(X^{(i)}_k\right)-\mathbb{E} X_jh_0(X_k)\right|\geq\epsilon\right)&\leq2\exp\left(-\frac{n\epsilon^2}{8M_0^{2}(2\sqrt{\Sigma_{jj}}+\sqrt{e}\,\mathbb{E} X_j)^2}\right),\\
\tau\left(n^{-1}\sum_{i=1}^n X^{(i)}_jh_0\left(X^{(i)}_k\right)-\mathbb{E} X_jh_0(X_k)\right)&\leq \frac{2M_0}{\sqrt{n}}\left(2\sqrt{\Sigma_{jj}}+\sqrt{e}\,\mathbb{E} X_j\right),\\
\left\|n^{-1}\sum_{i=1}^n X^{(i)}_jh_0\left(X_{k}^{(i)}\right)-\mathbb{E} X_jh_0(X_k)\right\|_{\psi_2}&\leq \frac{4M_0}{\sqrt{en}}\left(2\sqrt{\Sigma_{jj}}+\sqrt{e}\,\mathbb{E} X_j\right).
\end{align*}
\item For $j,k,\ell\in\{1,\ldots,p\}$, if $h_0$ is a function bounded by $M_0$, then 
\begin{equation}\label{lemma_sub-norms_bound2}
\|X_jX_kh_0(X_{\ell})-\mathbb{E} X_jX_kh_0(X_{\ell})\|_{\psi_1}\leq \frac{M_0}{2e}c_{\boldsymbol{X}}^2,
\end{equation} 
where $c_{\boldsymbol{X}}\equiv 2\max_j(2\sqrt{\Sigma_{jj}}+\sqrt{e}\,\mathbb{E}X_j)$. In particular, for any $\epsilon>0$,
\begin{multline*}
\P\left(\left|n^{-1}\sum_{i=1}^n X^{(i)}_jX^{(i)}_kh_0\left(X_{\ell}^{(i)}\right)-\mathbb{E} X_{j}X_kh_0(X_{\ell})\right|>\epsilon\right)\\
\leq2\exp\left(-\min\left(\frac{n\epsilon^2}{2M_0^2c_{\boldsymbol{X}}^4},\frac{n\epsilon}{2M_0c_{\boldsymbol{X}}^2}\right)\right).
\end{multline*}
\end{enumerate}
\end{lemma}

\begin{proof}[Proof of Lemma \ref{lemma_sub-norms}]

  \noindent 
1. Without loss of generality choose $j=1$. By the definition of sub-Gaussian parameters, we need to show that for all $t\in\R$,
\[\mathbb{E}\exp(t X_1)\leq \exp\left(t^2\Sigma_{11}/2+t\mathbb{E} X_1\right),\]
which is equivalent to
\begin{equation}\label{eq_proof_sub-norms_1}
t^2\Sigma_{11}/2+t\mathbb{E} X_1-\log\mathbb{E}\exp(tX_1)\geq 0\quad\forall t\in\R.
\end{equation}
Since the left-hand side of (\ref{eq_proof_sub-norms_1}) equals $0$ at $t=0$, it suffices to show that its derivative,
\begin{equation}\label{eq_proof_sub-norms_2}
t\Sigma_{11}+\mathbb{E} X_1-\frac{\d\log\mathbb{E} \exp(t X_1)}{\d t}=t\Sigma_{11}+\mathbb{E} X_1-\frac{\frac{\d\mathbb{E} \exp(t X_1)}{\d t}}{\mathbb{E}\exp(t X_1)},
\end{equation}
is non-negative on $(0,\infty)$ and non-positive on $(-\infty,0)$. By properties of moment-generating functions, $\frac{\d}{\d t}\mathbb{E}\exp(t X_1)$ evaluated at $t=0$ equals $\mathbb{E} X_1$, so (\ref{eq_proof_sub-norms_2}) equals $0$ at $t=0$. It in turn suffices to show the derivative of (\ref{eq_proof_sub-norms_2}), namely
\begin{equation}\label{eq_proof_sub-norms_3}
\Sigma_{11}-\frac{\d^2\log\mathbb{E}\exp(tX_1)}{\d t^2}
\end{equation}
is non-negative in $t\in\R$.

Given any vector $\boldsymbol{v}\in\mathbb{R}^p$, define $\mathbb{R}^p_+-\boldsymbol{v}\equiv\{\boldsymbol{u}-\boldsymbol{v}:\boldsymbol{u}\in\mathbb{R}^p_+\}$. By \citet{tal61}, denoting the first column of $\boldsymbol{\Sigma}$ as $\boldsymbol{\Sigma}_{1}$, the moment-generating function of the marginal distribution of $X_1$ is 
\[\frac{\int_{\mathbb{R}_+^p-(\boldsymbol{\mu}+t\boldsymbol{\Sigma}_{1})}\exp\left(-\frac{1}{2}\boldsymbol{x}^{\top}\boldsymbol{\Sigma}^{-1}\boldsymbol{x}\right)\d\boldsymbol{x}}{\int_{\mathbb{R}_+^p-\boldsymbol{\mu}}\exp\left(-\frac{1}{2}\boldsymbol{x}^{\top}\boldsymbol{\Sigma}^{-1}\boldsymbol{x}\right)\d\boldsymbol{x}}\exp\left(t\mu_1+\frac{1}{2}t^2\Sigma_{11}\right).\]
(\ref{eq_proof_sub-norms_3}) thus becomes
\[-\frac{\d^2}{\d t^2}\log\int_{\mathbb{R}_+^p-(\boldsymbol{\mu}+t\boldsymbol{\Sigma}_{1})}\exp\left(-\frac{1}{2}\boldsymbol{x}^{\top}\boldsymbol{\Sigma}^{-1}\boldsymbol{x}\right)\d\boldsymbol{x}.\]
Showing this is non-negative in $t\in\R$ is equivalent to showing that the integral itself is log-concave in $t$. But
\[\int_{\mathbb{R}_+^p-(\boldsymbol{\mu}+t\boldsymbol{\Sigma}_{1})}\exp\left(-\frac{1}{2}\boldsymbol{x}^{\top}\boldsymbol{\Sigma}^{-1}\boldsymbol{x}\right)\d\boldsymbol{x}=
\int_{\mathbb{R}^p}\exp\left(-\frac{1}{2}\boldsymbol{x}^{\top}\boldsymbol{\Sigma}^{-1}\boldsymbol{x}\right)\mathds{1}_{\mathbb{R}_+^p-\boldsymbol{\mu}}(\boldsymbol{x}+t\boldsymbol{\Sigma}_{1})\d\boldsymbol{x}\]
with $\exp\left(-\frac{1}{2}\boldsymbol{x}^{\top}\boldsymbol{\Sigma}^{-1}\boldsymbol{x}\right)$ log-concave in $\boldsymbol{x}$ and $\mathds{1}_{\mathbb{R}_+^p-\boldsymbol{\mu}}(\boldsymbol{x}+t\boldsymbol{\Sigma}_{1})$ log-concave in $(\boldsymbol{x},t)$ since $\mathbb{R}_+^p-\boldsymbol{\mu}$ is a convex set. Since log-concavity is closed under multiplication and integration over $\mathbb{R}^p$, the integral is indeed log-concave, and our proof of the bound on the sub-Gaussian parameter of $X_j-\mathbb{E}X_j$ is complete. The tail bound follows from \ref{lem_subgau_subexp_gaubound} of Lemma \ref{lem_subgau_subexp}.

Now by \ref{lem_subgau_subexp_norm} and \ref{lem_subgau_subexp_equivalence} of Lemma \ref{lem_subgau_subexp}, 
\[\|X_{j}\|_{\psi_2}\leq 2\sqrt{\Sigma_{jj}/e}+\mathbb{E} X_j.\] 
If $h_0$ is a function bounded by $M_0$, then by definition 
\[\|X_{j}h_0(X_{k})\|_{\psi_2}\leq M_0\left(2\sqrt{\Sigma_{jj}/e}+\mathbb{E} X_j\right).\]
By \ref{lem_subgau_subexp_norm} and \ref{lem_subgau_subexp_equivalence} of Lemma \ref{lem_subgau_subexp} again, 
\begin{align*}
\tau(X_{j}h_0(X_{k})-\mathbb{E} X_jh_0(X_k))&\leq \sqrt{e}\|X_{j}h_0(X_{k})-\mathbb{E} X_jh_0(X_k)\|_{\psi_2}\\
&\leq 2\sqrt{e}\|X_{j}h_0(X_{k})\|_{\psi_2}\\
&\leq 2M_0(2\sqrt{\Sigma_{jj}}+\sqrt{e}\,\mathbb{E} X_j).
\end{align*}
The tail bound thus follows from the first inequality using \ref{lem_subgau_subexp_gaubound} of Lemma \ref{lem_subgau_subexp}. By \ref{lem_subgau_subexp_equivalence} of the Lemma \ref{lem_subgau_subexp},
\begin{align*}
\tau\left(n^{-1}\sum_{i=1}^n X^{(i)}_{j}h_0\left(X^{(i)}_{k}\right)-\mathbb{E} X_jh_0(X_k)\right)&\leq \frac{2M_0}{\sqrt{n}}\left(2\sqrt{\Sigma_{jj}}+\sqrt{e}\,\mathbb{E} X_j\right),\\
\left\|n^{-1}\sum_{i=1}^n X^{(i)}_{j}h_0\left(X^{(i)}_{k}\right)-\mathbb{E} X_jh_0(X_k)\right\|_{\psi_2}&\leq \frac{4M_0}{\sqrt{en}}\left(2\sqrt{\Sigma_{jj}}+\sqrt{e}\,\mathbb{E} X_j\right).
\end{align*}

\noindent 2. By the proof of 1) of this lemma, $\|X_j\|_{\psi_2}\leq 2\sqrt{\Sigma_{jj}/e}+\mathbb{E} X_j$, and by \ref{lem_subgau_subexp_product} of Lemma \ref{lem_subgau_subexp}, 
\[\|X_jX_k\|_{\psi_1}\leq\left(2\sqrt{\Sigma_{jj}/e}+\mathbb{E} X_j\right)\left(2\sqrt{\Sigma_{kk}/e}+\mathbb{E} X_k\right)\leq \max_{j}\left(2\sqrt{\Sigma_{jj}/e}+\mathbb{E} X_j\right)^2.\] Since $h_{0}$ is a function bounded by $M_{0}$, by definition 
\[\|X_jX_k h_{0}(X_{\ell})\|_{\psi_1}\leq M_{0
}\max_j\left(2\sqrt{\Sigma_{jj}/e}+\mathbb{E} X_j\right)^2.\] Then by \ref{lem_subgau_subexp_norm} of Lemma \ref{lem_subgau_subexp} again, \[\|X_jX_kh_{0}(X_{\ell})-\mathbb{E} X_jX_kh_{0}(X_{\ell})\|_{\psi_1}\leq 2M_{0}\max_j\left(2\sqrt{\Sigma_{jj}/e}+\mathbb{E} X_j\right)^2.\] The tail bound then follows from \ref{lem_subgau_subexp_expbound} of Lemma \ref{lem_subgau_subexp}.
\end{proof}
Although not used for our consistency results, in the special case of $h_0\equiv 1$, we also have the following lemma. The notable difference between bounds (\ref{lemma_sub-norms2_bound}) below and (\ref{lemma_sub-norms_bound2}) from Lemma \ref{lemma_sub-norms}.2 is in the constants and dependency on $\mathbb{E}X_j$: The constants in the denominator in the right-hand side of  (\ref{lemma_sub-norms_bound2}) is smaller and thus gives a tighter bound, but (\ref{lemma_sub-norms2_bound}) is preferred when $\mathbb{E}X_j$ is notably large compared to $\sqrt{\Sigma_{jj}}$, since the constant is only linear in $\mathbb{E}X_j$.
\begin{lemma}\label{lemma_sub-norms2}
Consider the setting in Lemma \ref{lemma_sub-norms}. Then for $j,k\in\{1,\ldots,p\}$, for any $\epsilon>0$,
\begin{equation}\label{lemma_sub-norms2_bound}
\P\left(n^{-1}\left|\sum_{i=1}^n X^{(i)}_jX^{(i)}_k-\mathbb{E} X_jX_k\right|\geq\epsilon\right)\leq 4\exp\left(-\min\left(\frac{2n\epsilon^2}{C_1^2},\frac{n\epsilon}{C_1}\right)\right),
\end{equation}
where $C_1\equiv 91\max_j\Sigma_{jj}+72\max_j\mathbb{E} X_j\max_j\sqrt{\Sigma_{jj}}$.
\end{lemma}
\begin{proof}[Proof of Lemma \ref{lemma_sub-norms2}]
We use a proof similar to Lemma 1 in \citet{rav11} (note that 
$\mathbb{E}X_j$ may be nonzero in our case). Define
\[U^{(i)}_{jk}\equiv X^{(i)}_{j}+X^{(i)}_{k},\quad U_{jk}\equiv X_{j}+X_k,\quad V^{(i)}_{jk}\equiv X^{(i)}_{j}-X^{(i)}_{k},\quad V_{jk}\equiv X_j-X_k.\]
Since $X^{(i)}_{j}X^{(i)}_{k}=\frac{1}{4}\left({U^{(i)}_{jk}}^2-{V^{(i)}_{jk}}^2\right)$, by union bound we have
\begin{align*}
&\,\P\left(\left|n^{-1}\sum_{i=1}^nX^{(i)}_{j}X^{(i)}_{k}-\mathbb{E} X_jX_k\right|\geq\epsilon\right)\\
\leq&\,\P\left(\left|n^{-1}\sum_{i=1}^n{U^{(i)}_{jk}}^2-\mathbb{E} U_{jk}^2\right|\geq 2\epsilon\right)+\P\left(\left|n^{-1}\sum_{i=1}^n{V^{(i)}_{jk}}^2-\mathbb{E} V_{jk}^2\right|\geq 2\epsilon\right).
\end{align*}

We next define
\begin{alignat*}{3}
Z^{(i)}_{jk}&
\equiv{U^{(i)}_{jk}}^2-\mathbb{E} U_{jk}^2={A^{(i)}_{jk}}^2+B_{jk}^{(i)}+C_{jk},\quad &&\!\!\overline{X}^{(i)}_{j}\equiv X^{(i)}_{j}-\mathbb{E} X_{j},\\
A^{(i)}_{jk}&\equiv \overline{X}^{(i)}_{j}+\overline{X}^{(i)}_{k}, \quad B^{(i)}_{jk}\equiv 2(\mathbb{E} X_j+\mathbb{E} X_k)(\overline{X}^{(i)}_{j}+\overline{X}^{(i)}_{k}),\quad && C_{jk}\equiv-\E(\overline{X}^{(i)}_{j}+\overline{X}^{(i)}_{k})^2.
\end{alignat*}

Then since $\tau$ is a norm by \ref{lem_subgau_subexp_equivalence} of Lemma \ref{lem_subgau_subexp}, $A_{jk}$ is sub-Gaussian with parameter $\leq\sqrt{\Sigma_{jj}}+\sqrt{\Sigma_{kk}}$, and $B_{jk}$ is sub-Gaussian with parameter $\leq2(\mathbb{E} X_j+\mathbb{E} X_k)\left(\sqrt{\Sigma_{jj}}+\sqrt{\Sigma_{kk}}\right)$. Using \ref{lem_subgau_subexp_bounded_moments} of Lemma \ref{lem_subgau_subexp} together with  the inequality $(a+b+c)^q\leq (3\max\{a,b,c\})^q\leq 3^q(a^q+b^q+c^q)$ for all $a,b,c\geq 0$ and $q>0$, we have for any $q\geq 2$
\begin{align*}
(\E|Z_{jk}|^q)^{1/q}&\leq \left(3^q\left(\mathbb{E} |A_{jk}|^{2q}+\mathbb{E} |B_{jk}|^q+|C_{jk}|^q\right)\right)^{1/q}\\
&\leq 3^{1+1/q}\left((\E|A_{jk}|^{2q})^{1/q} + (\E|B_{jk}|^q)^{1/q}+|C_{jk}|\right)\\
&\leq 3^{1+1/q}\left(2^{1/q}(2q/e)\left(\sqrt{\Sigma_{jj}}+\sqrt{\Sigma_{kk}}\right)^2\right.\\&\quad\quad\quad\,\,\,\,\,\left.+2^{1/q}\sqrt{q/e}2(\mathbb{E} X_j+\mathbb{E} X_k)(\sqrt{\Sigma_{jj}}+\sqrt{\Sigma_{kk}})+\var(X_{j}+X_{k})\right).
\end{align*}

Using $\var(X+Y)\leq2(\var(X)+\var(Y))$ and the fact that
$\var(X_{j})=\var(X_j-\mathbb{E}X_j)\leq\tau^2(X_j-\mathbb{E}X_j)\leq\Sigma_{jj}$
(by \ref{lem_subgau_subexp_equivalence} of Lemma
\ref{lem_subgau_subexp} and 1) of Lemma \ref{lemma_sub-norms}, we then have
\begin{multline*}
\left(\frac{\E|Z_{jk}|^q}{q!}\right)^{1/q}\\
\leq3^{1+1/q}\frac{2^{3+1/q}(q/e)\max_j\Sigma_{jj}+2^{3+1/q}\sqrt{q/e}\max_j\mathbb{E} X_j\cdot\max_j\sqrt{\Sigma_{jj}}+4\max_j\Sigma_{jj}}{(q!)^{1/q}}.
\end{multline*}
Since all three coefficients involving $q$ are decreasing in $q\geq 2$, we have
\[\sup_{q\geq 2}\left(\frac{\E|Z_{jk}|^q}{q!}\right)^{1/q}\leq \left(48\sqrt{3}/e+6\sqrt{6}\right)\max_j\Sigma_{jj}+24\sqrt{6/e}\max_j\mathbb{E} X_j\max_j\sqrt{\Sigma_{jj}}.\]
Thus by \ref{lem_subgau_subexp_moment_based} of Lemma \ref{lem_subgau_subexp}, letting $B\equiv \left(91\max_j\Sigma_{jj}+72\max_{j}\mathbb{E} X_j\max_j\sqrt{\Sigma_{jj}}\right)$, we have for all $\epsilon>0$:
\[\P\left(n^{-1}\left|\sum_{i=1}^n Z_{jk}^{(i)}\right|\geq 2\epsilon\right)\leq 2\exp\left(-\min\left(\frac{2n\epsilon^2}{B^2},\frac{n\epsilon}{B}\right)\right).\]
A tail bound for the sample average of $V_{jk}^2$ can be similarly derived, and the result follows.
\end{proof}

\newpage

\section{Simulation Results for Erd\"os-R\'enyi Graphs}\label{ER}

We revisit the simulations from Section \ref{Numerical
  Experiments} but use Erd\"os-R\'enyi (ER) graphs in which
each possible edge is independently included with probability $\pi$.  Independent
uniform draws from $[0.5,1]$ are used to fill the
non-zero off-diagonal entries
 of the symmetric matrix $\mathbf{K}_0$. The
  diagonal elements are set such
that $\mathbf{K}_0$ has minimum eigenvalue $0.1$.  We choose $\pi=0.08$ for $n=1000$, and $\pi=0.02$ for $n=80$.

\subsection{Truncated GGMs}
In this section we present the results for truncated GGMs.
\subsubsection{Choice of \texorpdfstring{$\boldsymbol{h}$}{h}}

The results for truncated centered GGMs are reported in
Table~\ref{App_table_chooseh_centered} and Figure~\ref{App_plot_chooseh_centered}.  Those for truncated
non-centered GGMs using the profiled estimator are in Table~\ref{App_table_chooseh_noncentered} and
Figure~\ref{App_plot_chooseh_noncentered}.


\begin{table}[H]
\begin{center}
\scalebox{0.75}{
\begin{tabular}{|c|c|c|c|c|c|}
\hline
\multicolumn{6}{|c|}{Centered, $n=80$, multiplier 1.8647, ER}
\tabularnewline
\hline\hline
\multicolumn{3}{|c|}{$\min(\log(1+x),c)$} & \multicolumn{3}{c|}{$\min(x,c)$}
\tabularnewline
\hline
$c$ & Mean & sd & $c$ & Mean & sd
\tabularnewline\hline
$\infty$ & 0.632 & 0.036 & $\infty$ & 0.638 & 0.035
\tabularnewline
2 & 0.632 & 0.036 & 3 & 0.638 & 0.035
\tabularnewline
1 & 0.630 & 0.035 & 2 & 0.635 & 0.035
\tabularnewline 
0.5 & 0.613 & 0.033 & 1 & 0.623 & 0.033
\tabularnewline
\hline
\multicolumn{3}{|c|}{$\mathrm{MCP}(1,c)$} & \multicolumn{3}{c|}{$\mathrm{SCAD}(1,c)$}
\tabularnewline
\hline
$c$ & Mean & sd & $c$ & Mean & sd
\tabularnewline
\hline
10 & 0.637 & 0.035 & 10 & 0.638 & 0.035
\tabularnewline
5 & 0.636 & 0.036 & 5 & 0.637 & 0.035
\tabularnewline
1 & 0.617 & 0.033 & 2 & 0.632 & 0.035
\tabularnewline
\hline
\multicolumn{3}{|c|}{$x^{1.5}$: (0.627, 0.032)} & \multicolumn{3}{|c|}{$x^2$: (0.595,0.028)}
\tabularnewline
\hline
\hline
\multicolumn{3}{|c|}{GLASSO (0.553,0.029)} & \multicolumn{3}{|c|}{SPACE: (0.544, 0.026)}
\tabularnewline
\hline
\multicolumn{3}{|c|}{NS: (0.543,0.028)} & \multicolumn{3}{|c|}{SJ: (0.519,0.028)}
\tabularnewline
\hline
\end{tabular}}
\vskip0.05in
\scalebox{0.75}{
\begin{tabular}{|c|c|c|c|c|c|}
\hline
\multicolumn{6}{|c|}{Centered, $n=1000$, multiplier 1, ER}
\tabularnewline
\hline\hline
\multicolumn{3}{|c|}{$\min(\log(1+x),c)$} & \multicolumn{3}{c|}{$\min(x,c)$}
\tabularnewline
\hline
$c$ & Mean & sd & $c$ & Mean & sd
\tabularnewline\hline
$\infty$ & 0.716 & 0.016 & 2  & 0.710 & 0.016
\tabularnewline
2 & 0.716 & 0.016  & 3  & 0.710 & 0.016
\tabularnewline
1 & 0.715 & 0.016& $1$ & 0.710 & 0.017
\tabularnewline 
0.5  & 0.694 & 0.017& $\infty$ & 0.709 & 0.016
\tabularnewline
\hline
\multicolumn{3}{|c|}{$\mathrm{MCP}(1,c)$} & \multicolumn{3}{c|}{$\mathrm{SCAD}(1,c)$}
\tabularnewline
\hline
$c$ & Mean & sd & $c$ & Mean & sd
\tabularnewline
\hline
5 & 0.714 & 0.016 & 2 & 0.713 & 0.016
\tabularnewline
10 & 0.711 & 0.016 & 5 & 0.711 & 0.016
\tabularnewline
1 & 0.707 & 0.017& 10 & 0.710 & 0.016
\tabularnewline
\hline
\multicolumn{3}{|c|}{$x^{1.5}$: (0.678,0.016)} & \multicolumn{3}{|c|}{$x^2$: (0.64,0.017)}
\tabularnewline
\hline
\hline
\multicolumn{3}{|c|}{GLASSO: (0.675,0.016)} & \multicolumn{3}{|c|}{SPACE: (0.675,0.016)}
\tabularnewline
\hline
\multicolumn{3}{|c|}{NS: (0.675,0.016)} & \multicolumn{3}{|c|}{SJ: (0.624,0.017)}
\tabularnewline
\hline
\end{tabular}
\begin{tabular}{|c|c|c|c|c|c|}
\hline
\multicolumn{6}{|c|}{Centered, $n=1000$, multiplier 1.6438, ER}
\tabularnewline
\hline\hline
\multicolumn{3}{|c|}{$\min(\log(1+x),c)$} & \multicolumn{3}{c|}{$\min(x,c)$}
\tabularnewline
\hline
$c$ & Mean & sd & $c$ & Mean & sd
\tabularnewline\hline
$\infty$ & 0.796 & 0.014 & $\infty$ & 0.795 & 0.014
\tabularnewline
2 & 0.796 & 0.014 & 3  & 0.794 & 0.014
\tabularnewline
1 & 0.794 & 0.014 & 2 & 0.792 & 0.014
\tabularnewline 
0.5 & 0.772 & 0.015 & 1 & 0.784 & 0.015
\tabularnewline
\hline
\multicolumn{3}{|c|}{$\mathrm{MCP}(1,c)$} & \multicolumn{3}{c|}{$\mathrm{SCAD}(1,c)$}
\tabularnewline
\hline
$c$ & Mean & sd & $c$ & Mean & sd
\tabularnewline
\hline
5 & 0.796 & 0.014 & 5 & 0.795 & 0.014
\tabularnewline
10 & 0.796 & 0.014 & 10  & 0.795 & 0.014
\tabularnewline
1 & 0.778 & 0.015 & 2 & 0.793 & 0.014
\tabularnewline
\hline
\multicolumn{3}{|c|}{$x^{1.5}$: (0.757,0.015)} & \multicolumn{3}{|c|}{$x^2$: (0.693,0.016)}
\tabularnewline
\hline
\hline
\multicolumn{3}{|c|}{GLASSO: (0.675,0.016)} & \multicolumn{3}{|c|}{SPACE: (0.675,0.016)}
\tabularnewline
\hline
\multicolumn{3}{|c|}{NS: (0.675,0.016)} & \multicolumn{3}{|c|}{SJ: (0.624,0.017)}
\tabularnewline
\hline
\end{tabular}}
\end{center}
\caption{\vspace{-0.1in}\label{App_table_chooseh_centered} Mean and standard deviation of areas under the ROC curves (AUC) using different estimators in the centered setting, with $n=80$ and multiplier $1.8647$, or $n=1000$ and multipliers $1$ and $1.6438$. Methods include our estimator with different choices of $h$,  GLASSO, SPACE, neighborhood selection (NS), and Space JAM (SJ).}
\end{table}
\vspace{-0.5in}

\begin{figure}[ht]
\vspace{-0.in}
\centering
\subfloat[$n=80$, $\mathrm{mult}=1.8647$, ER]{\includegraphics[trim={0 0.5in 1.3in 0},clip,scale=0.26]{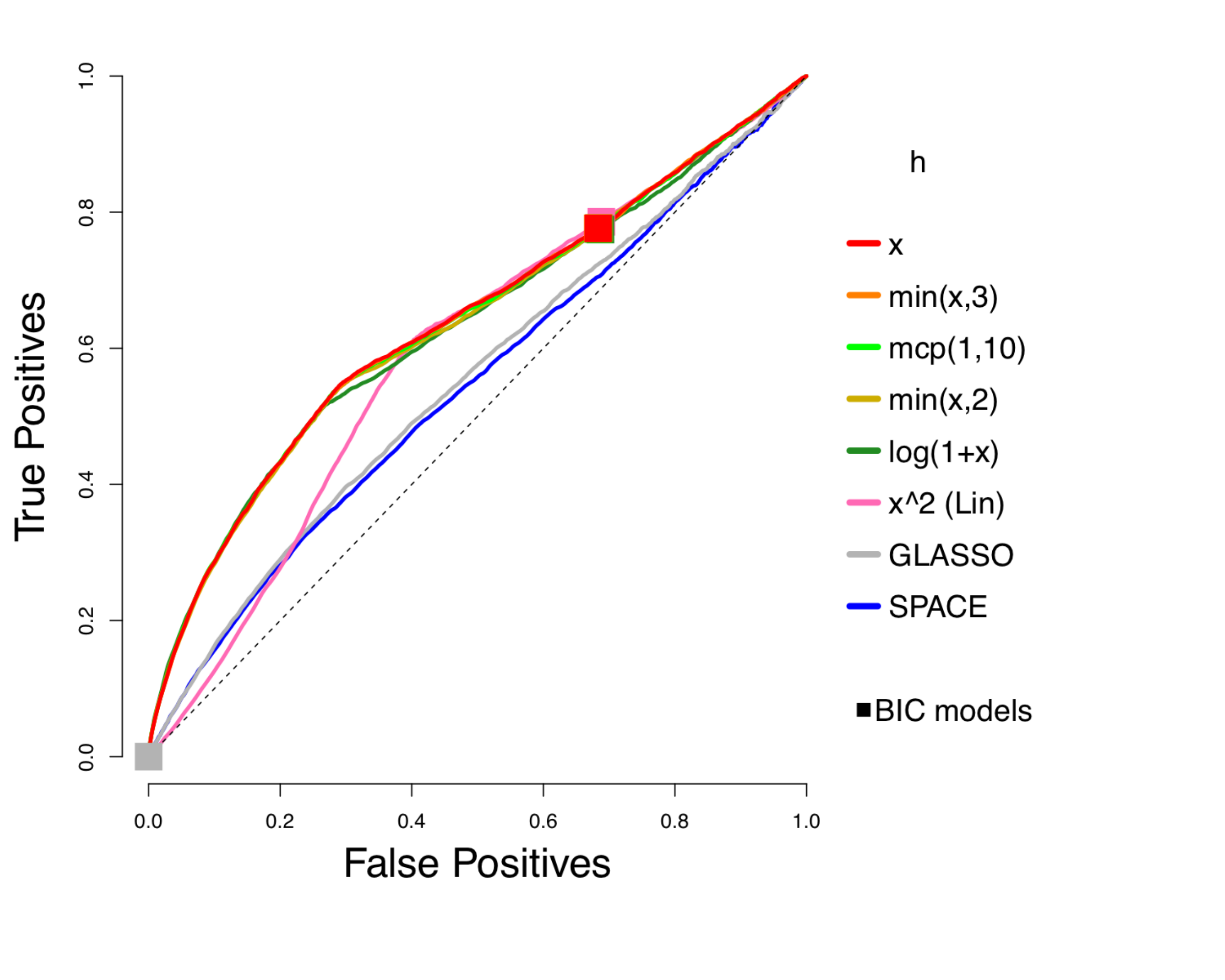}\hspace{-0.02in}}
\subfloat[$n=1000$, $\mathrm{mult}=1$, ER]{\includegraphics[trim={0
    0.5in 2in
    0},clip,scale=0.26]{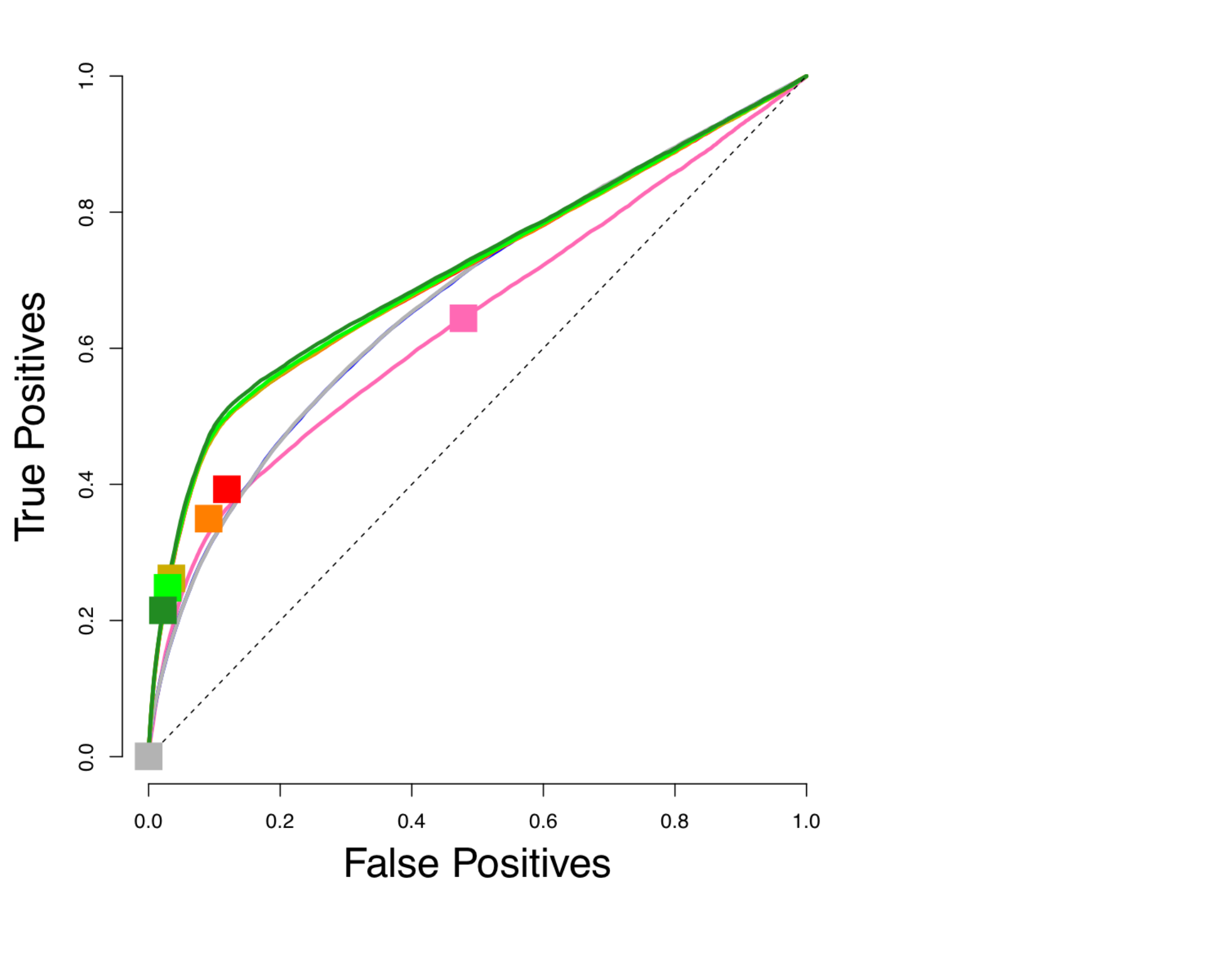}\hspace{-0.3in}}\subfloat[$n=1000$,
$\mathrm{mult}=1.6438$, ER]{\includegraphics[trim={0 0.5in 3in
    0},clip,scale=0.26]{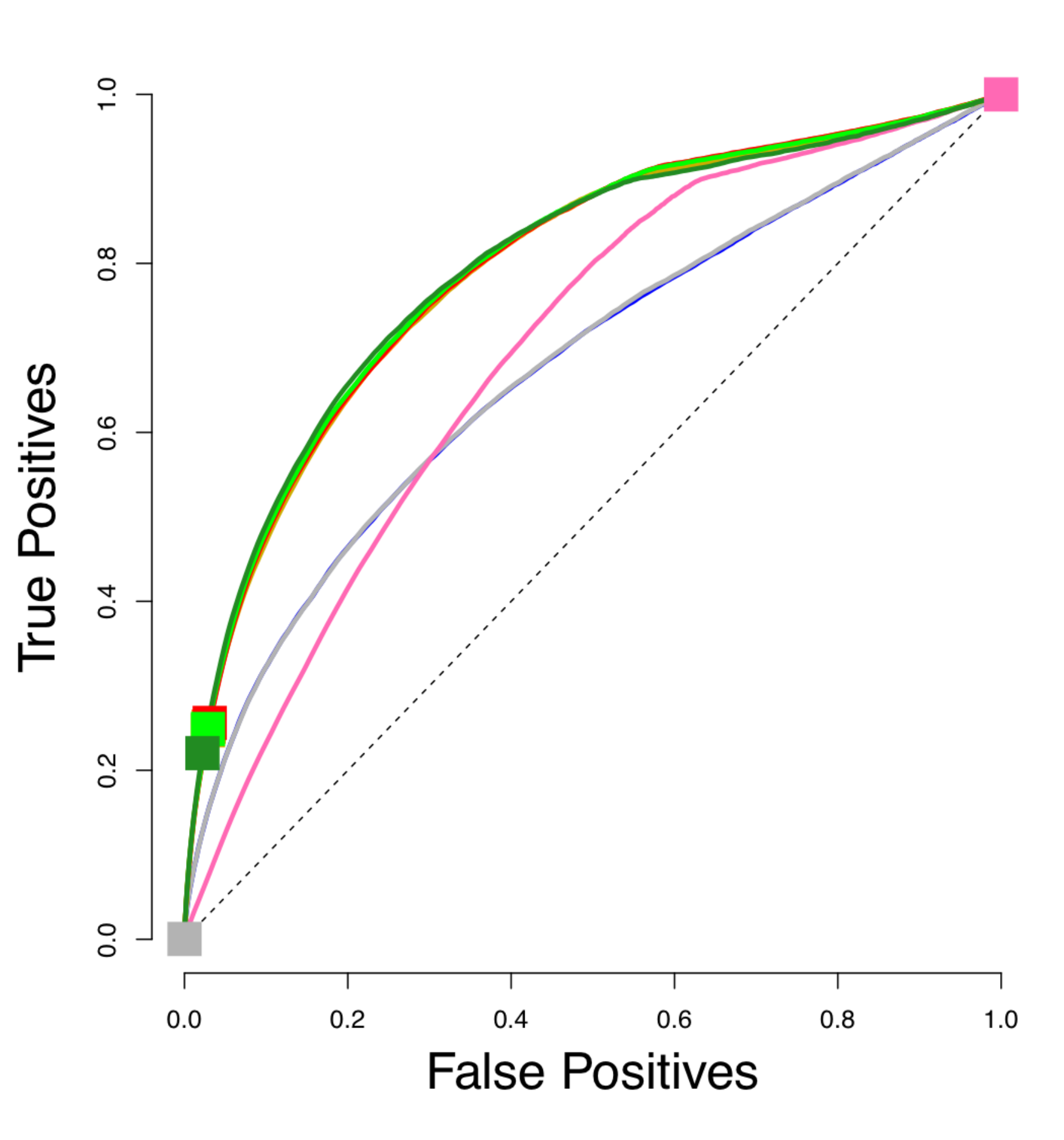}}
\caption{\label{App_plot_chooseh_centered} Average ROC curves of our \emph{centered} estimator for
  $m=100$ variables and two sample sizes $n$ under various choices of
  $h$, compared to SPACE and GLASSO, for the \emph{truncated centered GGM} case.  Squares indicate average true positive rate (TPR) and false positive rate (FPR) of models picked by eBIC with refitting for the estimator in the same color. }
\end{figure}
\vspace{-0.2in}


\begin{table}[H]
\vspace{-0.15in}
\begin{center}
\scalebox{0.75}{
\begin{tabular}{|c|c|c|c|c|c|}
\hline
\multicolumn{6}{|c|}{Non-centered profiled, $n=80$, multiplier 1.8647, ER}
\tabularnewline
\hline\hline
\multicolumn{3}{|c|}{$\min(\log(1+x),c)$} & \multicolumn{3}{c|}{$\min(x,c)$}
\tabularnewline
\hline
$c$ & Mean & sd & $c$ & Mean & sd
\tabularnewline\hline
1 & 0.588 & 0.034 & 3 & 0.588 & 0.033
\tabularnewline
$\infty$ & 0.588 & 0.034 & $\infty$ & 0.588 & 0.033
\tabularnewline
2 & 0.588 & 0.034 & 2 & 0.588 & 0.033
\tabularnewline 
0.5 & 0.576 & 0.033 & 1 & 0.583 & 0.033
\tabularnewline
\hline
\multicolumn{3}{|c|}{$\mathrm{MCP}(1,c)$} & \multicolumn{3}{c|}{$\mathrm{SCAD}(1,c)$}
\tabularnewline
\hline
$c$ & Mean & sd & $c$ & Mean & sd
\tabularnewline
\hline
5 & 0.588 & 0.033 & 5 & 0.588 & 0.033
\tabularnewline
10 & 0.588 & 0.033 & 10 & 0.588 & 0.033
\tabularnewline
1 & 0.581 & 0.033 & 2 & 0.587 & 0.033
\tabularnewline
\hline
\multicolumn{3}{|c|}{$x^{1.5}$: (0.582,0.028)} & \multicolumn{3}{|c|}{$x^2$: (0.576,0.028)}
\tabularnewline
\hline
\hline
\multicolumn{3}{|c|}{GLASSO: (0.572,0.033)} & \multicolumn{3}{|c|}{SPACE: (0.562,0.031)}
\tabularnewline
\hline
\multicolumn{3}{|c|}{NS: (0.560,0.032)} & \multicolumn{3}{|c|}{SJ: (0.535,0.027)}
\tabularnewline
\hline
\end{tabular}}
\vskip0.05in
\scalebox{0.75}{
\begin{tabular}{|c|c|c|c|c|c|}
\hline
\multicolumn{6}{|c|}{Non-centered profiled, $n=1000$, multiplier 1, ER}
\tabularnewline
\hline\hline
\multicolumn{3}{|c|}{$\min(\log(1+x),c)$} & \multicolumn{3}{c|}{$\min(x,c)$}
\tabularnewline
\hline
$c$ & Mean & sd & $c$ & Mean & sd
\tabularnewline\hline
2 & 0.692 & 0.022 & 1 & 0.687 & 0.022
\tabularnewline
$\infty$ & 0.692 & 0.022 & $\infty$  & 0.686 & 0.022
\tabularnewline
1 & 0.691 & 0.022 & 3 & 0.685 & 0.022
\tabularnewline 
0.5 & 0.684 & 0.02 & 2 & 0.685 & 0.022
\tabularnewline
\hline
\multicolumn{3}{|c|}{$\mathrm{MCP}(1,c)$} & \multicolumn{3}{c|}{$\mathrm{SCAD}(1,c)$}
\tabularnewline
\hline
$c$ & Mean & sd & $c$ & Mean & sd
\tabularnewline
\hline
5 & 0.689 & 0.022 & 2 & 0.687 & 0.022
\tabularnewline
1 & 0.689 & 0.020 & 5 & 0.687 & 0.022
\tabularnewline
10 & 0.687 & 0.022 & 10 & 0.686 & 0.022
\tabularnewline
\hline
\multicolumn{3}{|c|}{$x^{1.5}$: (0.663,0.020)} & \multicolumn{3}{|c|}{$x^2$: (0.638,0.019)}
\tabularnewline
\hline
\hline
\multicolumn{3}{|c|}{GLASSO (0.700,0.022)} & \multicolumn{3}{|c|}{SPACE: (0.699,0.022)}
\tabularnewline
\hline
\multicolumn{3}{|c|}{NS: (0.699,0.022)} & \multicolumn{3}{|c|}{SJ: (0.655,0.021)}
\tabularnewline
\hline
\end{tabular}
\begin{tabular}{|c|c|c|c|c|c|}
\hline
\multicolumn{6}{|c|}{Non-centered profiled, $n=1000$, multiplier 1.6438, ER}
\tabularnewline
\hline\hline
\multicolumn{3}{|c|}{$\min(\log(1+x),c)$} & \multicolumn{3}{c|}{$\min(x,c)$}
\tabularnewline
\hline
$c$ & Mean & sd & $c$ & Mean & sd
\tabularnewline\hline
2 & 0.705 & 0.021 & $\infty$ & 0.705 & 0.022
\tabularnewline
$\infty$ & 0.705 & 0.021 & 3 & 0.705 & 0.021
\tabularnewline
1 & 0.703 & 0.021 & 2 & 0.702 & 0.022
\tabularnewline 
0.5 & 0.683 & 0.019 & 1 & 0.695 & 0.021
\tabularnewline
\hline
\multicolumn{3}{|c|}{$\mathrm{MCP}(1,c)$} & \multicolumn{3}{c|}{$\mathrm{SCAD}(1,c)$}
\tabularnewline
\hline
$c$ & Mean & sd & $c$ & Mean & sd
\tabularnewline
\hline
5 & 0.706 & 0.021 & 10 & 0.705 & 0.022
\tabularnewline
10 & 0.706 & 0.022 & 5 & 0.705 & 0.022
\tabularnewline
1 & 0.690 & 0.019 & 2 & 0.703 & 0.022
\tabularnewline
\hline
\multicolumn{3}{|c|}{$x^{1.5}$: (0.689,0.021)} & \multicolumn{3}{|c|}{$x^2$: (0.664,0.019)}
\tabularnewline
\hline
\hline
\multicolumn{3}{|c|}{GLASSO (0.700,0.022)} & \multicolumn{3}{|c|}{SPACE: (0.699,0.022)}
\tabularnewline
\hline
\multicolumn{3}{|c|}{NS: (0.699,0.022)} & \multicolumn{3}{|c|}{SJ: (0.655,0.021)}
\tabularnewline
\hline
\end{tabular}}\vspace{-0.2in}
\end{center}\caption{\label{App_table_chooseh_noncentered} Mean and standard deviation of AUC using different profiled estimators in the non-centered
  setting, with $n=80$ and multiplier $1.8647$, or $n=1000$ and
  multipliers $1$ and $1.6438$. Methods as for Table~\ref{App_table_chooseh_centered}.}
\end{table}

\begin{figure}[H]
\centering
\vspace{-0.0in}
\subfloat[$n=80$, $\mathrm{mult}=1.8647$, ER]{\includegraphics[trim={0 0.5in 1.3in 0},clip,scale=0.26]{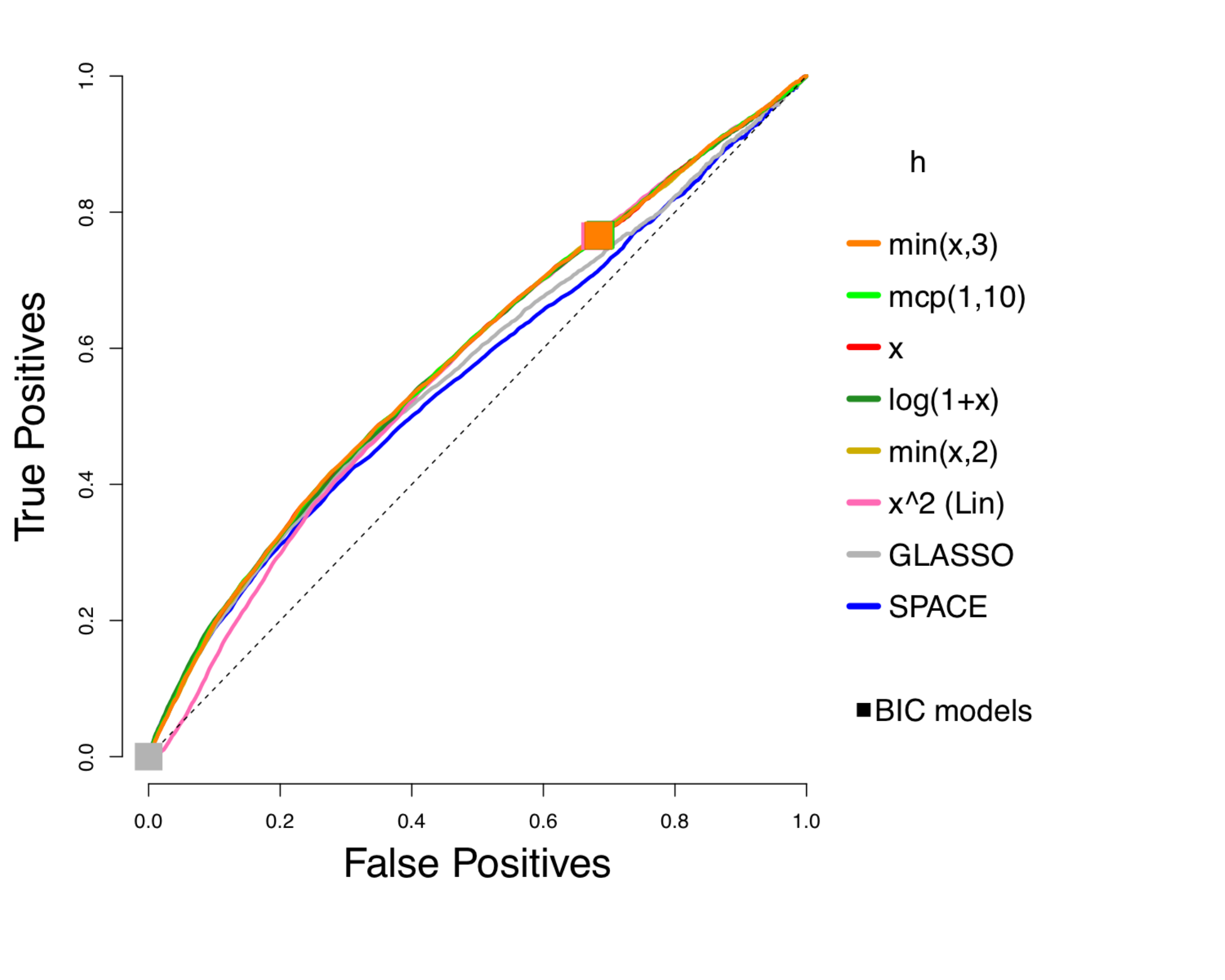}\hspace{-0.02in}}
\subfloat[$n=1000$, $\mathrm{mult}=1$, ER]{\includegraphics[trim={0 0.5in 2in 0},clip,scale=0.26]{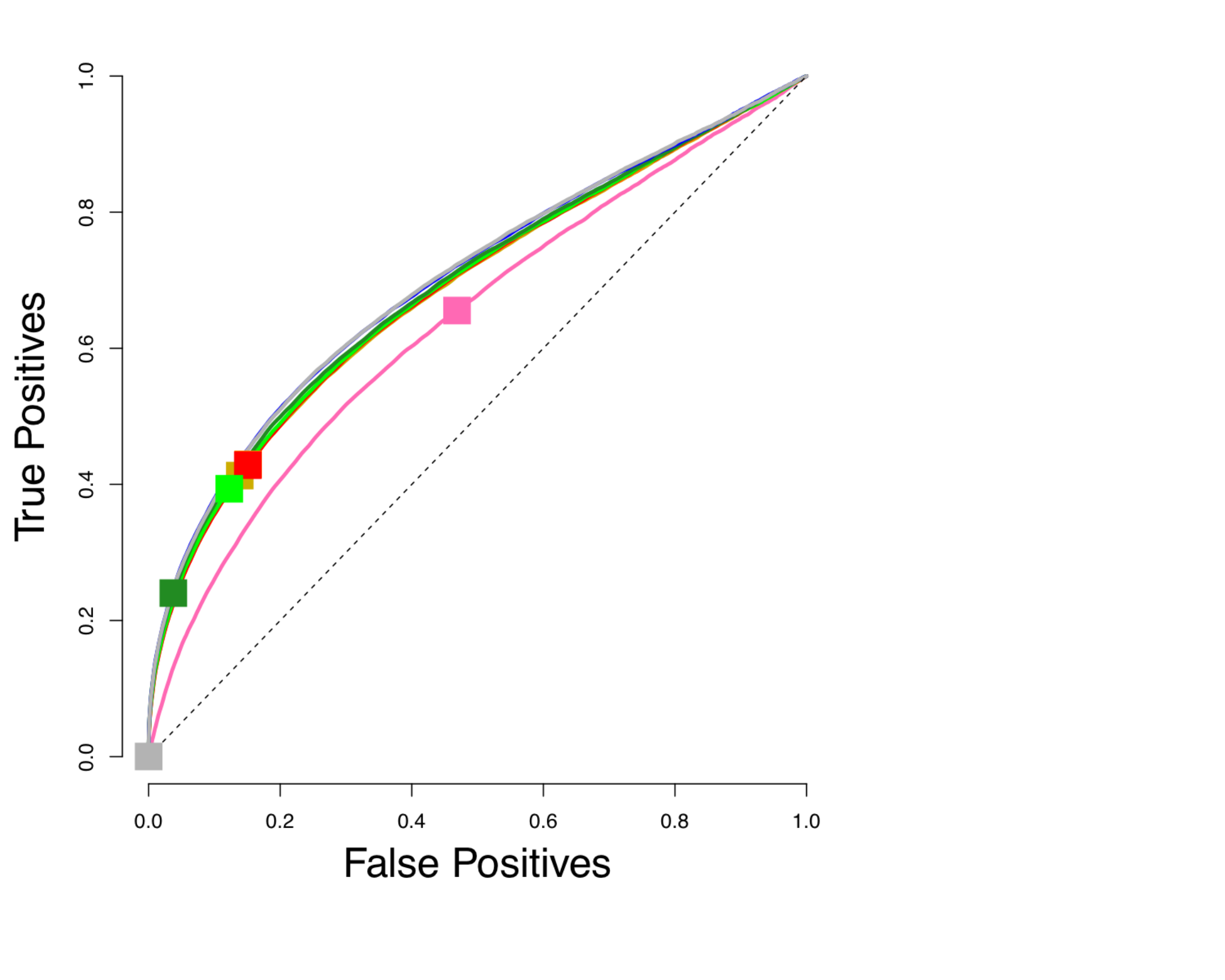}\hspace{-0.3in}}\subfloat[$n=1000$, $\mathrm{mult}=1.6438$, ER]{\includegraphics[trim={0 0.5in 3in 0},clip,scale=0.26]{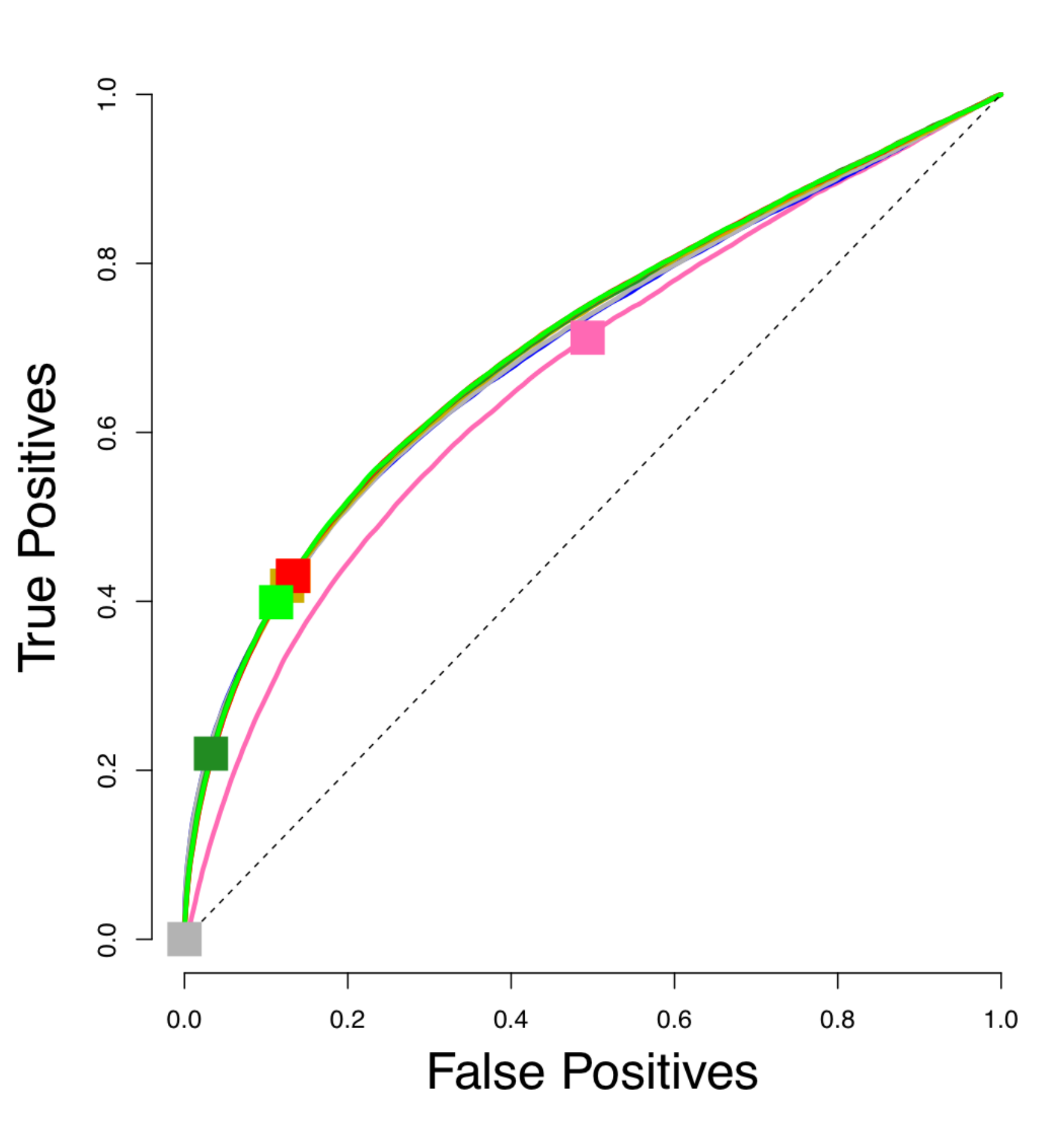}}
\caption{\label{App_plot_chooseh_noncentered} Average for the truncated non-centered GGM case. $n=80$ or $1000$, $m=100$.}
\end{figure}


\subsubsection{Choice of multiplier}
The results for truncated centered GGMs where each curve represents a different multiplier are shown in Figure \ref{App_plot_mixed}, and those for truncated non-centered GGMs are in Figure \ref{App_plot_lr}, where each curve corresponds to a different ratio $\lambda_{\mathbf{K}}/\lambda_{\boldsymbol{\eta}}$.
\begin{figure}[ht]
\centering
\vspace{-0.0in}
\subfloat[$n=80$, ER]{\includegraphics[scale=0.27]{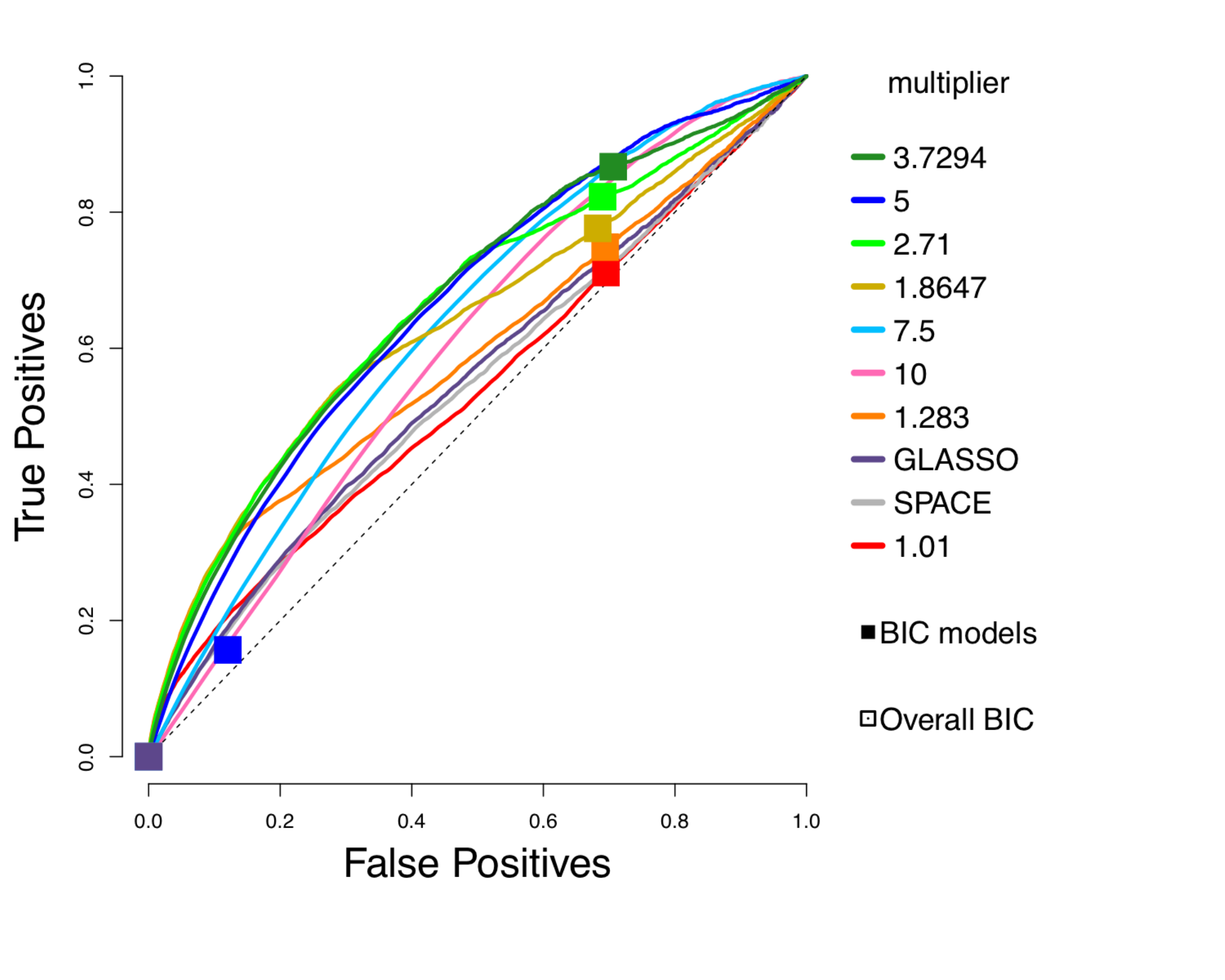}\hspace{-0.02in}}
\subfloat[$n=1000$, ER]{\includegraphics[scale=0.27]{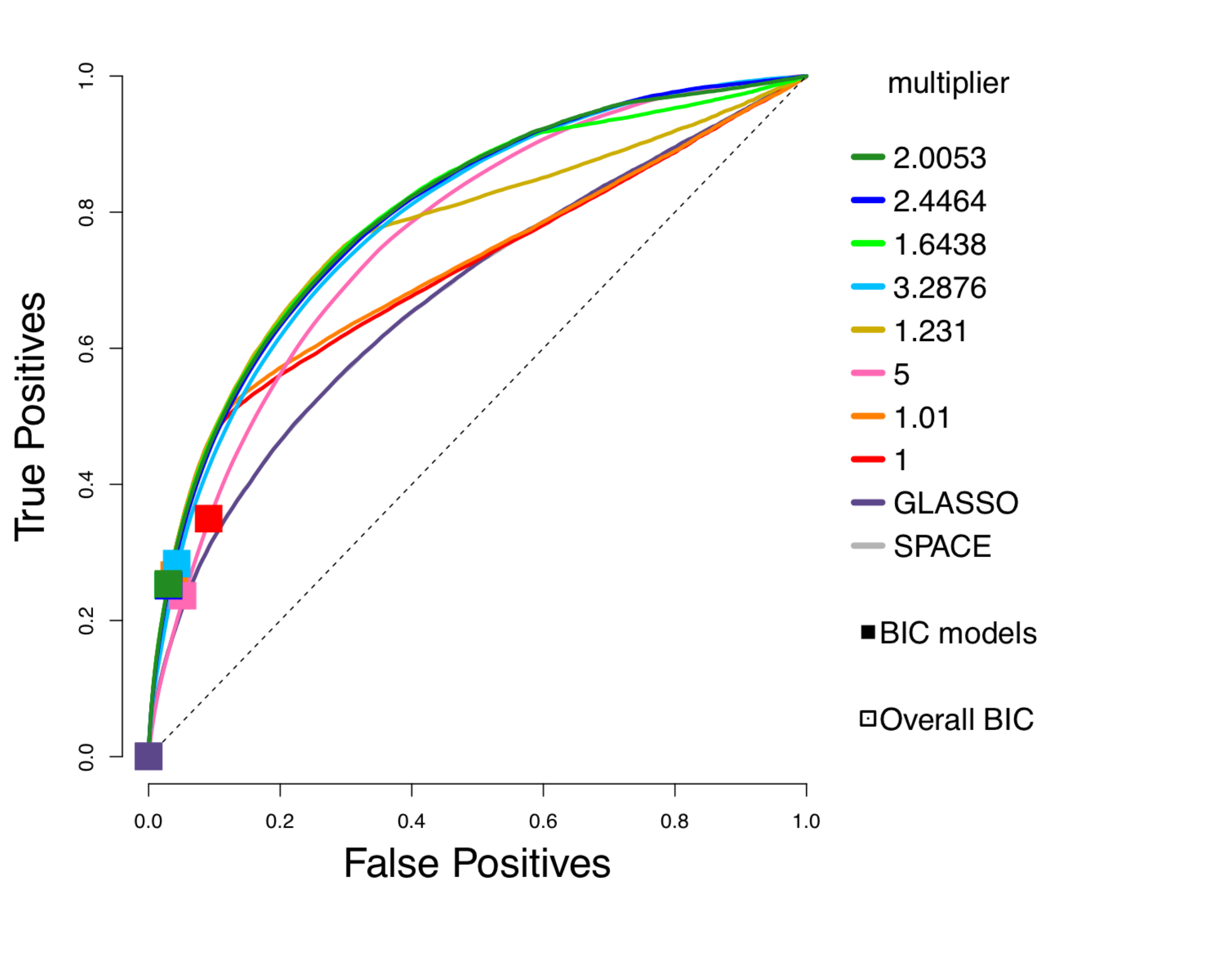}}
\caption{\label{App_plot_mixed} Performance of $\min(x,3)$ for truncated centered GGMs with different multipliers, compared to GLASSO and SPACE, in the centered setting, $n=80$ or 1000.}
\vspace{-0.0in}
\end{figure}

\begin{figure}[htp]
\centering
\vspace{-0.4in}
\subfloat[$n=80$, $\mathrm{mult}=1.7897$, ER]{\includegraphics[trim={0 0.5in 1.3in 0},clip,scale=0.27]{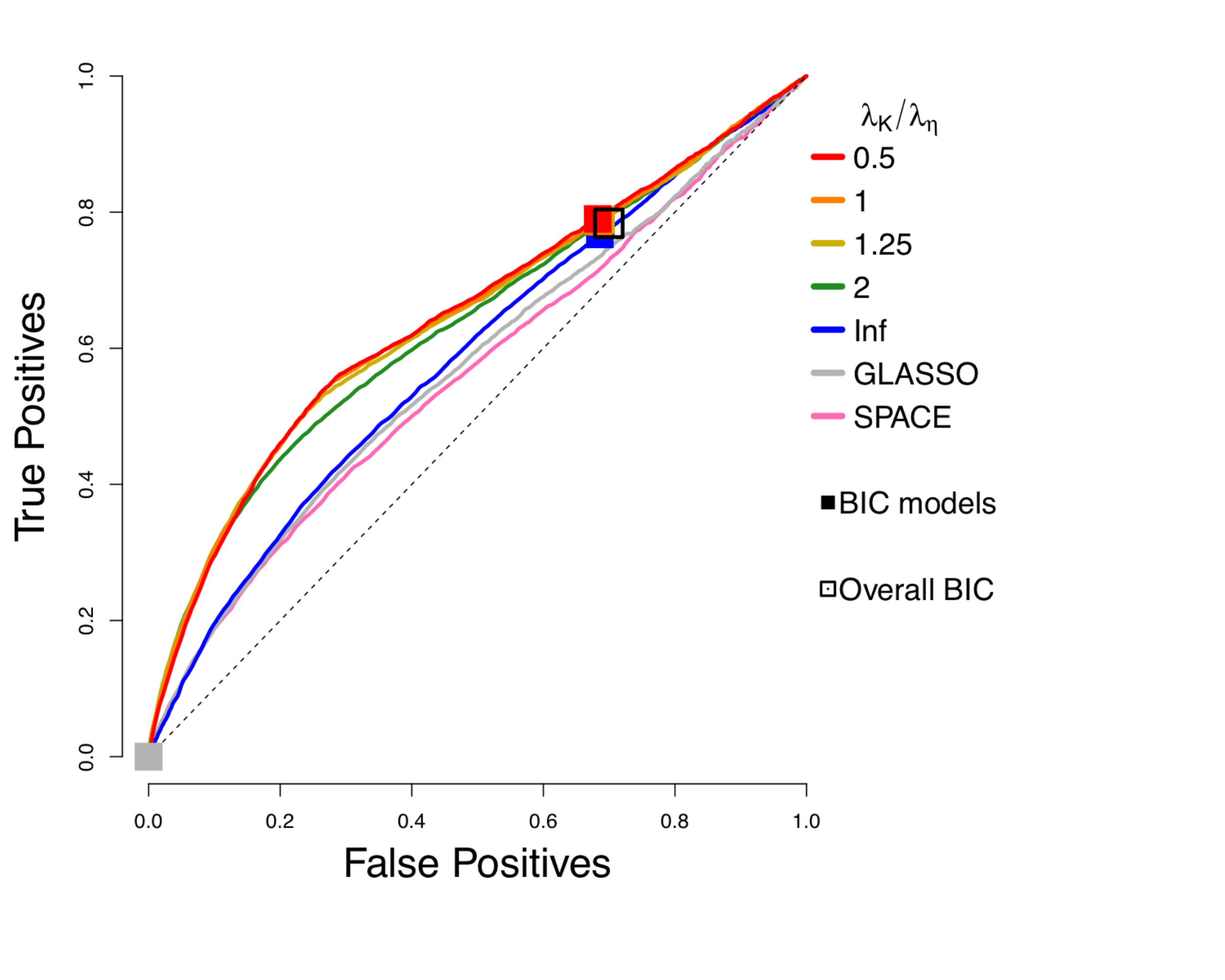}\hspace{-0.02in}}
\subfloat[$n=80$, $\mathrm{mult}=1.8647$, ER]{\includegraphics[trim={0 0.5in 2in 0},clip,scale=0.27]{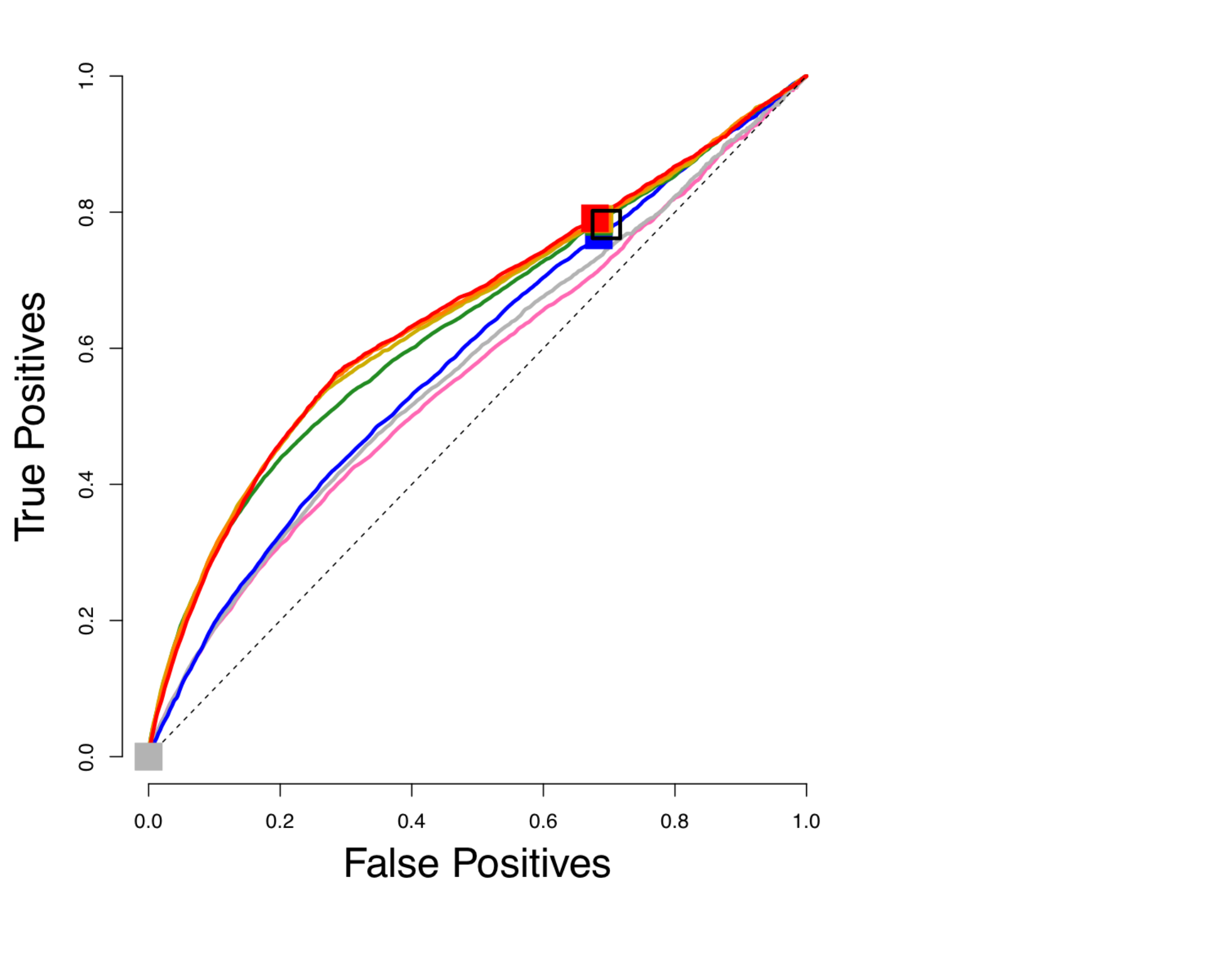}\hspace{-0.4in}}\\ \vspace{-0.1in}
\subfloat[$n=1000$, $\mathrm{mult}=1$, ER]{\includegraphics[trim={0 0.5in 2in 0},clip,scale=0.27]{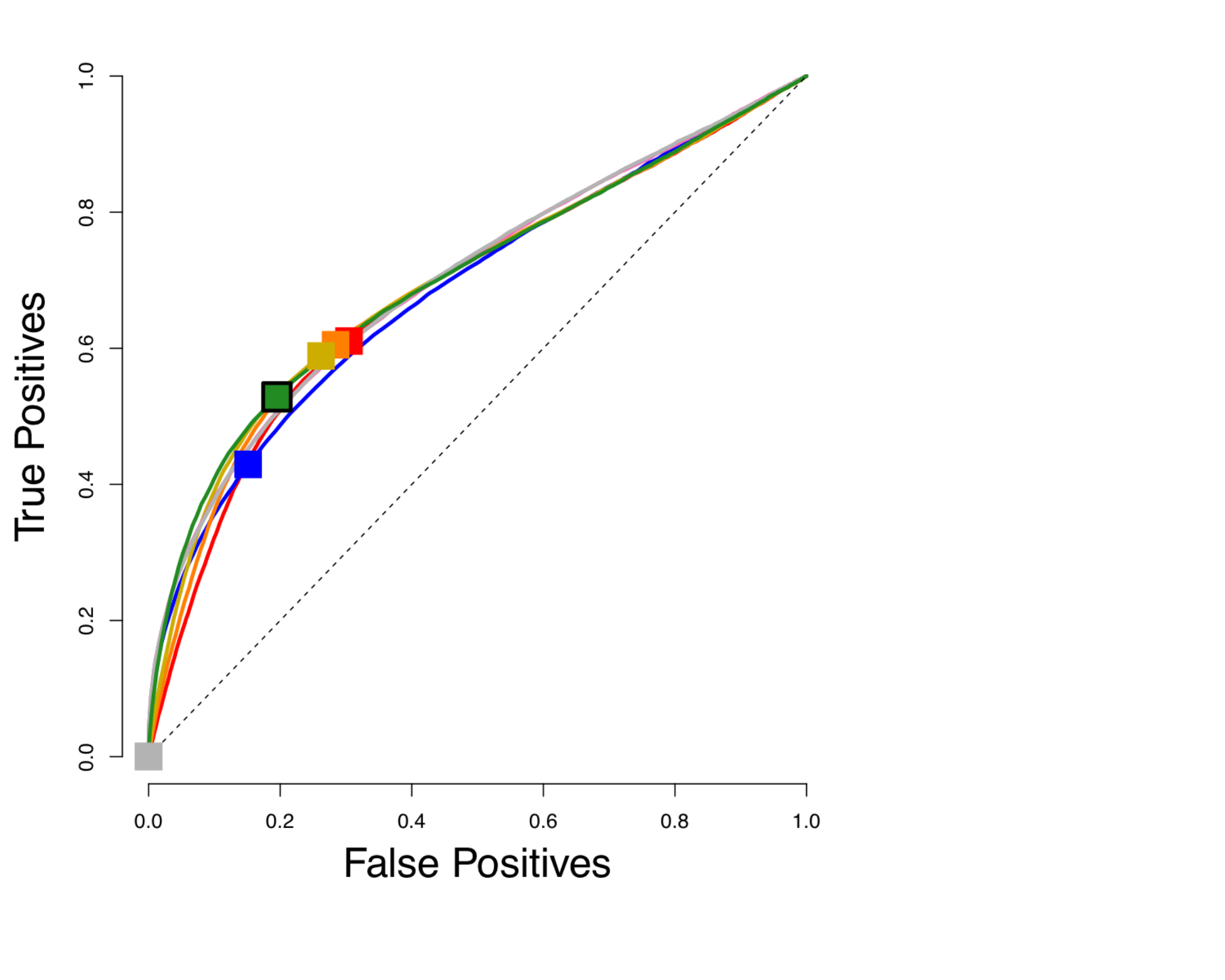}\hspace{-0.3in}}
\subfloat[$n=1000$, $\mathrm{mult}=1.2310$, ER]{\includegraphics[trim={0 0.5in 2in 0},clip,scale=0.27]{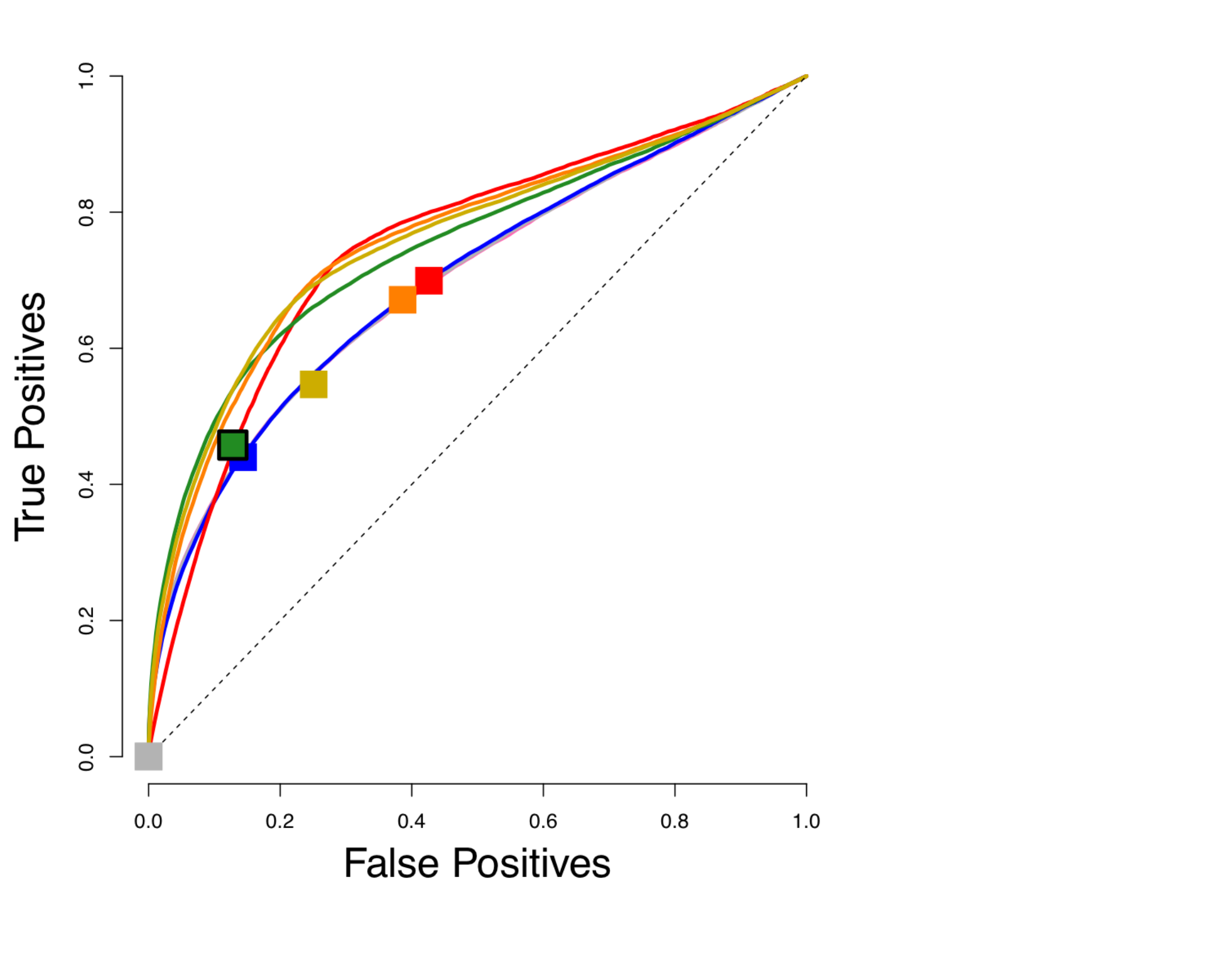}\hspace{-0.3in}}\subfloat[$n=1000$, $\mathrm{mult}=1.6438$, ER]{\includegraphics[trim={0 0.5in 2in 0},clip,scale=0.27]{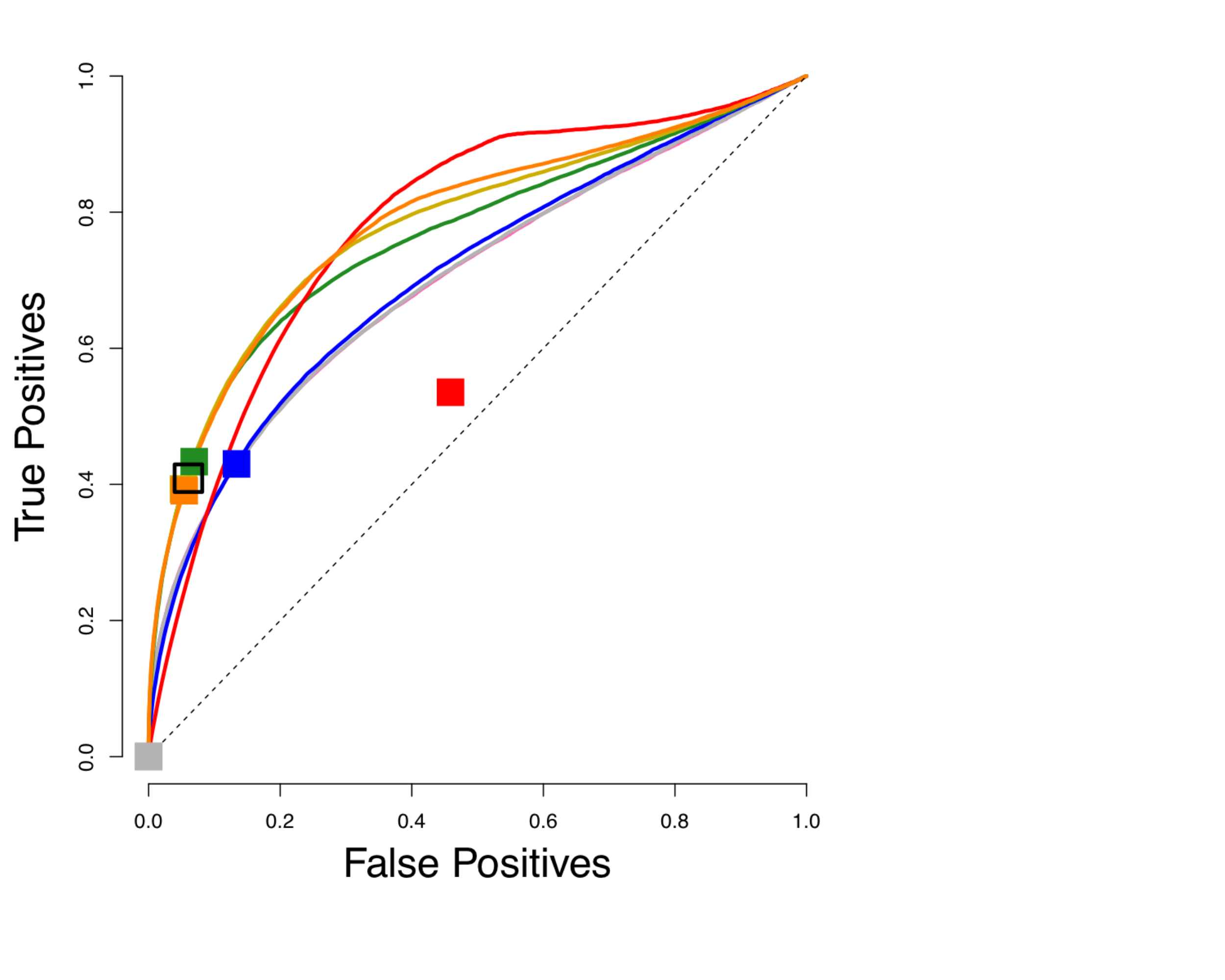}}
\caption{\label{App_plot_lr}Performance of the non-centered estimator with
  $h(x)=\min(x,3)$. Each curve corresponds to a different choice of
  $\lambda_{\mathbf{K}}/\lambda_{\boldsymbol{\eta}}$. 
  Squares indicate models picked by eBIC with refit.  The square with
  black outline has the highest eBIC among all models (combinations of
  $\lambda_{\mathbf{K}}$, $\lambda{\boldsymbol{\eta}}$). The multipliers correspond to medium or high for $n=80$, and low, medium and high for $n=1000$, respectively.}
 \vspace{-0.3in}
\end{figure}


\newpage
\subsection{Other \texorpdfstring{$a/b$}{a/b} Models}
Figure \ref{App_plot_exp} exhibits the results for the exponential models.
\begin{figure}[H]
\centering
\vspace{-0.1in}
\subfloat[$n=80$, $\boldsymbol{\eta}=-0.5\mathbf{1}_{100}$, profiled est, ER]{\includegraphics[scale=0.26]{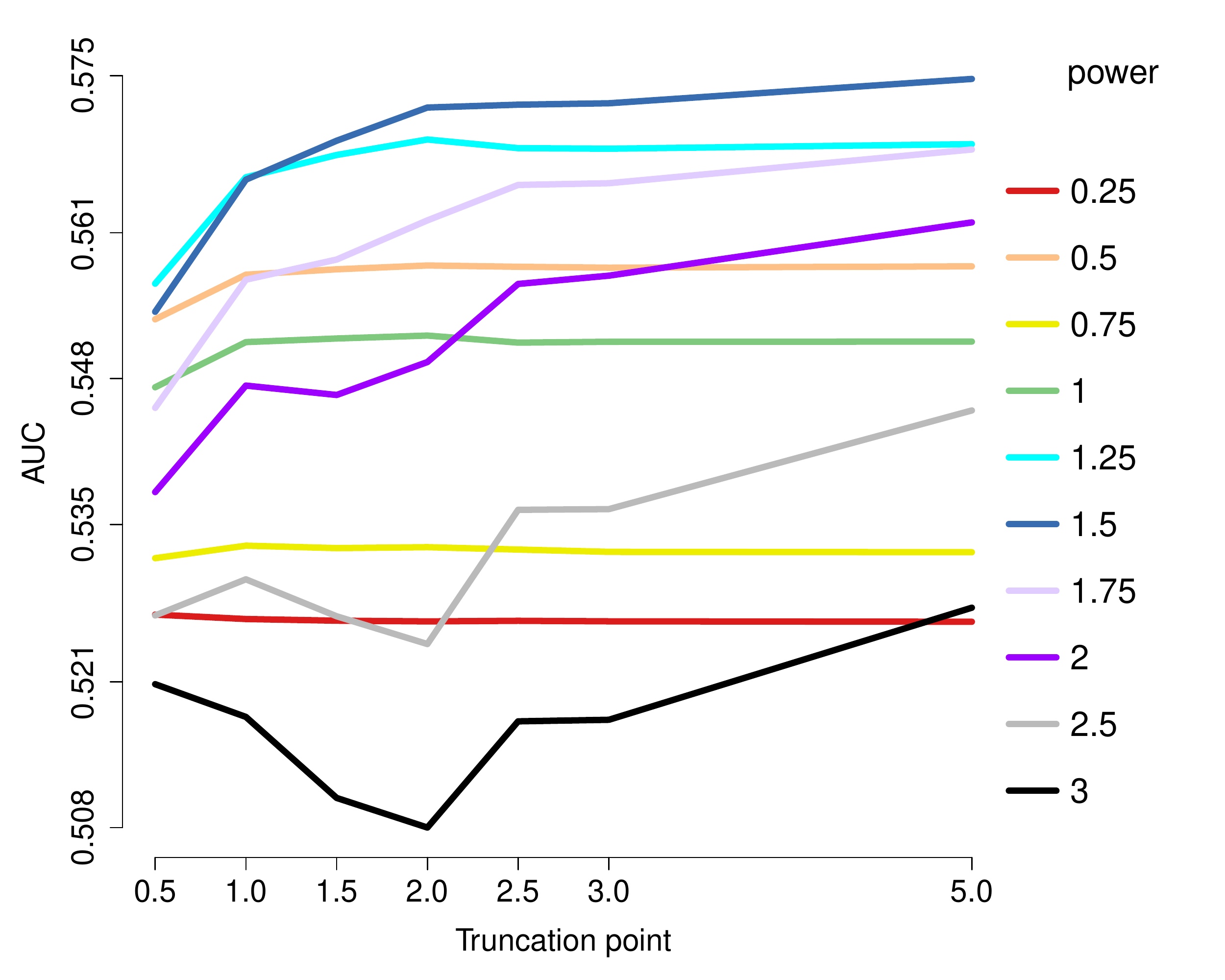}\hspace{-0.02in}}
\subfloat[$n=1000$, $\boldsymbol{\eta}=-0.5\mathbf{1}_{100}$, profiled est, ER]{\includegraphics[scale=0.26]{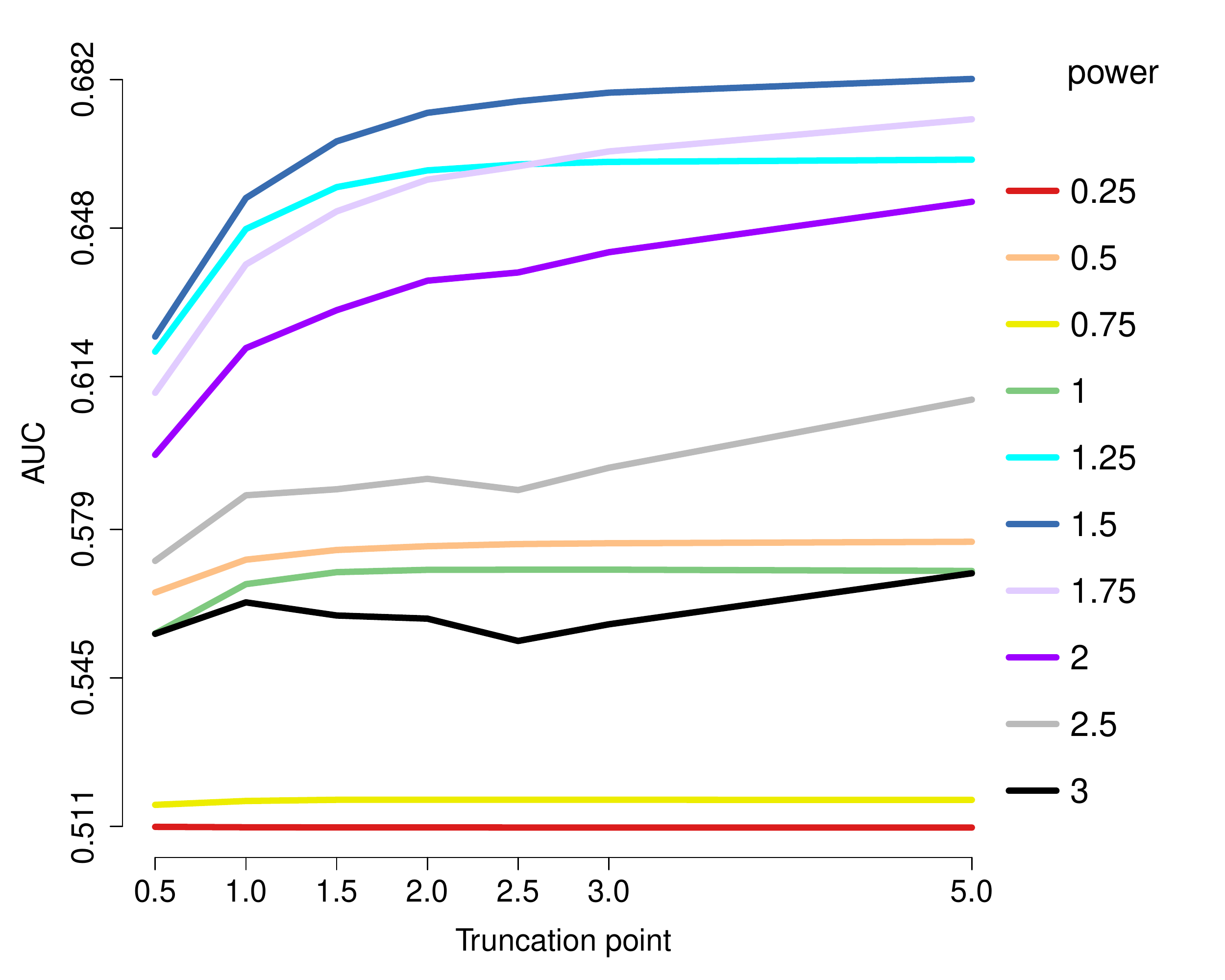}\hspace{-0.02in}}
\\ \vspace{-0.1in}
\subfloat[$n=80$, $\boldsymbol{\eta}=\boldsymbol{0}$, centered est, ER]{\includegraphics[scale=0.26]{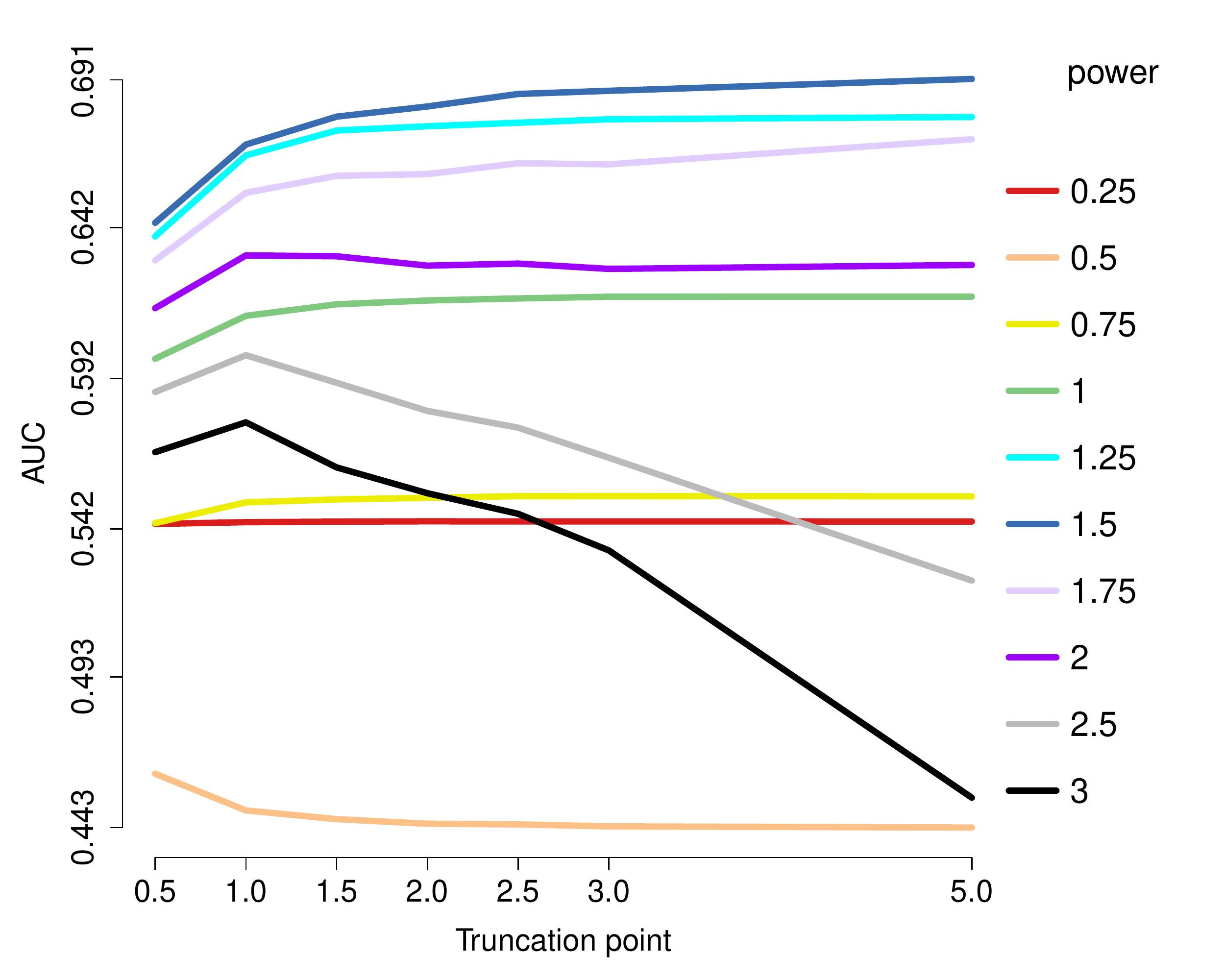}\hspace{-0.02in}}
\subfloat[$n=1000$, $\boldsymbol{\eta}=\boldsymbol{0}$, centered est, ER]{\includegraphics[scale=0.26]{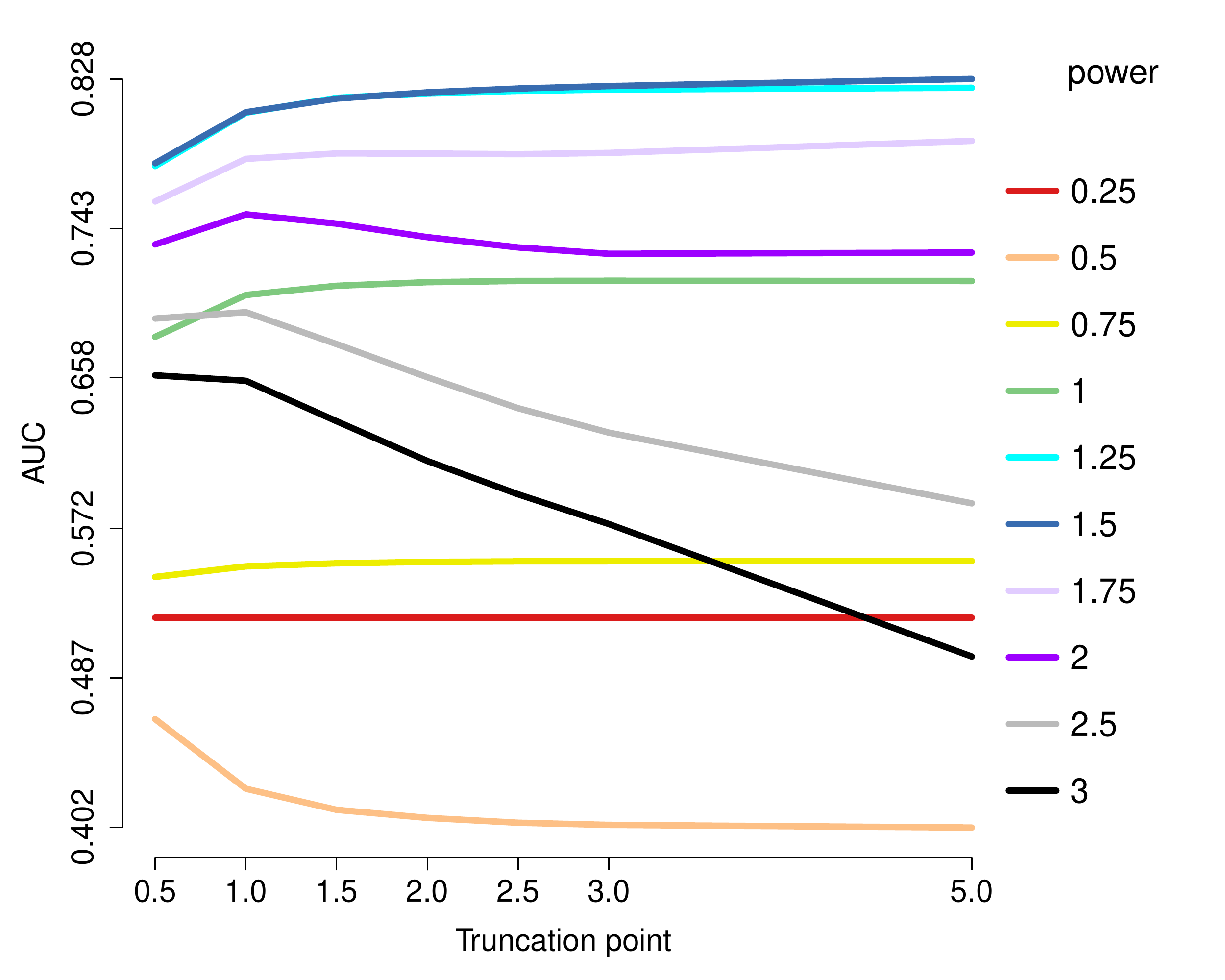}\hspace{-0.02in}}
\\ \vspace{-0.1in}
\subfloat[$n=80$, $\boldsymbol{\eta}=0.5\mathbf{1}_{100}$, profiled est, ER]{\includegraphics[scale=0.26]{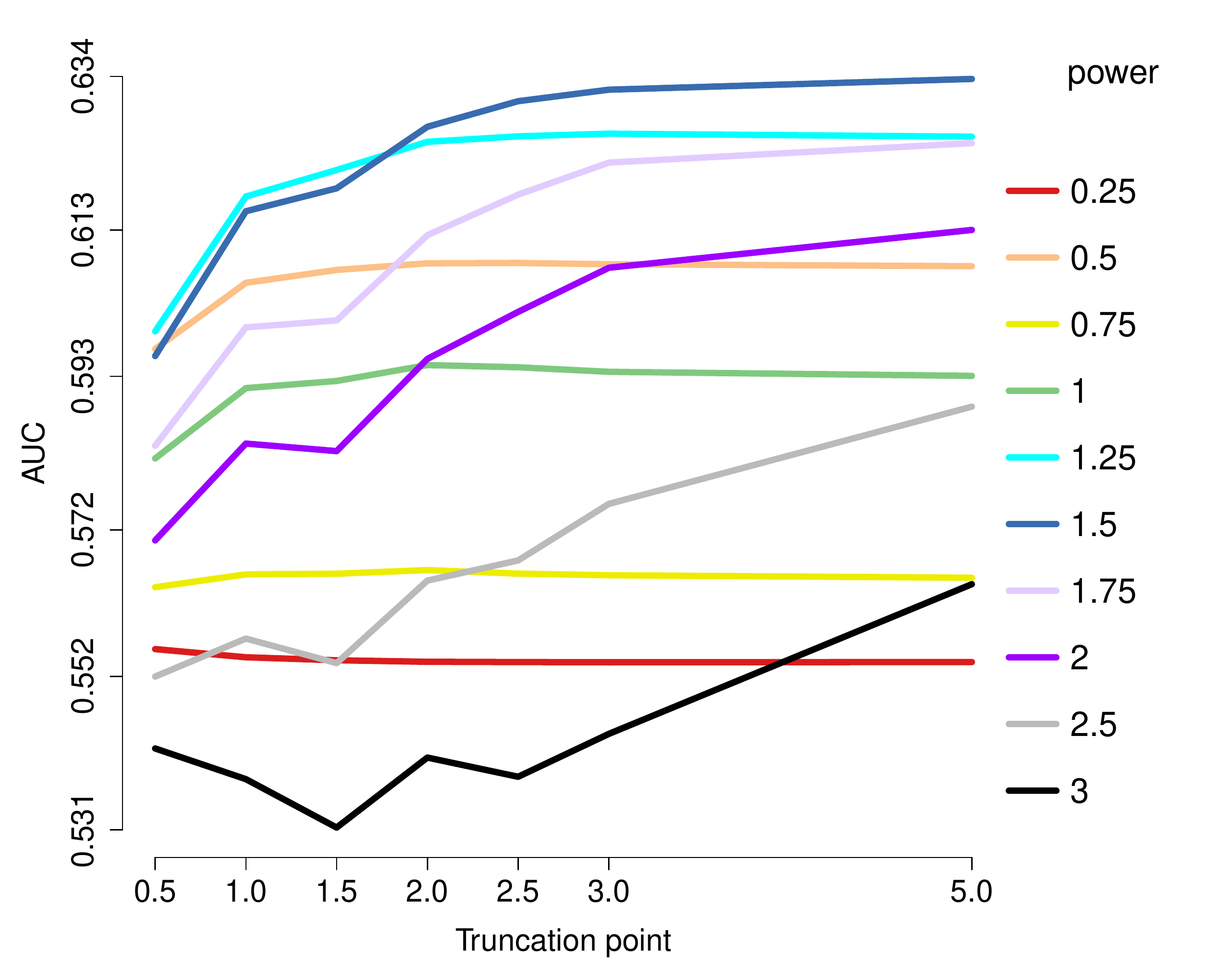}\hspace{-0.02in}}
\subfloat[$n=1000$, $\boldsymbol{\eta}=0.5\mathbf{1}_{100}$, profiled est, ER]{\includegraphics[scale=0.26]{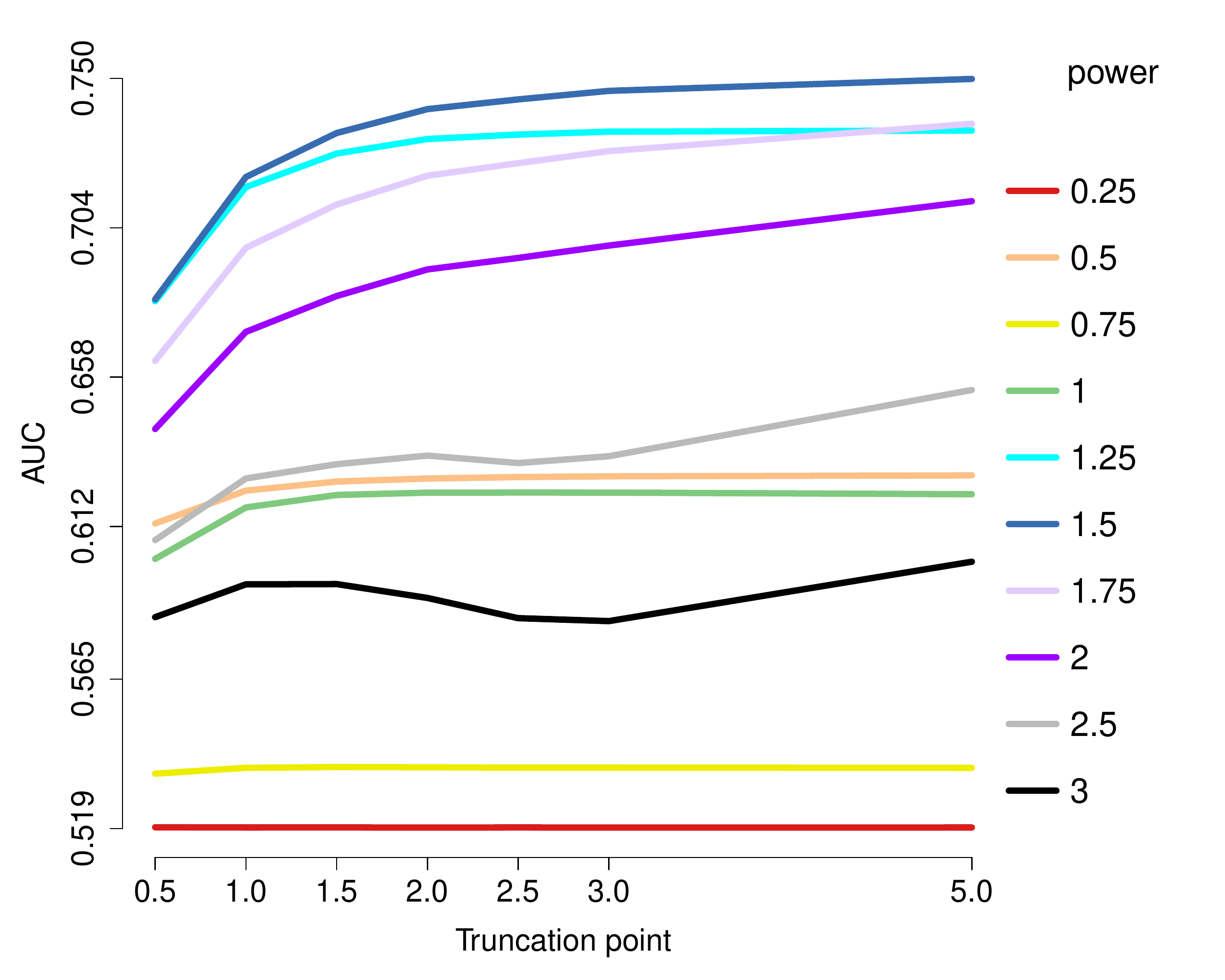}\hspace{-0.02in}}
\caption{\label{App_plot_exp}AUCs for edge recovery using generalized score matching for the exponential models. Each curve represents a different choice of power $p$ in $h(x)=\min(x^p,c)$, and the $x$ axis marks the truncation point $c$. Colors are sorted by $p$.}
\vspace{-0.5in}
\end{figure}

\newpage
\noindent Figure \ref{App_plot_gamma} displays the results for the gamma models.
\begin{figure}[htp]
\centering
\vspace{-0.0in}
\subfloat[$n=80$, $\boldsymbol{\eta}=-0.5\mathbf{1}_{100}$, profiled estimator, ER]{\includegraphics[scale=0.30]{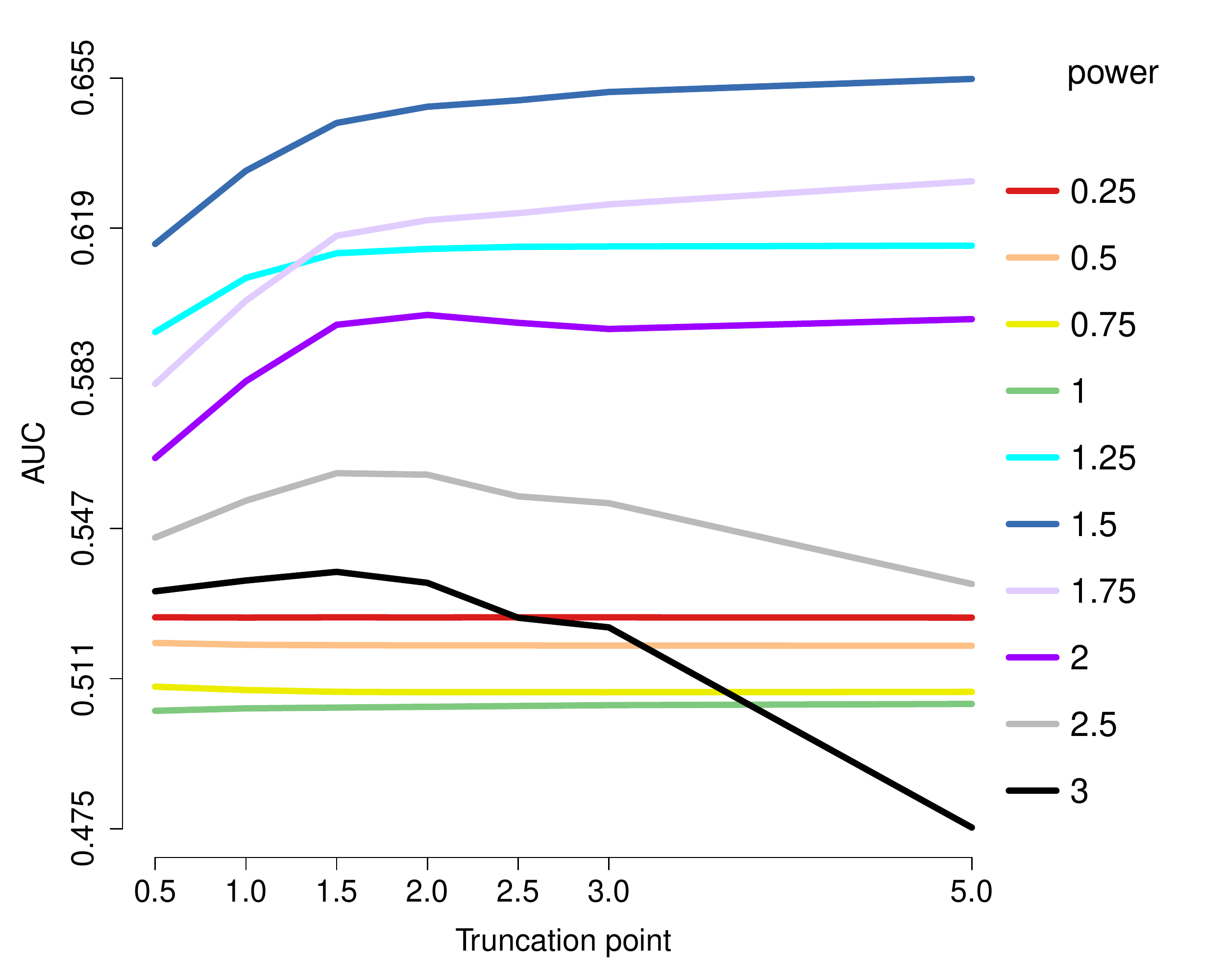}\hspace{-0.02in}}
\subfloat[$n=1000$, $\boldsymbol{\eta}=-0.5\mathbf{1}_{100}$, profiled estimator, ER]{\includegraphics[scale=0.30]{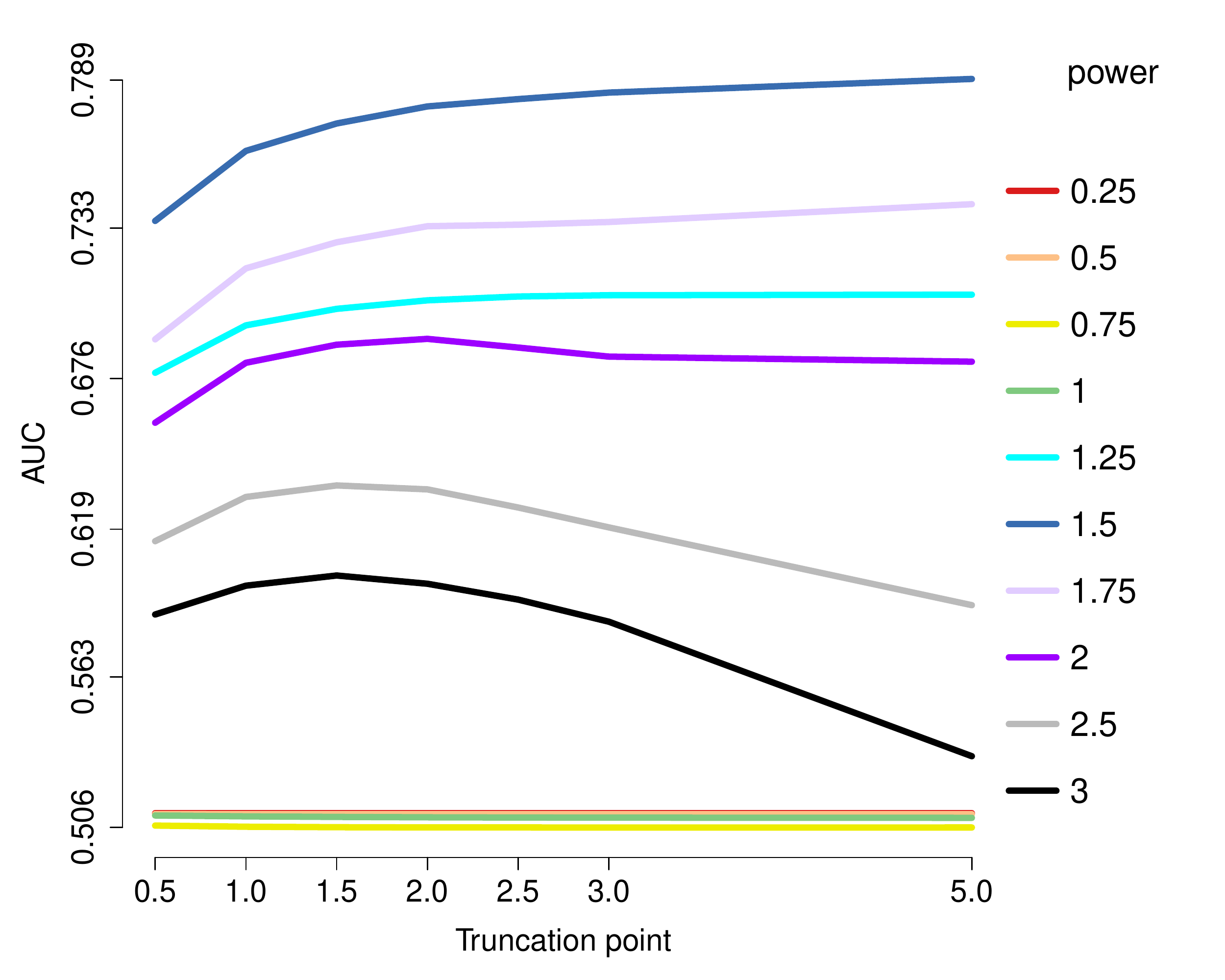}\hspace{-0.02in}}
\\ \vspace{-0.1in}
\subfloat[$n=80$, $\boldsymbol{\eta}=0.5\mathbf{1}_{100}$, profiled estimator, ER]{\includegraphics[scale=0.30]{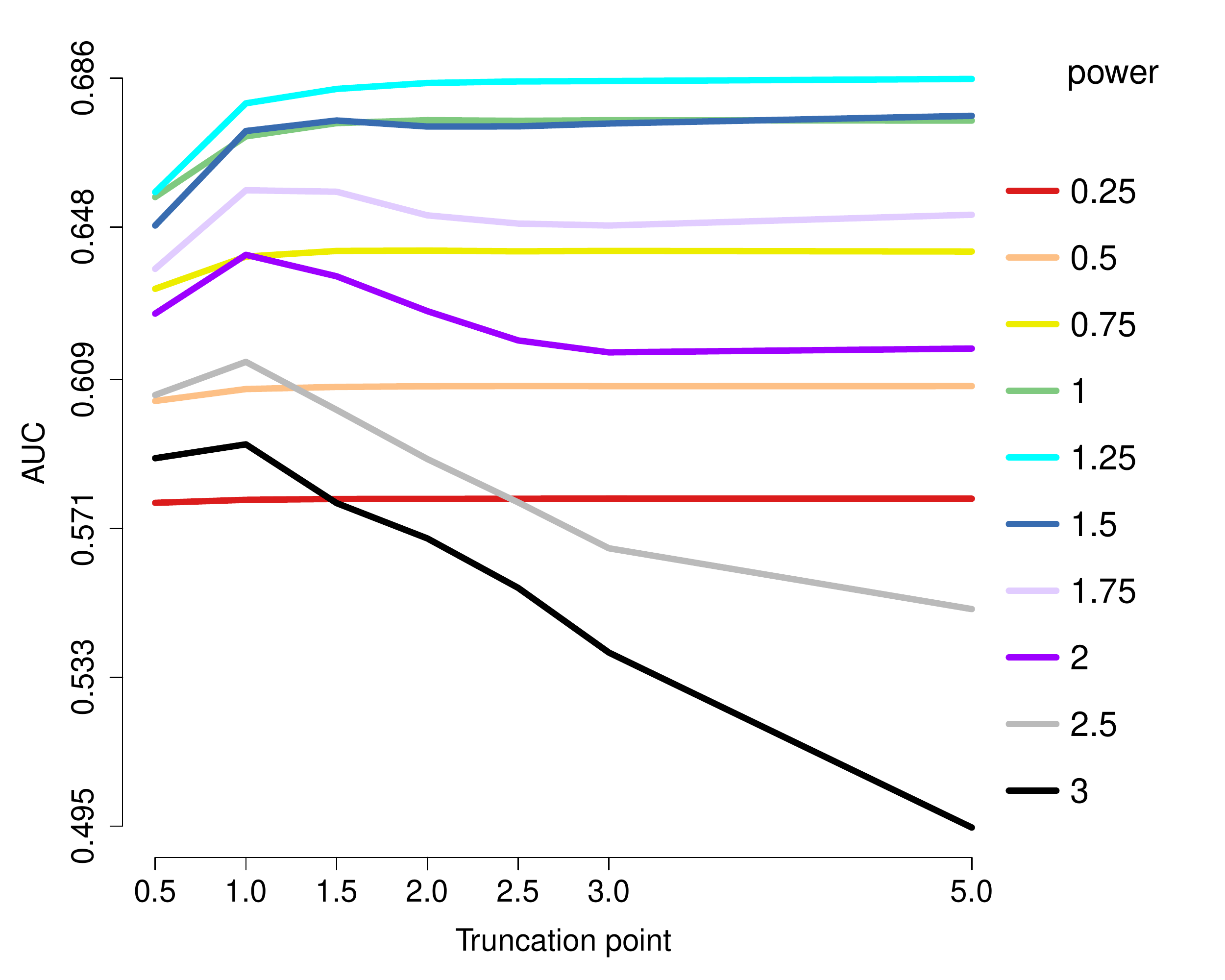}\hspace{-0.02in}}
\subfloat[$n=1000$, $\boldsymbol{\eta}=0.5\mathbf{1}_{100}$, profiled estimator, ER]{\includegraphics[scale=0.30]{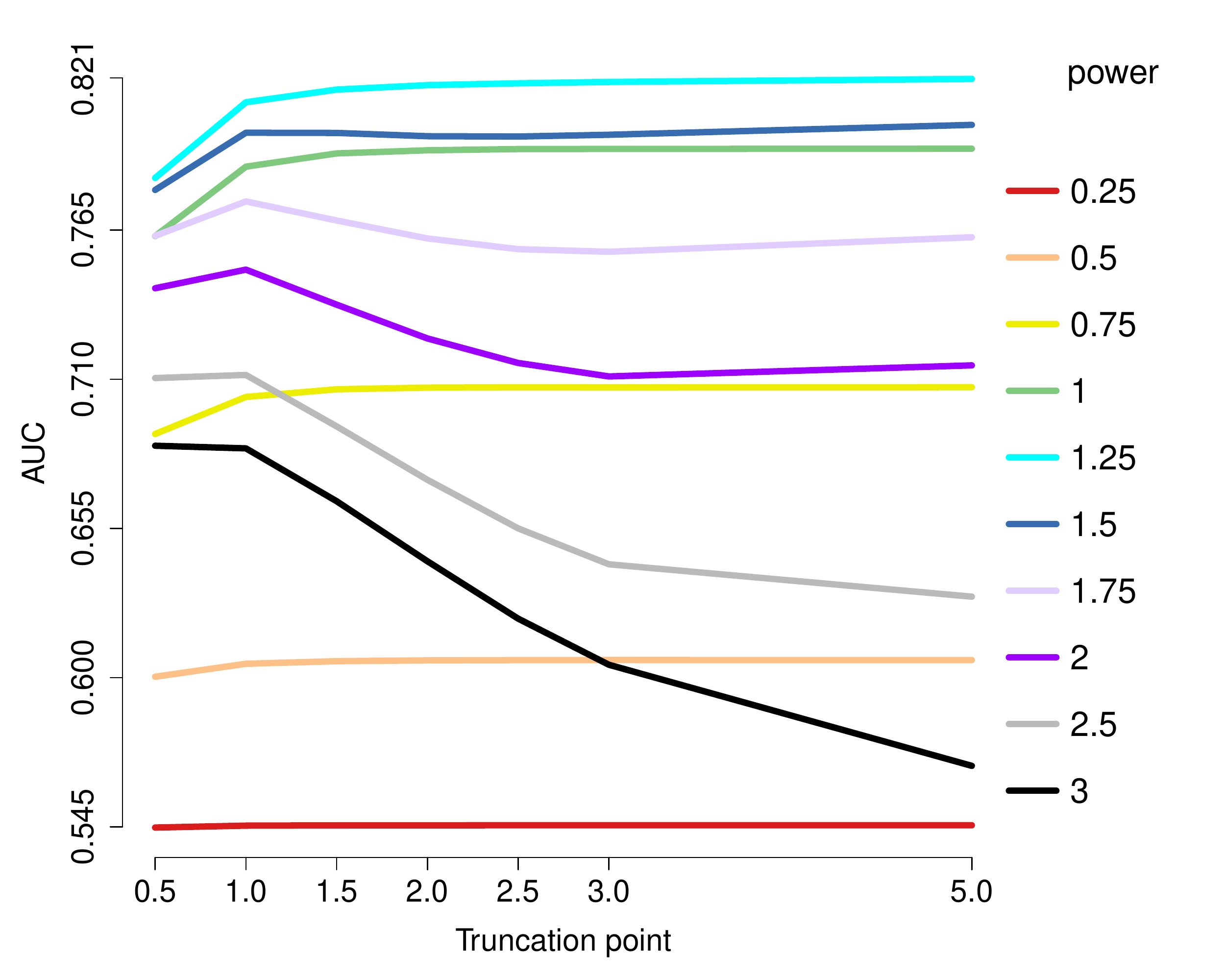}\hspace{-0.02in}}
\caption{\label{App_plot_gamma}AUCs for edge recovery using generalized score matching for the gamma models. Each curve represents a different choice of power $p$ in $h(x)=\min(x^p,c)$, and the $x$ axis marks the truncation point $c$. Colors are sorted by $p$.}
\end{figure}

\newpage
\noindent Figures \ref{App_plot_a1.5_b0.5} and \ref{App_plot_a1.5_b0} demonstrate the results for $a=3/2$ and $b=1/2$ or $b=0$, respectively.
\begin{figure}[H]
\centering
\subfloat[$n=80$, $\boldsymbol{\eta}=-0.5\mathbf{1}_{100}$, profiled est, ER]{\includegraphics[scale=0.26]{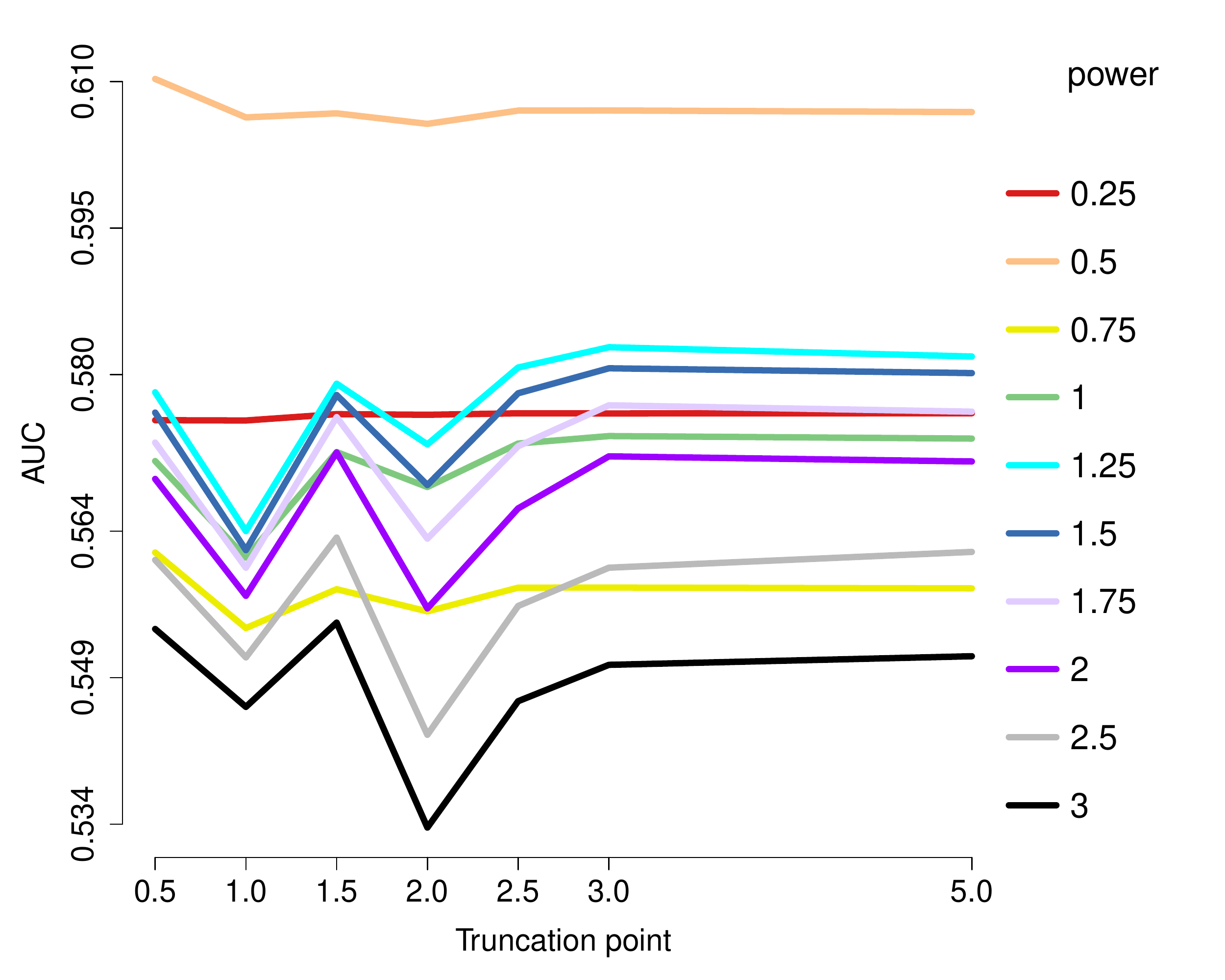}\hspace{-0.02in}}
\subfloat[$n=1000$, $\boldsymbol{\eta}=-0.5\mathbf{1}_{100}$, profiled est, ER]{\includegraphics[scale=0.26]{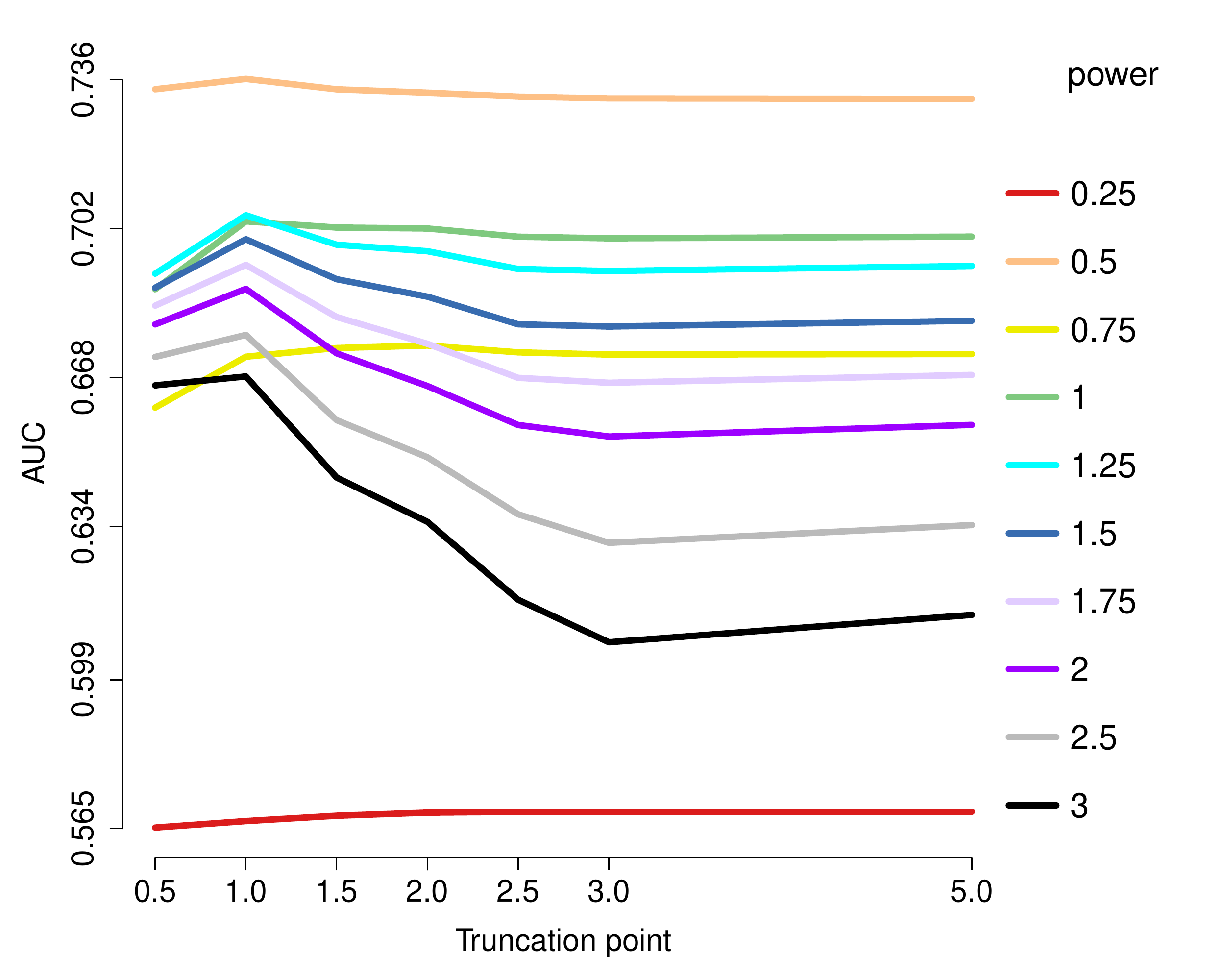}\hspace{-0.02in}}
\\ \vspace{-0.1in}
\subfloat[$n=80$, $\boldsymbol{\eta}=\boldsymbol{0}$, centered est, ER]{\includegraphics[scale=0.26]{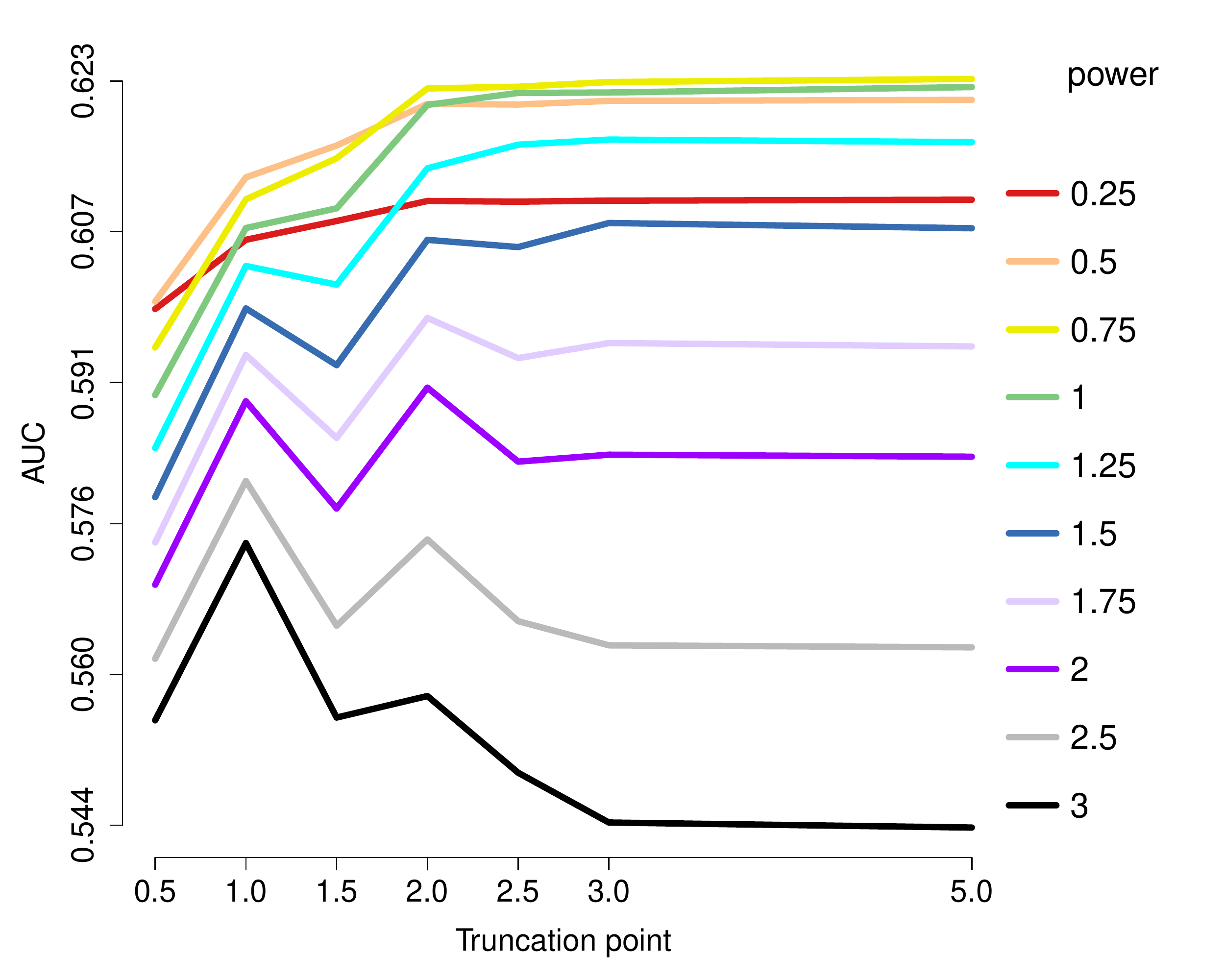}\hspace{-0.02in}}
\subfloat[$n=1000$, $\boldsymbol{\eta}=\boldsymbol{0}$, centered est, ER]{\includegraphics[scale=0.26]{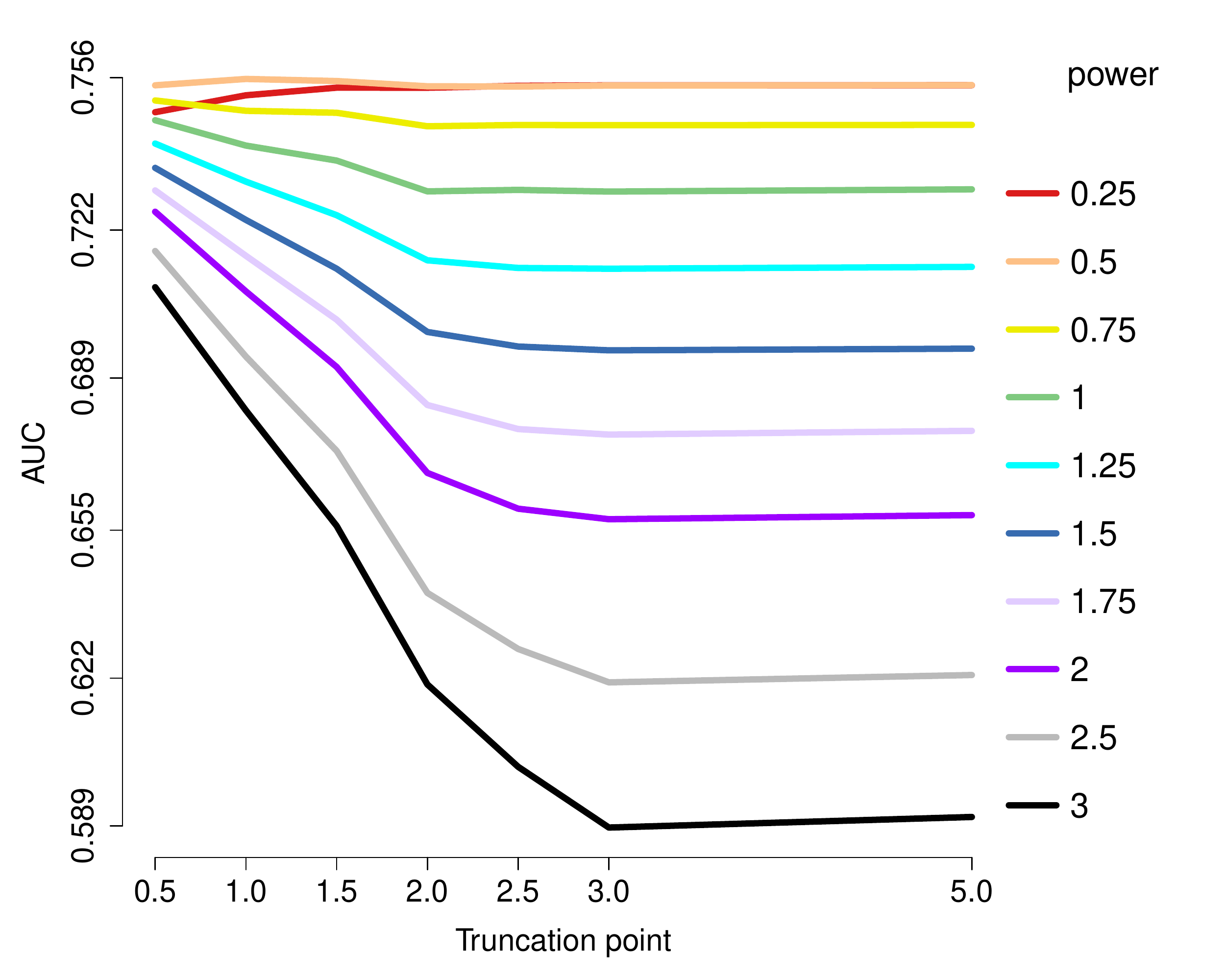}\hspace{-0.02in}}
\\ \vspace{-0.1in}
\subfloat[$n=80$, $\boldsymbol{\eta}=0.5\mathbf{1}_{100}$, profiled est, ER]{\includegraphics[scale=0.26]{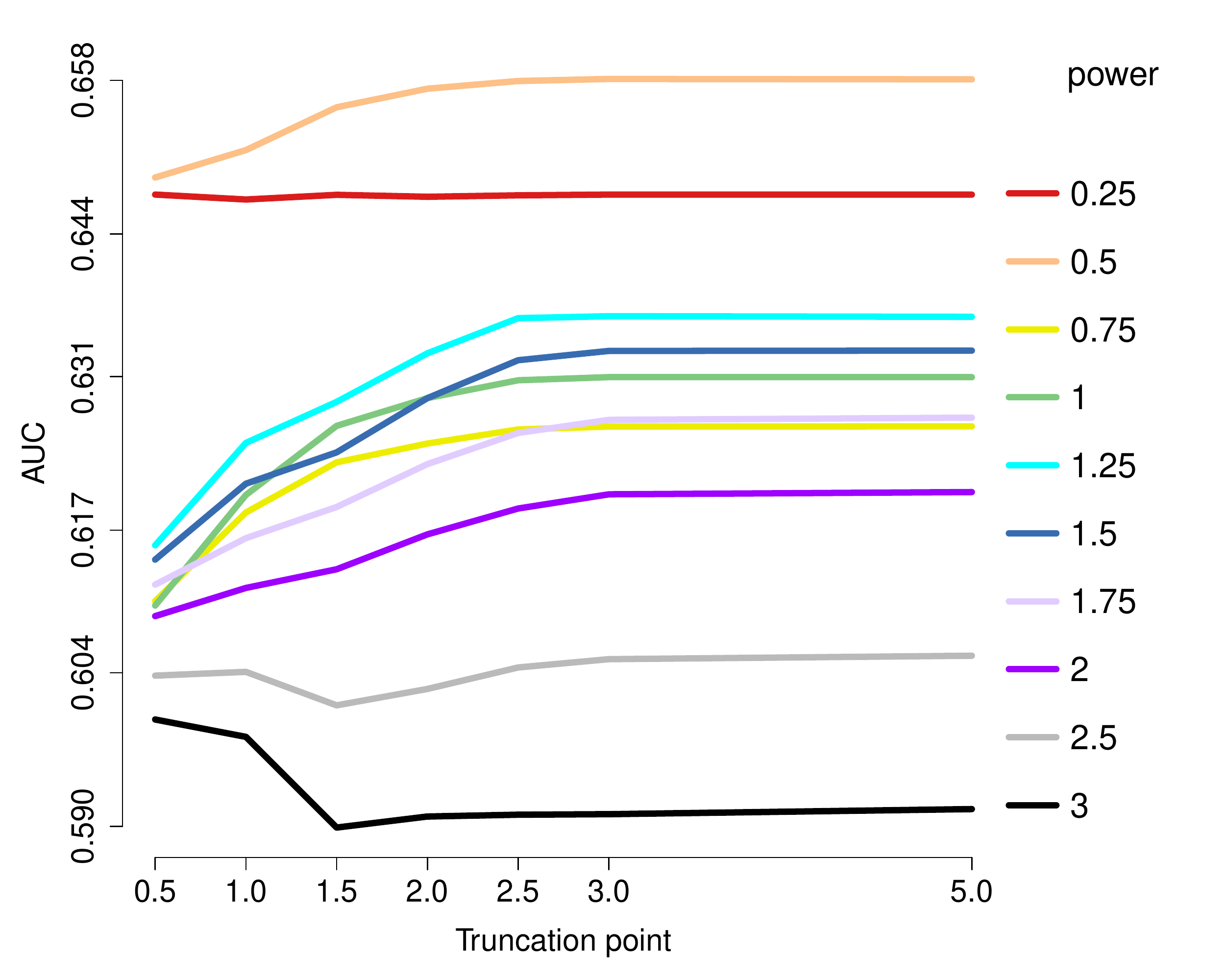}\hspace{-0.02in}}
\subfloat[$n=1000$, $\boldsymbol{\eta}=0.5\mathbf{1}_{100}$, profiled est, ER]{\includegraphics[scale=0.26]{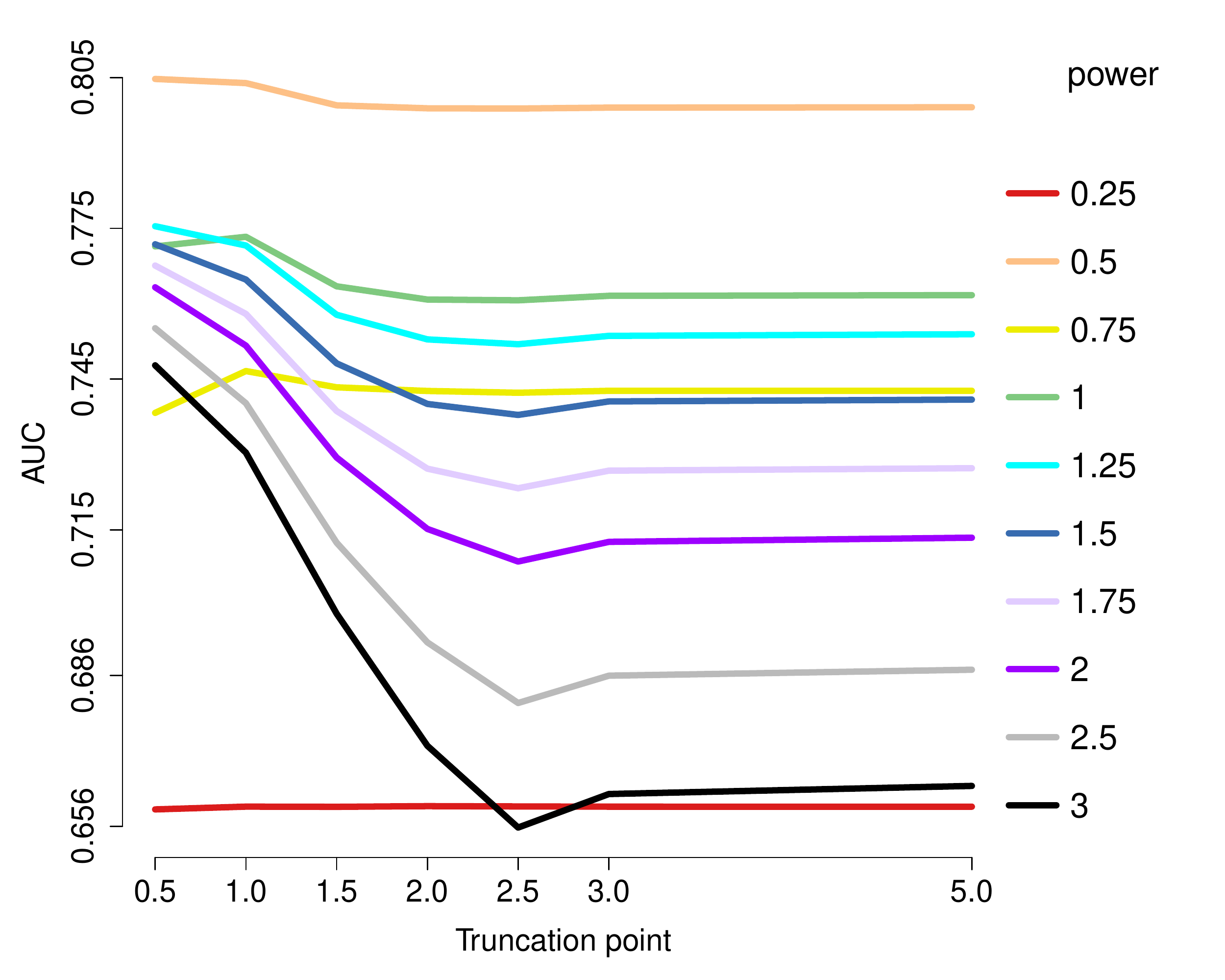}\hspace{-0.02in}}
\caption{\label{App_plot_a1.5_b0.5}AUCs for edge recovery using generalized score matching for $a=3/2$, $b=1/2$. Each curve represents a different choice of power $p$ in $h(x)=\min(x^p,c)$, and the $x$ axis marks the truncation point $c$. Colors are sorted by $p$.}
\end{figure}

\begin{figure}[H]
\centering
\vspace{-0.0in}
\subfloat[$n=80$, $\boldsymbol{\eta}=-0.5\mathbf{1}_{100}$, profiled estimator, ER]{\includegraphics[scale=0.30]{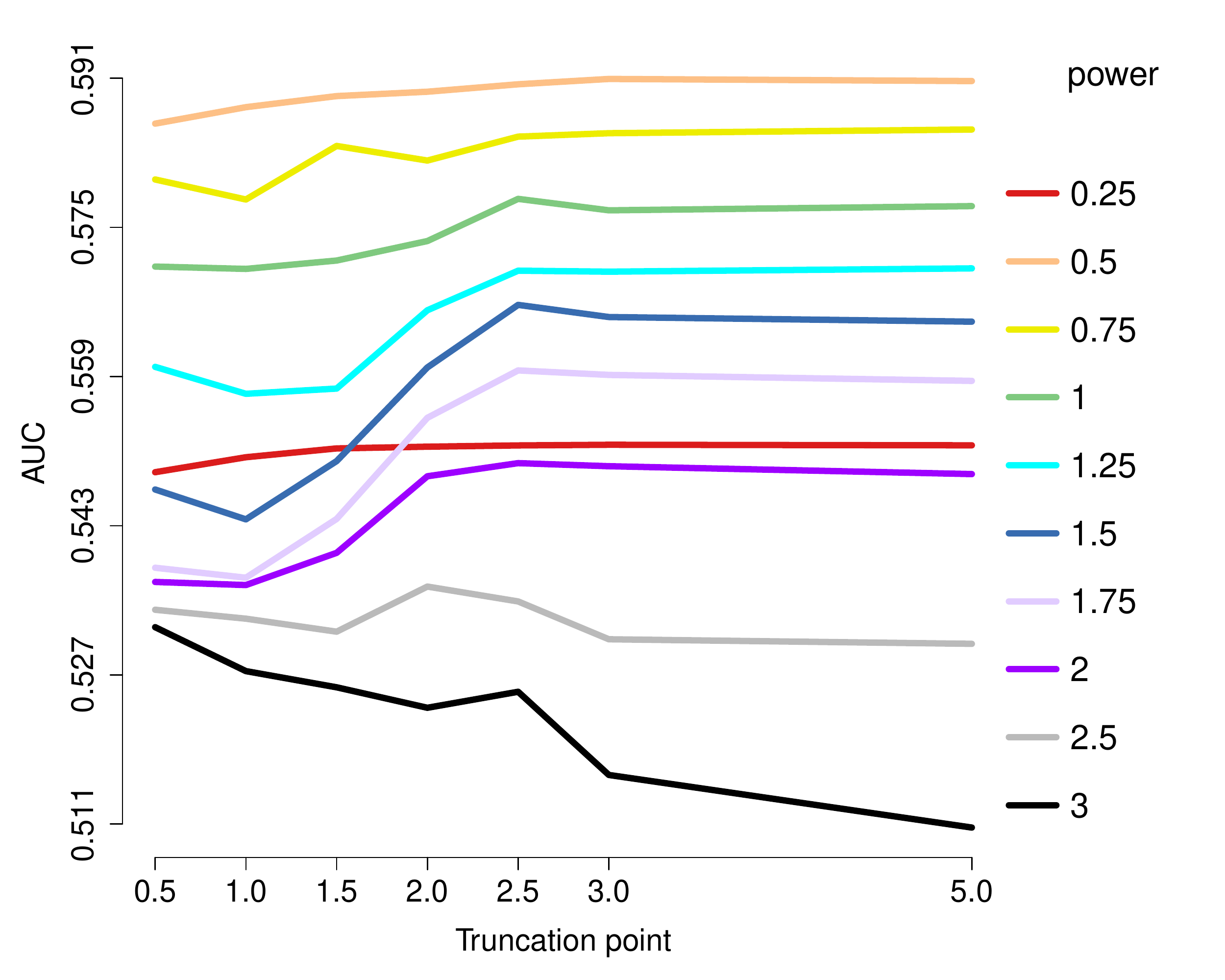}\hspace{-0.02in}}
\subfloat[$n=1000$, $\boldsymbol{\eta}=-0.5\mathbf{1}_{100}$, profiled estimator, ER]{\includegraphics[scale=0.30]{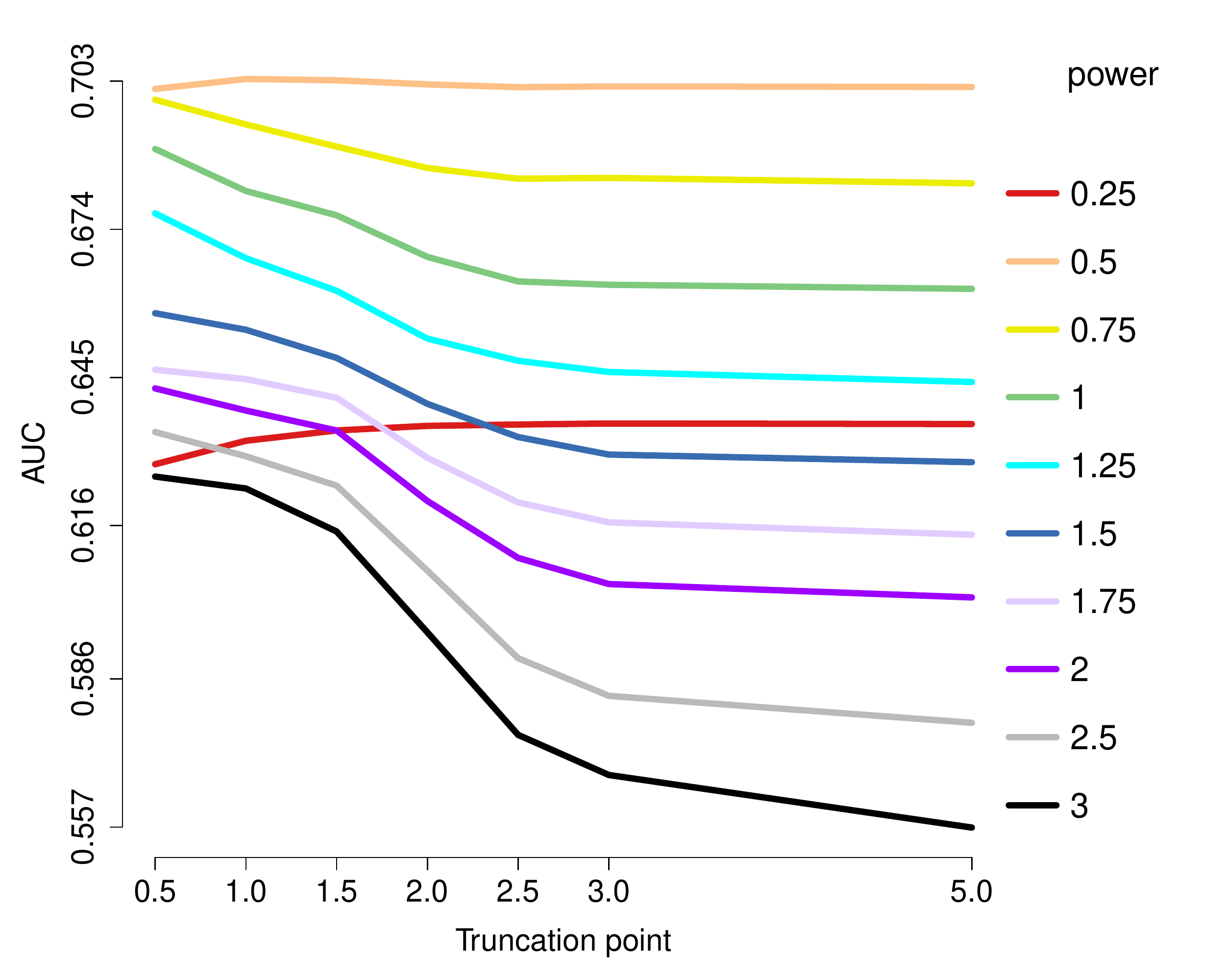}\hspace{-0.02in}}
\\ \vspace{-0.1in}
\subfloat[$n=80$, $\boldsymbol{\eta}=0.5\mathbf{1}_{100}$, profiled estimator, ER]{\includegraphics[scale=0.30]{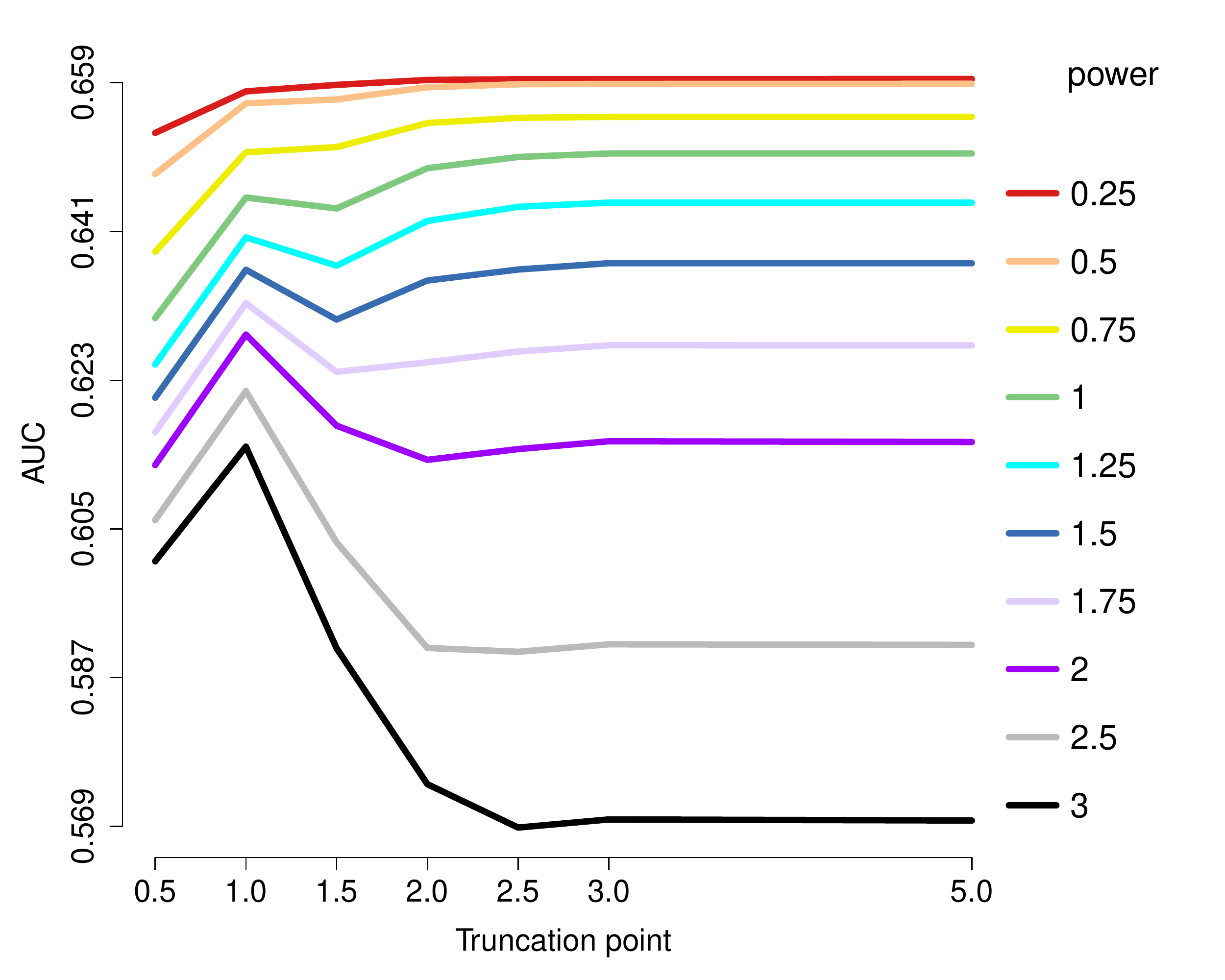}\hspace{-0.02in}}
\subfloat[$n=1000$, $\boldsymbol{\eta}=0.5\mathbf{1}_{100}$, profiled estimator, ER]{\includegraphics[scale=0.30]{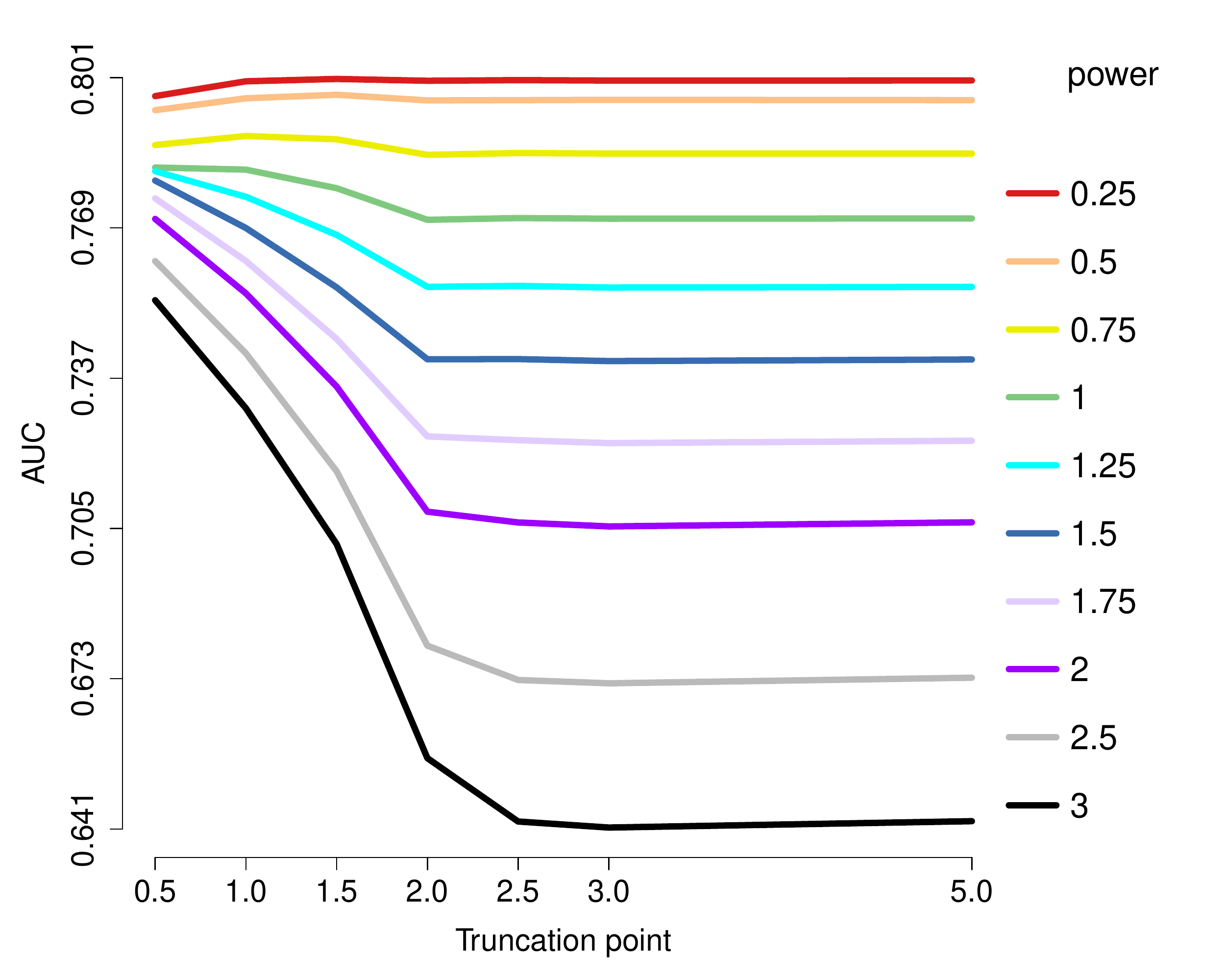}\hspace{-0.02in}}
\caption{\label{App_plot_a1.5_b0}AUCs for edge recovery using generalized score matching for $a=3/2$, $b=0$. Each curve represents a different choice of power $p$ in $h(x)=\min(x^p,c)$, and the $x$ axis marks the truncation point $c$. Colors are sorted by $p$.}
\end{figure}

\newpage

\bibliography{Paper}

\end{document}